\documentclass[11pt]{article}
\usepackage[utf8]{inputenc}
\usepackage{fullpage,amsmath,amssymb,bm,url,mathrsfs}

\usepackage{amssymb}
\usepackage{caption}
\usepackage{graphicx}
\usepackage{amsmath}
\usepackage{float}
\usepackage{fancyhdr}   
\usepackage{listings} 
\usepackage{xcolor} 
\usepackage{amsthm}
\usepackage{algpseudocode}
\usepackage{algorithm}
\usepackage{makecell}
\usepackage{comment,enumitem}
\usepackage[round]{natbib}
\usepackage{dsfont}
\usepackage{tablefootnote}
\usepackage{cleveref}
\usepackage{subcaption}
\usepackage{authblk} 

\newtheorem{theorem}{Theorem}
\newtheorem{lemma}{Lemma}
\newtheorem{definition}{Definition}

\newtheorem{assumption}{Assumption}

\newtheoremstyle{exampstyle}
{\topsep} 
{\topsep} 
{} 
{} 
{\bfseries} 
{.} 
{.5em} 
{} 

\theoremstyle{exampstyle}
\newtheorem{example}{Example}

\newcommand{\blue}[1]{{\color{blue}#1}}

\newcommand\fro[1]{\| #1 \|_{\rm{F}}}
\newcommand\op[1]{\| #1 \|}
\newcommand\nuc[1]{\| #1 \|_{*}}
\newcommand{\mat}[1]{\begin{bmatrix}#1 \\ \end{bmatrix}}
\newcommand{\inp}[2]{\langle #1,#2\rangle}
\newcommand{\linf}[1]{\|#1\|_{\ell_{\infty}}}
\newcommand{\lzero}[1]{\|#1\|_{0}}
\newcommand{\ltwo}[1]{\|#1\|_{\ell_{2}}}
\newcommand{\ltwoinf}[1]{\|#1\|_{2,\infty}}
\newcommand{\oneinf}[1]{\|#1\|_{1,\infty}}
\newcommand\psione[1]{\| #1 \|_{\psi_1}}
\newcommand\psionel[1]{\big\| #1 \big\|_{\psi_1}}
\newcommand\psitwo[1]{\| #1 \|_{\psi_2}}

\newcommand\indicator[1]{\mathds{1}(\calE_{#1})}

\def\spiki{\textsf{Spiki}}
\def\trace{\textsf{Trace}}

\def\incoh{\textsf{Incoh}}
\def\bSigma{{\boldsymbol \Sigma}}
\def\bTheta{{\boldsymbol \Theta}}
\def\bLambda{{\boldsymbol \Lambda}}

\def\up{\U_{\perp}}

\def\hosvd{\textsf{HOSVD}_{\r}}
\def\dof{\textsf{dof}}
\def\hat{\widehat}
\def\dmax{d_{\submax}}
\def\rmax{r_{\submax}}
\def\dmin{d_{\submin}}
\def\rmin{r_{\submin}}

\def\rank{\textsf{rank}}

\def\errinf{\textsf{Err}_{\infty}}
\def\errtwo{\textsf{Err}_2}
\def\errratio{\textsf{r}_{\textsf{2}/\infty}}
\def\trim{\textsf{Trim}}

\def\calE{{\mathcal E}}
\def\calF{{\mathcal F}}

\def\calI{{\mathcal I}}

\def\calM{{\mathcal M}}

\def\calP{{\mathcal P}}

\def\calR{{\mathcal R}}

\def\calY{{\mathcal Y}}

\def\bcalC{{\boldsymbol{\mathcal C}}}
\def\bcalD{{\boldsymbol{\mathcal D}}}
\def\bcalE{{\boldsymbol{\mathcal E}}}

\def\bcalG{{\boldsymbol{\mathcal G}}}

\def\bcalL{{\boldsymbol{\mathcal L}}}
\def\bcalM{{\boldsymbol{\mathcal M}}}

\def\bcalP{{\boldsymbol{\mathcal P}}}

\def\bcalR{{\boldsymbol{\mathcal R}}}
\def\bcalS{{\boldsymbol{\mathcal S}}}
\def\bcalT{{\boldsymbol{\mathcal T}}}

\def\bcalX{{\boldsymbol{\mathcal X}}}

\def\AA{{\mathbb A}}
\def\BB{{\mathbb B}}

\def\EE{{\mathbb E}}
\def\FF{{\mathbb F}}

\def\MM{{\mathbb M}}

\def\OO{{\mathbb O}}
\def\PP{{\mathbb P}}

\def\RR{{\mathbb R}}

\def\TT{{\mathbb T}}

\def\YY{{\mathbb Y}}

\def\a{{\boldsymbol a}}
\def\b{{\boldsymbol b}}
\def\c{{\boldsymbol c}}
\def\d{{\boldsymbol d}}
\def\e{{\boldsymbol e}}
\def\f{{\boldsymbol f}}
\def\g{{\boldsymbol g}}

\def\m{{\boldsymbol m}}

\def\r{{\boldsymbol r}}
\def\s{{\boldsymbol s}}
\def\t{{\boldsymbol t}}
\def\u{{\boldsymbol u}}
\def\v{{\boldsymbol v}}
\def\w{{\boldsymbol w}}
\def\x{{\boldsymbol x}}
\def\y{{\boldsymbol y}}

\def\A{{\boldsymbol A}}
\def\B{{\boldsymbol B}}
\def\C{{\boldsymbol C}}
\def\D{{\boldsymbol D}}

\def\G{{\boldsymbol G}}

\def\I{{\boldsymbol I}}

\def\L{{\boldsymbol L}}
\def\M{{\boldsymbol M}}
\def\N{{\boldsymbol N}}
\def\O{{\boldsymbol O}}
\def\P{{\boldsymbol P}}
\def\Q{{\boldsymbol Q}}
\def\R{{\boldsymbol R}}
\def\S{{\boldsymbol S}}
\def\T{{\boldsymbol T}}
\def\U{{\boldsymbol U}}
\def\V{{\boldsymbol V}}
\def\W{{\boldsymbol W}}
\def\X{{\boldsymbol X}}

\def\Z{{\boldsymbol Z}}

\def\submax{{\textsf{\tiny max}}}
\def\submin{{\textsf{\tiny min}}}
\def\dof{\overline{\textsf{dof}}}

\def\sfA{{\textsf{A}}}
\def\sfB{{\textsf{B}}}
\def\bDel{\boldsymbol{\Delta}}
\def\tr{\textsf{Tr}}
\def\vec{\textsf{Vec}}
\def\svd{\textsf{SVD}}
\def\reshape{\textsf{Reshape}}

\begin{document}

	\title{Online Tensor Learning: Computational and Statistical Trade-offs, Adaptivity and Optimal Regret}
	
%
%
%
	\author[1,2]{Jingyang Li}
	\author[3]{Jian-Feng Cai}
	\author[1]{Yang Chen}
	\author[3]{Dong Xia}
	
	\affil[1]{Department of Statistics, University of Michigan, Ann Arbor}
	\affil[2]{Department of Mathematics, University of Michigan, Ann Arbor}
	\affil[3]{Department of Mathematics, Hong Kong University of Science and Technology}
	\maketitle

\begin{abstract}
Large tensor learning algorithms are typically computationally expensive and require storing a vast amount of data. In this paper, we propose a unified online Riemannian gradient descent (oRGrad) algorithm for tensor learning, which is computationally efficient, consumes much less memory, and can handle sequentially arriving data while making timely predictions. The algorithm is applicable to both linear and generalized linear models. If the time horizon $T$ is known, oRGrad achieves statistical optimality by choosing an appropriate fixed step size. We find that noisy tensor completion particularly benefits from online algorithms by avoiding the trimming procedure and ensuring sharp entry-wise statistical error, which is often technically challenging for offline methods. The regret of oRGrad is analyzed, revealing a fascinating trilemma concerning the computational convergence rate, statistical error, and regret bound. By selecting an appropriate constant step size, oRGrad achieves an $O(T^{1/2})$ regret. We then introduce the adaptive-oRGrad algorithm, which can achieve the optimal $O(\log T)$ regret by adaptively selecting step sizes, regardless of whether the time horizon is known. The adaptive-oRGrad algorithm can attain a statistically optimal error rate without knowing the horizon. Comprehensive numerical simulations corroborate our theoretical findings. We show that oRGrad significantly outperforms its offline counterpart in predicting the solar F10.7 index with tensor predictors that monitor space weather impacts.
\end{abstract}
%
	

\section{Introduction}\label{sec:intro} 
The technological revolution in data collection and processing over the past decades has made large tensor data available in diverse fields, such as international trade flow \citep{cai2022generalized, lyu2021latent}, malaria parasite gene networks \citep{jing2021community, larremore2013network}, the BHL dataset \citep{lyu2022optimal, mai2021doubly}, EEG datasets \citep{hu2020matrix, huang2022robust}, hyperspectral imaging \citep{li2010tensor}, medical image analysis \citep{gandy2011tensor, wang2020learning}, multi-way recommender systems \citep{bi2018multilayer}, spatio-temporal analysis \citep{chen2020semiparametric, liu2022characterizing}, 4D scanning transmission electron microscopy \citep{han2022optimal}, and hypergraph network analysis \citep{ke2019community}.

Low-rank tensor models assume that the observed data are sampled from a statistical model characterized by an {\it unknown} but {\it low-rank} tensor $\bcalT^{\ast}$ of size $d_1\times\cdots\times d_m$. A tensor is considered low-rank if it is the sum of a few {\it rank-one} tensors, which can be understood as a generalization of low-rank matrix (see the formal definition in Section~\ref{sec:method}). The low-rank assumption substantially reduces the model complexity from $d^{\ast}:=d_1\cdots d_m$ to $O(md_{\submax})$, where $d_{\submax}:=\max_{j\in[m]}d_j$. The primary goal in these models is to estimate the latent tensor, a procedure known as {\it low-rank tensor learning}. There is a vast literature studying the computational and statistical aspects of low-rank tensor learning.  Tensor linear regression aims to recover a low-rank tensor from a collection of linear measurements and their respective (noisy) outcomes. Computationally efficient estimators are attainable via importance sketching \citep{zhang2020islet}, projected gradient descent \citep{chen2019non}, Burer-Monteiro type gradient descent \citep{han2022optimal}, Riemannian gradient descent \citep{shen2022computationally}, scaled gradient descent \citep{tong2021scaling}, and nuclear-norm penalized least squares via matricization \citep{mu2014square}. Tensor regression has been further investigated in generalized linear models to handle categorical responses \citep{cai2022generalized,han2022optimal,chen2019non}. Tensor completion refers to the problem of reconstructing a tensor by observing only a small fraction of its (noisy) entries. An incomplete list of representative works include convex programming via minimizing tensor nuclear norm \citep{yuan2016tensor}, vanilla or scaled gradient descent \citep{cai2019nonconvex,tong2021scaling}, alternating minimization \citep{jain2014provable}, Grassmannian gradient descent \citep{xia2019polynomial}, higher-order orthogonal iterations \citep{xia2021statistically}, sum-of-squares hierarchy \citep{barak2016noisy}, and Riemannian gradient descent \citep{kressner2014low, cai2022provable}.  Binary tensor learning or one-bit tensor completion aims to estimate $\bcalT^{\ast}$ from binary entry-wise observations \citep{wang2020learning, cai2022generalized}. 

The aforementioned works focused on {\it offline learning}, where a set of tensorial data is collected and a low-rank tensor model is fitted. This usually involves an iterative algorithm that uses all the data simultaneously, during which the data itself remains static. The offline learning framework has several limitations.  For many applications,  data are revealed sequentially, requiring timely predictions. A typical example is large-scale recommendation system \citep{chang2017streaming,zhang2019deep,davidson2010youtube,linden2003amazon}, where vast amounts of user-related feedback are gathered every minute. User preferences drift over time, affecting prediction accuracy. For example, a tweet that was popular last month may be less popular now. See also the application of predicting the solar index F10.7 in Section~\ref{sec:app-solar}, where we show that the online tensor learning method can significantly outperform its offline counterpart in prediction accuracy. 
Designing an algorithm that updates in real-time with newly-arrived data, referred to as {\it online learning}, is crucial. Another major issue with offline learning is the high computation and storage cost. The size of a tensor increases exponentially with its order. In contrast, an online learning algorithm \citep{langford2009sparse} updates in real-time using only one or a few observations, which are discarded afterwards, making it applicable to large-scale systems. Although this may seem wasteful, online methods can achieve performance comparable to offline counterparts that use all data together. This phenomenon has been observed in streaming principal component analysis (PCA) \citep{shamir2016convergence,jain2016streaming,allen2017first}, online sparse linear regression \citep{langford2009sparse,foster2016online,fan2018statistical}, and online sparse PCA \citep{yang2015streaming}.

An online algorithm based on convex programming was proposed by \cite{meka2008rank} for rank minimization over a polyhedral set. An online Riemannian gradient descent algorithm, equipped with an approximated SVD as the retraction, for estimating a low-rank matrix under a general loss function was studied in \cite{shalit2012online}. However, theoretical guarantees of algorithmic convergence and statistical performance are missing in both \cite{meka2008rank} and \cite{shalit2012online}. Exact matrix completion was studied by \cite{jin2016provable}, showing that a simple online gradient descent algorithm provably recovers the matrix within $\tilde O(r^2d_{\submax})$ iterations, where $r$ denotes the rank. 
A closely related study by \cite{han2022online} examines the online stochastic gradient descent algorithm for addressing the matrix contextual bandit problem. The study establishes the estimation error rate and proposes an inference procedure under the low-rank regression model. In \cite{ge2015escaping}, the authors demonstrated a strict saddle property for noiseless orthogonal tensor decomposition, showing that a randomly initialized stochastic gradient algorithm successfully recovers the latent tensor within a polynomial number of iterations. Their online tensor-based method is further utilized in \cite{huang2015online} to learn latent variable models with applications in community detection and topic modeling. Unfortunately, the statistical performance of their algorithm is still not provided in \cite{huang2015online}. See also \cite{yu2015accelerated,li2018online,mairal2010online,de2015global} and references therein. 

The statistical understanding of online methods is largely unknown for most popular low-rank models. While online gradient descent is effective in noiseless matrix and tensor decomposition \citep{ge2015escaping, jin2016provable}, it is well-recognized that random noise significantly alters the optimization landscape, causing even offline gradient descent algorithms to often become trapped in locally optimal solutions \citep{arous2019landscape}. This suggests that the dynamics of online tensor learning algorithms may be drastically different under random noise. Intuitively, each time a new noisy observation arrives, the online algorithm faces a dilemma: assign more weight for aggressive updates, which causes noise to accumulate quickly, or assign less weight for slower updates, allowing noise to accumulate more slowly. To the best of our knowledge, there is a lack of online algorithms for learning low-rank tensors under generalized linear models. This gap exists even for matrix cases. The primary advantage of online learning algorithms is their ability to reduce computational and storage costs. It remains unclear whether online methods offer additional benefits, particularly regarding statistical performance and algorithm design. Surprisingly, we find that online methods can achieve results considered extremely challenging for offline methods.

Regret measures the prediction performance of an online algorithm. In a seminal paper, \cite{hazan2007logarithmic} showed that the online gradient descent (OGD) algorithm achieves $O(\log T)$ regret for (strongly) convex programming, where $T$ is the time horizon. Regret analysis of online tensor learning is more challenging due to the non-convexity caused by the rank constraint. \cite{shi2023high} and \cite{zhou2020stochastic} derived $O(T^{1/2})$ for low-rank tensor bandit. Remarkably, we demonstrate that an adaptive online tensor learning algorithm can achieve the optimal $O(\log T)$ regret.

\paragraph*{Our contributions}
We investigate online tensor learning within a general framework, covering both linear and generalized linear models. Compared to \cite{jin2016provable} and \cite{ge2015escaping}, our method accommodates noisy and discrete-type observations, making it more suitable for statistical applications. We propose a computationally fast \underline{o}nline algorithm  based on \underline{R}iemannian \underline{grad}ient descent, referred to as the {\it oRGrad} algorithm.  At time $t$, with the current estimate $\bcalT_{t}$,  a new observation $\mathfrak{D}_t$ arrives, incurring  a one-time loss $\ell(\bcalT_{t}, \mathfrak{D}_t)$. Then, oRGrad updates the estimate as follows:
\begin{equation}\label{eq:oRGrad-1}
\bcalT_{t+1} \leftarrow \textsf{Retraction}\big(\bcalT_{t}-\eta\cdot \nabla_{\textsf{R}} \ell(\bcalT_{t}, \mathfrak{D}_t)\big),
\end{equation}
where $\eta$ is the step size, $\nabla_{\textsf{R}}$ represents the Riemannian gradient, and $\textsf{Retraction}(\cdot)$ projects a tensor onto the Riemannian manifold. If $\nabla_{\textsf{R}}$ is replaced with the vanilla gradient, (\ref{eq:oRGrad-1}) coincides with projected gradient descent \citep{chen2015fast, chen2019non}. The Riemannian gradient is low-rank, which significantly speeds up the computation of subsequent retractions \citep{cai2022generalized,kressner2014low,zheng2022riemannian}.  Our oRGrad algorithm uses higher order singular value decomposition (HOSVD, see the formal definition in Section~\ref{sec:method}) for the retraction step. More recently, \cite{luo2022tensor} demonstrated that Riemannian optimization is robust to over-specification of the underlying ranks. 


To summarize, we make the following contributions. 
\begin{enumerate}[label=(\alph*)]

\item We demonstrate that, under suitable conditions, after $t$ iterations, oRGrad, with high probability, achieves an estimate (\textit{informally}) satisfying 
\begin{equation}\label{eq:oRGrad-err-1}
\|\bcalT_t-\bcalT^{\ast}\|_{\rm F}^2\leq 2(1-\eta)^t \cdot \|\bcalT_0-\bcalT^{\ast}\|_{\rm F}^2+O\big(\eta\cdot \overline{\textsf{dof}}\cdot \sigma^2\big),
\end{equation}
where $\eta\in (0,1)$ is a {\it fixed} step size,  $\dof$ describes the degree of freedom for $\bcalT^{\ast}$, $\bcalT_0$ denotes the initial estimate, and $\sigma$ represents the noise level. The contraction rate is independent of the condition number of $\bcalT^{\ast}$, which is an advantage of the Riemannian gradient descent algorithm. This contraction dynamic is established for the various tensor models discussed above. The  {\it computational and statistical trade-off} is observed from (\ref{eq:oRGrad-err-1}). A larger step size $\eta$ leads to faster convergence of oRGrad, but the final output has a larger statistical error. If the time horizon $T$ is known, one can choose $\eta_{\textsf{stat}}:=C_0T^{-1}\log d_{\submax} $ for a large $C_0>0$. Then the final output, with high probability, achieves a statistical error $O_m(T^{-1}\dof \cdot\sigma^2 \log d_{\submax})$, under mild conditions. Amazingly, it matches the best rate, up to a logarithm factor,  in the minimax sense, attainable by offline methods even if all the sequential observations are presented at once. 

\item Noisy tensor (or matrix) completion benefits significantly from oRGrad. A major challenge in designing tensor completion algorithms is maintaining the so-called incoherence (see Section~\ref{sec:oNTC} for a formal definition). This is especially true for Riemannian optimization, which often involves SVD at each iteration. Existing algorithms \citep{cai2022generalized, cai2022provable} require additional trimming procedures after step (\ref{eq:oRGrad-1}).  By exploiting the online nature of oRGrad and a spectral representation tool from \cite{xia2021normal}, we can directly show that $\bcalT_t$ is already incoherent, bypassing the trimming step.
Another surprising advantage of the online method is that we can easily derive a sharp upper bound for the entry-wise error $\|\bcalT_t-\bcalT^{\ast}\|_{\ell_\infty}$, which is highly valuable for statistical inference \citep{chen2019inference, xia2021statistical} in low-rank models. Deriving the entry-wise error is notoriously challenging for noisy tensor (matrix) completion, where prior works \citep{chen2021bridging, wang2021entrywise} often rely on a complicated leave-one-out framework to analyze algorithmic dynamics finely. Our results suggest that, with a warm initialization,  $\tilde O_m(d_{\submax})$ randomly sampled entries suffice to complete the whole tensor. In contrast, most prior works still require a sample size $\tilde O_m\big((d^{\ast })^{1/2}\vee d_{\submax}\big)$ even with a warm initialization. Table~\ref{table:completion} presents the advantages of oRGrad in tensor completion.

\begin{table}
	\centering
	\begin{tabular}{ c c c c c}
		\textbf{Algorithm} &  \begin{tabular}{@{}c@{}} \textbf{Sample complexity} \\ {\small (given a warm initialization)}\end{tabular} & \textbf{Total flops} &  \textbf{Entry-wise Error}\\\hline\hline
		\thead{Nuclear norm minimization\\ \citep{yuan2016tensor}}&$\tilde O_m(d^{m/2})$& N/A&No\\ \hline
		\thead{Gradient descent \\ \citep{xia2019polynomial}} & $\tilde O_m(d^{m/2})$& N/A  & No\\ \hline
		\thead{Scaled GD \\ \citep{tong2021scaling} }& $\tilde O_m(d^{m/2})$ & $\tilde O_m(d^{m/2}\log(1/\epsilon))$ &No\\ \hline
		\thead{Vanilla GD \\ \citep{cai2019nonconvex} }& $\tilde O_m(d^{m/2})$ & $\tilde O_m(d^{m/2}\log(1/\epsilon))$ & Yes\\ \hline
		{\it oRGrad (this paper)} & $\tilde O_m(d)$ & $\tilde O_m(d^2\log(1/\epsilon))$\tablefootnote{Total flops of oRGrad can be further reduced to $O(d\log \epsilon^{-1})$ if an appropriate retraction is chosen. Here, we opt to HOSVD as retraction for technical convenience, but it leads to $O(d^2\log \epsilon^{-1})$ total flops. }  & Yes\\
		\hline
	\end{tabular}
	\caption{Comparison with existing algorithms on tensor completion {\it assuming a warm initialization is provided} . Here the latent $m$-th order tensor is of size $d\times \cdots\times d$ whose ranks are assumed to be constants. The total flops is calculated when an $\epsilon$-accurate estimate is reached, i.e., $\|\bcalT_t-\bcalT^{\ast}\|_{\rm F} \|\bcalT^{\ast}\|_{\rm F}^{-1}\leq \epsilon$. Note that \cite{cai2019nonconvex} deals with a CP-format tensor, whereas others focus on Tucker-format tensors.}
	\label{table:completion}
\end{table}

\item The merits of oRGrad are further demonstrated in binary tensor learning. Assuming a logistic link, we show that oRGrad naturally preserves incoherence, eliminating the need for additional trimming. To the best of our knowledge, this is the first result of its kind. Finally, with a properly chosen step size relative to the time horizon, a minimax optimal error rate (up to logarithmic factors) can be achieved.

\item We show that oRGrad, equipped with a fixed step size $\eta$, achieves a regret $O(\eta^{-1}\fro{\bcalT_0 -\bcalT^*}^2 + \eta T\dof\sigma^2)$, highlighting a trilemma concerning the computational convergence, statistical error, and regret. Compared to the statistically optimal value $\eta_{\textsf{stat}} \asymp T^{-1}$, the regret bound is minimized at $\eta_{\textsf{regret}} \asymp T^{-1/2}$, allowing oRGrad to achieve the $O(T^{1/2})$ regret. We then propose the adaptive-oRGrad algorithm, which selects step sizes adaptively, and show that adaptive-oRGrad attains the {\it optimal} $O(\log T)$ regret. To the best of our knowledge, our result represents the first logarithmic regret in the high-dimensional online learning literature. Moreover, we show that the adaptive-oRGrad algorithm can achieve statistical optimality or the optimal $O(\log T)$ regret if the step sizes are adaptively selected, even when the time horizon is unknown.

\item We apply the oRGrad algorithm to predict the solar index F10.7 using ionospheric total electron content data. Compared to its offline counterpart, our numerical experiments demonstrate that oRGrad achieves significantly higher prediction accuracy.

\end{enumerate}

\section{Methodology}\label{sec:method}

\subsection{Background and notations}\label{sec:notation}

We use calligraphic-font bold-face letters (e.g. $\bcalT,\bcalX$) to denote tensors, bold-face capital letters (e.g. $\T,\X$) for matrices, bold-face lower-case letters (e.g. $\t,\x$) for vectors and blackboard bold-faced letters (e.g. $\RR,\MM$) for sets. We use square brackets with subscripts (e,g. $[\bcalT]_{i_1,i_2,i_3}$) to represent corresponding entries of tensors. 
Denote $\fro{\cdot}$ the Frobenius norm of tensors, and $\|\cdot\|_{\ell_p}$ the $\ell_p$-norm of tensors for $0< p\leq +\infty$. Specifically, $\linf{\bcalX}$ represents the largest magnitude of the entries of $\bcalX$. We denote $C,C_1,C_2,c,c_1,c_2,\ldots$ some absolute constants whose actual values might vary at different appearances. For nonnegative $A$ and $B$, the notation $A\lesssim B$ (equivalently, $B\gtrsim A$) means that there exists an absolute constant $C>0$ such that $A\leq CB$; $A\asymp B$ is equivalent to $A\lesssim B$ and $A\gtrsim B$, simultaneously. For any $1\leq \alpha\leq 2$, the Orlicz norm of a random variable is defined by  $\|X\|_{\psi_{\alpha}}:=\min\{C>0: \EE\exp\{|X/C|^{\alpha}\}\leq 2\}$.

An $m$-th order tensor $\bcalT\in\RR^{d_1\times\cdots\times d_m}$ means that its $j$-th dimension has size $d_j$.  
The $j$-th {\it matricization} $\calM_j(\cdot):\RR^{d_1\times\cdots\times d_m}\rightarrow \RR^{d_j\times d_j^-}$ with $d_j^- := d^*/d_j$. We also denote $(\cdot)_{(j)} := \calM_j(\cdot)$ for simplicity. If $m=3$, we have $[\calM_1(\bcalT)\big]_{i_1, (i_2-1)d_3+i_3}=[\bcalT]_{i_1,i_2,i_3}$ for $\forall i_j\in[d_j]$. The collection 
$\rank(\bcalT):=\big(\rank(\calM_1(\bcalT)),\cdots, \rank(\calM_m(\bcalT))\big)^{\top}$
is called the {\it Tucker ranks} of $\bcalT$. 
Given a matrix $\W_j\in \RR^{p_j\times d_j}$ for any $j\in[m]$, the mode-$j$ marginal product, denoted by $\times_j$,  between $\bcalT$ and $\W_j$ is defined by 
$
[\bcalT\times_j \W_j]_{i_1,\cdots,i_m}:=\sum\nolimits_{k=1}^{d_j}[\bcalT]_{i_1,\cdots,i_{j-1},k,i_{j+1},\cdots,i_m}\cdot [\W_j]_{i_j,k}, \ \forall i_{j'}\in [d_{j'}] \textrm{ for } j'\neq j; \forall i_j\in[p_j]. 
$
If $\bcalT$ has Tucker ranks $\r=(r_1,\cdots, r_m)^{\top}$, there exist $\bcalC\in\RR^{r_1\times\cdots\times r_m}$ and $\U_j\in\RR^{d_j\times r_j}$ satisfying $\U_j^{\top}\U_j=\I_{r_j}$ for all $j\in[m]$ such that $
\bcalT=\bcalC\cdot(\U_1,\cdots,\U_m):=\C\times_1\U_1\times_2\cdots\times_m \U_m
$, known as Tucker decomposition. More details of tensor algebra can be found in \cite{kolda2009tensor}.

Define $r^* = r_1\cdots r_m$, $d_{\submax} := \max_{i\in [m]}d_i$, and $d_{\submin} := \min_{i\in [m]} d_i$. Let $\d:=(d_1,\cdots, d_m)$. Denote $\MM_{\r}$ the collection of all $m$-way tensors of size $d_1\times \cdots\times d_m$ whose Tucker ranks are at most $\r$. 
Let $\dof := r^* + \sum_{i=1}^m d_ir_i$ be the degree of freedom of $\bcalT\in\MM_{\r}$.

\subsection{Generalized low-rank tensor learning}\label{sec:generalproblem}
A collection of tensorial data $\mathfrak{D}:=\{\mathfrak{D}_t: t=0, 1,\cdots,T-1\}$ is sequentially observed,  where we assume for now that the time horizon $T$ is known.  The scenario where $T$ is unknown requires a more involved treatment and will be specifically investigated in Section~\ref{sec:adaptive}. The sequence $\{\mathfrak{D}_t\}_{t\leq T-1}$ is i.i.d. sampled from a distribution characterized by an {\it unknown} tensor $\bcalT^{\ast}\in\RR^{d_1\times\cdots\times d_m}$ with  ranks $\r=(r_1,\cdots,r_m)$.  Here,  $r_j\ll d_j$. Without loss of generality,  let $\mathfrak{D}_t=(\bcalX_t, Y_t)$,  where $\bcalX_t$ is a covariate tensor of size $d_1\times\cdots\times d_m$ and $Y_t\in\RR$ denotes its corresponding response. 

Let $\ell(\cdot, \cdot): \RR^{d_1\times\cdots\times d_m}\times (\RR^{d_1\times\cdots\times d_m}, \RR)\rightarrow \RR$ be a loss function such that the incurred loss for the $t$-th observation is $\ell(\bcalT, \mathfrak{D}_t)$ at the estimate $\bcalT$.  Offline tensor learning is formulated as the following non-convex optimization program:
\begin{equation}\label{eq:offline}
\min_{\bcalT\in\MM_{\r}}\ \mathfrak{L}(\bcalT, \mathfrak{D}):=\sum_{t=0}^{T-1} \ell(\bcalT, \mathfrak{D}_t).
\end{equation}
Some particularly interesting examples are as follows. 

\begin{example}[linear regression]\label{ex:regression}
The observation satisfies $Y_t=\langle \bcalX_t, \bcalT^{\ast}\rangle+\epsilon_t$, where $\langle\cdot,\cdot \rangle$ denotes the Euclidean inner product, and the noise $\epsilon_t$ is {\it centered} sub-Gaussian with proxy variance $\sigma^2$. The square loss gives $\ell(\bcalT, \mathfrak{D}_t)=(Y_t-\langle \bcalX_t, \bcalT\rangle)^2/2$.  It finds applications in quantum state tomography \citep{gross2010quantum,xia2016estimation}, spatio-temporal forecasting \citep{arroyo2021inference}, multi-task learning \citep{chen2011integrating}, and 3D imaging processing \citep{guo2011tensor}, among others. 
\end{example}

\begin{example}[logistic regression]\label{ex:logisticregression}
Conditioned on $\bcalX_t$, the response follows a Bernoulli distribution,  $Y_t\sim {\rm Ber}(p_t)$ with  $p_t=f(\langle \bcalX_t, \bcalT^{\ast}; \sigma\rangle)$, where $f(u;\sigma)=(1+e^{-u/\sigma})^{-1}$ is the {\it logistic link}. The negative log-likelihood as a loss function is given by
$
\ell (\bcalT, \mathfrak{D}_t) = -Y_t\log \big(f(\inp{\bcalX_t}{\bcalT})\big) - (1-Y_t)\log \big(1 - f(\inp{\bcalX_t}{\bcalT})\big).
$
\end{example}

\begin{example}[Poisson regression]\label{ex:poissonregression}
Conditioned on $\bcalX_t$,  the response $Y_t$ follows the Poisson distribution $Y_t\sim \text{Pois}(I\exp(\inp{\bcalX_t}{\bcalT^*}))$, where $I>0$ stands for the intensity parameter. 
The negative log-likelihood is
$
\ell(\bcalT,\mathfrak{D}_t) = -I^{-1}Y_t\inp{\bcalX_t}{\bcalT} + \exp(\inp{\bcalX_t}{\bcalT}).
$
It finds applications in hyper-spectral imaging \citep{zhang2021low}, positron-emission tomography \citep{ollinger1997positron}, astronomical images \citep{molina1994hierarchical}, photon-limited imaging \citep{salmon2014poisson}, and online click-through data analysis \citep{shan2016predicting}.  
\end{example}

\begin{example}[noisy tensor completion]\label{ex:completion}
The pair $(\bcalX_t, Y_t)$ represents a random noisy entry of $\bcalT^{\ast}$. It is often assumed that $\bcalX_t$ is uniformly sampled from $\big\{(d^{\ast})^{1/2}\bcalE_{\omega}: \omega\in[d_1]\times\cdots\times [d_m] \big\}$, where the $\omega$-th entry of $\bcalE_{\omega}$ is one and all other entries are zero. The observation satisfies $Y_t=\langle \bcalX_t, \bcalT^{\ast}\rangle+\epsilon_t$, where the noise $\epsilon_t$ is {\it centered} sub-Gaussian with proxy variance $\sigma^2$. The square loss is given by $\ell(\bcalT, \mathfrak{D}_t)=(Y_t-\langle \bcalX_t, \bcalT\rangle)^2/2$. Besides low-rankness, the tensor $\bcalT^{\ast}$ possess an additional structure known as {\it incoherence} (see Section~\ref{sec:oNTC} for more details). Tensor completion has been studied by \cite{barak2016noisy, yuan2016tensor, bi2018multilayer,xia2019polynomial} and references therein. 
\end{example}

\begin{example}[binary tensor learning]\label{ex:binary}
The covariate $\bcalX_t$ is sampled as in noisy tensor completion, while $Y_t$ follows a Bernoulli distribution $Y_t\sim {\rm Ber}(p_t)$ with $p_t=f(\langle \bcalX_t, \bcalT^{\ast}\rangle; \sigma)$ being the logistic link. The loss function is typically the negative log-likelihood. Binary tensor learning has been studied by \cite{cai2013max, davenport20141, wang2020learning, han2022optimal, cai2022generalized} and references therein.
\end{example}

The objective function in (\ref{eq:offline}) can be minimized, {\it at least locally}, by several gradient-type algorithms \citep{xia2019polynomial, cai2019nonconvex,han2022optimal, cai2022generalized,kressner2014low}. The convergence performance and statistical behavior of these algorithms are well-understood in the offline setting. 
At the $t$-th iteration, with the current estimate $\bcalT_{t}$, these algorithms compute the gradient $\nabla\mathfrak{L}(\bcalT_{t},\mathfrak{D})$ defined over full dataset. Their theoretical investigations crucially rely on a certain concentration property of $\nabla\mathfrak{L}(\bcalT_{t}, \mathfrak{D})$, for which summing over a large dataset is usually necessary. This property, under suitable conditions, typically guarantees a contraction of error with high probability: 
\begin{align}\label{eq:offline-error}
\|\bcalT_{t+1}-\bcalT^{\ast}\|_{\rm F}\leq (1-\gamma)\cdot \|\bcalT_{t}-\bcalT^{\ast}\|_{\rm F}+\textsf{Statistical Error},
\end{align}
where $\gamma\in(0,1)$ is a constant independent of dimensions $\d$ and horizon $T$. 

\subsection{Online Riemannian gradient descent}\label{sec:algorithm}
We propose an online tensor learning algorithm called oRGrad, based on Riemannian gradient descent. Unlike the conventional method \citep{kressner2014low,cai2022generalized},  it computes the gradient using a single observation. When a new observation $\mathfrak{D}_t$ arrives, with the current estimate $\bcalT_t$, the Riemannian gradient $\calP_{\TT_{t}}\nabla \ell(\bcalT_{t}, \mathfrak{D}_t)$ is calculated. Here, $\calP_{\TT_{t}}$ denotes the projection on $\TT_{t}$,  the tangent space of the manifold $\MM_{\r}$ at $\bcalT_{t}$. 
The Riemannian gradient has a rank of at most $2\r$ dimension-wise, facilitating the subsequence computation of low-rank approximation. Closed-form expressions for the Riemannian gradient are well-known; see the Appendix for details. Notably, in some applications, the naive gradient is inherently low-rank. For example, in online tensor completion, it is rank-one, making the Riemannian gradient unnecessary. In such cases, the naive gradient suffices. In this section, we focus on a constant step size and a known time horizon, which facilitates a clear presentation of the dynamics of oRGrad. This serves as the foundation for the adaptive oRGrad algorithm discussed in Section~\ref{sec:adaptive}, where step sizes are adaptively chosen and the time horizon is unknown. 

ORGrad then updates the estimate to $\bcalT_{t}^+=\bcalT_{t}-\eta\cdot \calP_{\TT_{t}}\nabla \ell (\bcalT_{t}, \mathfrak{D}_t)$,  which is typically not an element of $\MM_{\r}$.  To address this,  we apply the higher order singular value decomposition (HOSVD), denoted by $\textsf{HOSVD}_{\r}$,  to retract $\bcalT_t^+$ back into $\MM_{\r}$.  Let $\U_{t,j}^+$ denote the top-$r_j$ left singular vectors of $\calM_j(\bcalT_{t}^{+})$.  Then,  
$$
\textsf{HOSVD}_{\r}(\bcalT_{t}^+):=\bcalT_{t}^+\cdot\big(\U_{t,1}^{+}\U_{t,1}^{+\top}, \cdots, \U_{t,m}^{+}\U_{t,m}^{+\top}\big).  
$$
Finally,  the estimate is updated to $\bcalT_{t+1}\leftarrow \textsf{HOSVD}_{\r}(\bcalT_{t}^+)$.  The detailed steps of oRGrad are enumerated in Algorithm~\ref{alg:oRGrad}. 
Refer to the Appendix for further details on the computational cost.

\begin{algorithm}[H]
	\caption{Online Riemannian gradient descent (oRGrad) -- {\it if time horizon is known}}
	\begin{algorithmic}
		\State{\textbf{Input:} initialization $\bcalT_0$, time horizon $T$ and step size $\eta>0$;}
		\For{$t=0, 1,2,\ldots,T-1$}
		\State{$\bcalG_{t} \leftarrow \nabla \ell(\bcalT_{t},\mathfrak{D}_t)$;}
		\State{$\bcalT_{t}^+ \leftarrow \bcalT_{t} -\eta\cdot \calP_{\TT_{t}}(\bcalG_{t})$;}
		\State{$\bcalT_{t+1} \leftarrow \textsf{HOSVD}_{\r}(\bcalT_{t}^+)$;}
		\EndFor
		\State{\textbf{Output:} $\bcalT_{T}$.}
	\end{algorithmic}
	\label{alg:oRGrad}
\end{algorithm}

\paragraph*{Online initialization} Note that oRGrad requires an initial estimator $\bcalT_0$. Our theorems require $\bcalT_0$ to be sufficiently close to the true value, as is typically necessary in the literature \citep{han2022optimal,cai2022generalized}. The design of the online initialization algorithms varies depending on specific applications, which will be provided in later sections.

Compared to its offline counterpart,  the convergence dynamics of oRGrad exhibit more local volatility.  The rationale is simple: at each iteration, the gradient $\nabla \ell(\bcalT_{t}, \mathfrak{D}_t)$ is computed on a single datum,  which can have remarkably high variance.  Consequently, contraction behavior like (\ref{eq:offline-error}) does not hold true for oRGrad.  Instead,  we can show,  {\it informally},  that
\begin{align}\label{eq:online-error}
\EE_{\leq t} \|\bcalT_{t+1}-\bcalT^{\ast}\|_{\rm F}\leq (1-\gamma)\cdot \EE_{\leq t-1} \|\bcalT_{t}-\bcalT^{\ast}\|_{\rm F}+\textsf{Statistical Error}.
\end{align}
Essentially,  contraction only holds in expectation,  and the event $\{\|\bcalT_{t+1}-\bcalT^{\ast}\|_{\rm F}> \|\bcalT_{t}-\bcalT^{\ast}\|_{\rm F}\}$ occurs with a non-negligible probability.  The expectation $\EE_{\leq t}$ on LHS of (\ref{eq:online-error}) is taken w.r.t. the randomness of the $\sigma$-field $\FF_t:=\sigma\big((\bcalX_s, Y_s): 0\leq s\leq t\big)$. We also denote $\EE_t$ as the expectation w.r.t. the randomness of $(\bcalX_t, Y_t)$, conditioning on the $\sigma$-filed $\FF_{t-1}$.

The subsequent sections are devoted to rigorously studying the computational and statistical performances of oRGrad for several popular low-rank models.   Some notations will frequently appear in these sections. The {\it signal strength} of $\bcalT^{\ast}$ is defined by $\lambda_{\submin}:=\min_{k\in[m]}\sigma_{r_k}\big(\calM_k(\bcalT^{\ast})\big)$, where $\sigma_r(\cdot)$ denotes the $r$-th largest singular value of a matrix. The {\it condition number} of $\bcalT^{\ast}$ is defined by $\kappa_0:=\lambda_{\submin}^{-1}\lambda_{\submax}$, with $\lambda_{\submax} = \max_{k\in [m]} \|\calM_k(\bcalT^{\ast})\|$.

\section{Online Generalized Tensor Regression}\label{sec:regression}
Throughout this section, we focus on the {\it sub-Gaussian} design under Assumption~\ref{assump:X-design}.

\begin{assumption}\label{assump:X-design}
There exist absolute constants $c_0, C_0>0$ such that 
$
\EE \langle \bcalX_t, \bcalM\rangle=0, \ c_0\leq \EE \langle \bcalX_t, \bcalM\rangle^2/\|\bcalM\|_{\rm F}^2\leq C_0,\ {\rm and}\ \big\|\langle\bcalX_t, \bcalM \rangle \big\|_{\psi_2}\leq C_0 \|\bcalM\|_{\rm F} 
$
for all $t\leq T$ and any $\bcalM$,  where $\|\cdot\|_{\psi_2}$ denotes the Orlicz norm. 
\end{assumption}

Suppose that the loss function is given by 
$
\ell(\bcalT, \mathfrak{D}_t)=h(\langle \bcalT, \bcalX_t\rangle, Y_t),
$
where the function $h(\cdot, \cdot): \RR\times \YY \mapsto \RR$ is locally smooth and strongly convex,  satisfying Assumption~\ref{assump:GLM-loss}.  Here,  $\YY$ denotes the sample space of $Y_t$. 

\begin{assumption}\label{assump:GLM-loss} For any $\alpha>0$, there exist $\gamma_{\alpha}, \mu_{\alpha}>0$ such that 
$$
\gamma_{\alpha}(\theta_1-\theta_2)^2\leq (\theta_1-\theta_2)\big(h_{\theta}(\theta_1, y)-h_{\theta}(\theta_2, y)\big)\leq \mu_{\alpha} (\theta_1-\theta_2)^2
$$
for all $|\theta_1|, |\theta_2|\leq \alpha$ and $y\in \YY$. Here,  $h_{\theta}(\theta,y) = \frac{\partial h}{\partial \theta}(\theta, y)$ is the partial derivative.
\end{assumption}

The range of $\alpha$ depends on specific applications.  For example,  $\alpha=\infty$ and $\gamma_{\infty}=\mu_{\infty}=1$ if $h(\theta, y):=(\theta-y)^2/2$.  
The following assumption restricts $\bcalT^*$ to a bounded space. 

\begin{assumption}\label{assump:GLM-trueT}
There exists a large constant $C>0$ such that $C\|\bcalT^{\ast}\|_{\rm F}\log^{1/2}\dmax\leq \alpha$. Moreover, 
$\EE_{Y_t} h_{\theta}(\langle \bcalT^{\ast}, \bcalX_t\rangle, Y_t) = 0.$
\end{assumption}


Define $g_t(\inp{\bcalX_t}{\bcalT^*}): = \EE_{Y_t}h_{\theta}^2(\inp{\bcalX_t}{\bcalT^*},Y_t)$.  
The statistical accuracy of oRGrad is characterized by 
$
\errtwo^2 := \max_{t=0}^T \max_{|\theta|\leq\alpha}g_t(\theta).
$
For example,  $\errtwo^2 = O(\sigma^2)$ in linear regression.  Another important quantity is 
$
	\errinf := \max_{t=0}^T |h_{\theta}(\inp{\bcalX_t}{\bcalT^*},Y_t)|,  
$
which provides a uniform bound for the noise.  For notational simplicity,  we denote $\errratio:=\errtwo/\errinf$. 

\begin{theorem}\label{thm:gen}
Suppose Assumptions~\ref{assump:X-design}-\ref{assump:GLM-trueT} hold,  the initialization $\bcalT_0\in\MM_{\r}$ satisfies $\fro{\bcalT_0 - \bcalT^*}\leq c_m\mu_{\alpha}^{-1}\gamma_{\alpha}\lambda_{\submin}$ for some sufficient small constant $c_m>0$.  Also assume $C_0\errratio^{-2}\log\dmax\leq\dof$,  the step size $\eta$ satisfies 
\begin{itemize}
		\item [(1)] $C_1\eta\gamma_{\alpha}^{-1}\max\bigg\{\errratio^2\dof\log\dmax, \errratio^2\mu_{\alpha}^2\dof\log\dmax,\mu_{\alpha}\log\dmax,\log^3\dmax\bigg\}\leq 1$,
		\item[(2)] $C_2\eta\gamma_{\alpha}\errratio^{-4}\log^3\dmax\leq 1$ and $C_3\eta\mu_{\alpha}\dof\log^{5/2}\dmax\leq 1$,
\end{itemize}
and the signal strength satisfies
$$
	\lambda_{\submin}^2/ \errinf^2\geq C_4\eta\errratio^{-2}\gamma_{\alpha}^{-1}\dof\log^2\dmax
\quad
{\rm and}
\quad
	\lambda_{\submin}^2/\errtwo^2\geq C_m
	\eta\gamma_{\alpha}^{-3}\mu_{\alpha}^2\dof,
$$
where $C_0,\ldots,C_4>0$ are absolute constants,  and $c_m,  C_m>0$ depends only on $m$.  
Then, there exists an absolute constant $C>0$ such that, with probability exceeding $1-14T\dmax^{-10}$,  for all $t\leq T$,  Algorithm~\ref{alg:oRGrad} guarantees 
$$
\fro{\bcalT_t-\bcalT^*}^2 \leq 2\Big(1-\frac{\eta\gamma_{\alpha}}{8}\Big)^t\fro{\bcalT_0 - \bcalT^*}^2 + C\eta\gamma_{\alpha}^{-1}\dof\errtwo^2.
$$
\end{theorem}

Theorem~\ref{thm:gen} demonstrates that the oRGrad algorithm converges linearly. The contraction rate is even independent of the condition number of $\bcalT^{\ast}$. 
The step size balances computational convergence and statistical error: a larger step size speeds up convergence but increases statistical error and demands a higher SNR. In the following subsections, we explore various applications. Due to space constraint, the application of tensor logistic regression is provided in the Appendix.

\subsection{Linear regression}
Recall that the loss function $h(\theta,y) = \frac{1}{2}(\theta - y)^2$,  and the noise is sub-Gaussian with proxy variance $\sigma^2$.   Assumptions~\ref{assump:GLM-loss} and \ref{assump:GLM-trueT} hold  with $\alpha = \infty$ and $ \gamma_{\alpha} = \mu_{\alpha} = 1$.  Furthermore, we have $\errtwo^2= O(\sigma^2)$ and $\PP\big(\errinf\leq  C\sigma\log^{1/2}\dmax\big)\geq 1 - 2T\dmax^{-10}$  for some absolute constant $C>0$.  This immediately implies the following theorem.  

\begin{theorem}\label{thm:regression}
Suppose Assumption~\ref{assump:X-design} holds,  and the initialization $\bcalT_0\in\MM_{\r}$ satisfies $\fro{\bcalT_0 - \bcalT^*}\leq c_m\lambda_{\submin}$ for some sufficiently small constant $c_m>0$.  Assume also $C_0\log^5\dmax\leq\dof$,  $\eta$ satisfies $C_1\eta\dof\leq 1$, and the signal-to-noise ratio (SNR) satisfies 
$
			\lambda_{\submin}^2/\sigma^2\geq C_m\eta\dof\log^4\dmax,
$
where $C_0, C_1>0$ are absolute constants, and $c_m, C_m>0$ depend only on $m$. 	Then, there exists an absolute constant $C>0$ such that,  with probability exceeding $1-16T\dmax^{-10}$,  for all $t\leq T$,  Algorithm~\ref{alg:oRGrad} guarantees
\begin{equation}\label{eq:thm:regression}
	\fro{\bcalT_t-\bcalT^*}^2 \leq 2\Big(1-\frac{\eta}{8}\Big)^t\cdot\fro{\bcalT_0 - \bcalT^*}^2 + C\eta\dof\cdot\sigma^2.
\end{equation}
\end{theorem}

For ease of interpretation,  assume $m,\rmax\asymp O(1)$ and $d_j \asymp  d$ for all $j\in[m]$.  The oRGrad algorithm converges the fastest by setting $\eta=\eta_{\textsf{comp}}:=c_md^{-1}\cdot \min\big\{\lambda_{\submin}^2/(\sigma^2\log^4d),\ 1\}$.  However,  this aggressive step size results in an estimator that is not even consistent.  If the time horizon $T$ is known, we can set $\eta=\eta_{\textsf{stat}}:=C_m T^{-1}\log d$ for a sufficiently large constant $C_m>0$ depending only on $m$.  When $T\leq d^{C_m}, \lambda_{\submin}/\sigma\leq d^{C_m}$ for a large constant $C_m>0$, this leads to 
$$
\fro{\bcalT_T - \bcalT^*}^2 =O_m\bigg( \sigma^2\cdot \frac{\dof\log d}{T}\bigg),
$$
holding with high probability, which is minimax optimal up to the logarithmic factor. See, e.g., \cite{chen2019non, zhang2020islet,han2022optimal}.  A similar rate is achieved in \cite{han2022online} for online inference under  matrix linear bandit,  requiring stronger SNR and initialization conditions.  However,  the bandit problem is usually more challenging because finding a tradeoff between exploration and exploitation is crucial.  

\subsubsection{Online initialization}  
Our online initialization algorithm consists of two main stages: subspace estimation and core tensor estimation.  See Algorithm~\ref{alg:init:linear-regression} for the detailed steps,  where $T_1$ and $T_2$ represent the sample sizes of the first and second stage,  respectively.  
Each time step in the first stage requires $O(d^* + \sum_{j=1}^m d_j^2)$ storage.  In the second stage, the storage requirement is reduced to  $O\big((r^*)^2\big)$.    The first stage of Algorithm~\ref{alg:init:linear-regression} is an online implementation of the second order moment method \citep{xia2019polynomial} for spectral initialization.  
 The conventional offline method requires $O\big(T_0(d^* + 1)\big)$ storage if $T_0$ observations are used for initialization.  

\begin{algorithm}
	\caption{Online initialization for tensor linear regression}
	\begin{algorithmic}
		\State{Set $\bcalM_0 = {\bf 0},  \N_{1,0}=\cdots=\N_{m,0}={\bf 0}$; }
		\For{$t = 1,\cdots, T_1$}  \Comment{second order moment}
		\State{\hspace{0.1cm}$\N_{j,t} = \N_{j,t-1} + Y_t\calM_j(\bcalM_{t-1})\calM_j(\bcalX_t)^\top + Y_t\calM_j(\bcalX_t)\calM_j(\bcalM_{t-1})^\top,\quad \forall j\in[m]$;}
		\State{\hspace{0.1cm}$\bcalM_t = \bcalM_{t-1} + Y_t\bcalX_t;$}
		\EndFor 
		\State{$\hat\U_j = \svd_{r_j}(\hat\N_j),\ {\rm where}\ \hat\N_j = \frac{1}{T_1(T_1-1)}\N_{j,T_1},\ \forall j\in[m]$;} \Comment {spectral method}
		\For{$t = T_1+1,\cdots, T_1+T_2$} \Comment{Gram matrix}
		\State{Compute $\bcalC_t = \bcalX_t\times_1\hat\U_1^{\top}\cdots \times_m\hat\U_m^{\top}$, }
		\State{\hspace{1.2cm}$\X_{t-T_1} = \X_{t - T_1-1} + \vec(\bcalC_t)\vec(\bcalC_t)^\top \quad {\rm and}\quad \R_{t-T_1} = \R_{t-T_1 -1} + Y_t\bcalC_t$;}
		\EndFor
		\State{$\hat\bcalC = \reshape(\hat\c,\r),\ {\rm where}\ \hat\c = \X_{T_2}^{-1}\R_{T_2}$;} \Comment{linear regression}
		\State{\textbf{Output: }$\hat\bcalT = \hat\bcalC\times_1\hat\U_1\cdots \times_m\hat\U_m.$}
	\end{algorithmic}
	\label{alg:init:linear-regression}
\end{algorithm}

\begin{theorem}\label{thm:init:linear-regression}
	Suppose the sample sizes satisfy $T_1 \geq C_1 r_{\submin}\kappa_0^3\big((d^*)^{1/2} + d_{\submax}\kappa_0^3\big)\log^2\dmax$ and $T_2 \geq C_3(r^* + \log \dmax)$,  and the SNR $\lambda_{\submin}^2/\sigma^2\geq C_2 T_1^{-1}\big((d^*)^{1/2} + d_{\submax}\kappa_0^3\big)\kappa_0\log^2\dmax$ and $\lambda_{\submin}^2/\sigma^2\geq C_4T_2^{-1}(r^* + \log \dmax)$ hold 	for some constant $C_1,C_2, C_3,C_4>0$ depending only on $c_m$. 
	Then,  with probability exceeding $1-8\dmax^{-10}$,  Algorithm \ref{alg:init:linear-regression} outputs $\hat\bcalT$ satisfying 
	$
		\fro{\hat\bcalT - \bcalT^*}\leq c_m\lambda_{\submin},  
	$
	where $c_m\in(0,1)$ is a constant depending only on $m$.  
\end{theorem}

Theorem~\ref{thm:init:linear-regression} requires a total sample size on the order of $((d^*)^{1/2} + \dmax)\log \dmax$,  which matches the best-known existing results.  See,  e.g.,  \cite{shen2022computationally},  \cite{han2022optimal} and references therein.

\subsection{Poisson regression}
Recall from Example \ref{ex:poissonregression} that $Y_t|\bcalX_t \sim\text{Pois}\big(I\exp(\inp{\bcalX_t}{\bcalT^*})\big)$ with $I$ being the {\it intensity} parameter.  Assumption \ref{assump:GLM-loss} holds for the loss $h(\theta, y) = 	-I^{-1}y\theta+ e^{\theta}$ with $\gamma_{\alpha} = e^{-\alpha}$ and $\mu_{\alpha} = e^{\alpha}$ for any $\alpha>0$.   Without loss of generality,   we assume $I\geq 1$.

\begin{theorem}\label{thm:poissonregression}
Suppose that Assumptions~\ref{assump:X-design} and \ref{assump:GLM-trueT} hold,  the initialization satisfies $\fro{\bcalT_0 - \bcalT^*}\leq c_me^{-2\alpha}\lambda_{\submin}$ for some sufficiently small $c_m>0$, 
	$\dof\geq \tilde C_1 I^{-1}\log^3\dmax$ with $\tilde C_1 = C_1e^{-\alpha}\log^{-2}(1+2I^{-1}e^{-\alpha})$, and 
	$ \tilde C_2\eta\dof\max\{I\log^{-1}\dmax, \log^{5/2}\dmax\}\leq 1 $ with $ \tilde C_2 = C_2 e^{\alpha}\log^2(Ie^{-\alpha}\log 2 + 1)$,
where $C_1, C_2>0$ are absolute constants and $c_m$ depends only on $m$. Then,  there exists an absolute constant $C_3>0$ such that,  with probability at least $1-16T\dmax^{-10}$,  for all $t\leq T$,  Algorithm~\ref{alg:oRGrad} guarantees
$$
\fro{\bcalT_t-\bcalT^*}^2 \leq 2\Big(1-\frac{\eta e^{-\alpha}}{8}\Big)^t\fro{\bcalT_0 - \bcalT^*}^2 + C_3e^{2\alpha}\eta I^{-1}\dof.
$$
\end{theorem}

If $m,\rmax,\alpha \asymp O(1)$ and $d_j \asymp d$,  Theorem~\ref{thm:poissonregression} requires $\eta\leq \eta_{\textsf{comp}}=\tilde O(\frac{1}{dI})$,  where $\tilde O$ hides logarithmic factors.   By setting $\eta=\eta_{\textsf{stat}}:=C_mT^{-1}e^{\alpha}\log d$ with a large enough constant $C_m>0$, we get
$$
\fro{\bcalT_T-\bcalT^*}^2 =O_m\bigg( e^{3\alpha} I^{-1}\cdot \frac{\dof\log d}{T}\bigg).
$$
Note that $I^{-1/2}$ is often interpreted as noise level \citep{han2022optimal,cai2022generalized}.

\subsubsection{Online initialization}\label{sec:init-poisson}
We implement online spectral initialization by unfolding a tensor into a matrix of balanced sizes. Define $\calI: = \arg\min_{\calI\in [m]}\max\{d_{\calI}, d_{\calI}^{-}\}$, where $d_{\calI}: = \prod_{i\in\calI}d_i$ and $d_{\calI}^{-}:=d^{\ast}/d_{\calI}$. Let $\calM_{\calI}:\RR^{d_1\times\cdots\times d_m}\rightarrow \RR^{D_1\times D_2}$ denote the linear operator which unfolds a tensor into a matrix of size $d_{\calI}\times d_{\calI}^{-}$. Define $r_{\calI}$ and $r_{\calI}^{-}$ similarly. 

\begin{algorithm}
	\caption{Online initialization for Poisson regression}
	\begin{algorithmic}
		\State{Set $\M_0 = {\bf 0}$,  $\calI\subset[m]$, $R=\max\{r_{\calI}, r_{\calI}^{-}\}$, and $g(y):=\log\big((y+1/2)/I\big)$; }
		\For{$t = 1,\cdots, T$}
		\State{$\M_t = \M_{t-1} + g(Y_t)\calM_{\calI}\big(\bcalX_t\big)$;}
		\EndFor
		\State{$\widetilde\bcalT = \calM_{\calI}^{-1}(\hat\M),$ where $\widehat \M=\svd_R(\M_{T})$;}
		\State{$(\widehat\bcalC,\widehat\U_1,\cdots,\widehat\U_m) = \hosvd(\widetilde\bcalT)$;}
		\State{\textbf{Output: }$\widehat\bcalT = \widehat\bcalC\times_1\widehat\U_1\cdots \times_m\widehat\U_m$.}
	\end{algorithmic}
	\label{alg:init:poisson-regression:spectral}
\end{algorithm}

\begin{theorem}\label{thm:init:poisson}
Let $D_{\submax}:=\max\{d_{\calI}, d_{\calI}^{-}\}$ and $R=\max\{r_{\calI}, r_{\calI}^{-}\}$. Suppose that $I \geq C_{1,m}R\lambda_{\submin}^{-2}e^{5\alpha}\log \dmax$ and $T \geq C_{2,m}D_{\submax}\log(D_{\submax})e^{4\alpha}(1+\alpha)^2\lambda_{\submin}^{-2}$	for some constants $C_{1,m},C_{2,m}>0$ depending only on $m$. Then, with probability exceeding $1-5T\dmax^{-100}$, the output of Algorithm~\ref{alg:init:poisson-regression:spectral} satisfies
$
\fro{\hat\bcalT - \bcalT^*}\leq c_{0,m}e^{-2\alpha}\lambda_{\submin} 
$
for some small constant $c_{0,m}>0$. 
\end{theorem}

If $d_j\asymp d$ and $\lambda_{\submin}, m=O(1)$,  Theorem~\ref{thm:init:poisson} requires  a total sample size on the order $d^{\lceil\frac{m}{2}\rceil}\log d$ for a warm initialization. In light of the online nature, Algorithm~\ref{alg:init:poisson-regression:spectral} consumes only $O(d^{\ast})$ storage.

\section{Online Noisy Tensor Completion}\label{sec:oNTC}
Online noisy tensor completion (Example~\ref{ex:completion}) aims to reconstruct a tensor by sequentially observing its entries with noise. We assume that $\bcalX_t$ is uniformly sampled from the scaled orthonormal basis $\big\{(d^{\ast})^{1/2}\bcalE_{\omega}: \omega\in[d_1]\times\cdots\times [d_m] \big\}$ and the noise $\epsilon_t$ is centered sub-Gaussian with proxy variance $\sigma^2$. We equip oRGrad with the square loss $\ell(\bcalT, \mathfrak{D}_t):=(\langle \bcalX_t, \bcalT\rangle-Y_t)^2/2$.

For an orthonormal matrix $\U\in\RR^{d\times r}$ satisfying $\U^{\top}\U=\I_r$, the incoherence of $\U$ is defined as 
$
\incoh(\U) := (d/r)\max_{i\in[d]}\ltwo{\U^\top\e_i}^2.
$
The smaller value of $\incoh(\U)$ indicates that the ``information'' carried by $\U$ is more evenly distributed across its rows. Tensor completion becomes an ill-posed problem if some entries are significantly larger than others. The following incoherence assumption rules out these ill-posed scenarios. 

\begin{assumption}\label{assump:incoherence}
	Let $\bcalT^* = \bcalC^*\times_{j=1}^m\U_j^*\in\MM_{\r}$ with $\{\U_{j}^*\}_{j=1}^m$ being orthonormal matrices. There exists a $\mu>0$ such that $\incoh(\bcalT^{\ast}):=\max_{j=1}^m\incoh(\U_j^*)\leq\mu$.
\end{assumption}

\begin{theorem}\label{thm:localrefinement:completion}
 Suppose that Assumption \ref{assump:incoherence} holds,  the initialization satisfies $\bcalT_0\in\MM_{\r}$, $\fro{\bcalT_0 - \bcalT^*}\leq c\lambda_{\submin}$ for some $c\in(0, 0.5)$, $\incoh(\bcalT_0)\leq10\mu \kappa_0^2$,  the step size satisfies $C_{0,m}\eta\kappa_0^{4m+2}\mu^{2m-1}\dmax(r^*)^2\rmin\leq 1$, 
and the signal-to-noise ratio satisfies
$
(\lambda_{\submin}/\sigma)^2\geq C_{0,m}\eta\mu^{2m-2}\kappa_0^{4m-8}(\dmax^2/\dof)(r^{\ast}/r_{\submin})^2\log^2\dmax,
$
where $C_{0,m}>0$ and $c>0$ are absolute constants depending on $m$ only. Then, there exist absolute constants $C_1>0$  and $C_{2,m}, C_{3,m}>0$ depending only on $m$ such that,  with probability exceeding $1-3T\dmax^{-10}$, for all $t\leq T$, Algorithm~\ref{alg:oRGrad} guarantees
	\begin{align*}
		\fro{\bcalT_t - \bcalT^*}^2&\leq 2\Big(1-\frac{\eta}{4}\Big)^t\fro{\bcalT_0-\bcalT^*}^2+C_1\eta \dof\sigma^2\\
		\linf{\bcalT_t-\bcalT^*} &\leq C_{2,m}\kappa_0^{m+2}\mu^{m/2}\Big(1-\frac{\eta}{4}\Big)^{\frac{t}{2}}\lambda_{\submax}\sqrt{\frac{r^*}{d^*}} + C_{3,m}\kappa_0^{m+3}\mu^{m/2}\sigma\sqrt{\frac{\eta r^*\dof}{d^*}}. 
	\end{align*}
\end{theorem}

Let us discuss the implications of Theorem~\ref{thm:localrefinement:completion}.  For ease of interpretation, we assume $m,\rmax,\mu,\kappa_0\asymp O(1)$ and $d_j \asymp d$. 

\paragraph*{Linear convergence and reduced sample size} Theorem~\ref{thm:localrefinement:completion} demonstrates that oRGrad converges linearly in both the Frobenius norm and sup-norm.  This implies that oRGrad delivers an $\epsilon$-accurate estimate in noiseless tensor completion after $T=\Omega(\eta^{-1}\log\epsilon^{-1})$ iterations. By selecting $\eta=\eta_{\textsf{comp}}\asymp d^{-1}$, the required number of iterations becomes $T=\Omega(d\log\epsilon^{-1})$. This represents a significant improvement over existing results.  For instance, \cite{yuan2016tensor,xia2019polynomial,xia2021statistically,tong2021scaling,cai2019nonconvex} all require a sample size condition of $\tilde{\Omega}(d^{m/2})$  in offline tensor completion, even when a warm initialization is provided. Nevertheless, we remark that a sample size of $\tilde{\Omega}(d^{m/2})$ is still necessary to obtain a desirable initialization.  


\paragraph*{Theoretical and technical benefits} If the time horizon $T$ is known, one can set  $\eta=\eta_{\textsf{stat}}\asymp C_mT^{-1}\log d_{\submax}$ for some large constant $C_m>0$ so that oRGrad outputs an estimator with a Frobenius-norm error rate 
$$
\|\bcalT_T-\bcalT^{\ast}\|_{\rm F}^2=O_m\bigg(\sigma^2\cdot \frac{\dof\log d_{\submax}}{T}\bigg), 
$$
 which is minimax optimal up to a logarithmic factor \citep{xia2021statistically}. A surprising theoretical benefit of online algorithm for noisy tensor completion is that one can easily derive the {\it entry-wise} error rate, which is usually much more challenging yet practically useful. The existing literature on offline noisy matrix/tensor completion \citep{cai2019nonconvex, chen2021bridging} often resorts to the rather complicated leave-one-out analysis framework to establish the entry-wise error rate. Benefited from the online nature of oRGrad, we can apply martingale techniques and derive a sharp upper bound for the entry-wise error. Indeed, oRGrad outputs an estimator with a sup-norm error rate
$$
\|\bcalT_T-\bcalT^{\ast}\|_{\ell_\infty}^2=O_m\bigg(\frac{\sigma^2}{d^{\ast}}\cdot \frac{\dof \log d_{\submax}}{T}\bigg), 
$$
suggesting that the entry-wise error is approximately $1/d^{\ast}$ of the Frobenius-norm error.  A technical benefit of online algorithm is that it does not require trimming. Maintaining the incoherence condition is crucial in analyzing the convergence of tensor completion algorithms. Most existing literature (except \cite{cai2019nonconvex, chen2021bridging}) applies an additional trimming procedure to ensure the incoherence property. Interestingly, we can take advantage of martingale techniques and  prove the incoherence property during the update of the oRGrad algorithm.  

\subsection{Initialization}
We apply the second-order moment method, originally proposed by \cite{xia2019polynomial}, to obtain a warm initialization for tensor completion (see also \cite{xia2021statistically}). This approach is motivated by the fact that $\U_j^*$ constitutes the top $r_j$ eigenvectors of the following matrix:
\begin{align*}
	\EE \frac{1}{T_1(T_1-1)}\sum_{1\leq i<i'\leq T_1}Y_iY_{i'}(\calM_j(\bcalX_{i})\calM_j(\bcalX_{i'})^\top+\calM_j(\bcalX_{i'})\calM_j(\bcalX_{i})^\top).
\end{align*}
The detailed implementation is provided in Algorithm~\ref{alg:init:tensor-completion}. Given $\tilde\U\in\RR^{d\times r}$,  we define $\trim(\tilde\U, \mu)=\bar \U(\bar \U^{\top}\bar \U)^{-1/2}$, where $[\bar \U]_{i:}=[\tilde \U]_{i:}\cdot \min\big\{1, (\mu r/d)^{1/2}/\|[\tilde \U]_{i:}\|_{\ell_2}\big\}$. 

\begin{algorithm}
	\caption{Initialization for tensor completion}
	\begin{algorithmic}
		\State{Set $\X_0={\bf 0}$ and collect data $\{\bcalX_t, Y_t\}_{t=1}^{T_1}$;} \Comment{spectral initialization}
		\State{Compute $\tilde\N_j: = \frac{1}{T_1(T_1-1)}\sum_{1\leq i<i'\leq T_1} Y_iY_{i'}(\calM_j(\bcalX_{i})\calM_j(\bcalX_{i'})^\top+\calM_j(\bcalX_{i'})\calM_j(\bcalX_{i})^\top)$;}
		\State{Trimming and SVD: $\hat\U_j = \trim(\tilde\U_j,\mu)$ where $\tilde\U_j = \svd_{r_j}(\tilde\N_j)$;}
		\For{$t = T_1+1,\cdots, T_1+T_2$} \Comment{core tensor initialization}
		\State{Compute $\bcalC_t = \bcalX_t\times_1\hat\U_1^{\top}\cdots \times_m\hat\U_m^{\top}$}
		\State{\hspace{1.2cm}$\X_{t-T_1} = \X_{t - T_1-1} + \vec(\bcalC_t)\vec(\bcalC_t)^\top$ and $\R_{t-T_1} = \R_{t-T_1 -1} + Y_t\bcalC_t$;}
		\EndFor
		\State{$\hat\bcalC = \reshape(\hat\c,\r)$ where $\hat\c = \X_{T_2}^{-1}\R_{T_2}$;}
		\State{\textbf{Output: }$\hat\bcalT = \hat\bcalC\times_1\hat\U_1\cdots \times_m\hat\U_m$.}
	\end{algorithmic}
	\label{alg:init:tensor-completion}
\end{algorithm}

\begin{theorem}\label{thm:init:completion}
 Suppose that Assumption \ref{assump:incoherence} holds, $T_1 \geq C_1 \big((d^*)^{1/2} + d_{\submax}r^*\rmin^{1/2}\mu^m\kappa_0^3\big)\allowbreak r^*\rmin^{1/2}\mu^m\kappa_0^3\log\dmax$, $T_2 \geq C_3\mu^m r^*\log \dmax$, and $\lambda_{\submin}^2/\sigma^2\geq C_4T_2^{-1}\mu^mr^*\log\dmax+C_2 T_1^{-1}\allowbreak  \big((d^*)^{1/2} + d_{\submax}\rmin^{1/2}\mu^m\kappa_0^3r^*\big)\kappa_0\rmin^{1/2}\log\dmax$ for some constant $C_1,C_2, C_3,C_4>0$ depending only on $m$. Then,  with probability exceeding $1-4T_1\dmax^{-100}$, the output of Algorithm \ref{alg:init:tensor-completion} satisfies $\fro{\hat\bcalT - \bcalT^*}\leq c_{0,m}\lambda_{\submin}$ for some small constant $c_{0,m}\in(0,1/2)$. 
\end{theorem}

Theorem~\ref{thm:init:completion} shows that $\tilde\Omega\big(((d^*)^{1/2} + \dmax)\log\dmax\big)$ randomly sampled entries suffice to provide a warm initialization.

\section{Online Binary Tensor Learning}\label{sec:oBinary}
The covariate $\bcalX_t$ is sampled in the same manner as in tensor completion (see Section~\ref{sec:oNTC}), but the response $Y_t$ is binary, following a Bernoulli distribution $Y_t\sim {\rm Ber}(p_t)$ with $p_t = f(\inp{\bcalX_t}{\bcalT^*}; \sigma)$. Here, $f(\theta;\sigma)=f(\theta):= (1+e^{-\theta/\sigma})^{-1}$ is the logistic link function. The loss function $\ell(\bcalT, \mathfrak{D}_t)=h(\langle \bcalT, \bcalX_t\rangle, Y_t)$ is the negative log-likelihood $h(\theta, y) = -y\log(f(\theta)) - (1-y)\log (1-f(\theta))$. 
The following assumption imposes an upper bound on the entrywise magnitude of $\bcalT^{\ast}$. 

\begin{assumption}\label{assump:binary}
There exist $\mu,  \alpha>0$ such that $\incoh(\bcalT^*)\leq\mu$ and $2C_0(r^*/\rmax)^{1/2}\fro{\bcalT^*}\leq \alpha$, where $C_0 = 20^{m/2}\kappa_0^{m+1}\mu^{m/2}$. 
\end{assumption}

Define $L_{\alpha}:=\sup_{|\theta|\leq \alpha} f'(\theta)\big(f(\theta)(1-f(\theta))\big)^{-1}$ and 
\begin{align*}
	\gamma_{\alpha} &:= \min\bigg\{\inf_{|\theta|\leq\alpha}\frac{(f'(\theta))^2 - f^{''}(\theta)f(\theta)}{f^2(\theta)},\quad\inf_{|\theta|\leq\alpha}\frac{f^{''}(\theta)(1-f(\theta)) + (f'(\theta))^2}{(1-f(\theta))^2}\bigg\},\\
	\mu_{\alpha}&:= \max\bigg\{\sup_{|\theta|\leq\alpha}\frac{(f'(\theta))^2 - f^{''}(\theta)f(\theta)}{f^2(\theta)},\quad\sup_{|\theta|\leq\alpha}\frac{f^{''}(\theta)(1-f(\theta)) + (f'(\theta))^2}{(1-f(\theta))^2}\bigg\}.
\end{align*}

Without loss of generality, we focus on the regime where $\dmax\asymp \dmin, \rmax\asymp \rmin$ and fix $\sigma=1$.  In this case,  we have $L_{\alpha}=1, \gamma_{\alpha}=e^{\alpha}(1+e^{\alpha})^{-2}$ and $\mu_{\alpha}=1/4$.

\begin{theorem}\label{thm:binary}
Suppose Assumption \ref{assump:binary} holds and the initialization $\bcalT_0\in\MM_{\r}$ satisfies $\fro{\bcalT_0 - \bcalT^*}\leq c_m\gamma_{\alpha}\lambda_{\submin}$	for some small constant $c_m>0$ depending on $m$ only.   If  $C_{1}\eta\dmax\log\dmax \leq 1$,  $\lambda_{\submin}^2\geq C_2\eta \max\{\gamma_{\alpha}^{-3},\log\dmax\} \dmax $,  and  $\lambda_{\submax}\leq (C_3\dmax)^{-1}$,  where $
 	C_1 = \max\big\{(\kappa_0^2\mu)^{m-1}r^*/\rmin,1\big\}$,  $C_2 = \max\big\{\kappa_0^{4m-4}\mu^{2m-2},\rmin^{-1}\kappa_0^{2m-4}\mu^{m-2}\big\}\cdot r^*\rmin^{-1}$,  and $ C_3 = 20^m\kappa_0^{3m+1}\mu^{\frac{3m}{2}-1}(r^*)^{3/2}/\rmax$,  then there exists an absolute constant $C>0$ such that,  with probability exceeding $1-5Td^{-10}$, for all $t\leq T$,  Algorithm~\ref{alg:oRGrad} guarantees 
$$
\fro{\bcalT_t - \bcalT^*}^2\leq 2\Big(1-\frac{1}{4}\eta\gamma_{\alpha}\Big)^t\fro{\bcalT_0-\bcalT^*}^2+C\eta\gamma_{\alpha}^{-1} \dof.
$$
\end{theorem}

There is an implicit step size requirement in Theorem~\ref{thm:binary} arising from the conditions of $\lambda_{\submin}$ and $\lambda_{\submax}$.  Specifically, we require $\eta\leq c_1\min\big\{\gamma_{\alpha}^3, d^{-3}\log^{-1}d \big\} $,  assuming $d_j\asymp d$ and that $m, \mu, \kappa_0, r_j=O(1)$ for simplicity.  This constraint is necessary to control the incoherence during the iterations of oRGrad.  Additionally, the condition on $\lambda_{\submax}$ is required for controlling the higher-order derivatives of the link function.

By choosing a step size such that $\eta\gamma_{\alpha}\asymp C_m T^{-1}\log \dmax$ with a sufficiently large constant $C_m>0$,  Algorithm~\ref{alg:oRGrad} produces an estimator $\bcalT_T$ satisfying 
$$
\|\bcalT_T-\bcalT^{\ast}\|_{\rm F}^2=O_m\bigg(\gamma_{\alpha}^{-1}\cdot \frac{\dof \log\dmax}{T}\bigg),
$$
which matches existing ones \cite{wang2020learning,han2022optimal,cai2022generalized} in offline binary tensor learning up to logarithmic factors.  Theorem~\ref{thm:binary} requires a sample size of $T=\Omega(d^3)$,  regardless of the dimensionality $m$.  This requirement is due to the implicit constraint on the step size $\eta$.

While our result provides the first theoretical guarantee for online binary tensor learning and achieves an optimal error rate under additional conditions, many open problems remain, such as determining the minimal sample size requirement.

\subsection{Initialization}
The initialization for binary tensor learning is more challenging. Here, we combine a convex optimization method \cite{davenport20141} with HOSVD for initialization. Following the notations in Section \ref{sec:init-poisson}, we unfold the tensor $\bcalT^{\ast}$ into a matrix of balanced dimensions $d_{\calI}\times d_{\calI}^{-}$.

\begin{algorithm}
	\caption{Initialization for binary tensor learning}
	\begin{algorithmic}
		\State{Collect data $\{\bcalX_t, Y_t\}_{t=1}^{T_1}$ and solve}
		\begin{align*}
		\hat\M: = \arg\max_{\M}\ 	L(\M): = &\sum_{t=1}^{T_1}\Big(Y_t\log f(\inp{\calM_{\calI}(\bcalX_t)}{\M}) + (1-Y_t)\log \big(1-f(\inp{\calM_{\calI}(\bcalX_t)}{\M})\big)\Big),\\
		\text{s.t.~}& \nuc{\M}\leq \alpha\sqrt{R}\quad {\rm and}\quad \linf{\M}\leq \alpha/\sqrt{d^*}.
	    \end{align*}
		\State{$\tilde\bcalT = \calM_{\calI}^{-1}(\hat\M)$ and find $(\hat\bcalC,\hat\U_1,\cdots,\hat\U_m) = \hosvd(\tilde\bcalT)$}
		\State{\textbf{Output: }$\hat\bcalT = \hat\bcalC\times_1\hat\U_1\cdots \times_m\hat\U_m$}
	\end{algorithmic}
	\label{alg:init:binary}
\end{algorithm}

\begin{theorem}\label{thm:init-binary}
Suppose Assumption \ref{assump:binary} holds,  let $c_m\in(0,1)$ be a small constant,  and assume that 
	\begin{align*}
		\lambda_{\submin} \geq (C\alpha L_{\alpha}\gamma^{-3})^{1/2}\max\bigg\{\left(\frac{D_{\submax}^2R\log(D_{\submax})}{T_1^2}\right)^{1/4}, \left(\frac{RD_{\submax}}{T_1}\right)^{1/4}\bigg\},
	\end{align*}
where $D_{\submax}:=\max\{d_{\calI}, d_{\calI}^{-}\}$, $R=\max\{r_{\calI}, r_{\calI}^{-}\}$, and $C>0$ depends only on $c_m$.  Then,  with probability at least $1-D_{\submax}^{-1}$,  the output of Algorithm~\ref{alg:init:binary} satisfies  $\fro{\hat\bcalT - \bcalT^*}\leq c_m \gamma_{\alpha}\lambda_{\submin}. $
\end{theorem}

\section{Sub-Optimal Regret of oRGrad using Constant Step Size}\label{sec:regretbound}
Online learning is particularly useful when an immediate prediction is required upon observing the covariate $\bcalX_t$. Regret performance is crucial for evaluating the prediction accuracy of online algorithms. See \cite{hazan2007logarithmic}, \cite{zhang2004solving} and references therein. Fundamentally, regret measures the difference in prediction performance between an online learner and a static player who has the advantage of hindsight and can make predictions as if all observations were available simultaneously. In the current and upcoming sections, we focus on the regret analysis of the oRGrad algorithm. For simplicity, our analysis is confined to linear cases, including tensor linear regression and noisy tensor completion,  i.e.,  the response satisfies $\EE Y_t = \langle \bcalX_t, \bcalT^{\ast} \rangle$.

At time $t$,  given the current estimate $\bcalT_t$, the online learner receives a new observation containing only the covariate information $\bcalX_t$.  The learner is then obliged to make a prediction before the true response $Y_t$ is revealed.  In the context of the linear model, a reasonable prediction is $\hat Y_t=\langle\bcalX_t, \bcalT_t \rangle$.  Prediction accuracy is measured by the difference between $\hat Y_t$ and the expected response $\EE Y_t$. The regret of oRGrad is thus defined by 
$$
{\rm Regret}_T({\rm oRGrad})=R_T:=\EE\bigg[\sum_{t=0}^{T-1}\frac{1}{2}\big(\hat Y_t - \EE Y_t\big)^2\bigg].
$$
Oftentimes, regret can also be defined as the difference between the cumulative loss incurred by the online learner and that of the best offline player. Specifically, let the loss function be $\ell_t(\bcalT):=\ell(\bcalT, \mathfrak{D}_t)=(Y_t-\langle \bcalX_t, \bcalT\rangle)^2/2$ and then the regret of oRGrad can also be defined as
 $R_T:=\inf_{\bcalT\in\MM_{\r}} \sum_{t=0}^{T-1}\EE \big[\ell_t(\bcalT_t)-\ell_t(\bcalT)\big]$.  In both linear regression and tensor completion, the infimum is attained at  $\bcalT^{\ast}$.  Using either definition, and due to the independence between $\bcalX_t$ and $\bcalT_t$,  one can show that 
$$
R_T = \frac{1}{2} \sum_{t = 0}^{T-1}\EE_{\leq t-1} \fro{\bcalT_{t} - \bcalT^*}^2,
$$
where $\EE_{-1}\fro{\bcalT_0 -\bcalT^*}^2 = \fro{\bcalT_0 -\bcalT^*}^2$.
Note that a trivial bound on $R_T$ is $O(T)$.  Online algorithms are deemed effective when the achieved regret is  $o(T)$.  
The regret performance of oRGrad for both linear regression and tensor completion is summarized in the following theorem.
 
\begin{theorem}\label{thm:regret}
Suppose that the conditions of Theorem \ref{thm:regression} (respectively, Theorem \ref{thm:localrefinement:completion}) hold and assume that the time horizon satisfies $Cm^{1/2}\leq T\leq \dmax^{100m}$ for some absolute constant $C>0$. There exists an absolute constant $C_1>0$ such that the sequence output by oRGrad for online tensor linear regression (online tensor completion, respectively) achieves the  regret
\begin{equation}\label{eq:thm:regret}
R_T \leq C_1\big( \eta^{-1}\fro{\bcalT_0 -\bcalT^*}^2 + \eta\cdot \sigma^2 T\dof \big).
\end{equation}
\end{theorem}

The upper bound on the time horizon $T$ is due to technical reasons and is relatively weak,  as the constant $100$ can be replaced by any absolute constant.  If the  horizon $T$ is known, one can set an appropriate step size to minimize the right-hand side of eq. (\ref{eq:thm:regret}).  By setting $\eta\asymp \|\bcalT_0-\bcalT^{\ast}\|_{\rm F}\sigma( T\cdot \dof)^{-1/2}$, oRGrad achieves an $O(T^{1/2})$ regret or,  more precisely,   
\begin{align*}
		R_T =O\Big( T^{1/2}\cdot \dof^{1/2} \cdot \sigma\|\bcalT_0-\bcalT^{\ast}\|_{\rm F} \Big) .
\end{align*}
The regret is proportional to the initialization error $\|\bcalT_0-\bcalT^{\ast}\|_{\rm F}$,  implying that oRGrad predicts more accurately if a better initialization is available.  Additionally,  the regret is proportional to the degree of freedom,  reflecting the role played by model complexity.

A square-root regret bound is typical in online learning (see, e.g., \cite{zhang2004solving}). However, it is well-known in the online learning literature \cite{hazan2007logarithmic} that an optimal regret of  $O(\log T)$ is attainable when the associated optimization program is convex and the loss function is strongly convex. This logarithmic regret can be achieved by the online gradient descent algorithm with decaying step sizes. In our case, we are restricted to a fixed step size, and it is unclear whether our derived regret of $O(T^{1/2})$ is optimal.  Interestingly, it matches the recently established regret bound for low-rank bandit learning \citep{shi2023high}, although the definition of regret in bandit problem differs from ours. We emphasize that logarithmic regret is attainable even when the time horizon is unknown, provided that adaptive choices of step sizes are allowed—a topic we shall explore in the next section.

\paragraph*{Trilemma among computational convergence, statistical error, and regret}
Theorems~\ref{thm:regression} (or Theorem~\ref{thm:localrefinement:completion}, respectively) and~\ref{thm:regret} suggest a trilemma in choosing a \emph{constant} step size for online linear regression (or online noisy tensor completion, respectively). If computational convergence is the primary concern, we can choose a large step size  $\eta=\eta_{\textsf{comp}}\asymp d^{-1}$,  allowing oRGrad to converge within $T=\tilde{O}(d)$ iterations.  On the other hand,  if the goal is to eventually obtain a statistically optimal estimator,  we should set a small step size $\eta=\eta_{\textsf{stat}}\asymp T^{-1}$.  Lastly, as discussed above, the regret bound in equation~(\ref{eq:thm:regret}) is minimized when the step size is fixed at $\eta=\eta_{\textsf{regret}}\asymp  (Td)^{-1/2}$.   It appears that statistical optimality and sharp regret performance cannot be simultaneously achieved with only a fixed step size. Fortunately, we demonstrate in the next section that these two goals can be achieved concurrently by adaptively selecting step sizes.

\section{Optimal Adaptive Online Learning when Horizon is Unknown}\label{sec:adaptive}
While online gradient descent with a fixed step size is convenient for algorithm implementation, it has two limitations. First, oRGrad cannot achieve both statistical optimality and sharp regret simultaneously. Second, determining the optimal step sizes,  $\eta_{\textsf{stat}}$ and $\eta_{\textsf{regret}}$,  requires knowledge of the time horizon $T$,  which is unrealistic in most applications. To tackle these challenges, we introduce the Adaptive-oRGrad algorithm, which adaptively selects the step sizes. Notably, we demonstrate that Adaptive-oRGrad can achieve both statistical optimality and an optimal $O(\log T)$ regret simultaneously, regardless of whether the true time horizon is known. Without loss of generality,  we now assume $T$ is unknown, but the algorithm and theoretical results are also applicable when the horizon $T$ is known.  Similarly to Section~\ref{sec:regretbound}, we focus on the linear cases, including tensor linear regression and noisy tensor completion.

In the initial phase of the adaptive-oRGrad algorithm, we select a step size $\eta_0$ and a phase length $t_0$.  As the algorithm progresses, the step size decreases while the phase length doubles.  More precisely,  during the $k$-th phase, oRGrad is executed with a step size $\eta=2^{-k}\eta_0$ for a duration of $2^{k-1}t_0$ steps. The implementation details of Adaptive-oRGrad algorithm can be found in Algorithm~\ref{alg:adaptive}. 

\begin{algorithm}
	\caption{Adpative-oRGrad -- {\it if time horizon is unknown}}
	\begin{algorithmic}
		\State{\textbf{Input:} Initial step size $\eta_0>0$ and initial phase length $t_0>0$.}
		\For{$k = 1,2,3,\ldots$}
		\State{Set $\eta_k = 2^{-k}\eta_0$}
		\For{$t\in\{(2^{k-1}-1)t_0,\ldots,(2^{k}-1)t_0-1\}$}
		\State{Run oRGrad with step size $\eta=  \eta_k$}
		\EndFor
		\EndFor
	\end{algorithmic}
	\label{alg:adaptive}
\end{algorithm}

The adaptive step size schedule of the Adaptive-oRGrad algorithm is motived by the fact that,  for any unknown time horizon $T$,  approximately half of the iterations are executed with a step size of order $T^{-1}$. This enables adaptive-oRGrad to deliver a statistically optimal estimator as if the time horizon was known  in advance.  Notably,  an optimal $O(\log T)$ regret is also attainable if the initial step size and phase length are carefully chosen. For simplicity, we assume $m, \rmax, \kappa_0, \mu_0\asymp O(1)$ and $d_j\asymp d$ for all $j\in[m]$. The following theorem characterizes the statistical and regret performance of Adaptive-oRGrad. 

\begin{theorem}\label{thm:adaptive}
Suppose the conditions of Theorem \ref{thm:regression} (Theorem \ref{thm:localrefinement:completion}, respectively) hold with the initial step size $\eta_0$ satisfying the conditions of $\eta$ therein. Adaptive-oRGrad algorithm guarantees the following statistical and regret performance for online tensor linear regression (online tensor completion, respectively):
	\newline
\textbf{For statsitical optimality.} If $\eta_0 t_0\geq C_1$ for some absolute constant $C_1 >0$,
		then with probability exceeding $1-16T\dmax^{-10}$,
		\begin{align}\label{eq:thm:adaptive-1}
			\fro{\bcalT_T-\bcalT^*}^2  \leq 2\exp\Big(-c_1t_0\eta_0\log(T/t_0)\Big)\cdot \fro{\bcalT_0-\bcalT^*}^2 + C_{2}t_0\eta_0\sigma^2\cdot \frac{\dof}{T},
		\end{align}
		where $c_1, C_{2}>0$ are some absolute constants.
\newline
\textbf{For optimal regret.} If the initial step size $\eta_0 = O\big((\lambda_{\submin}/\sigma) \cdot (t_0\dof)^{-1/2}\big)$ and the initial phase length $t_0 \geq C_m\dof\cdot\max\{(\sigma/\lambda_{\submin})^2\log^8 d, (\lambda_{\submin}/\sigma)^2\}$ for some constant $C_m>0$ depending only on $m$, then 
\begin{align}
	R_T = O\Big((t_0\dof)^{1/2}\sigma\lambda_{\submin}\cdot \log(T/t_0)\Big). \label{eq:thm:adaptive-2}
\end{align}
\end{theorem}

By Theorem~\ref{thm:adaptive} and by defining $t_0\eta_0/T=:\tilde{\eta}\geq C_1T^{-1}$, the error bound becomes 
\begin{align*}
	\fro{\bcalT_T-\bcalT^*}^2  \leq 2\exp\Big(-C_2\tilde\eta T\log(T/t_0)\Big)\cdot \fro{\bcalT_0-\bcalT^*}^2 + C_{3}\tilde\eta\cdot\dof\sigma^2.
\end{align*}
This should be compared with Theorem \ref{thm:regression} (using the inequality $1+x\leq e^x$), which states that 
\begin{align*}
	\fro{\bcalT_T-\bcalT^*}^2  
	\leq 2\exp(-c_1\eta T)\cdot\fro{\bcalT_0-\bcalT^*}^2 + C\eta\cdot \dof\sigma^2.
\end{align*}
This comparison suggests that,  by introducing an additional term $\log(T/t_0)$ in the exponent,  Adaptive-oRGrad can achieve a faster convergence rate.  This is reasonable because, under the adaptive schedule of step sizes, Adaptive-oRGrad chooses more aggressive step sizes except in the final phase, which enforces faster computational convergence. Nevertheless, both algorithms can deliver the statistically optimal rate $\tilde{O}_p(\sigma^2\cdot \dof /T)$. 

Surprisingly, Theorem~\ref{thm:adaptive} suggests that the Adaptive-oRGrad algorithm can achieve optimal regret. Suppose that the conditions before eq. (\ref{eq:thm:adaptive-2}) hold and that the initial phase length is chosen as $t_0\asymp \dof\cdot\max\{(\sigma/\lambda_{\submin})^2\log^8 d, (\lambda_{\submin}/\sigma)^2\}$.  Then,  eq. (\ref{eq:thm:adaptive-2}) yields the following regret bound:
$$
R_{T}=O\Big(\max\{\sigma^2\log^4d, \lambda_{\submin}^2\} \dof\cdot \log T\Big),
$$
which is exactly (up to the $\log d$ factors) the best regret performance achievable when $\lambda_{\submin}/\sigma\asymp 1$.  This optimality holds even if all available information is exploited at each iteration.  Indeed, at each time $t$, suppose that $\bcalT_t$ is estimated using all the accumulated data before time $t$, i.e., an offline estimate with a sample size of $t$.  The minimax optimal error rate  in squared Frobenius norm that $\bcalT_t$ can achieve is $O_p(\sigma^2\cdot \dof/t)$ \citep{cai2022generalized, han2022optimal, xia2021statistically} .  Consequently, the optimistic accumulated error rate is $\EE\big[\sum_{t\leq T} \|\bcalT_t-\bcalT^{\ast}\|_{\rm F}^2\big]\asymp \sigma^2\dof\cdot \log T$, dictating a lower bound of $O(\sigma^2\dof\cdot \log T)$ for the regret performance. 
For general SNR, the regret bound is tight with respect to $T$.  
From a theoretical perspective,  we can replace $\lambda_{\submin}$ in eq. \eqref{eq:thm:adaptive-2} with $\fro{\bcalT_0-\bcalT^*}$.  However, this replacement will affect the choice of the step size $\eta_0$ since $\fro{\bcalT_0-\bcalT^*}$ is unknown. 

Finally, it is evident that there are scenarios where the conditions prior to equations~(\ref{eq:thm:adaptive-1}) and~(\ref{eq:thm:adaptive-2}) are simultaneously satisfied. In such cases, the Adaptive-oRGrad algorithm can achieve both statistical optimality and optimal regret.

\section{Numerical Experiments and Real Data Analysis}\label{sec:nu}
This section presents both numerical simulation results and real data examples.  We specifically focus on online tensor linear regression and online tensor completion. Our analysis verifies that the chosen step size effectively balances the trade-off between convergence rate and final error rate. Additionally, we adjust the noise level to evaluate its impact on estimation accuracy.  
Throughout this section,  we frequently use the relative error,  defined by
$
\textsf{relative error} = \fro{\hat\bcalT - \bcalT^*}/\fro{\bcalT^*}, 
$
where $\hat\bcalT$ is the output of the oRGrad algorithm and $\bcalT^*$ is the unknown tensor. Due to space constraint, one real data example is provided in the Appendix.  

\subsection{Online tensor linear regression}\label{sec:nu:lr}
In this section, we conduct numerical experiments on online tensor linear regression, focusing on the low-rank tensor $\bcalT^*\in\RR^{d\times d\times d}$ with $d = 30$ and Tucker rank $\r =(2,2,2)$. We generate $\bcalT^*$ in the factorization form $\bcalT^* =\bcalC^*\times_1\U_1^*\times_2\U_2^*\times_3\U_3^*$,  where $\bcalC^*$ has independent and identically distributed (i.i.d.) entries drawn from  $N(0,1)$, and $\U_1^*,\U_2^*,\U_3^*$ have i.i.d. entries drawn from a uniform distribution $\text{Unif}(0,1)$. The corresponding value of $\lambda_{\submin}$ is approximately $2.3$.
We use Algorithm \ref{alg:init:linear-regression} for initialization. 

\paragraph*{Rank Selection}
In order to determine the rank $r_j$ in practice,  we may plot the singular values of $\hat\N_j$ defined in Algorithm \ref{alg:init:linear-regression},  i.e.,  the scree plot as shown in  Figure \ref{fig:rank}.  Notably, there is a significant drop after the second-largest singular values across each dimension.  More precisely,  for each dimension $j\in[m]$,  we set $r_j = \min\big\{k: \sum_{j'=1}^k(\sigma_{j'})^2/\sum_{j'=1}^{d_j}(\sigma_{j'})^2>0.9\big\}$, where $\sigma_1,\cdots, \sigma_{d_j}$ are the non-increasing eigenvalues of $\hat\N_j$.  
\begin{figure}[htbp]
	\centering
	\begin{subfigure}[b]{0.3\textwidth}
		\centering
		\includegraphics[width=\textwidth]{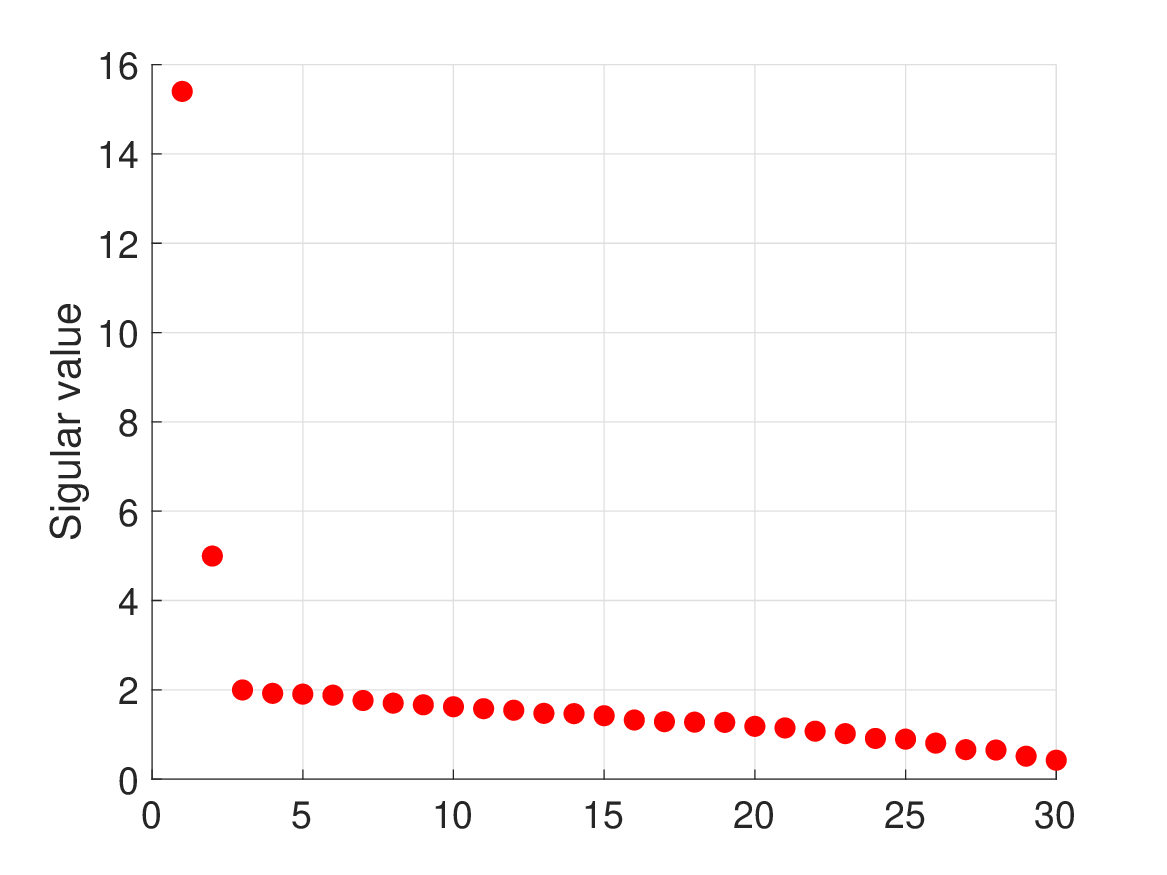}
		\caption{1st dimension}
		\label{fig:fig1}
	\end{subfigure}
	\hfill
	\begin{subfigure}[b]{0.3\textwidth}
		\centering
		\includegraphics[width=\textwidth]{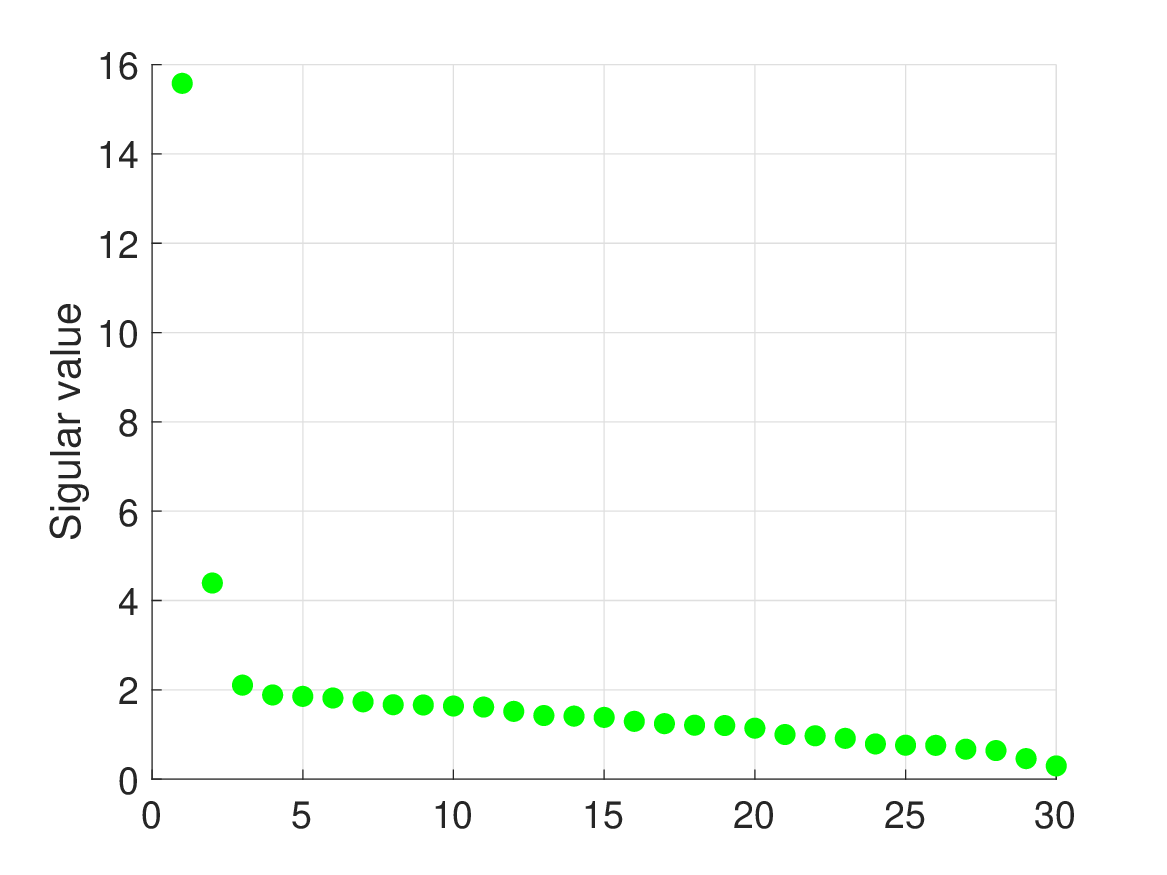}
		\caption{2nd dimension}
		\label{fig:fig2}
	\end{subfigure}
	\hfill
	\begin{subfigure}[b]{0.3\textwidth}
		\centering
		\includegraphics[width=\textwidth]{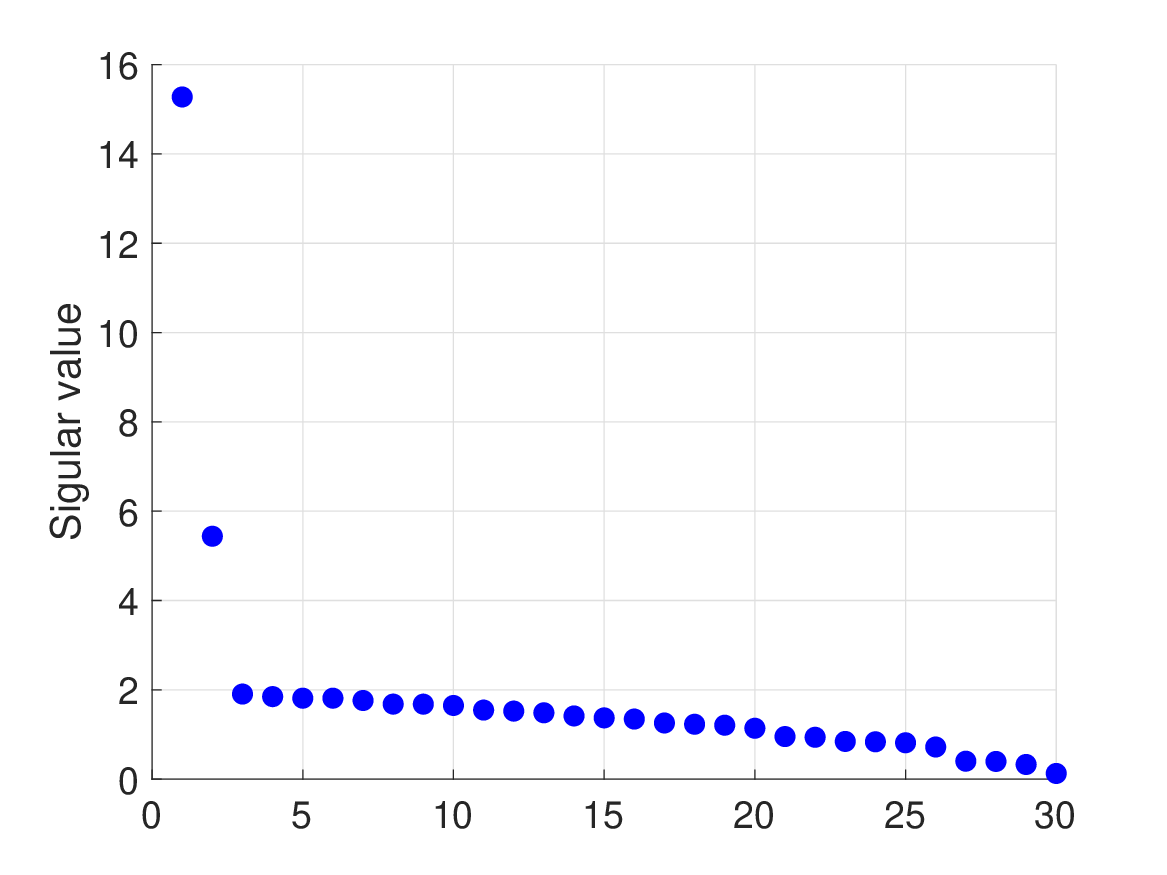}
		\caption{3rd dimension}
		\label{fig:fig3}
	\end{subfigure}
	\caption{Scree plot for rank selection}
	\label{fig:rank}
\end{figure}


In the first experiment, we vary the step size $\eta$ among the values $\{5 \times 10^{-4}, 1 \times 10^{-3}, 5 \times 10^{-3}\}$ while keeping the noise level fixed at $\sigma = 1$. The horizon $T$ is set to $15,000$. The experimental results are presented in the left panel of Figure \ref{fig:step size}. The figure clearly shows that a larger step size (represented by the yellow curve) leads to faster convergence but results in a higher error rate. Conversely, a smaller step size (represented by the blue curve) yields a more accurate estimator, albeit requiring a longer convergence time to reach a stable state. These findings validate our theoretical claims.
\begin{figure}
	\centering
	\begin{subfigure}[b]{.98\linewidth}
		\includegraphics[width=0.5\textwidth]{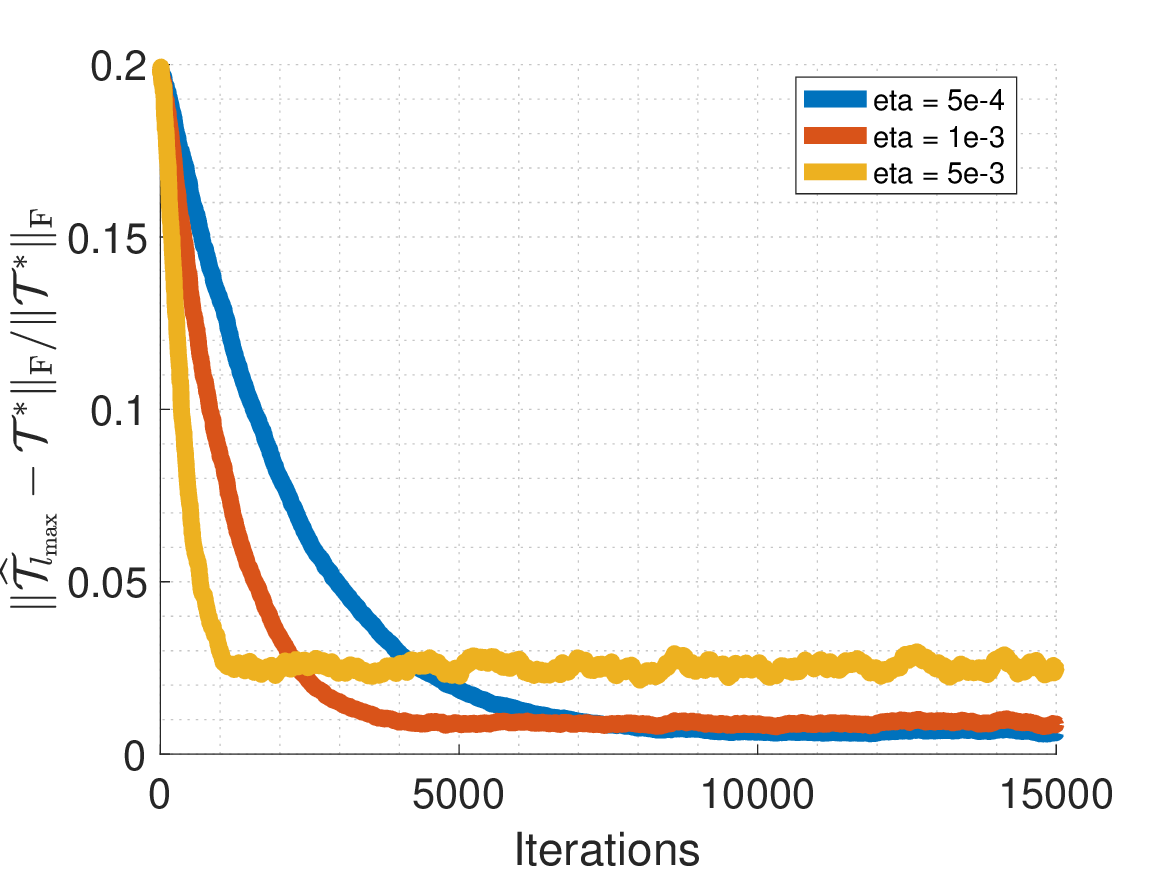}
		\includegraphics[width=0.5\textwidth]{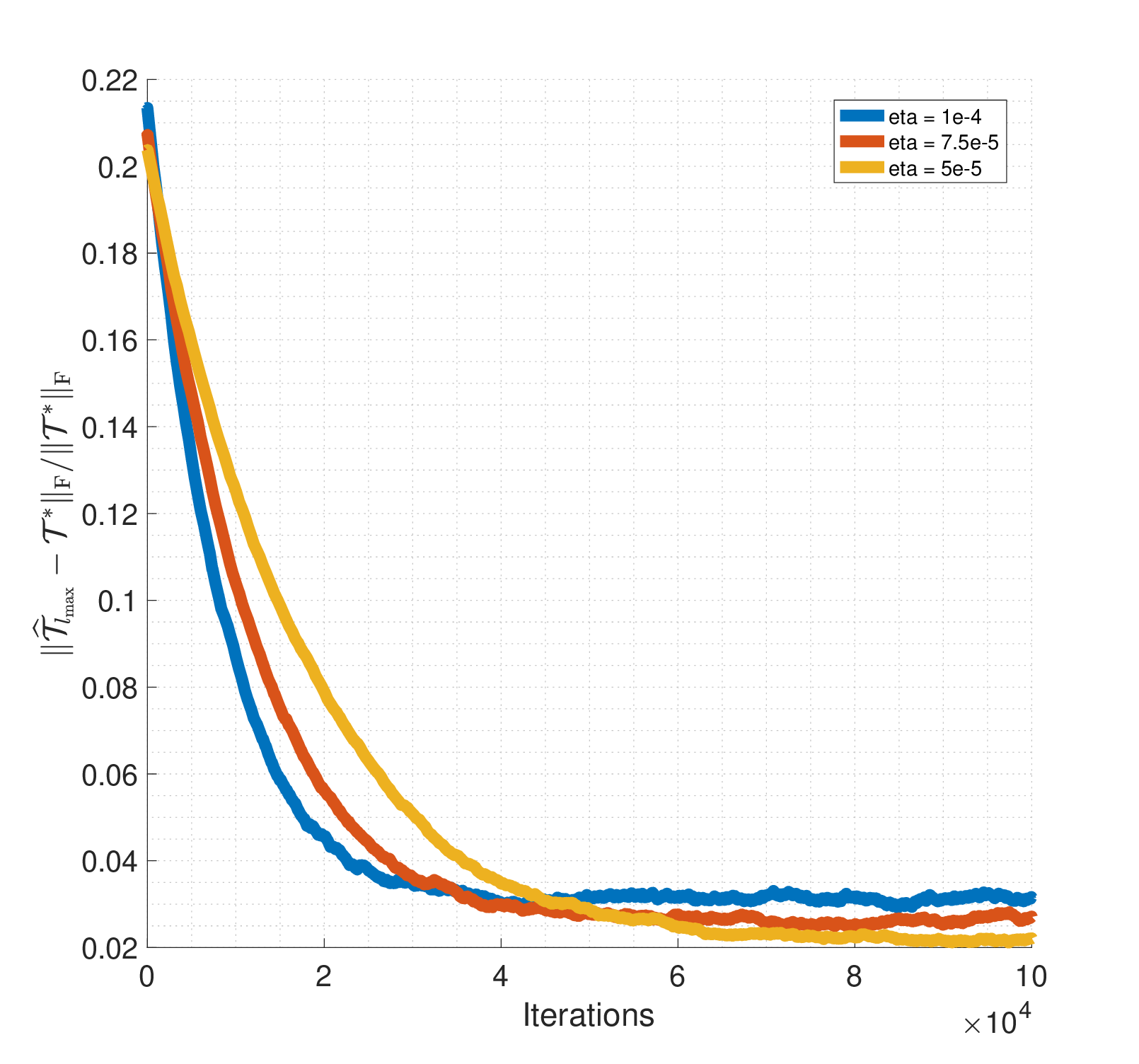}
	\caption{Convergence dynamics of oRGrad for online tensor linear regression (left) and online tensor completion (right) with different step sizes.}
	\label{fig:step size}
	\end{subfigure}
	\begin{subfigure}[b]{.98\linewidth}
	\includegraphics[width=0.5\textwidth]{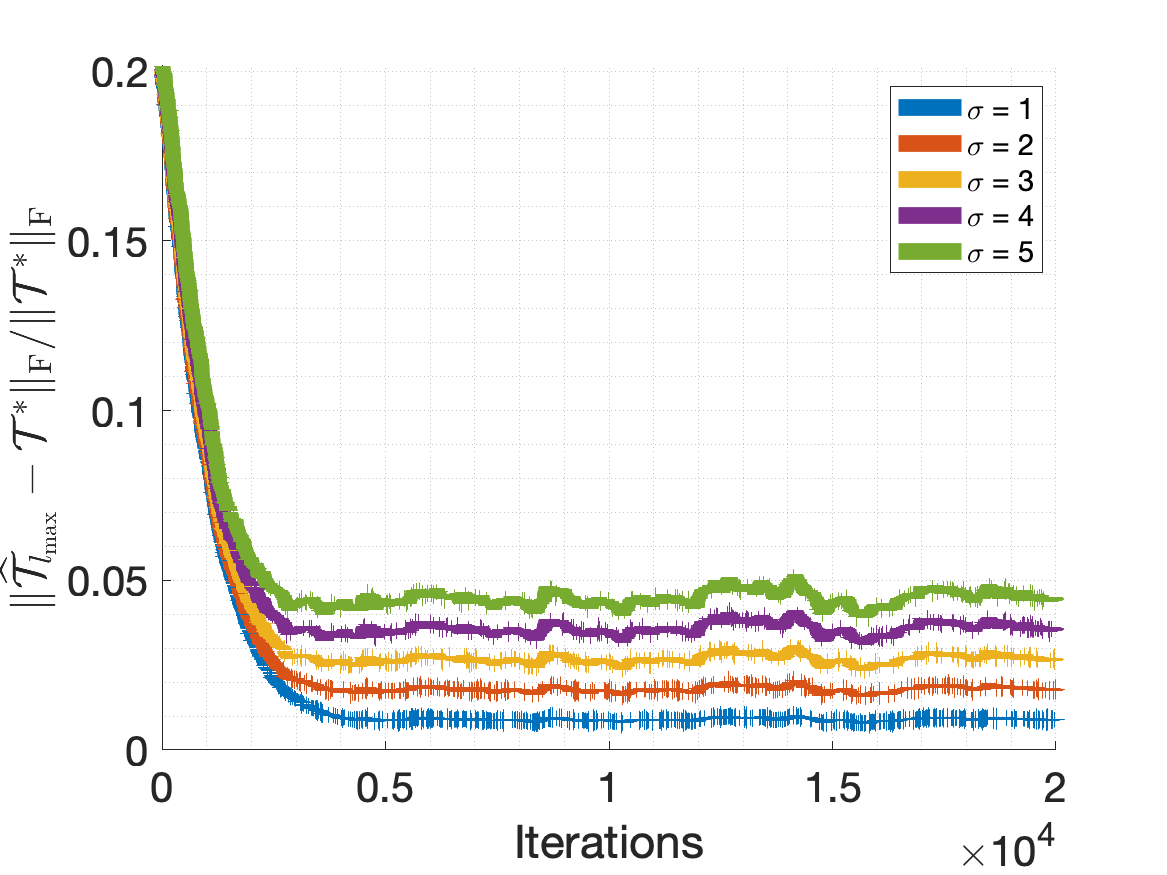}
	\includegraphics[width=0.5\textwidth]{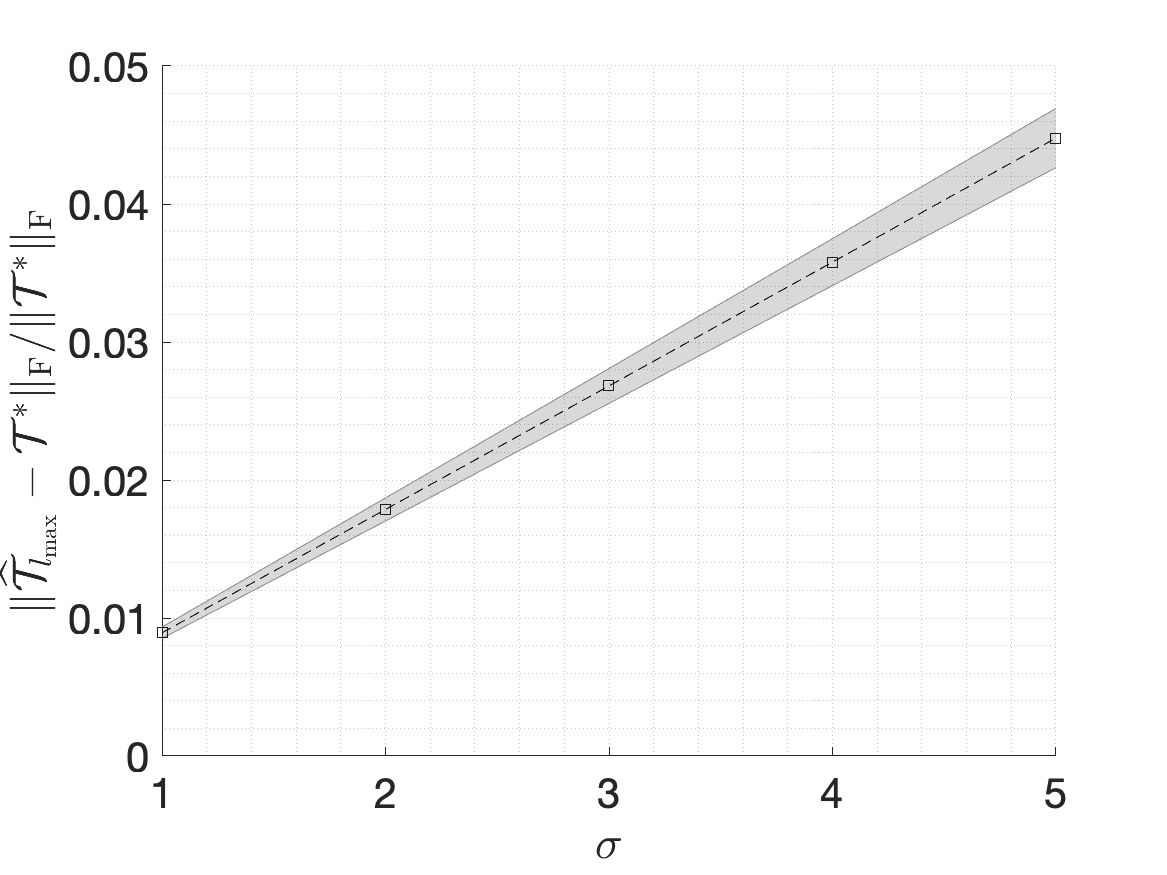}
	\caption{Left: Convergence dynamics of oRGrad for online tensor linear regression with different noise levels; Right: error bar plot based on 10 independent trials for different noise levels. }
	\label{fig:tr_noise}
\end{subfigure}
	\begin{subfigure}[b]{.98\linewidth}
	\includegraphics[width=0.5\textwidth]{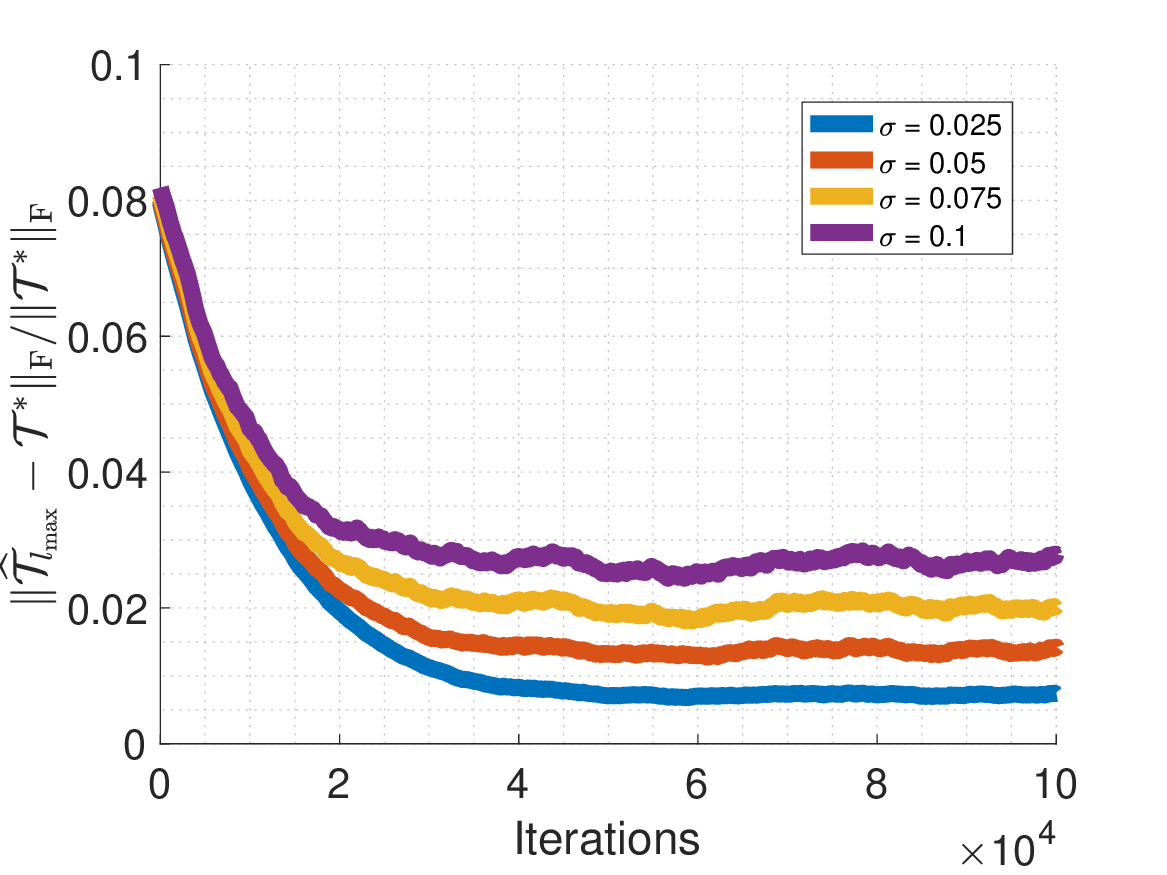}
	\includegraphics[width=0.5\textwidth]{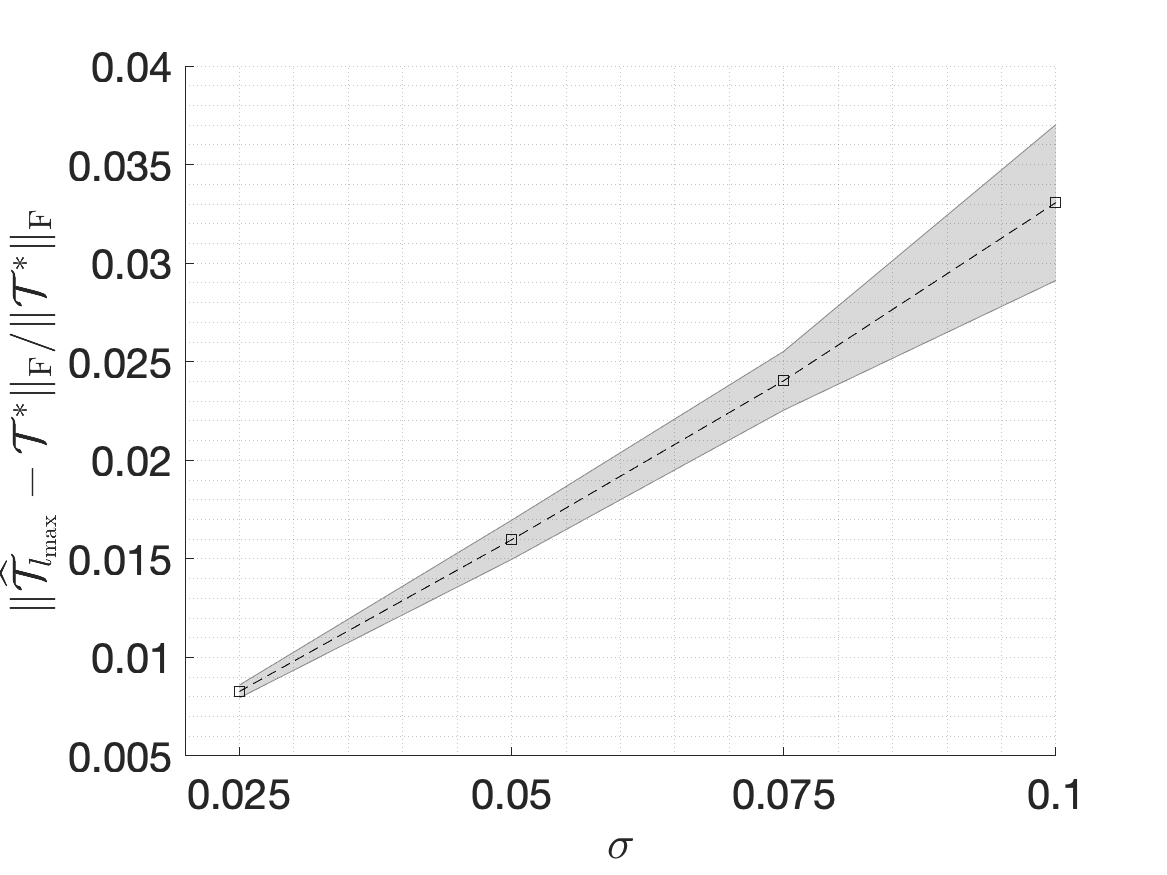}
	\caption{Left: Convergence dynamics of oRGrad for online tensor completion with different noise levels; Right: error bar plot based on 10 independent trials for different noise levels. }
	\label{fig:tc_noise}
\end{subfigure}
\caption{Convergence dynamics of oRGrad for online tensor linear regression and completion.}
\end{figure}

In the second experiment, we investigate the effect of varying noise levels while keeping the step size fixed at $\eta = 1 \times 10^{-3}$. To ensure convergence to a stable estimator, we set the horizon $T$ to $20,000$ steps. The noise level $\sigma$ takes on values $\{1, 2, 3, 4, 5\}$. For each $\sigma$, we conduct a single convergence dynamics experiment and perform 10 independent trials for the error bar plot. The results are presented in Figure \ref{fig:tr_noise}.
From the left panel, we observe that the convergence rates do not significantly vary across different noise levels. However, the right panel reveals a proportional relationship between the final error rate and the noise level, thereby confirming our theoretical findings.

We also present the average per-step runtime as a function of the tensor dimension. We fix the rank $\mathbf{r} = (2, 2, 2)'$ while varying the tensor dimensions $d \times d \times d$ for $d \in \{20, 25, \ldots, 70\}$. For each dimension, we set the horizon to $T = 10,000$. The results are displayed in Figure \ref{fig:runtime}. As shown, the average per-step runtime scales linearly with $d^3$, which aligns with the computational cost detailed in Appendix \ref{sec:RGrad}.

\begin{figure}[htbp]
	\centering
	\includegraphics[width=0.6\textwidth]{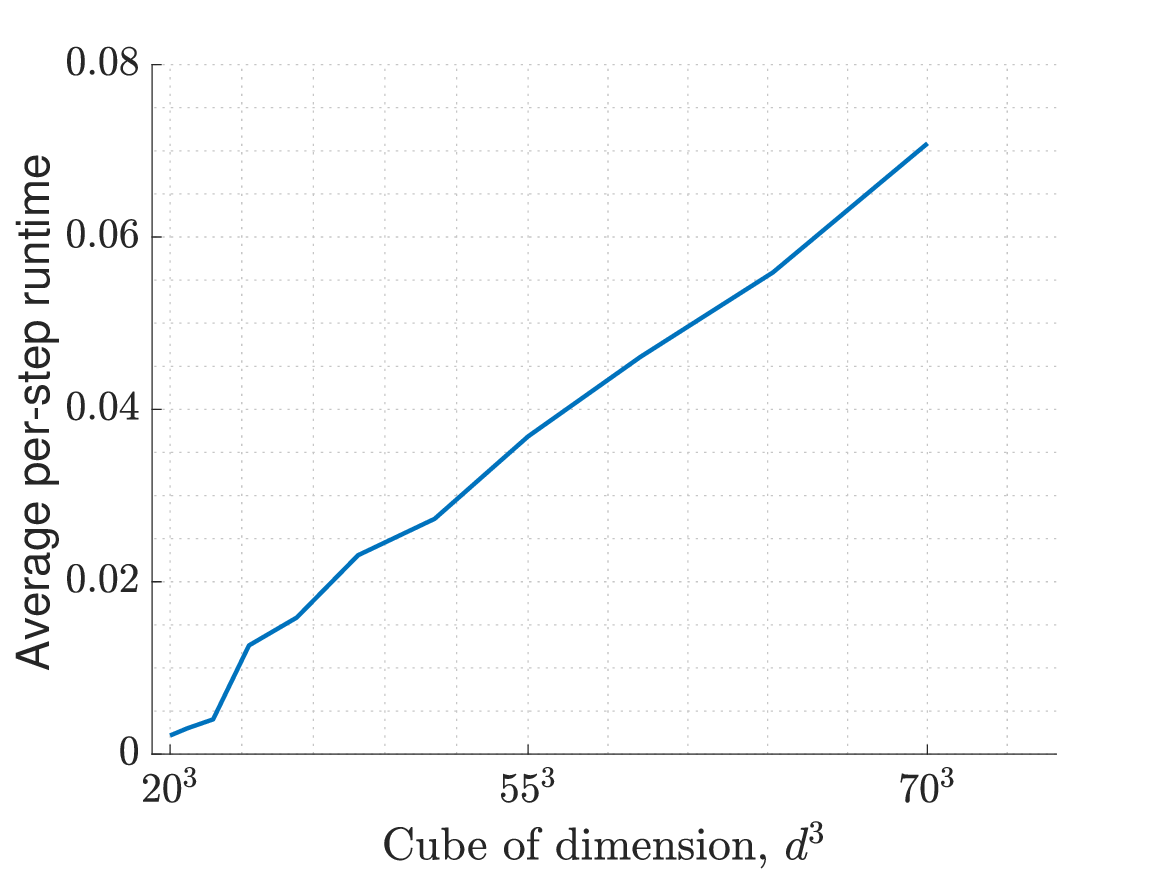}
	\caption{Average per-step runtime versus cubic of dimension}
	\label{fig:runtime}
\end{figure}

\subsection{Online tensor completion}
In this section, we conduct numerical experiments on online tensor completion, focusing on a low-rank tensor $\bcalT^* \in \mathbb{R}^{d \times d \times d}$ with dimension $d = 75$ and Tucker rank $\mathbf{r} = (2, 2, 2)'$. We generate $\bcalT^*$ by applying the higher-order singular value decomposition ($\textsf{HOSVD}_{\r}$) to a random tensor with entries independently and uniformly drawn from the interval $[0,1]$. The corresponding value of $\lambda_{\submin}$ is approximately $0.6$, and the incoherence measure $\mu$ is approximately $3.2$. To ensure accurate initialization, we collect data at the beginning of the procedure and employ the second-order moment method proposed in \cite{xia2021statistically}.

In the first experiment, we vary the step size $\eta$ among the values $\{5\times 10^{-5},7.5\times 10^{-5},1\times 10^{-4}\}$ while keeping the noise level fixed at $\sigma = 0.1$. The total number of iterations is set to $T = 100,000$. The experimental results are depicted in the right panel of Figure \ref{fig:step size}. From the figure, we observe that a larger step size (indicated by the yellow curve) leads to faster convergence but also incurs a higher error rate. Conversely, a smaller step size (represented by the blue curve) yields a more accurate estimator at the cost of a longer convergence time to reach a stable state. 

In the second experiment, we investigate the impact of varying noise levels while maintaining a fixed step size of $\eta = 7.5 \times 10^{-5}$. To ensure convergence to a stable estimator, we execute a sufficient number of steps with $T = 100,000$. The noise level $\sigma$ takes on the values $\{0.025, 0.05, 0.075, 0.1\}$. Convergence dynamics are recorded for each noise level in separate trials, and the error bar plot is based on 10 independent trials. The outcomes are presented in Figure \ref{fig:tc_noise}. 
From the left panel, it is evident that the convergence rates don't change across different noise levels. However, the right panel demonstrates a proportional relationship between the final error rate and the noise level, which aligns with our theoretical findings.

\subsection{Regret analysis}
In this section, we analyze the regret of oRGrad under both constant step size and adaptive settings. We consider a tensor $\bcalT^* \in \mathbb{R}^{d \times d \times d}$ with dimension $d = 30$ and Tucker rank $\mathbf{r} = (2, 2, 2)$, generated similarly to those in previous sections. The corresponding value of $\lambda_{\submin}$ is approximately 0.5, and the noise level is fixed at $\sigma = 0.1$ throughout the experiment.

We vary the horizon $T$ from $5,000$ to $200,000$. Following the suggestion in the remark after Theorem \ref{thm:regret}, we set the step size to $\eta = \frac{0.01}{\sqrt{T}}$ for the constant step size regime. We plot regret versus the square root of the horizon ($\sqrt{T}$) in the left panel of Figure \ref{fig:regret}. Additionally, we adopt an adaptive choice of step sizes as suggested by Algorithm \ref{alg:adaptive}, and the resulting regret versus $\log(T)$ is plotted in the right panel of Figure \ref{fig:regret}. From the figures, we observe that the regret scales linearly with respect to $\sqrt{T}$, aligning with the predictions of Theorem \ref{thm:regret} under a constant step size. In the adaptive setting, the regret scales linearly with respect to $\log T$, which significantly improves upon the constant step size regime.

\begin{figure}[h]
	\begin{subfigure}[b]{.98\linewidth}
	\includegraphics[width=0.5\textwidth]{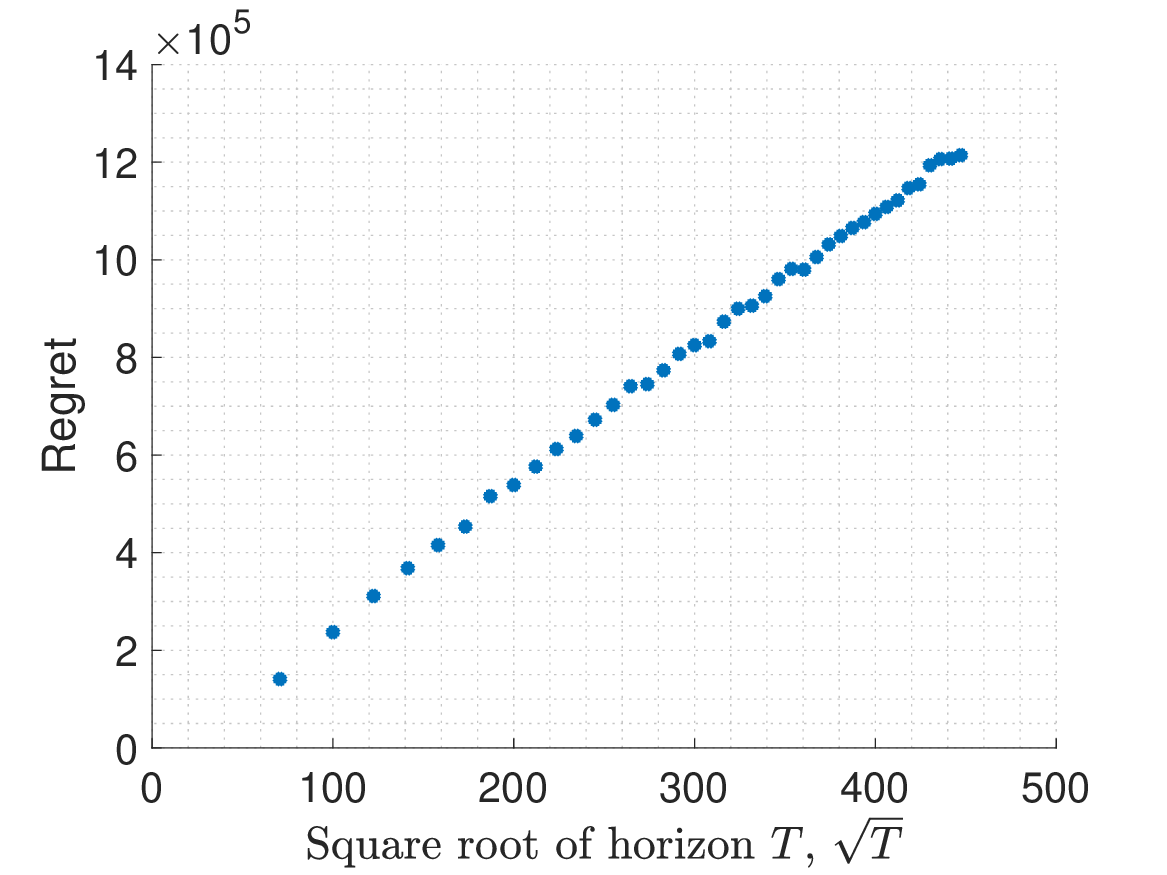}
	\includegraphics[width=0.5\textwidth]{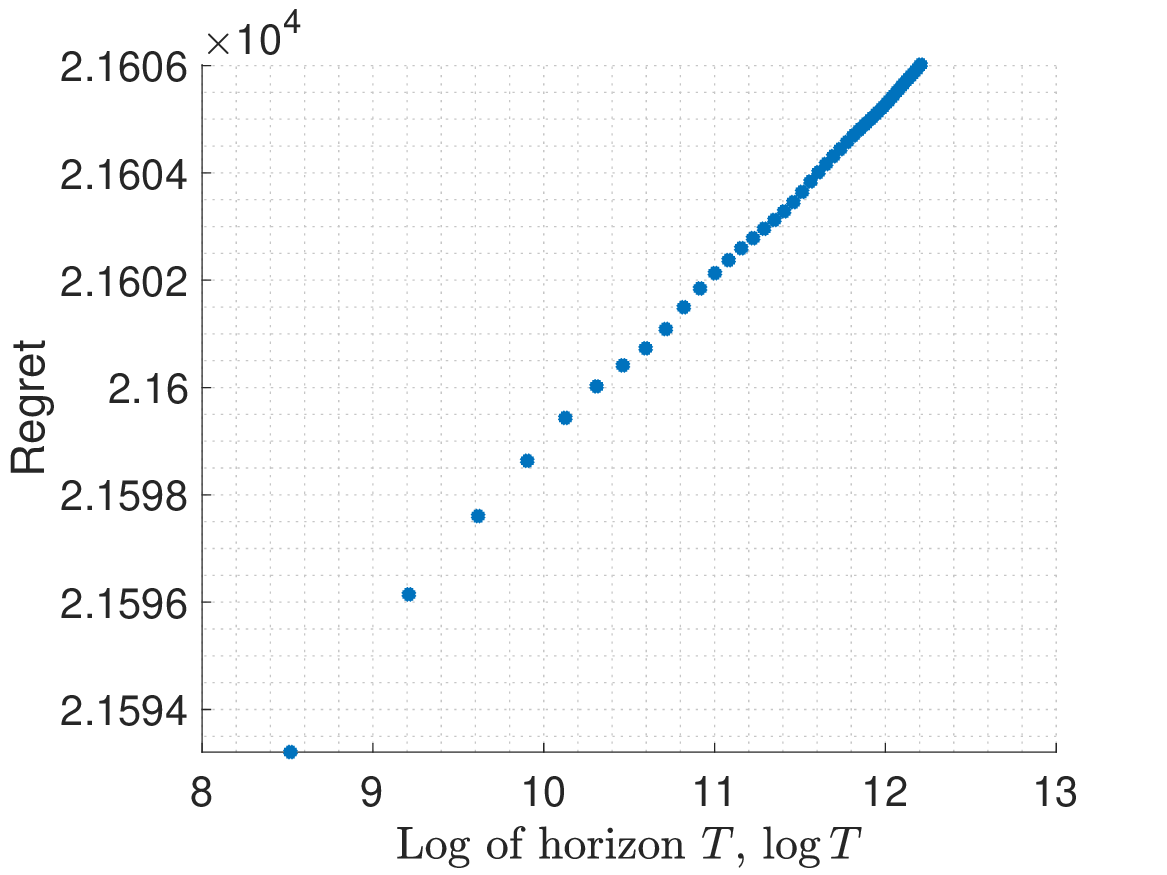}
\end{subfigure}
	\caption{Regret performance. Left: a constant step size; Right: adaptive step sizes. }
\label{fig:regret}
\end{figure}

\subsection{Real data: Prediction of solar index F10.7}\label{sec:app-solar}
In this section, we target at real-time prediction of solar index F10.7 from ionospheric total electron content (TEC) data. The TEC values derived from multi-frequency Global Navigation Satellite System (GNSS) signals and the relevant products have become one of the most utilized parameters in the space weather and ionospheric research community. The F10.7 index measures solar activity based on the radio flux at 10.7 cm (2800 MHz) from the sun's outer layers. It has been consistently recorded in Canada since 1947 and varies from below 50 to above 300 solar flux units over a solar cycle. Researchers have been investigating how the solar activities as measured by F10.7 and other indices can impact the changes of global TEC maps, see e.g., \cite{wang2023forecast}. However, the relationship between the F10.7 and the TEC maps is still not very clear based on data-driven approaches as shown in \cite{wang2023forecast}, though physics implies that these two quantities are highly related to each other. At a particular timestamp, the TEC is measured with a spatial resolution of 1 latitude by 1 longitude, and the temporal cadence is 5 minutes.
The original TEC data from the Madrigal TEC database \citep{rideout2006automated} has more than 80\% of the data missing on average. We use a completed version of the TEC data~\citep{sun2023complete} using the VISTA algorithm proposed in \citep{sun2022matrix} \footnote{This dataset is made publicly available in \url{https://deepblue.lib.umich.edu/data/concern/data_sets/nc580n00z?locale=en}.}(see Figure \ref{fig:TEC} for an example of one completed measurement).
The response solar index F10.7 is measured hourly \footnote{The response can be downloaded from \url{https://omniweb.gsfc.nasa.gov/form/dx1.html}} (see Figure \ref{fig:F107} for the solar index in the year 2020). 
We use the TEC data and solar flare index in the year 2020. We predict the F10.7 index using the data measured within an hour-long window one day (24 hours) ahead. For example, we predict the F10.7 index at 2 pm on July 20th using the TEC data measured between 1 p.m.- 2 p.m. (12 measurements in total) on July 19th. We down-sample the data in the spatial domain by applying a $6\times 6$ maximum kernel to the original dataset. This approach helps reduce redundancy in the original image, emphasizing the active regions with higher values that carry more significant information. Additionally, it enables us to use the offline setting as a benchmark, as it otherwise would require excessive storage.
After pre-processing the data, we obtain 8760 samples, where each covariate sample is of size $16\times 31\times 12$. We adopt the rank selection method mentioned in Section \ref{sec:nu:lr}, which indicates the rank $\r = (4,4,4)$. We use the first 6000 samples for the initialization; and for the rest of the samples, we perform online prediction. 
\begin{figure}[htbp]
	\centering
	\begin{subfigure}[b]{0.45\textwidth}
		\centering
		\includegraphics[width=\textwidth]{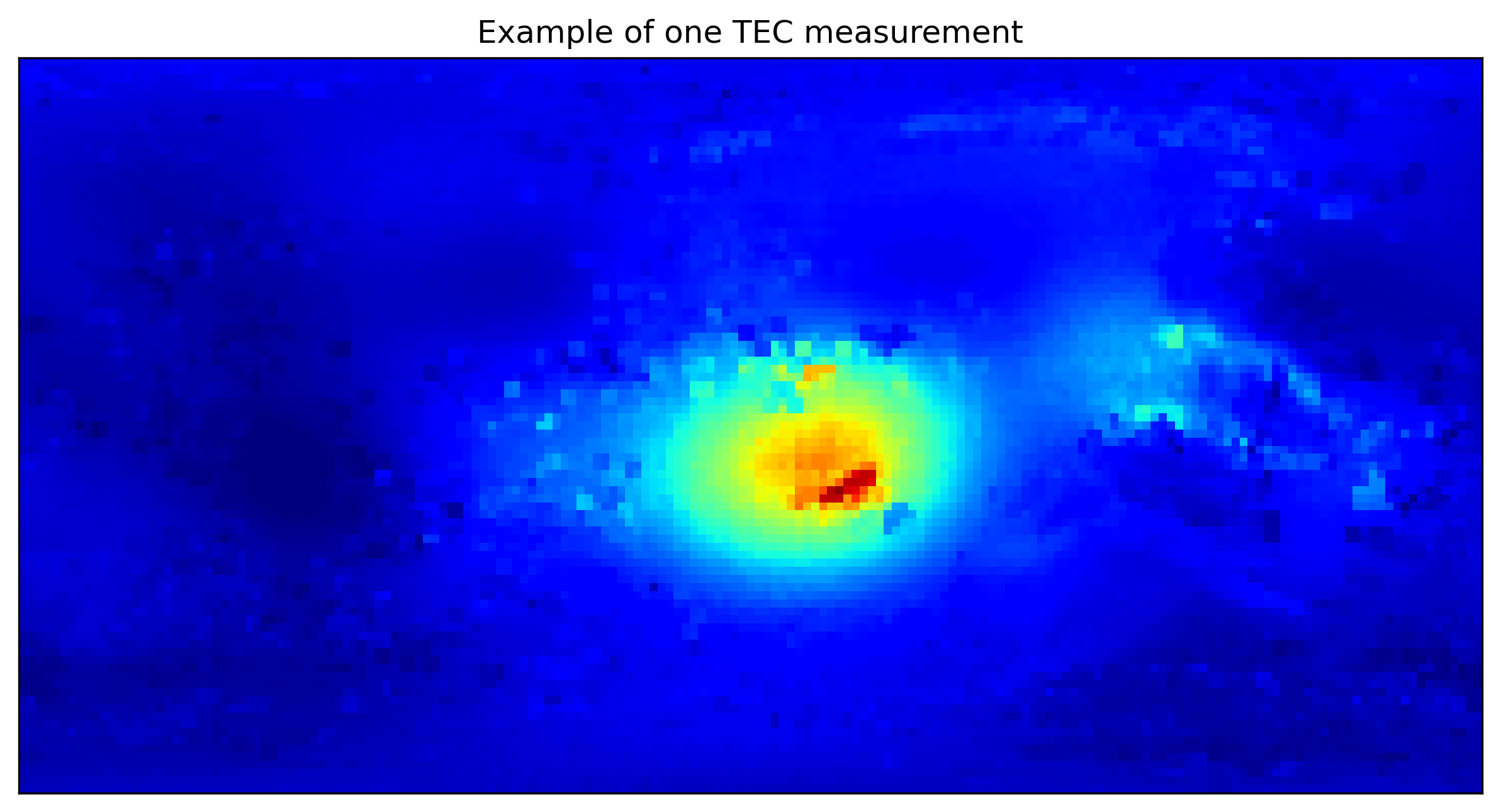}
		\caption{Example of one TEC measurement}
		\label{fig:TEC}
	\end{subfigure}
	\hfill
	\begin{subfigure}[b]{0.45\textwidth}
		\centering
		\includegraphics[width=\textwidth]{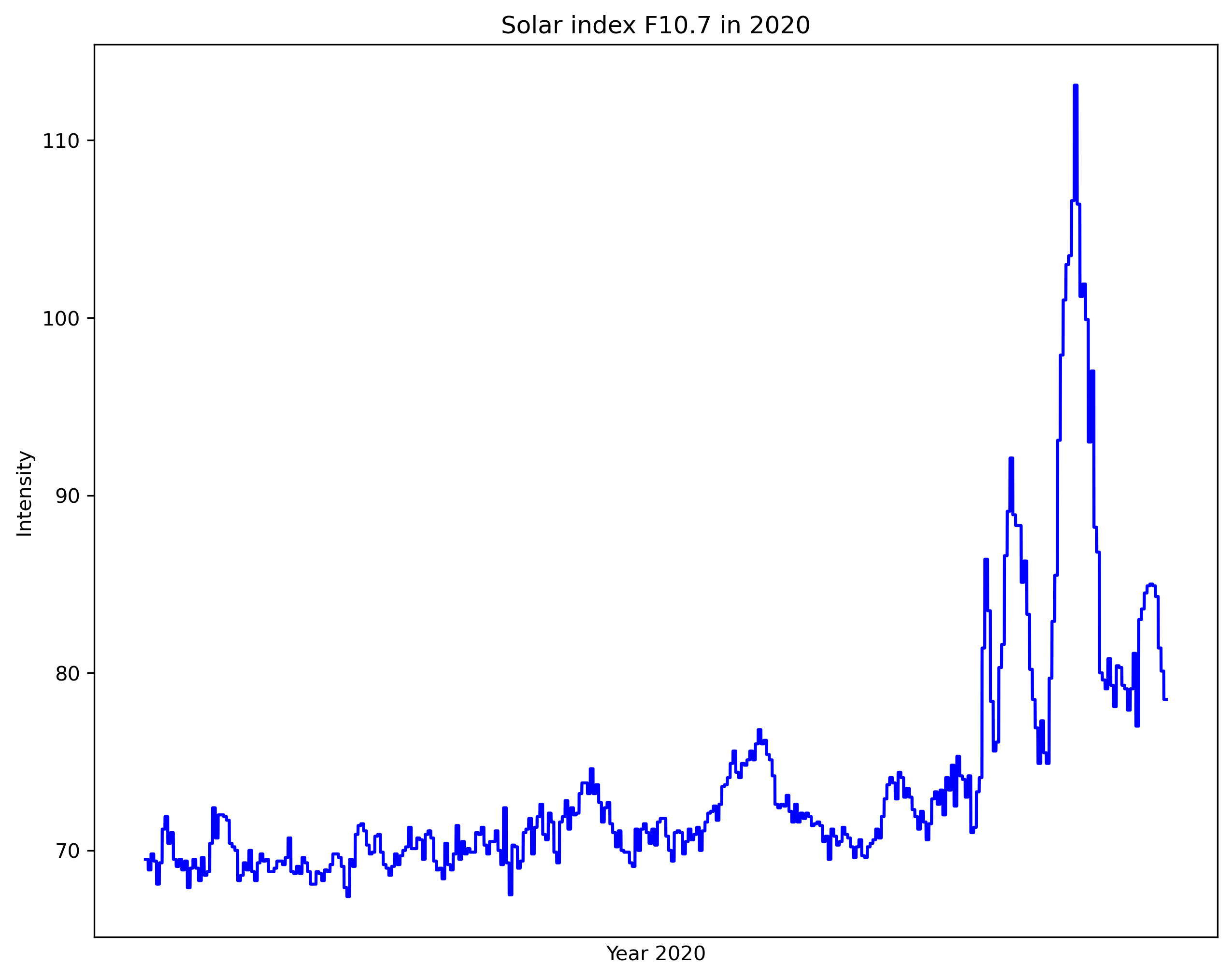}
		\caption{Solar index F10.7 in the year 2020}
		\label{fig:F107}
	\end{subfigure}
	\caption{Example of one slice in the tensor covariate and the response}
	\label{fig:TEC+f107}
\end{figure}
\begin{figure}[htbp]
	\centering
	\includegraphics[width=0.8\textwidth]{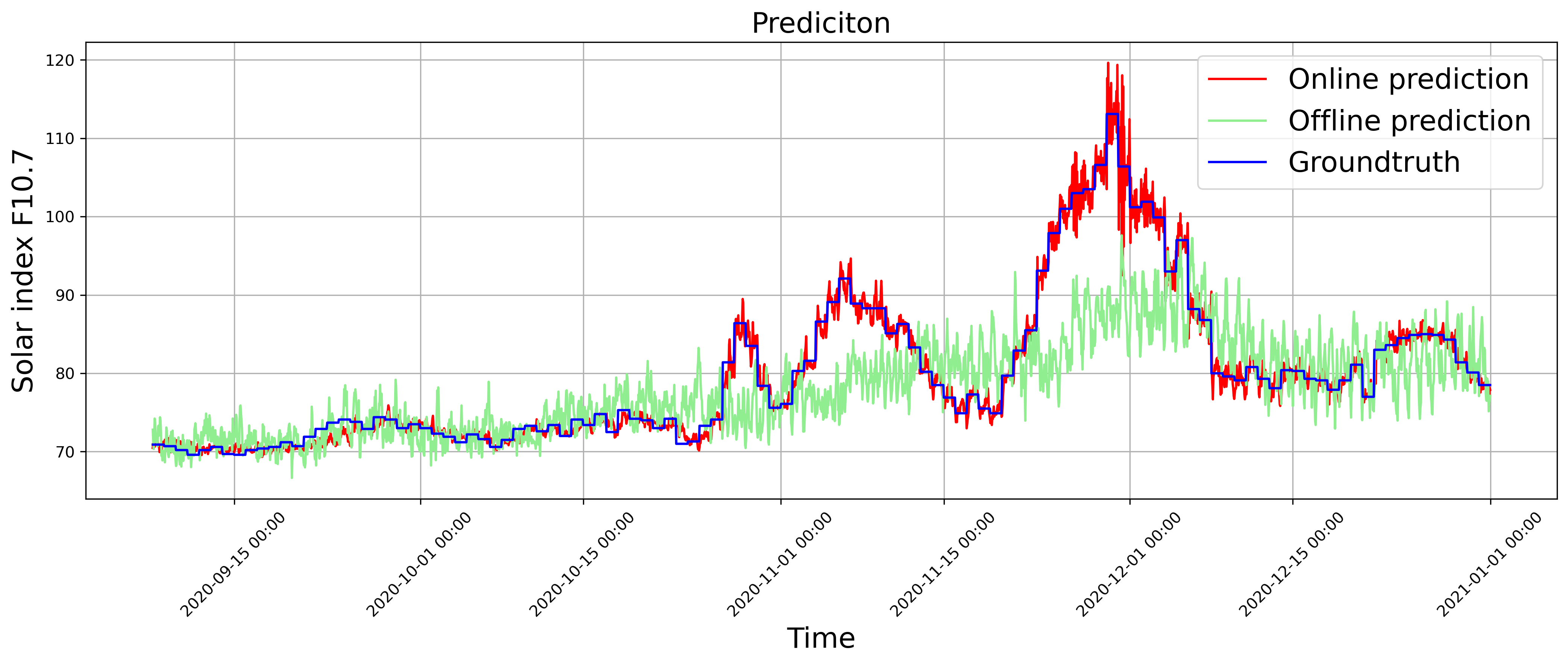}
	\caption{Predicting solar index F10.7 by online/offline RGrad and comparison with the ground truth (true recording).}
	\label{fig:f107pred}
\end{figure}
As a comparison, we also use the offline-RGrad for our task. We use 3000 samples for initialization and another 3000 samples for refinement, and test the prediction accuracy using the rest of the samples. The results of the prediction are displayed in Figure \ref{fig:f107pred}. And the relative error in prediction and correlations are displayed in Table \ref{table:TEC}.

\begin{table}[h]
\centering
	\begin{tabular}{c||c|c}
		& Relative prediction error& Correlation\\ 
		\hline
		Online RGrad& \textbf{0.018}&\textbf{0.988} \\
		\hline
		Offline RGrad&0.085 & 0.724 
	\end{tabular}
	\caption{Comparison with offline RGrad in predicting solar index F10.7.}
		\label{table:TEC}
\end{table}
It is clear from both Figure \ref{fig:f107pred} and Table \ref{table:TEC} that online RGrad has a better performance and is able to make accurate real-time prediction. From 2020-09-07 to 2020-10-25, it is non-stormy period and both online RGrad and offline RGrad have similar performance. However, there are three stormy periods from 2020-10-26 to the end of 2020. Since the offline RGrad is trained using the data from non-stormy periods, it can not capture the change in covariate and performs poorly in predicting the stormy period. While online RGrad updates the parameter from time to time and is thus able to make accurate predictions even in the stormy periods. This adaptivity of stormy periods is of great importance in space weather monitoring.

\section*{Acknowledgements}
Jian-Feng Cai's research was partially supported by Hong Kong RGC Grant GRF 16306821, 16309219 and 16310620. Dong Xia's research was partially supported by Hong Kong RGC Grant GRF 16300121 and 16301622. Yang Chen's research was partially supported by NSF DMS 2113397, NSF PHY 2027555, NSF AGS 2419187, NASA 22-SWXC22\_2-0005, and NASA 22-SWXC22\_2-0015.


\bibliographystyle{plainnat} 
\bibliography{reference.bib}       

\newpage
\appendix

\section{Experiment on MovieLens 100k Dataset}
We obtained the MovieLens 100k Dataset from the official website \url{https://grouplens.org/datasets/movielens/}, which comprises 100,000 ratings given by 1000 users for 1700 movies. The data is transformed into a matrix of size $1000\times 1700$ with a significant number of missing entries. For our experiment, we utilize the earliest 80,000 ratings as the training data, while the remaining 20,000 ratings serve as the test data.

To address the problem of online matrix completion, we employ the oRGrad algorithm and compare its performance against the offline RGrad method proposed in \cite{wei2016guarantees}, which serves as a benchmark. We evaluate the results using the mean absolute error (MAE) metric:
$$\text{MAE} = \frac{1}{|\Omega_{\text{test}}|} \| [\hat \M - \M^*]_{\Omega_{\text{test}}}\|_{\ell_1},$$
where $\Omega_{\text{test}}$ is indices of test data and $|\Omega_{\text{test}}|$ is the number of test data, $\hat\M$ is the estimator output by either oRGrad or offline RGrad, and $\M^*$ represents the underlying matrix. 
The result are collected in Table \ref{table:movielens}.
\begin{table}[h]
	\centering
	\begin{tabular}{|c || c|c|c|c|}
		\hline
		MAE& $r = 2$& $r = 5$ & $r= 10$& $r = 15$\\ 
		\hline
		Online & 0.4871&0.4863 & 0.4808 & 0.4902\\ 
		\hline
		Offline (Benchmark) & 0.4749 & 0.4694 &0.4635 &0.4859\\ 
		\hline
	\end{tabular}
	\caption{Comparison with benchmark (offline RGrad).}
	\label{table:movielens}
\end{table}

Although offline RGrad outperforms oRGrad slightly in terms of MAE, it is important to note that offline RGrad is a method that relies on collecting data and lacks the ability to provide timely updates. In contrast, oRGrad sacrifices some accuracy but offers the advantage of providing timely updates without the need for data collection, which can be a time-consuming process.

\section{Tensor logistic regression}\label{sec:logistic}
Low-rank logistic regression has been studied in \cite{taki2021minimax,hung2013matrix,shi2014sparse}. In particular, \cite{taki2021minimax} provided a minimax lower bound for matrix logistic regression. In contrast, online low-rank logistic regression has been far less explored in the literature. Similarly as Example~\ref{ex:logisticregression}, consider the binary observation $Y_t\sim\text{Ber}(p_t)$ where $p_t = f(\inp{\bcalX_t}{\bcalT^*})$ and $f:\RR\rightarrow[0,1]$ is an {\it inverse link function}. Common choices of $f$ include the logistic link $f(\theta) = (1+e^{-\theta/\sigma})^{-1}$ and the probit link $f(\theta) = 1 - \Phi(-\theta/\sigma)$ where $\sigma>0$ is a scaling parameter.  The corresponding loss function is the negative log-likelihood defined as $h(\theta,y) = -y\log(f(\theta)) - (1-y)\log (1-f(\theta))$. For any given $\alpha>0$, define
\begin{align*}
\gamma_{\alpha} &:= \min\big\{\inf_{|\theta|\leq\alpha}\frac{(f'(\theta))^2 - f^{''}(\theta)f(\theta)}{f^2(\theta)},\quad\inf_{|\theta|\leq\alpha}\frac{f^{''}(\theta)(1-f(\theta)) + (f'(\theta))^2}{(1-f(\theta))^2}\big\},\\
\mu_{\alpha}&:= \max\big\{\sup_{|\theta|\leq\alpha}\frac{(f'(\theta))^2 - f^{''}(\theta)f(\theta)}{f^2(\theta)},\quad\sup_{|\theta|\leq\alpha}\frac{f^{''}(\theta)(1-f(\theta)) + (f'(\theta))^2}{(1-f(\theta))^2}\big\}.
\end{align*}
and 
\begin{align*}
	L_{\alpha} = \sup_{|\theta|\leq \alpha}\frac{f'(\theta)}{f(\theta)(1-f(\theta))},
\end{align*}
which characterizes $\errinf$ and $\errtwo$. In binary learning problems, $L_{\alpha}$ and $\lambda_{\submin}/L_{\alpha}$ are often regarded as noise level \citep{wang2020learning,davenport20141,cai2013max,cai2022generalized,han2022optimal} and the SNR, respectively.


\paragraph*{Logistic link}  If $f(\theta) = (1+e^{-\theta/\sigma})^{-1}$, we have $f'(\theta)=\sigma^{-1}\cdot f(\theta)(1-f(\theta))$ and $f''(\theta)=\sigma^{-2}\cdot f(\theta)(1-f(\theta))(1-2f(\theta))$. As a result, 
\begin{align*}
	\gamma_{\alpha} = \min_{|\theta|\leq \alpha}\frac{1}{\sigma^2}f(\theta)(1-f(\theta)) = \frac{e^{\alpha/\sigma}}{(1+e^{\alpha/\sigma})^2\sigma^2},\quad {\rm and}\quad 
	\mu_{\alpha} = \max_{|\theta|\leq \alpha}\frac{1}{\sigma^2}f(\theta)(1-f(\theta)) =\frac{1}{4\sigma^2}.
\end{align*}
As $\alpha$ becomes larger, $\gamma_{\alpha}$ approaches zero exponentially fast.  It is therefore typical to assume $\|\bcalT^{\ast}\|_{\rm F}$ is upper bounded. 


\begin{theorem}\label{thm:logisticregression}
Suppose Assumptions~\ref{assump:X-design} and \ref{assump:GLM-trueT} hold, the initialization satisfies $\fro{\bcalT_0 - \bcalT^*}\leq c_m\gamma_{\alpha}\mu_{\alpha}^{-1}\cdot \lambda_{\submin}$ for some sufficiently small constant $c_m>0$, $C_0\log\dmax\leq\dof$, the learning rate $\eta$ satisfies $\tilde C_1\eta\dof\log^{5/2}\dmax\leq 1$ with $\tilde C_1 = C_1\max\{\gamma_{\alpha}^{-1},\gamma_{\alpha}^{-1}\mu_{\alpha}^2,\gamma_{\alpha}^{-1}\mu_{\alpha}\dof^{-1}\}$, and the SNR satisfies 
	\begin{align*}
		\frac{\lambda_{\submin}^2}{L_{\alpha}^2}\geq C_m\eta\gamma_{\alpha}^{-1}\dof\max\{\log^2\dmax,\gamma_{\alpha}^{-2}\mu_{\alpha}^2\},
	\end{align*}
where $c_m, C_m>0$ depend only on $m$ and $C_0, C_1$ are absolute constants.  Then there exists an absolute constant $C>0$ such that, for any given horizon $T$, the oRGrad algorithm guarantees that, with probability exceeding $1-14T\dmax^{-10}$, for all $t\leq T$,
$$
\fro{\bcalT_t-\bcalT^*}^2 \leq 2\Big(1-\frac{\eta\gamma_{\alpha}}{8}\Big)^t\fro{\bcalT_0 - \bcalT^*}^2 + C\eta L_{\alpha}^2\gamma_{\alpha}^{-1}\cdot\dof.
$$
\end{theorem}

\begin{proof}[Proof of Theorem~\ref{thm:logisticregression}]
It suffices to validate the conditions of Theorem~\ref{thm:gen}. Indeed, we have $h_{\theta}(\theta, y)=-y f'(\theta)/f(\theta)+(1-y)f'(\theta)/(1-f(\theta))$ and $g_t(\theta)=(f'(\theta))^2/(f(\theta)(1-f(\theta)))$.  It is clear that Assumption~\ref{assump:GLM-loss} holds from the definitions of $\gamma_{\alpha}$ and $\mu_{\alpha}$. Meanwhile, we have $\errtwo^2=L_{\alpha}^2/4, \errinf=L_{\alpha}$, and $\errratio=1/2$. The proof is concluded by Theorem~\ref{thm:gen}.  
\end{proof}

For ease of exposition, assume $m,\rmax,\mu_{\alpha},\gamma_{\alpha} \asymp O(1)$ and $d_j\asymp d$ for $j\in[m]$. Theorem~\ref{thm:logisticregression} requires an upper bound of step size $\eta\leq \eta_{\textsf{comp}}:= c(d\log^2d)^{-1}\cdot \min\big\{\lambda_{\submin}^2L_{\alpha}^{-2}, \log^{-1/2}d\big\}$ for some small constant $c>0$. Similarly, we observe a computational and statistical trade-off. To achieve statistical optimality, given any known horizon $T$, we can set $\eta=\eta_{\textsf{stat}}:=C_m(T\gamma_{\alpha})^{-1}\log d$ and Theorem~\ref{thm:logisticregression} gives rise to 
$$
\fro{\bcalT_T - \bcalT^*}^2 =O_m\bigg( d^{-C_m}\fro{\bcalT_0-\bcalT^*}^2  + L_{\alpha}^2\cdot \frac{\dof\log d}{T\gamma_{\alpha}^2}\bigg),
$$
which is valid when SNR $\lambda_{\submin}^2/L_{\alpha}^2=\Omega\big(T^{-1}d\log^3d\big)$ and horizon $T=\Omega\big(d\log^{7/2}d\big)$. If we choose a sufficiently large $C_m$ so that the first term is negligible, we end up with an optimal error rate matching the minimax lower bound established in \cite{taki2021minimax}. 

\section{Efficiency of Riemannian Gradient Descent}\label{sec:RGrad}
The efficiency of RGrad lies in the step of computing HOSVD. At each step, we will need to compute the following:
\begin{align*}
	\hosvd(\bcalT - \calP_{\TT}\bcalG), 
\end{align*}
where $\bcalT$ is the current estimator whose Tucker rank is $\r$, and we absorb the step-size into the gradient $\bcalG$. Now we illustrate how we RGrad can help reduce the computational cost. We denote $\bcalT = \bcalC\times_{j=1}^m\U_j$ be its Tucker decomposition. Then (see e.g. \citep{kressner2014low})
\begin{align*}
	\calP_{\TT}\bcalG = \bcalG\times_{j=1}^m(\U_j\U_j^\top) + \sum_{i=1}^m\bcalC\times_{j\in[m]\backslash i}\U_j\times \W_i,
\end{align*}
where $\W_i = (\I - \U_i\U_i^\top)\calM_i(\bcalG)(\U_m\otimes\cdots\otimes \U_{i+1}\otimes \U_{i-1}\otimes\cdots\otimes \U_1)\calM_i(\bcalC)^\dagger\in\RR^{d_i\times r_i}$. 
Then 
\begin{align*}
	\bcalT - \calP_{\TT}\bcalG &= \bcalC\times_{j=1}^m\U_j - \bcalG\times_{j=1}^m(\U_j\U_j^\top) -\sum_{i=1}^m\bcalC\times_{j\in[m]\backslash i}\U_j\times \W_i\\
	&=(\bcalC- \bcalG\times_{j=1}^m\U_j^\top)\times_{j=1}^m\U_j -\sum_{i=1}^m\bcalC\times_{j\in[m]\backslash i}\U_j\times \W_i\\
	&= \bcalL\times_{j=1}^m[\U_j \quad \W_j],
\end{align*}
where $\bcalL\in\RR^{2r_1\times\cdots\times 2r_m}$ such that $\bcalL(1:r_1,\cdots,1:r_m) = \bcalC- \bcalG\times_{j=1}^m\U_j^\top$, and $\bcalL(1:r_1,\cdots,r_j+1:2r_j,\cdots,1:r_m) = -\bcalC$ for all $j\in[m]$. Now we set $[\U_j \quad \W_j] = \Q_j\R_j$ be its compact QR decomposition, where $\Q_j\in\OO_{d_j,2r_j}$, and $\R_j\in\RR^{2r_j\times 2r_j}$. Then we have
\begin{align*}
	\bcalT - \calP_{\TT}\bcalG &= (\bcalL\times_{j=1}^m\R_j)\times_{j=1}^m\Q_j. 
\end{align*}
Here $\bcalL\times_{j=1}^m\R_j$ is a tensor of size $2r_1\times \cdots\times 2r_m$. And let the best rank $\r$ Tucker decomposition of it be
\begin{align*}
	\hosvd(\bcalL\times_{j=1}^m\R_j) = \tilde\bcalC\times_{j=1}^m\tilde\U_j,
\end{align*}
where $\tilde\bcalC\in\RR^{r_1\times\cdots\times r_m}$ and $\tilde\U_j\in\OO_{d_j,r_j}$. Then we have
\begin{align*}
	\hosvd(\bcalT - \calP_{\TT}\bcalG) =  \tilde\bcalC\times_{j=1}^m(\Q_j\tilde\U_j).
\end{align*}
Notice here it suffices to compute the HOSVD of a size $2\r$ tensor instead of a size $\d$ tensor, and this explains the efficiency of RGrad. 
\textit{Computational cost analysis. } In computing the HOSVD of the gradient update, we need to compute $\W_j$, the QR decomposition of $[\U_j \quad \W_j]$, and the $\hosvd(\bcalL\times_{j=1}^m\R_j)$ (including the tensor product $\bcalL\times_{j=1}^m\R_j$). When $d_1 = \cdots = d_m=d$, and $r_1 = \cdots =r_m=r$, the total computational cost is $O(md^mr+dr^m+r^{m+1}) = O(md^mr)$. 

Alternatively, if one uses projected gradient descent, that is we need to compute $	\hosvd(\bcalT - \bcalG)$ at each time, where $\bcalT - \bcalG$ is a full rank tensor. 
Therefore we need to compute $m$ SVD of full rank matrices, which requires the computation cost of order $O(md^{m+1})$.

\section{Introduction to the Representation Formula of Spectral Projectors}
This section reviews the representation formula of spectral projectors. We display this section specifically for our usage in analysis. Interested readers are referred to \cite{xia2021normal} for more details. 

\subsection{Symmetric Case}\label{sec:symetric}
Let $\A,\X$ be $d\times d$ symmetric matrices and $\text{rank}(\A)\leq r$ and $\A$ admits an eigen-decomposition 
$\A = \bTheta\bLambda\bTheta^\top$ with $\bTheta\in\RR^{d\times r}$ and $\bLambda = \text{diag}(\lambda_1,\cdots,\lambda_r)\in\RR^{r\times r}$. And the noise matrix $\X$ satisfies $\op{\X}\leq \frac{1}{2}\min_{i=1}^r|\lambda_i|$. Suppose we observe $\hat \A = \A+\X$ but $\A,\X$ are unknown and our aim is to estimate $\bTheta$. We denote $\hat\bTheta$ the $d\times r$ matrix containing the eigenvectors of $\hat\A$ with largest $r$ eigenvalues in absolute values. Now we are able to write $\hat\bTheta\hat\bTheta^\top - \bTheta\bTheta^\top$ explicitly. Follow Theorem 1 in \cite{xia2021normal}, we have 
$$\hat\bTheta\hat\bTheta^\top - \bTheta\bTheta^\top = \sum_{k\geq 1}\S_k,$$
where $\S_k$ has the following form
\begin{align}\label{Sk}
	\S_k = \sum_{s_1+\cdots+s_{k+1} = k}(-1)^{1+\lzero{\s}}\P^{-s_1}\X\cdots\X\P^{-s_{k+1}}
\end{align}
with $\P^0  = \bTheta_{\perp}\bTheta_{\perp}^\top=: \P^{\perp}$ and $\P^{-s} =  \bTheta\bLambda^{-s}\bTheta^\top$ for $s\geq 1$. In our problem setting only the first two terms are carefully dealt with so here we write out them explicitly
\begin{align*}
	\S_1 &= \P^{\perp}\X\P^{-1} + \P^{-1}\X\P^{\perp},\\
	\S_2 &= \P^{\perp}\X\P^{-2}\X\P^{\perp} + \P^{-2}\X\P^{\perp}\X\P^{\perp} +\P^{\perp}\X\P^{\perp}\X\P^{-2} \\
	&\quad - (\P^{-1}\X\P^{-1}\X\P^{\perp} + \P^{\perp}\X\P^{-1}\X\P^{-1} + \P^{-1}\X\P^{\perp}\X\P^{-1}).
\end{align*} 

\subsection{Asymmetric Case}\label{sec:assymetric}
Let $\M\in\RR^{d_1\times d_2}$ be a rank $r$ matrix that admits the compact SVD $\M = \U\bSigma\V^\top$ and $\Z\in\RR^{d_1\times d_2}$ be the noise matrix such that $\op{\Z}\leq\frac{1}{2}\sigma_r$. We observe $\hat\M = \M +\Z$ but $\M,\Z$ are unknown and our aim is to estimate $\U,\V$. Let $\hat\U$ and $\hat\V$ be $\hat \M$'s top-$r$ left and right singular vectors. We can reduce this case to the symmetric case using lifting. More specifically, set 
$$\hat\A = \mat{0&\hat\M\\ \hat\M^\top&0},\quad \A = \mat{0&\M\\ \M^\top&0},\quad \X = \mat{0&\Z\\ \Z^\top&0}.$$
As a result, $\hat\A$ admits a decomposition,
$$\A = \mat{0&\U\\ \V&0}\mat{0&\bSigma\\ \bSigma&0}\mat{0&\V^{\top}\\ \U^{\top}&0}.$$
So the projector onto the subspace spanned by non-zero eigenvectors of $\A$ is $\mat{\U\U^\top& 0\\ 0&\V\V^\top}$. So we can apply the result in the previous section, and we see
\begin{align*}
	\mat{\hat\U\hat\U^\top - \U\U^\top& 0\\ 0&\hat\V\hat\V^\top - \V\V^\top} = \sum_{k\geq 1}\S_k,
\end{align*}
where $\S_k$ is defined in \eqref{Sk} with $\P^{0} = \P^{\perp} = \mat{\U_{\perp}\U_{\perp}^{\top}& 0\\ 0& \V_{\perp}\V_{\perp}^{\top}}$ and for $s\geq 1$, 
\begin{equation*}
	\P^{-s}=\begin{cases}
		\mat{0&\U\bSigma^{-s}\V^\top\\\V\bSigma^{-s}\U^\top & 0}, & \text{if $s$ is odd},\\
		\mat{\U\bSigma^{-s}\U^\top&0\\ 0& \V\bSigma^{-s}\V^\top }, & \text{if $s$ is even}.
	\end{cases}
\end{equation*}
Now we take the left singular vectors as an example, that is, we are interested in $\hat\U\hat\U^\top - \U\U^\top$. From the above discussion, it has the following form
\begin{align*}
	\hat\U\hat\U^\top - \U\U^\top = \sum_{k\geq 1} \S_k^{\text{left}},
\end{align*}
where $\S_k^{\text{left}}$ is the top left block matrix of $\S_k$. 
For example, 
\begin{align}
	\S_1^{\text{left}} &= \P_{\U}^\perp\Z\V\bSigma^{-1}\U^\top + \U\bSigma^{-1}\V^\top\Z^\top\P_{\U}^\perp\label{S1},\\
	\S_{2,1}^{\text{left}} &=  \P_{\U}^\perp\Z\V\bSigma^{-2}\V^\top\Z^\top \P_{\U}^\perp \label{S21},
\end{align}
where $	\S_{2,1}^{\text{left}}$ is the term corresponding to $\s = (0,2,0)$.
Then for any $k\geq 1$ and $\s$ satisfies $s_1 + \cdots + s_{k+1} = k$, $\S_k(\s)$ must be of the following form
\begin{align}\label{Sks}
	\S_k(\s) = \A_1\bSigma^{-s_1}\B_1\bSigma^{-s_2}\cdots\B_k\bSigma^{-s_{k+1}}\A_2^\top,
\end{align}
where $\A_1,\A_2\in\{\U,\U^\top\}$, $\B_{i}\in\{\U^\top\Z\V, \U^\top_{\perp}\Z\V,\U^\top\Z\V_{\perp},\U_{\perp}^\top\Z\V_{\perp} \text{ or their transpose}\}$, and $\bSigma^0$ is the identity matrix whose size depends on the adjacent matrices.

\section{Martingale Concentration Inequality}
\begin{theorem}[Azuma-Hoeffding Inequality]\label{thm:martingale}
Suppose $X_n, n\geq 1$ is a martingale such that $X_0 = 0$ and $|X_{i} - X_{i-1}|\leq d_i, 1\leq i\leq n$ almost surely for some constant $d_i,1\leq i\leq n$. Then for every $t\geq 0$, 
$$\PP(|X_n|\geq t)\leq 2\exp(-\frac{t^2}{2\sum_{i=1}^nd_i^2}).$$
\end{theorem}

\section{Technical Lemmas}
The following lemma is a stronger version of Lemma 15.2 in \cite{cai2022generalized} when the perturbation has certain structure. 
\begin{lemma}\label{lemma:perturbation}
	Let $\bcalT = \bcalS\cdot(\V_1,\ldots,\V_m)$ be the tensor with Tucker rank $\r = (r_1,\ldots,r_m)$. Let $\bcalD = \calP_{\TT}(\bcalX)\in\RR^{d_1\times \cdots \times d_m}$ be a perturbation tensor such that $8\sigma_{\max}(\bcalD)\leq \lambda_{\submin}$. Then for $\calR = \hosvd$ we have 
	$$\fro{\calR(\bcalT +\bcalD) - \bcalT-\bcalD} \leq \frac{59m\fro{\bcalD}^2}{\lambda_{\submin}}.$$
\end{lemma}
\begin{proof}
	We introduce some notations that will be used throughout the proof. For an orthogonal matrix $\U\in\RR^{d\times r}$, let $\P_{\U} = \U\U^\top$ be the projector and $\up\in\RR^{d\times (d-r)}$ be the orthogonal complement of $\U$ and $\P_{\U}^{\perp} = \up\up^\top$. For a tensor $\bcalT$, we use $\T_i = \bcalT_{(i)}$ to denote its $i$-th unfolding.
	
	Without loss of generality, we prove the lemma when $m = 3$. First notice 
	$$\calR(\bcalT +\bcalD) = (\bcalT +\bcalD)\cdot(\P_{\U_1},\P_{\U_2},\P_{\U_3}),$$
	where $\U_i$ are the leading $r_i$ left singular vectors of $\T_i+\D_i$. Following Theorem 1 in \cite{xia2021normal}, we have for all $i$, 
	$$\P_{\U_i} - \P_{\V_i} = \S_{i,1} + \sum_{k\geq 2}\S_{i,k},$$
	where $\S_{i,1} = (\T_i^\top)^{\dagger}\D_i^\top\P_{\V_i}^{\perp} + \P_{\V_i}^{\perp}\D_i\T_i^{\dagger}$ and $\S_{i,k}$ satisfy $\op{\S_{i,k}}\leq (\frac{4\sigma_{\max}(\bcalD)}{\lambda_{\submin}})^k$ and
	the explicit form of $\S_{i,k}$ can be found in \cite{xia2021normal}. Here we denote $\A^{\dagger}$ the pseudo-inverse of $\A$, i.e., $\A^{\dagger} = \R\bSigma^{-1}\L^\top$ given the compact SVD of $\A = \L\bSigma\R^\top$.
	
	For the sake of brevity, we denote $\S_i = \sum_{k\geq 1} \S_{i,k}$. As a result of $8\sigma_{\max}(\bcalD)\leq \lambda_{\submin}$, we see
	$$\op{\S_i}\leq \sum_{k\geq 1}(\frac{4\sigma_{\max}(\bcalD)}{\lambda_{\submin}})^k \leq \frac{8\sigma_{\max}(\bcalD)}{\lambda_{\submin}}.$$
	Therefore,
	\begin{align}\label{eqTU}
		&\quad\bcalT\cdot(\P_{\U_1},\P_{\U_2},\P_{\U_3}) = \bcalT\cdot(\P_{\V_1} + \S_1,\P_{\V_2}+\S_2,\P_{\V_3}+\S_3)\notag\\
		&= \underbrace{\bcalT\cdot(\P_{\V_1},\P_{\V_2},\P_{\V_3})}_{\text{zeroth order term} = \bcalT} + \underbrace{\bcalT\cdot(\S_1,\P_{\V_2},\P_{\V_3}) + \bcalT\cdot(\P_{\V_1},\S_2,\P_{\V_3}) + \bcalT\cdot(\P_{\V_1},\P_{\V_2},\S_3)}_{\text{first order term}}\notag\\
		&\quad+\underbrace{\bcalT\cdot(\S_1,\S_2,\P_{\V_3}) + \bcalT\cdot(\P_{\V_1},\S_2,\S_3) + \bcalT\cdot(\S_1,\P_{\V_2},\S_3)}_{\text{second order term}} + \underbrace{\bcalT\cdot(\S_1,\S_2,\S_3)}_{\text{third order term}}.
	\end{align}
Now we consider $\bcalT\cdot(\P_{\U_1},\P_{\U_2},\P_{\U_3}) - \bcalT$. The zeroth order term is exactly $\bcalT$. For the first order term, we take $\bcalT\cdot(\S_1,\P_{\V_2},\P_{\V_3})$ as an example. And we consider $\calM_1(\bcalT\cdot(\S_1,\P_{\V_2},\P_{\V_3}))$, 
\begin{align}\label{eq1st}
	\calM_1\big(\bcalT\cdot(\S_1,\P_{\V_2},\P_{\V_3})\big) &= \S_1\T_1 = \S_{1,1}\T_1 + \sum_{k\geq 2}\S_{1,k}\T_1\notag\\
	&= \P_{\V_1}^{\perp}\D_1\T_1^\dagger\T_1 + \sum_{k\geq 2}\S_{1,k}\T_1\notag\\
	&= \calM_1\big(\bcalD\cdot(\P_{\V_1}^{\perp},\P_{\V_2},\P_{\V_3})\big) + \sum_{k\geq 2}\S_{1,k}\T_1,
\end{align}
where in the second line we use the fact $\bcalP_{\V_1}^{\perp}\T_1 = \boldsymbol{0}$.
This implies
$\fro{\S_1\T_1} \leq 2\fro{\bcalD}$ and 
 meanwhile $$\fro{\sum_{k\geq 2}\S_{1,k}\T_1} \leq \lambda_{\submin}\sum_{k\geq 2}(\frac{4\sigma_{\max}(\bcalD)}{\lambda_{\submin}})^k \leq 32\fro{\bcalD}^2/\lambda_{\submin}.$$
For the second order term, we take $\bcalT\cdot(\S_1,\S_2,\P_{\V_3})$ as an example and consider $\calM_1\big(\bcalT\cdot(\S_1,\S_2,\P_{\V_3})\big)$,
\begin{align}\label{eq2nd}
	\fro{\calM_1\big(\bcalT\cdot(\S_1,\S_2,\P_{\V_3})\big)} = \fro{\S_1\T_1(\P_{\V_3}\otimes \S_2)} \leq \fro{\S_1\T_1}\op{\P_{\V_3}\otimes \S_2}\leq \frac{16\fro{\bcalD}^2}{\lambda_{\submin}}.
\end{align}
For the third order term, we can similarly show $\fro{\bcalT\cdot(\S_1,\S_2,\S_3)} \leq \frac{\fro{\bcalD}^2}{\lambda_{\submin}}$. Putting this and \eqref{eqTU} - \eqref{eq2nd} together and we see 
\begin{align}\label{eq:TUR1}
	\bcalT\cdot(\P_{\U_1},\P_{\U_2},\P_{\U_3}) - \bcalT &= \bcalD\cdot(\P_{\V_1}^{\perp},\P_{\V_2},\P_{\V_3}) + \bcalD\cdot(\P_{\V_1},\P_{\V_2}^{\perp},\P_{\V_3}) \notag\\
	&\quad+ \bcalD\cdot(\P_{\V_1},\P_{\V_2},\P_{\V_3}^{\perp}) + \bcalR_1
\end{align}
with the remainder satisfying $\fro{\bcalR_1}\leq \frac{49m\fro{\bcalD}^2}{\lambda_{\submin}}$. We now consider $\bcalD\cdot(\P_{\U_1},\P_{\U_2},\P_{\U_3})$. Expanding this gives us
\begin{align}\label{eqDU}
	&\quad\bcalD\cdot(\P_{\U_1},\P_{\U_2},\P_{\U_3}) = \bcalD\cdot(\P_{\V_1} + \S_1,\P_{\V_2}+\S_2,\P_{\V_3}+\S_3)\notag\\
	&= \bcalD\cdot(\P_{\V_1},\P_{\V_2},\P_{\V_3}) + \bcalD\cdot(\S_1,\P_{\V_2},\P_{\V_3}) + \bcalD\cdot(\P_{\V_1},\S_2,\P_{\V_3}) + \bcalD\cdot(\P_{\V_1},\P_{\V_2},\S_3)\notag\\
	&\quad+\bcalD\cdot(\S_1,\S_2,\P_{\V_3}) + \bcalD\cdot(\P_{\V_1},\S_2,\S_3) + \bcalD\cdot(\S_1,\P_{\V_2},\S_3) +\bcalD\cdot(\S_1,\S_2,\S_3)\notag\\
	&=: \bcalD\cdot(\P_{\V_1},\P_{\V_2},\P_{\V_3}) + \bcalR_2.
\end{align}
Using a similar idea we can show that $\fro{\bcalR_2}\leq \frac{10m\fro{\bcalD}^2}{\lambda_{\submin}}$. 
Finally since $\bcalD = \calP_{\TT}(\bcalX)$ has the following form,
$$\bcalD = \calP_{\TT}(\bcalX) = \bcalX\times_{i=1}^3\P_{\V_i} + \sum_{i = 1}^3\bcalS\times_{j\in[3]\backslash i}\V_j\times_i \W_i.$$
Therefore $\bcalD\cdot(\P_{\V_1}^{\perp},\P_{\V_2}^{\perp},\P_{\V_3}) = \bcalD\cdot(\P_{\V_1}^{\perp},\P_{\V_2}^{\perp},\P_{\V_3}^{\perp}) = \boldsymbol{0}$.
Notice
\begin{align}\label{eq:DUR2}
	\bcalD &= \bcalD\cdot(\P_{\V_1}+\P_{\V_1}^{\perp}, \P_{\V_2}+\P_{\V_2}^{\perp},\P_{\V_3}+\P_{\V_3}^{\perp}) = \bcalD\cdot(\P_{\V_1},\P_{\V_2},\P_{\V_3})\notag\\
	&\quad + \bcalD\cdot(\P_{\V_1}^{\perp},\P_{\V_2},\P_{\V_3}) +\bcalD\cdot(\P_{\V_1},\P_{\V_2}^{\perp},\P_{\V_3})+\bcalD\cdot(\P_{\V_1},\P_{\V_2},\P_{\V_3}^{\perp}).
\end{align}
From \eqref{eq:TUR1},\eqref{eqDU} and \eqref{eq:DUR2}, we see 
\begin{align*}
	\calR(\bcalT+\bcalD) - \bcalT-\bcalD = \bcalR_1 + \bcalR_2,
\end{align*}
which implies $\fro{\calR(\bcalT+\bcalD) - \bcalT-\bcalD} \leq 59m\frac{\fro{\bcalD}^2}{\lambda_{\submin}}$.
\end{proof}

\begin{lemma}\label{lemma:projectgaussian}
	Let $\X\in\RR^{d_1\times d_2}$ be the random matrix with i.i.d. standard normal entries. Then for any orthogonal matrices $\U\in\RR^{d_1\times r_1}$, $\V\in\RR^{d_2\times r_2}$,
	$\psionel{\fro{\U^T\X\V}^2}\lesssim r_1r_2.$
\end{lemma}
\begin{proof}
	First we have $\fro{\U^T\X\V}^2\leq r_1\wedge r_2\op{\U^T\X\V}$. Now the result follows that $\U^T\X\V$ has the same distribution as a matrix of size $r_1\times r_2$ with i.i.d. standard normal entries, whose $\psi_1$ norm is bounded by $O(r_1\vee r_2)$.
\end{proof}

\begin{lemma}\label{lemma:psioneofptx}
	Let $\bcalT = \bcalC\cdot(\U_1,\cdots,\U_m)\in\RR^{d_1\times \cdots\times d_m}$ be a Tucker rank $\r$ tensor. Let $\TT$ be its corresponding tangent plane. Let $\bcalX\in\RR^{d_1\times \cdots\times d_m}$ be the random tensor having i.i.d. standard normal entries that is independent of $\bcalT$. Then
	$\psionel{\fro{\calP_{\TT}\bcalX}^2} \lesssim \dof.$
\end{lemma}
\begin{proof}
	Notice $\calP_{\TT}\bcalX$ must be of the following form:
	$$\calP_{\TT}\bcalX = \bcalX\times_{i=1}^m\P_{\U_i} + \sum_{i = 1}^m \bcalC\times_{j\in[m]\backslash i}\U_j\times \W_i,$$
	where $\W_i = \big(\bcalX\times_{j\in[m]\backslash i}\U_j^\top\times_i\P_{\U_i}^{\perp}\big)_{(i)}\bcalC_{(i)}^\dagger$. This implies each of the $m+1$ components are mutually orthogonal under the standard inner product. As a result,
	$$\fro{\calP_{\TT}\bcalX}^2 = \fro{\bcalX\times_{i=1}^m\P_{\U_i}}^2 + \sum_{i = 1}^m \fro{\bcalC\times_{j\in[m]\backslash i}\U_j\times \W_i}^2.$$
	
	For the first term, we have $\fro{\bcalX\times_{i=1}^m\P_{\U_i}}^2 = \fro{\U_1\bcalX_{(1)}(\U_m\otimes \cdots\otimes \U_2)^\top}^2$. 
Then from Lemma \ref{lemma:projectgaussian}, we have 
	$$\psionel{\fro{\bcalX\times_{i=1}^m\P_{\U_i}}^2} \lesssim r^*.$$
	
	On the other hand, for each $i\in[m]$, 
	\begin{align*}
		\fro{\bcalC\times_{j\in[m]\backslash i}\U_j\times \W_i}^2 &= \fro{\W_i\bcalC_{(i)}(\otimes_{j\neq i}\U_j)^\top}^2 = \fro{\P_{\U_i}^{\perp}\bcalX_{(i)}(\otimes_{j\neq i}\U_j)\bcalC_{(i)}^\dagger\bcalC_{(i)}(\otimes_{j\neq i}\U_j)^\top}^2\\
		&\leq \fro{\bcalX_{(i)}\O_i}^2,
	\end{align*}
	where $\O_i\in\RR^{d_i^-\times r_i}$ is an orthogonal matrix since $(\otimes_{j\neq i}\U_j)\bcalC_{(i)}^\dagger\bcalC_{(i)}(\otimes_{j\neq i}\U_j)^\top$ is a projector. And therefore
	$$\psionel{\fro{\bcalC\times_{j\in[m]\backslash i}\U_j\times \W_i}^2}\lesssim d_ir_i.$$
Using a triangular inequality and we get the desired result.
\end{proof}

\begin{lemma}[Lemma 3.2 of \cite{koch2010dynamical}]\label{lemma:ptperp}
Let $\bcalT\in\RR^{d_1\times \cdots\times d_m}$ be a Tucker rank $\r$ tensor. Let $\TT$ be its corresponding tangent plane. Then for any $\bcalX\in\RR^{d_1\times \cdots\times d_m}$ such that $\fro{\bcalX-\bcalT} \leq \frac{\sigma_{\min}(\bcalT)}{16m(m+3)}$ we have 
$$\fro{\calP_{\TT}^{\perp}\bcalX} \leq \frac{8m(m+3)}{\sigma_{\min}(\bcalT)}\fro{\bcalT-\bcalX}^2.$$
\end{lemma}

\begin{lemma}[Incoherence implies spikiness]\label{lemma:incohspiki}
Let $\bcalT\in\RR^{d_1\times \cdots\times d_m}$ be a Tucker rank $\r$ tensor with the decomposition $\bcalT = \bcalC\times_{j=1}^m\U_j$. Denote the condition number of $\bcalT$ as $\kappa_0 = \sigma_{\max}(\bcalT)/\sigma_{\min}(\bcalT)$. Suppose $\incoh(\U_j)\leq\mu_0$, then $\spiki(\bcalT) \leq \sqrt{r^*/\rmax}\kappa_0\cdot \mu_0^{m/2}$, where $r^* = r_1\cdots r_m$, $\rmax = \max_{j=1}^mr_j$.
\end{lemma}
\begin{proof}
	Notice that $$\linf{\bcalT} \leq \sigma_{\max}(\bcalC)\prod_{j=1}^m\ltwoinf{\U_j}\leq \sigma_{\max}(\bcalT)\prod_{j=1}^m\sqrt{\frac{\mu_0r_j}{d_j}}.$$
	On the other hand, 
	$\fro{\bcalT}\geq \sqrt{\rmax}\sigma_{\min}(\bcalT)$. Therefore
	$$\spiki(\bcalT) \leq \sqrt{r_1\cdots r_m/\rmax}\kappa_0\cdot \mu_0^{m/2}.$$
\end{proof}

\begin{lemma}\label{pseudoinversedistance}
	Let $\M_1,\M_2\in\RR^{d_1\times d_2}$ be two rank $r$ matrices. Denote $\sigma_{\submin} = \min\{\sigma_{\submin}(\M_1), \sigma_{\submin}(\M_2)\}$ and $\sigma_{\submax} = \max\{\sigma_{\submax}(\M_1), \sigma_{\submax}(\M_2)\}$. Suppose $\op{\M_1 - \M_2}\leq\frac{1}{2}\sigma_{\submin}$, then
	$$\op{\M_1^\dagger - \M_2^\dagger}\leq (4\sqrt{2}+1)\frac{\sigma_{\submax}}{\sigma_{\submin}^3}\op{\M_1-\M_2}.$$
\end{lemma}
\begin{proof}
	Let $\M_1 = \U_1\bSigma_1\V_1^\top$, be the compact SVD of $\M_1$ and 
	$\M_2 = \U_2'\bSigma_2'\V_2'^\top$ 
	be the compact SVD of $\M_1,\M_2$.
	Define
	\begin{align*}
		\L = \arg\min_{\L\in\OO_r}\op{\U_1-\U_2'\L}, \quad \R = \arg\min_{\R\in\OO_r}\op{\V_1-\V_2'R},
	\end{align*}
where $\OO_r$ is the set of all orthogonal matrices of size $r\times r$. 
Denote $\U_2 = \U_2'\L, \V_2 = \V_2'\R, \bSigma_2 = \L^\top\bSigma_2'\R$. Then $\M_2 = \U_2\bSigma_2\V_2^\top$
but $\bSigma_2$ is not necessarily diagonal. 
As a result, we can write $\M_1^\dagger = \V_1\bSigma_1^{-1}\U_1^\top, \M_2^\dagger = \V_2\bSigma_2^{-1}\U_2^\top$. 
Now since $\op{\M_1 - \M_2}\leq\frac{1}{2}\sigma_{\submin}$, using Wedin's sin$\Theta$ Theorem, we obtain 
\begin{align}\label{alldistance}
	\max\{\op{\U_1- \U_2}, \op{\V_1- \V_2}\}\leq \frac{\sqrt{2}\op{\M_1-\M_2}}{\sigma_{\submin}}.
\end{align}
We first bound $\op{\bSigma_1 - \bSigma_2}$. In fact,
\begin{align*}
	\bSigma_1 - \bSigma_2 &= \U_1^\top\U_1\bSigma_1\V_1^\top\V_1 -\U_2^\top\U_2\bSigma_2\V_2^\top\V_2 \\
	&= (\U_1-\U_2)^\top\M_1\V_1 + \U_2^\top(\M_1-\M_2)\V_1 + \U_2^\top\M_2(\V_1 - \V_2).	
\end{align*}
Using triangle inequality and \eqref{alldistance}, we obtain
\begin{align*}
	\op{\bSigma_1 - \bSigma_2} \leq (2\sqrt{2}+1)\frac{\sigma_{\submax}}{\sigma_{\submin}}\op{\M_1-\M_2}.
\end{align*}
Next we bound $\op{\bSigma_1^{-1} - \bSigma_2^{-1}}$ using this inequality,
\begin{align}\label{sigmainverse}
	\op{\bSigma_1^{-1} - \bSigma_2^{-1}} = \op{\bSigma_1^{-1}\bSigma_1(\bSigma_1^{-1} - \bSigma_2^{-1})\bSigma_2\bSigma_2^{-1}} \leq \sigma_{\submin}^{-2}\op{\bSigma_1-\bSigma_2}\leq (2\sqrt{2}+1)\frac{\sigma_{\submax}}{\sigma^3_{\submin}}\op{\M_1-\M_2}.
\end{align}
Now we are ready to bound $\op{\M_1^\dagger - \M_2^\dagger}$,
\begin{align*}
	\M_1^\dagger - \M_2^\dagger = \V_1\bSigma_1^{-1}\U_1^\top - \V_2\bSigma_2^{-1}\U_2^\top = (\V_1-\V_2)\bSigma_1^{-1}\U_1^\top + \V_2(\bSigma_1^{-1}-\bSigma_2^{-1})\U_1^\top + \V_2\bSigma_2^{-1}(\U_1 - \U_2)^\top.
\end{align*}
Using triangle inequality, \eqref{alldistance} and \eqref{sigmainverse}, we obtain
\begin{align*}
	\op{\M_1^\dagger - \M_2^\dagger}\leq (4\sqrt{2}+1)\frac{\sigma_{\submax}}{\sigma_{\submin}^3}\op{\M_1-\M_2}.
\end{align*}
\end{proof}

\section{Proofs}
\subsection{Proof of Theorem \ref{thm:gen}}
Define the event 
\begin{align*}
	&\calE_t = \bigg\{\forall 0\leq l\leq t, \fro{\bcalT_l - \bcalT^*}^2\leq 2(1-\frac{\eta\gamma_{\alpha}}{8})^{l}\fro{\bcalT_0 - \bcalT^*}^2 + C\eta\gamma_{\alpha}^{-1}\dof\errtwo^2,\\
	&\hspace{1.5cm}\fro{\calP_{\TT_l}(\bcalX_l)}\leq C(\dof\cdot \log\dmax)^{1/2},
	|\inp{\bcalX_l}{\calP_{\TT_l}(\bcalT_l-\bcalT^*)}|\leq C\fro{\calP_{\TT_l}(\bcalT_l-\bcalT^*)}\log^{1/2}\dmax,\\ 
	&\hspace{1.5cm}|\inp{\bcalT^*}{\bcalX_l}| \leq C\fro{\bcalT^*}\log^{1/2}\dmax,
	\quad|\inp{\bcalT_l}{\bcalX_l}| \leq C\fro{\bcalT_l}\log^{1/2}\dmax,\\
	&\hspace{8cm}|\inp{\bcalX_l}{\bcalT_l-\bcalT^*}|\leq C\fro{\bcalT_l-\bcalT^*}\log^{1/2} \dmax\bigg\}.
\end{align*}
Then under $\calE_t$ and Assumption \ref{assump:GLM-trueT},
\begin{align*}
	|\inp{\bcalT^*}{\bcalX_t}| &\leq C\fro{\bcalT^*}\log^{1/2}\dmax\leq \alpha,\\
	|\inp{\bcalT_t}{\bcalX_t}| &\leq C\fro{\bcalT_t}\log^{1/2}\dmax \leq C\fro{\bcalT^*}\log^{1/2}\dmax + C\fro{\bcalT_t-\bcalT^*}\log^{1/2}\dmax\\
	&\leq C\fro{\bcalT^*}\log^{1/2}\dmax  + C\fro{\bcalT_0-\bcalT^*}\log^{1/2}\dmax + C\big(\eta\gamma_{\alpha}^{-1}\dof\errtwo^2\log\dmax\big)^{1/2}\\
	&\leq \alpha,
\end{align*}
where the last inequality holds as long as $\eta\gamma_{\alpha}^{-1}\dof\errtwo^2\log\dmax\lesssim \alpha^2$.

\noindent\textit{Step 1: Relation between $\fro{\bcalT_{t+1} - \bcalT^*}^2\cdot\indicator{t}$ and $\fro{\bcalT_{t} - \bcalT^*}^2\cdot\indicator{t-1}$.}
\begin{align}\label{eq:contraction:gen}
	&\fro{\bcalT_{t+1} - \bcalT^*}^2\cdot\indicator{t} = \fro{\hosvd(\bcalT_t^+) -\bcalT_t^+ + \bcalT_t^+ - \bcalT^*}^2\cdot\indicator{t} \notag\\
	&\quad\leq (1+\delta^{-1}) \fro{\hosvd(\bcalT_t^+) -\bcalT_t^+}^2\cdot\indicator{t}  + (1+\delta)\fro{\bcalT_t^+ - \bcalT^*}^2\cdot\indicator{t}.
\end{align}
Since $\bcalG_t = h_{\theta}(\inp{\bcalT_t}{\bcalX_t},Y_t)\bcalX_t$, and $|\inp{\bcalT_t}{\bcalX_t}|,|\inp{\bcalT^*}{\bcalX_t}|\leq \alpha$, from Assumption \ref{assump:GLM-loss},
\begin{align}\label{eq:h}
	|h_{\theta}(\inp{\bcalT_t}{\bcalX_t},Y_t)| &\leq |h_{\theta}(\inp{\bcalT_t}{\bcalX_t},Y_t) -h_{\theta}(\inp{\bcalT^*}{\bcalX_t},Y_t)| + |h_{\theta}(\inp{\bcalT^*}{\bcalX_t},Y_t)|\notag\\
	&\leq \mu_{\alpha}|\inp{\bcalT_t-\bcalT^*}{\bcalX_t}|  +  |h_{\theta}(\inp{\bcalT^*}{\bcalX_t},Y_t)|,
\end{align}
we obtain under $\calE_t$,
\begin{align*}
	\fro{\calP_{\TT_t}(\bcalG_t)} &\leq C\bigg(\mu_{\alpha}\fro{\bcalT_t-\bcalT^*}\cdot\log^{1/2}\dmax  +  |h_{\theta}(\inp{\bcalT^*}{\bcalX_t},Y_t)|\bigg)(\dof\cdot\log\dmax)^{1/2}\\
	&\leq C\bigg(\mu_{\alpha}\fro{\bcalT_t-\bcalT^*}\cdot\log^{1/2}\dmax  +  \errinf\bigg)(\dof\cdot\log\dmax)^{1/2}.
\end{align*}
Under $\calE_t$, this implies
\begin{align*}
	\fro{\calP_{\TT_t}(\bcalG_t)} &\leq C\mu_{\alpha}\log\dmax\dof^{1/2}\big(\fro{\bcalT_0-\bcalT^*} + (\eta\gamma_{\alpha}^{-1}\dof\errtwo^2)^{1/2}\big) + C\errinf(\dof\cdot\log\dmax)^{1/2}\\
	&\leq C\mu_{\alpha}\log\dmax\dof^{1/2}\fro{\bcalT_0-\bcalT^*} +C\errinf(\dof\cdot\log\dmax)^{1/2},
\end{align*}
as long as $\eta\gamma_{\alpha}^{-1}\dof\mu_{\alpha}^2\log\dmax\frac{\errtwo^2}{\errinf^2}\lesssim 1$.
Now as long as $\eta\mu_{\alpha}\log\dmax(\dof)^{1/2}\fro{\bcalT_0-\bcalT^*}\lesssim\lambda_{\submin}$ and $\eta\errinf(\dof\log\dmax)^{1/2}\lesssim \lambda_{\submin}$,
$\eta\fro{\calP_{\TT_t}(\bcalG_t)}\leq \frac{1}{8}\lambda_{\submin}$, we use Lemma \ref{lemma:perturbation} with $\bcalD = -\eta\calP_{\TT_t}(\bcalG_t), \bcalT = \bcalT_t$, and we have 
\begin{align*}
	&\quad\fro{\hosvd(\bcalT_t^+) - \bcalT_t^+}^2\cdot\mathds{1}(\calE_t) \leq C_m\frac{\eta^4\fro{\calP_{\TT_t}(\bcalG_t)}^4\cdot\mathds{1}(\calE_t)}{\lambda_{\submin}^2}\\
	&\leq C_m \eta^4\mu_{\alpha}^4\fro{\bcalT_t-\bcalT^*}^4\cdot\log^{2}\dmax (\dof\cdot\log\dmax)^{2}\lambda_{\submin}^{-2}\cdot\mathds{1}(\calE_t) +C_m\eta^4 \errinf^4(\dof\cdot\log\dmax)^{2}\lambda_{\submin}^{-2}\\
	&\leq \frac{\eta^2\gamma_{\alpha}^2}{12}\fro{\bcalT_t-\bcalT^*}^2\cdot\mathds{1}(\calE_t) +C_m\eta^4 \errinf^4(\dof\cdot\log\dmax)^{2}\lambda_{\submin}^{-2},
\end{align*}
as long as $\eta^2\gamma_{\alpha}^{-2}\mu_{\alpha}^4\fro{\bcalT_0-\bcalT^*}^2\dof^2\log^4\dmax\lambda_{\submin}^{-2}\lesssim 1$ and $\eta^3\gamma_{\alpha}^{-3}\mu_{\alpha}^4\dof^3\log^5\dmax\errtwo^2\lambda_{\submin}^{-2}\lesssim 1$.

\noindent\textit{\underline{Lower bound for $\EE_t\inp{\bcalT_t - \bcalT^*}{\calP_{\TT_t}(\bcalG_t)}\cdot\indicator{t}$.}} 
We denote
\begin{align}\label{event:y}
		\calY_t &=\bigg\{\fro{\bcalT_t - \bcalT^*}^2\leq 2(1-\frac{\eta\gamma_{\alpha}}{8})^{t}\fro{\bcalT_0 - \bcalT^*}^2 + C\eta\gamma_{\alpha}^{-1}\dof\errtwo^2,\notag\\
		&\hspace{1.5cm}\fro{\calP_{\TT_t}(\bcalX_t)}\leq C(\dof\cdot \log\dmax)^{1/2},
		|\inp{\bcalX_t}{\calP_{\TT_t}(\bcalT_t-\bcalT^*)}|\leq C\fro{\calP_{\TT_t}(\bcalT_t-\bcalT^*)}\log^{1/2}\dmax,\notag\\ 
		&\hspace{1.5cm}|\inp{\bcalT^*}{\bcalX_t}| \leq C\fro{\bcalT^*}\log^{1/2}\dmax,
		\quad|\inp{\bcalT_t}{\bcalX_t}| \leq C\fro{\bcalT_t}\log^{1/2}\dmax,\notag\\
		&\hspace{8cm}|\inp{\bcalX_t}{\bcalT_t-\bcalT^*}|\leq C\fro{\bcalT_t-\bcalT^*}\log^{1/2} \dmax\bigg\}.
\end{align}
And $\bcalG_t^* = \nabla l(\bcalT^*,\mathfrak{D}_t) = h_{\theta}(\inp{\bcalT^*}{\bcalX_t},Y_t)\bcalX_t$. Then from Assumption \ref{assump:GLM-trueT}, $\EE_{Y_t}\bcalG_t^* = 0$. Since $\calY_t$ is independent of $Y_t$, $\EE_t\bcalG_t^* \cdot\mathds{1}(\calY_{t}) = 0$.
So we have $\EE_t\calP_{\TT_t}(\bcalG_t^*) \cdot\mathds{1}(\calY_{t}) = 0$ since $\calP_{\TT_t}$ is linear operator and  the expectation is taken with respect to $\bcalX_t,Y_t$ following the definition of $\EE_t$. 
Then $\EE_t\inp{\bcalT_t - \bcalT^*}{\calP_{\TT_t}(\bcalG_t)}\cdot\indicator{t} = \mathds{1}(\calE_{t-1})\cdot\EE_t\inp{\bcalT_t - \bcalT^*}{\calP_{\TT_t}(\bcalG_t)}\cdot\mathds{1}(\calY_{t})$. Meanwhile,
\begin{align*}
&\quad\EE_t\inp{\bcalT_t - \bcalT^*}{\calP_{\TT_t}(\bcalG_t)}\cdot\mathds{1}(\calY_{t}) \\
&= \EE_t\inp{\bcalT_t - \bcalT^*}{\bcalG_t-\bcalG_t^* }\cdot\mathds{1}(\calY_{t}) - \EE_t\inp{\bcalT_t - \bcalT^*}{\calP_{\TT_t}^\perp(\bcalG_t-\bcalG_t^* )}\cdot\mathds{1}(\calY_{t}).
\end{align*}
And since $|\inp{\bcalT_t}{\bcalX_t}|,|\inp{\bcalT^*}{\bcalX_t}|\leq \alpha$, using Assumption \ref{assump:GLM-loss},
\begin{align*}
	\EE_t\inp{\bcalT_t - \bcalT^*}{\bcalG_t-\bcalG_t^* }\cdot\mathds{1}(\calY_{t}) &= \EE_t\inp{\bcalT_t - \bcalT^*}{\bcalX_t}\big(h_{\theta}(\inp{\bcalT_t}{\bcalX_t},Y_t) - h_{\theta}(\inp{\bcalT^*}{\bcalX_t},Y_t)\big)\cdot\mathds{1}(\calY_{t})\\
	&\geq \gamma_{\alpha}\EE_t\inp{\bcalT_t - \bcalT^*}{\bcalX_t}^2\cdot\mathds{1}(\calY_{t})\\
	&=\gamma_{\alpha}\EE_t\inp{\bcalT_t - \bcalT^*}{\bcalX_t}^2\cdot\big(1 - \mathds{1}(\calY_{t}^c)\big)\\
	&\geq \gamma_{\alpha}\fro{\bcalT_t - \bcalT^*}^2 - \gamma_{\alpha}\big(\EE_t\inp{\bcalT_t - \bcalT^*}{\bcalX_t}^4\cdot\PP(\calY_{t}^c)\big)^{1/2}\\
	&\geq \frac{4}{5}\gamma_{\alpha}\fro{\bcalT_t - \bcalT^*}^2,
\end{align*}
where the last inequality holds since $\PP(\calY_{t}^c)\leq 14\dmax^{-10}$. On the other hand, 
\begin{align*}
	\EE_t\inp{\bcalT_t - \bcalT^*}{\calP_{\TT_t}^\perp(\bcalG_t-\bcalG_t^* )}\cdot\mathds{1}(\calY_{t})&\leq \EE_t|\inp{\bcalT_t - \bcalT^*}{\calP_{\TT_t}^\perp(\bcalX_t)}|\cdot|h_{\theta}(\inp{\bcalT_t}{\bcalX_t},Y_t) - h_{\theta}(\inp{\bcalT^*}{\bcalX_t},Y_t)|\\
	&\leq C\mu_{\alpha}\fro{\calP_{\TT_t}^\perp\bcalT^*}\fro{\bcalT_t-\bcalT^*}\leq C_m\mu_{\alpha}\frac{\fro{\bcalT_t-\bcalT^*}^3}{\lambda_{\submin}}.
\end{align*}
Now as long as 
$\fro{\bcalT_0-\bcalT^*}\lesssim_m \frac{\gamma_{\alpha}}{\mu_{\alpha}}\lambda_{\submin}$ and 
$\eta\gamma_{\alpha}^{-1}\dof\frac{\mu_{\alpha}^2}{\gamma_{\alpha}^2}\lesssim_m\frac{\lambda_{\submin}^2}{\errtwo^2}$,
we obtain 
\begin{align}\label{lowerbound:inp}
	\EE_t\inp{\bcalT_t - \bcalT^*}{\calP_{\TT_t}(\bcalG_t)}\cdot\indicator{t} \geq \frac{3}{4}\gamma_{\alpha}\fro{\bcalT_t - \bcalT^*}^2\cdot\indicator{t-1}.
\end{align}
We derive the bound for $\EE_t\fro{\calP_{\TT_t}\bcalG_t}^2\cdot\indicator{t}$. In fact, from \eqref{eq:h} and Lemma \ref{lemma:psioneofptx}, and $|\inp{\bcalT^*}{\bcalX_t}|\leq\alpha$,
\begin{align*}
	\EE_t\fro{\calP_{\TT_t}\bcalG_t}^2\cdot\indicator{t} &\leq \EE_t\fro{\calP_{\TT_t}\bcalX_t}^2\cdot\indicator{t-1} \big(\mu_{\alpha}^2\inp{\bcalT_t-\bcalT^*}{\bcalX_t}^2+ h^2_{\theta}(\inp{\bcalT^*}{\bcalX_t},Y_t)\big)\\
	&=\EE_{\bcalX_t}\fro{\calP_{\TT_t}\bcalX_t}^2\cdot\indicator{t-1} \big(\mu_{\alpha}^2\inp{\bcalT_t-\bcalT^*}{\bcalX_t}^2+\EE_{Y_t} h^2_{\theta}(\inp{\bcalT^*}{\bcalX_t},Y_t)\big)\\
	&\leq C\dof\mu_{\alpha}^2\fro{\bcalT_t-\bcalT^*}^2\cdot\indicator{t-1} + C\dof\errtwo^2,
\end{align*}
where in the last inequality is from the definition of $\errtwo^2$.
And therefore 
\begin{align*}
	\fro{\bcalT_t^+ - \bcalT^*}^2\cdot\indicator{t} &= \fro{\bcalT_t - \bcalT^* -\eta\calP_{\TT_t}\bcalG_t}^2\cdot\indicator{t}\\
	&=\fro{\bcalT_t - \bcalT^*}^2\cdot\indicator{t} - 2\eta\inp{\bcalT_t - \bcalT^*}{\calP_{\TT_t}\bcalG_t}\cdot\indicator{t} + \eta^2\fro{\calP_{\TT_t}\bcalG_t}^2\cdot\indicator{t}\\
	&= \bigg(\fro{\bcalT_t - \bcalT^*}^2 - 2\eta\EE_t\inp{\bcalT_t - \bcalT^*}{\calP_{\TT_t}\bcalG_t}\bigg)\cdot\indicator{t}\\ 
	&\quad+\bigg(2\eta\EE_t\inp{\bcalT_t - \bcalT^*}{\calP_{\TT_t}\bcalG_t} - 2\eta\inp{\bcalT_t - \bcalT^*}{\calP_{\TT_t}\bcalG_t}\bigg)\cdot\indicator{t}\\
	&\quad + \eta^2\EE_t\fro{\calP_{\TT_t}\bcalG_t}^2\cdot\indicator{t}+ \eta^2(\fro{\calP_{\TT_t}\bcalG_t}^2 - \EE_t \fro{\calP_{\TT_t}\bcalG_t}^2)\cdot\indicator{t}\\
	&\leq (1-\frac{3}{2}\eta\gamma_{\alpha})\fro{\bcalT_t - \bcalT^*}^2\cdot\indicator{t-1}+C\eta^2\dof(\mu_{\alpha}^2\fro{\bcalT_t-\bcalT^*}^2\cdot\indicator{t-1} + \errtwo^2)\\
	&\quad+ \eta^2(\fro{\calP_{\TT_t}\bcalG_t}^2 - \EE_t \fro{\calP_{\TT_t}\bcalG_t}^2)\cdot\indicator{t}\\
	&\quad +\bigg(2\eta\EE_t\inp{\bcalT_t - \bcalT^*}{\calP_{\TT_t}\bcalG_t} - 2\eta\inp{\bcalT_t - \bcalT^*}{\calP_{\TT_t}\bcalG_t}\bigg)\cdot\indicator{t}\\
	&\leq (1-\eta\gamma_{\alpha})\fro{\bcalT_t - \bcalT^*}^2\cdot\indicator{t-1}+C\eta^2\dof\errtwo^2\\
	&\quad+ \eta^2(\fro{\calP_{\TT_t}\bcalG_t}^2 - \EE_t \fro{\calP_{\TT_t}\bcalG_t}^2)\cdot\indicator{t}\\
	&\quad +\bigg(2\eta\EE_t\inp{\bcalT_t - \bcalT^*}{\calP_{\TT_t}\bcalG_t} - 2\eta\inp{\bcalT_t - \bcalT^*}{\calP_{\TT_t}\bcalG_t}\bigg)\cdot\indicator{t},
\end{align*}
where the last inequality holds as long as $\eta \dof\frac{\mu_{\alpha}^2}{\gamma_{\alpha}}\lesssim 1$.
Now we set $\delta = \frac{1}{2}\eta\gamma_{\alpha}$ in \eqref{eq:contraction:gen}, and we get
\begin{align*}
	\fro{\bcalT_{t+1} - \bcalT^*}^2\cdot\indicator{t} &\leq (1-\frac{1}{2}\eta\gamma_{\alpha})\fro{\bcalT_t - \bcalT^*}^2\cdot\indicator{t-1} + C\eta^2\dof\errtwo^2\\
	&\quad  +(1+\frac{1}{2}\eta\gamma_{\alpha})\eta^2(\fro{\calP_{\TT_t}\bcalG_t}^2 - \EE_t \fro{\calP_{\TT_t}\bcalG_t}^2)\cdot\indicator{t}\\
	&\quad +2\eta(1+\frac{1}{2}\eta\gamma_{\alpha})\bigg(\EE_t\inp{\bcalT_t - \bcalT^*}{\calP_{\TT_t}\bcalG_t} - \inp{\bcalT_t - \bcalT^*}{\calP_{\TT_t}\bcalG_t}\bigg)\cdot\indicator{t}\\
	&\quad + \frac{\eta\gamma_{\alpha}}{4}\fro{\bcalT_t-\bcalT^*}^2\cdot\mathds{1}(\calE_t) +C_m\eta^3\gamma_{\alpha}^{-1}\errinf^4(\dof\cdot\log\dmax)^{2}\lambda_{\submin}^{-2}\\
	&\leq (1-\frac{1}{4}\eta\gamma_{\alpha})\fro{\bcalT_t - \bcalT^*}^2\cdot\indicator{t-1} + C\eta^2\dof\errtwo^2\\
	&\quad  +(1+\frac{1}{2}\eta\gamma_{\alpha})\eta^2(\fro{\calP_{\TT_t}\bcalG_t}^2 - \EE_t \fro{\calP_{\TT_t}\bcalG_t}^2)\cdot\indicator{t}\\
	&\quad +2\eta(1+\frac{1}{2}\eta\gamma_{\alpha})\bigg(\EE_t\inp{\bcalT_t - \bcalT^*}{\calP_{\TT_t}\bcalG_t} - \inp{\bcalT_t - \bcalT^*}{\calP_{\TT_t}\bcalG_t}\bigg)\cdot\indicator{t},
\end{align*}
where the last inequality holds as long as $\frac{\errinf^2}{\errtwo^2}\eta\gamma_{\alpha}^{-1}\errinf^2\lambda_{\submin}^{-2}\dof\log^2\dmax\lesssim_m 1$. Telescoping this and we get
\begin{align*}
	\fro{\bcalT_{t+1} - \bcalT^*}^2&\cdot\indicator{t} \leq (1-\frac{\eta\gamma_{\alpha}}{4})^{t+1}\fro{\bcalT_0 - \bcalT^*}^2 + C\eta\gamma_{\alpha}^{-1}\dof\errtwo^2\\
	& + \sum_{l=0}^t\underbrace{(1-\frac{\eta\gamma_{\alpha}}{4})^{t-l}(1+\frac{1}{2}\eta\gamma_{\alpha})\eta^2(\fro{\calP_{\TT_l}\bcalG_l}^2 - \EE_l \fro{\calP_{\TT_l}\bcalG_l}^2)\cdot\indicator{l}}_{=:D_l}\\
	&+ \sum_{l=0}^t\underbrace{(1-\frac{\eta\gamma_{\alpha}}{4})^{t-l}2\eta(1+\frac{1}{2}\eta\gamma_{\alpha})\bigg(\EE_l\inp{\bcalT_l - \bcalT^*}{\calP_{\TT_l}\bcalG_l} - \inp{\bcalT_l - \bcalT^*}{\calP_{\TT_l}\bcalG_l}\bigg)\cdot\indicator{l}}_{=:F_l}.
\end{align*}

\hspace{1cm}

\noindent\textit{Step 2: Martingale concentration inequality.}
Now we use Azuma-Hoeffding inequality (c.f. Theorem \ref{thm:martingale}) to bound $|\sum_lD_l|$ and $|\sum_l F_l|$. Notice under the event $\calE_l$, 
\begin{align*}
	\fro{\calP_{\TT_l}\bcalG_l}^2\leq C\bigg(\mu_{\alpha}^2\fro{\bcalT_l-\bcalT^*}^2\cdot\log\dmax\cdot\indicator{l}  +  \errinf^2\bigg)(\dof\cdot\log\dmax).
\end{align*}
And therefore
\begin{align*}
	|D_l| \leq C(1-\frac{\eta\gamma_{\alpha}}{4})^{t-l}\eta^2\bigg(\mu_{\alpha}^2\fro{\bcalT_l-\bcalT^*}^2\cdot\log\dmax\cdot\indicator{l}   +  \errinf^2\bigg)(\dof\cdot\log\dmax) =: d_l.
\end{align*}
And
\begin{align*}
	\sum_{l=0}^{t} d_l^2 &\leq \sum_{l=0}^{t} C(1-\frac{\eta\gamma_{\alpha}}{4})^{2t-2l}\eta^4 \bigg(\mu_{\alpha}^4\fro{\bcalT_l-\bcalT^*}^4\cdot\log^2\dmax\cdot\indicator{l}   +  \errinf^4\bigg)(\dof^2\cdot\log^2\dmax)\\
	&\leq \sum_{l=0}^{t} C(1-\frac{\eta\gamma_{\alpha}}{4})^{2t-2l}\eta^4 \bigg(\mu_{\alpha}^4(1-\frac{\eta\gamma_{\alpha}}{8})^{2l}\fro{\bcalT_0-\bcalT^*}^4\cdot\log^2\dmax\cdot\indicator{l}  +  \errinf^4\bigg)(\dof^2\cdot\log^2\dmax)\\
	&\leq C\eta^3\gamma_{\alpha}^{-1}\mu_{\alpha}^4\dof^2\log^4\dmax(1-\frac{\eta\gamma_{\alpha}}{8})^{2t+2}\fro{\bcalT_0-\bcalT^*}^4 + C\eta^3\gamma_{\alpha}^{-1}\dof^2\cdot\log^2\dmax\errinf^4,
\end{align*}
where the second inequality holds as long as $\frac{\errtwo^4}{\errinf^4}\eta^2\mu_{\alpha}^4\gamma_{\alpha}^{-2}\dof^2\log^2\dmax\lesssim 1$. So with probability exceeding $1-2\dmax^{-10}$,
\begin{align}\label{D:gen}
	|\sum_{l=0}^{t} D_l| &\leq C(\eta^3\gamma_{\alpha}^{-1}\mu_{\alpha}^4\dof^2\log^5\dmax(1-\frac{\eta\gamma_{\alpha}}{8})^{2t+2}\fro{\bcalT_0-\bcalT^*}^4)^{1/2} + C(\eta^3\gamma_{\alpha}^{-1}\dof^2\cdot\log^3\dmax\errinf^4)^{1/2}\notag\\
	&\leq \frac{1}{2}(1-\frac{\eta\gamma_{\alpha}}{8})^{t+1}\fro{\bcalT_0-\bcalT^*}^2 + C\eta\gamma_{\alpha}^{-1}\dof\errtwo^2,
\end{align}
as long as $\eta^3\gamma_{\alpha}^{-1}\mu_{\alpha}^4\dof^2\log^5\dmax\lesssim 1$ and $\eta\gamma_{\alpha}\frac{\errinf^4}{\errtwo^4}\log^3\dmax\lesssim 1$. On the other hand, 
\begin{align*}
	&\quad|\inp{\bcalT_l - \bcalT^*}{\calP_{\TT_l}\bcalG_l}\cdot\indicator{l}| = |\inp{\bcalT_l - \bcalT^*}{\calP_{\TT_l}(\bcalG_l-\bcalG_l^*)}\cdot\indicator{l}| + |\inp{\bcalT_l - \bcalT^*}{\calP_{\TT_l}\bcalG_l^*}\cdot\indicator{l}|\\
	&\leq |\inp{\bcalT_l - \bcalT^*}{\calP_{\TT_l}\bcalX_l}|\cdot|\inp{\bcalT_l - \bcalT^*}{\bcalX_l}|\cdot\indicator{l} + |h_{\theta}(\inp{\bcalX_l}{\bcalT^*},Y_l)|\cdot|\inp{\bcalT_l - \bcalT^*}{\calP_{\TT_l}\bcalX_l}|\cdot\indicator{l}\\
	&\leq C\fro{\bcalT_l-\bcalT^*}^2\log\dmax\cdot\indicator{l} + C\errinf\fro{\bcalT_l-\bcalT^*}\sqrt{\log\dmax}\cdot\indicator{l}.
\end{align*}
And thus 
\begin{align*}
	|F_l| \leq C(1-\frac{\eta\gamma_{\alpha}}{4})^{t-l}\eta\bigg(\fro{\bcalT_l-\bcalT^*}^2\log\dmax\cdot\indicator{l} + \errinf\fro{\bcalT_l-\bcalT^*}\sqrt{\log\dmax}\cdot\indicator{l}\bigg)=:f_l.
\end{align*}
Also,
\begin{align*}
	\sum_{l=0}^{t}f_l^2 &\leq \sum_{l=0}^{t}C(1-\frac{\eta\gamma_{\alpha}}{4})^{2t-2l}\eta^2 \bigg(\fro{\bcalT_l-\bcalT^*}^4\cdot\log^2\dmax\cdot\indicator{l}   +  \errinf^2\fro{\bcalT_l-\bcalT^*}^2\cdot\log\dmax\cdot\indicator{l}\bigg)\\
	&\leq \sum_{l=0}^{t}C(1-\frac{\eta\gamma_{\alpha}}{4})^{2t-2l}\eta^2\bigg[(1-\frac{\eta\gamma_{\alpha}}{8})^{2l}\fro{\bcalT_0-\bcalT^*}^4\log^2\dmax + \eta^2 \gamma_{\alpha}^{-2}\dof^2\errtwo^4\log^2\dmax\\
	&\hspace{3cm}+ (1-\frac{\eta\gamma_{\alpha}}{8})^{l}\fro{\bcalT_0-\bcalT^*}^2\errinf^2\log\dmax + \eta\gamma_{\alpha}^{-1}\dof\errtwo^2\errinf^2\log\dmax\bigg]\\
	&\leq C(1-\frac{\eta\gamma_{\alpha}}{8})^{2t+2}\eta\gamma_{\alpha}^{-1}\fro{\bcalT_0-\bcalT^*}^4\log^2\dmax  + C\eta^3\gamma_{\alpha}^{-3}\dof^2\errtwo^4\log^2\dmax\\
	&\quad+ C(1-\frac{\eta\gamma_{\alpha}}{8})^{t+1}\eta\gamma_{\alpha}^{-1}\fro{\bcalT_0-\bcalT^*}^2\errinf^2\log\dmax + C\eta^2\gamma_{\alpha}^{-2}\dof\errtwo^2\errinf^2\log\dmax\\
	&\leq C(1-\frac{\eta\gamma_{\alpha}}{8})^{2t+2}\eta\gamma_{\alpha}^{-1}\fro{\bcalT_0-\bcalT^*}^4\log^2\dmax+ C(1-\frac{\eta\gamma_{\alpha}}{8})^{t+1}\eta\gamma_{\alpha}^{-1}\fro{\bcalT_0-\bcalT^*}^2\errinf^2\log\dmax\\
	&\quad + C\eta^2\gamma_{\alpha}^{-2}\dof\errtwo^2\errinf^2\log\dmax,
\end{align*}
where the last line holds as long as $\eta\gamma_{\alpha}^{-1}\dof\errtwo^2\errinf^{-2}\log\dmax\lesssim 1$. Now using the inequality $2ab\leq a^2+b^2$ to the second term and we obtain 
\begin{align*}
	\sum_{l=0}^{t}f_l^2 &\leq \frac{1}{4}(1-\frac{\eta\gamma_{\alpha}}{8})^{2t+2}\fro{\bcalT_0-\bcalT^*}^4(\log\dmax)^{-1} 
	+ C\eta^2\gamma_{\alpha}^{-2}\dof\errtwo^2\errinf^2\log^2\dmax,
\end{align*}
as long as $\eta\gamma_{\alpha}^{-1}\log^3\dmax\lesssim 1$ and $\frac{\errinf^2}{\errtwo^2}\log^2\dmax\lesssim\dof$. So with probability exceeding $1-2\dmax^{-10}$,
\begin{align}\label{F:gen}
	|\sum_{l=0}^{t} F_l| &\leq \frac{1}{2}(1-\frac{\eta\gamma_{\alpha}}{8})^{t+1}\fro{\bcalT_0-\bcalT^*}^2
	+ C\eta\gamma_{\alpha}^{-1}\dof^{1/2}\errtwo\errinf\log^{3/2}\dmax\notag\\
	&\leq \frac{1}{2}(1-\frac{\eta\gamma_{\alpha}}{8})^{t+1}\fro{\bcalT_0-\bcalT^*}^2
	+ C\eta\gamma_{\alpha}^{-1}\dof\errtwo^2,
\end{align}
where the last inequality holds as long as $\frac{\errinf^2}{\errtwo^2}\log^2\dmax\lesssim\dof$. Now from \eqref{D:gen} and \eqref{F:gen}, we obtain
\begin{align*}
	\fro{\bcalT_{t+1} - \bcalT^*}^2\cdot\indicator{t} \leq 2(1-\frac{\eta\gamma_{\alpha}}{8})^{t+1}\fro{\bcalT_0 - \bcalT^*}^2 + C\eta\gamma_{\alpha}^{-1}\dof\errtwo^2\log\dmax.
\end{align*}

\noindent\textit{Step 3: Controlling the probability.}
Now we bound the probability of the event $\calE_T$. Notice 
\begin{align*}
	&\quad\PP(\calE_{t}\cap\calE_{t+1}^c) \leq \PP\bigg(\calE_{t}\cap \{\fro{\bcalT_{t+1}-\bcalT^*}^2\geq 2(1-\frac{\eta\gamma_{\alpha}}{8})^{t+1}\fro{\bcalT_0 - \bcalT^*}^2 + C\eta\gamma_{\alpha}^{-1}\dof\errtwo^2.\}\bigg) \\
	&+ \PP\bigg(\fro{\calP_{\TT_{t+1}}(\bcalX_{t+1})}\gtrsim\sqrt{\dof \log d}\bigg) + \PP\bigg(|\inp{\bcalX_{t+1}}{\bcalT_{t+1}-\bcalT^*}|\gtrsim\fro{\bcalT_{t+1}-\bcalT^*}\sqrt{\log d}\bigg) \\
	&+ \PP\bigg(|\inp{\bcalX_{t+1}}{\calP_{\TT_{t+1}}(\bcalT_{t+1}-\bcalT^*)}|\gtrsim\fro{\calP_{\TT_{t+1}}(\bcalT_{t+1}-\bcalT^*)}\sqrt{\log d}\bigg) \\
	&+\PP\bigg(|\inp{\bcalT^*}{\bcalX_{t+1}}| \leq C\fro{\bcalT^*}\log^{1/2}\dmax\bigg) +
	\PP\bigg(|\inp{\bcalT_{t+1}}{\bcalX_{t+1}}| \leq C\fro{\bcalT_{t+1}}\log^{1/2}\dmax\bigg)\\
	&\leq 14d^{-10},
\end{align*}
where the last inequality holds from Lemma \ref{lemma:psioneofptx}, and the other four are Gaussian random variables since $\bcalX_t$ is independent of $\bcalT_t, \bcalT^*$. 
And thus
$$\PP(\calE_T^c) = \sum_{t= 1}^T\PP(\calE_{t-1}\cap\calE_t^c)\leq 14Td^{-10}.$$

\subsection{Proof of Theorem \ref{thm:init:linear-regression}}
\begin{proof}
	Notice for each $j\in[m]$, $\hat\U_j$ is actually the top $r_j$ left singular vectors of the following matrix:
	\begin{align*}
		\hat\N_j = \frac{1}{T_1(T_1-1)}\sum_{1\leq i<i'\leq T_1} Y_iY_{i'}(\calM_j(\bcalX_i)\calM_j(\bcalX_{i'})^\top+\calM_j(\bcalX_{i'})\calM_j(\bcalX_i)^\top).
	\end{align*}
	We denote $\N_j = \calM_j(\bcalT^*)\calM_j(\bcalT^*)^\top$. 
	Using Wedin's sin$\Theta$ theorem, we have
	\begin{align*}
		\op{\hat\U_j\hat\U_j^\top- \U_j^*\U_j^{*\top}} \leq \frac{\sqrt{2}\op{\hat\N_j - \N_j}}{\lambda_{\submin}^2}. 
	\end{align*}
	Now $\hat\N_j  -\N_j$ is a U-statistics of order 2, using standard decoupling techniques for U-statistics (see e.g. Theorem 3.4.1 in \cite{de2012decoupling}), we have 
	\begin{align*}
		\PP(\op{\hat\N_j - \N_j}\geq t )\leq 15\PP(\op{\tilde\N_j - \N_j}\geq t), 
	\end{align*}
	where 
	\begin{align*}
		\tilde\N_j = \frac{1}{2T_1(T_1-1)}\sum_{i\neq i'} Y_i\tilde Y_{i'}(\calM_j(\bcalX_i)\calM_j(\tilde\bcalX_{i'})^\top+\calM_j(\tilde\bcalX_{i'})\calM_j(\bcalX_{i})^\top),
	\end{align*}
	with $\tilde\bcalX_i, \tilde Y_i$ i.i.d. copy of $\bcalX_i,Y_i$ such that 
	\begin{align*}
		\tilde Y_{i} = \inp{\tilde\bcalX_{i}}{\bcalT^*} +\tilde\epsilon_i. 
	\end{align*}
	For notation simplicity, we drop the subscript $j$, and we denote $m_1 = d_j$, $m_2=  d_j^-$, $\M = \calM_j(\bcalT^*)\in\RR^{m_1\times m_2}$, and $\X_i = \calM_j(\bcalX_i)$. 
	We define 
	\begin{align*}
		\S_1 = \bDel_1 + \Z_1, \quad \S_2 = \bDel_2 + \Z_2,
	\end{align*}
	where 
	\begin{align*}
		&\bDel_1 = \left(\frac{1}{T_1}\sum_{i=1}^{T_1}\inp{\X_i}{\M}\X_i - \M\right),  &\bDel_2 = \left(\frac{1}{T_1}\sum_{i=1}^{T_1}\inp{\tilde\X_i}{\M}\tilde\X_i - \M\right),\\
		&\Z_1 = \frac{1}{T_1}\sum_{i=1}^{T_1}\epsilon_i\X_i, &\Z_2 = \frac{1}{T_1}\sum_{i=1}^{T_1}\tilde\epsilon_i\tilde\X_i.
	\end{align*}
	Recall we write $\N_j = \M\M^\top$ and thus 
	\begin{align*}
		\tilde\N_j - \N_j &= \frac{T_1}{2(T_1 - 1)}(\S_1\S_2^\top +\S_2\S_1^\top)+ \frac{T_1}{2(T_1-1)}(\S_1+\S_2)\M^\top + \frac{T_1}{2(T_1-1)}\M(\S_1+\S_2)^\top \\
		&\quad+ \frac{1}{T_1-1}\left(\frac{1}{2T_1}\sum_{i=1}^{T_1}Y_i\tilde Y_i(\X_i\tilde\X_i^\top + \tilde\X_i\X_i^\top) - \M\M^\top\right).
	\end{align*}
	We denote $M = \max\{m_1,m_2\}$. 
	Using matrix Bernstein inequality (see e.g. \cite{koltchinskii2011nuclear}) and Lemma 2.1 in \cite{koltchinskii2016perturbation}, we have the following event 
	\begin{align*}
		\calE_1 = \bigg\{\max\{\op{\Z_1},\op{\Z_2}\} \leq C\sigma\frac{\sqrt{M}\log M}{\sqrt{T_1}}\bigg\}
	\end{align*}
	holds with probability exceeding $1-M^{-10}$. And from matrix Bernstein inequality, 
	\begin{align*}
		\calE_2 = \bigg\{\max\{\op{\bDel_1}, \op{\bDel_2}\}\leq \fro{\M}\frac{\sqrt{M}\log M}{\sqrt{T_1}}\bigg\}
	\end{align*}
	holds with probability exceeding $1-M^{-10}$. We now proceed our proof conditioning on $\calE_1\cap \calE_2$. 
	
	\hspace{1cm}
	
	\noindent\textit{Upper bound for $\op{\S_1\S_2^\top}, \op{\S_2\S_1^\top}$. }
	We only consider the upper bound for $\op{\S_2\S_1^\top}$. Notice $\S_2$ is independent of $Y_i\X_i$. We shall proceed conditioning on $\S_2$. In fact, 
	\begin{align*}
		\S_2\S_1^\top = \frac{1}{T_1}\sum_{i=1}^{T_1}(Y_i\S_2\X_i^\top - \S_2\M^\top), 
	\end{align*}
	and that 
	\begin{align*}
		\EE \S_2(Y_i\X_i^\top -\M^\top)(Y_i\X_i -\M)\S_2^\top \preccurlyeq \sigma^2m_2\S_2\S_2^\top + 3m_1\fro{\M}^2\S_2\S_2^\top,
	\end{align*}
	and 
	\begin{align*}
		\EE (Y_i\X_i -\M)\S_2^\top\S_2(Y_i\X_i^\top -\M^\top) \preccurlyeq \sigma^2\tr(\S_2^\top\S_2)\I_{m_1} + 3\fro{\M}^2\tr(\S_2^\top\S_2)\I_{m_1}. 
	\end{align*}
	Moreover, we have 
	\begin{align*}
		&\quad \bigg\|\op{\S_2(Y_i\X_i^\top - \M^\top)}\bigg\|_{\psi_1} \leq \bigg\|\op{\S_2(\inp{\X_i}{\M}\X_i^\top - \M^\top)}\bigg\|_{\psi_1} + \bigg\|\epsilon_i\op{\S_2\X_i^\top}\bigg\|_{\psi_1}\\
		&\leq \psitwo{\inp{\X_i}{\M}}\cdot\psitwo{\op{\S_2\X_i^\top}} + \op{\S_2\M^\top} + \psitwo{\epsilon_i}\psitwo{\op{\S_2\X_i^\top}}\\
		&\lesssim (\fro{\M} +\sigma)\cdot\sqrt{m_1}\op{\S_2},
	\end{align*}
	where in the last line we use the fact that $\psitwo{\op{\S_2\X_i^\top}} \lesssim \sqrt{m_1}\op{\S_2}$ (see \cite{vershynin2011spectral}). 
	Now using the matrix Bernstein inequality (see e.g. Proposition 2 in \cite{koltchinskii2016perturbation}), we have with probability exceeding $1-M^{-10}$, the following event holds
	\begin{align*}
		\calE_3 = \bigg\{\op{\S_2\S_1^\top}\lesssim (\fro{\M} +\sigma)\cdot\op{\S_2}\frac{\sqrt{m_1}}{\sqrt{T_1}}\log M\bigg\}.
	\end{align*}
	And thus under $\calE_1\cap \calE_2\cap\calE_3$, 
	\begin{align*}
		\op{\S_2\S_1^\top}\lesssim (\fro{\M}^2 +\sigma^2)\cdot\frac{\sqrt{m_1M}}{T_1}\log^2 M.
	\end{align*}
	
	\hspace{1cm}
	
	\noindent\textit{Upper bound for $\op{\S_1\M^\top}, \op{\S_2\M^\top}$. }
	Notice 
	\begin{align*}
		\S_1\M^\top  = \frac{1}{T_1}\sum_{i=1}^{T_1}(Y_i\X_i - \M)\M^\top.
	\end{align*}
	And it is easy to verify that 
	\begin{align*}
		&\quad\max\bigg\{\big\|\EE(Y_i\X_i - \M)\M^\top\M(Y_i\X_i - \M)\big\|, \big\|\EE\M(Y_i\X_i - \M)(Y_i\X_i - \M)\M^\top\big\|\bigg\}\\
		&\lesssim m_1(\sigma^2 + \fro{\M}^2)\op{\M}^2, 
	\end{align*}
	and 
	\begin{align*}
		&\quad\bigg\|\op{(Y_i\X_i - \M)\M^\top}\bigg\|_{\psi_1} \leq \psitwo{\epsilon_i}\psitwo{\op{\X_i\M^\top}} + \op{\M}^2 + \psitwo{\inp{\X_i}{\M}}\psitwo{\op{\X_i\M^\top}}\\
		&\lesssim \sqrt{m_1}(\sigma + \fro{\M})\op{\M}. 
	\end{align*}
	Using matrix Bernstein inequality again, and we see with probability exceeding $1-M^{-10}$, the following event holds,
	\begin{align*}
		\calE_4 = \bigg\{\op{\S_1\M^\top} \lesssim (\sigma + \fro{\M})\op{\M}\sqrt{\frac{m_1}{T_1}}\log M\bigg\}. 
	\end{align*}

	\hspace{1cm}
	
	\noindent\textit{Upper bound for $\frac{1}{T_1-1}\left(\frac{1}{2T_1}\sum_{i=1}^{T_1}Y_i\tilde Y_i(\X_i\tilde\X_i^\top + \tilde\X_i\X_i^\top) - \M\M^\top\right)$. }
	In fact, we have
	\begin{align*}
		\frac{1}{T_1}\sum_{i=1}^{T_1}Y_i\tilde Y_i\X_i\tilde\X_i^\top- \M\M^\top &= \frac{1}{T_1}\sum_{i=1}^{T_1}Y_i\X_i(\tilde Y_i \tilde \X_i^\top - \M^\top) + \frac{1}{T_1}\sum_{i=1}^{T_1}(Y_i\X_i - \M)\M^\top.
	\end{align*}
	Here the second term is $\S_1\M^\top$, which is just bounded above. Now we conditioned on $Y_i\X_i, i = 1,\cdots, T_1$. One can similarly show 
	\begin{align*}
		\bigg\|\op{Y_i\X_i(\tilde Y_i \tilde \X_i^\top - \M^\top)}\bigg\|_{\psi_1} \lesssim \sqrt{m_1}(\sigma + \fro{\M})\op{Y_i\X_i}. 
	\end{align*}
	And 
	\begin{align*}
		&\quad\max\bigg\{\big\|\EE \sum_{i=1}^{T_1}Y_i^2\X_i(\tilde Y_i\tilde\X_i - \M)(\tilde Y_i\tilde\X_i - \M)\X_i^\top\big\|, \big\|\EE\sum_{i=1}^{T_1}Y_i^2(\tilde Y_i\tilde\X_i - \M)\X_i^\top\X_i(\tilde Y_i\tilde\X_i - \M)\big\|\bigg\}\\
		&\lesssim m_1(\sigma^2 + \fro{\M}^2)\max\bigg\{\left\|\sum_{i=1}^{T_1}Y_i^2\X_i\X_i^\top\right\|, \left\|\sum_{i=1}^{T_1}Y_i^2\X_i^\top\X_i\right\|\bigg\}.
	\end{align*}
	And we have from matrix Bernstein inequality again, and we see with probability exceeding $1-M^{-10}$, the following event holds,
	\begin{align*}
		\calE_5 = \bigg\{&\op{\frac{1}{T_1}\sum_{i=1}^{T_1}Y_i\X_i(\tilde Y_i \tilde \X_i^\top - \M^\top)} \lesssim \sqrt{m_1}(\sigma + \fro{\M})\frac{\log M}{T_1}\max_{i}\op{Y_i\X_i}\\
		&\quad + \sqrt{m_1}(\sigma + \fro{\M})\frac{\sqrt{\log M}}{T_1}\sqrt{\max\bigg\{\left\|\sum_{i=1}^{T_1}Y_i^2\X_i\X_i^\top\right\|, \left\|\sum_{i=1}^{T_1}Y_i^2\X_i^\top\X_i\right\|\bigg\}}
		\bigg\}.
	\end{align*}
	Now we consider the following event 
	\begin{align*}
		\calE_6 = \bigg\{\max_i \op{Y_i\X_i}\lesssim \sqrt{M}(\sigma + \fro{\M})\log T_1\bigg\}.
	\end{align*}
	Since $\psione{\op{Y_i\X_i}}\lesssim \sqrt{M}(\sigma + \fro{\M})$, $\calE_6$ holds with probability exceeding $1-T_1^{-10}$. 
	So we conclude on $\calE_5\cap\calE_6$, we have 
	\begin{align*}
		\op{\frac{1}{T_1}\sum_{i=1}^{T_1}Y_i\X_i(\tilde Y_i \tilde \X_i^\top - \M^\top)} \lesssim (\fro{\M}^2 + \sigma^2)\sqrt{\frac{m_1M\log M\log^2 T_1}{T_1}}
	\end{align*}
	
	\hspace{1cm}
	
	\noindent\textit{Finalize the proof for subspace. }
	Under $\bigcap_{k=1}^6\calE_k$, we have 
	\begin{align*}
		\op{\tilde\N_j - \N_j} &\lesssim (\sigma^2 + \fro{\M^2}) \frac{(m_1M)^{1/2}\log^2M}{T_1} + (\sigma + \fro{\M})\op{\M}\sqrt{\frac{m_1}{T_1}}\log M\\
		&\quad + (\sigma^2 + \fro{\M^2})\frac{(m_1M\log M \log^2 T_1)^{1/2}}{T_1^{3/2}}\\
		&\lesssim (\sigma^2 + \fro{\M^2}) \frac{(m_1M)^{1/2}\log^2M}{T_1} + (\sigma + \fro{\M})\op{\M}\sqrt{\frac{m_1}{T_1}}\log M
	\end{align*}
	as long as $T_1\geq \log^2 T_1$. We now plug in the $m_1, m_2, M$ and we obtain 
	\begin{align*}
		\op{\tilde\N_j - \N_j} &\lesssim_m(\sigma^2 + \fro{\M^2}) \frac{\big((d^*)^{1/2} + d_j\big)\log^2\dmax}{T_1} + (\sigma + \fro{\M})\op{\M}\sqrt{\frac{d_j}{T_1}}\log \dmax, 
	\end{align*}
	where $\lesssim_m$ hides constant depending only on the dimension $m$ of the tensor. Finally, we conclude with probability exceeding $1-6\dmax^{-10}$, 
	\begin{align*}
		\op{\hat\U_j\hat\U_j^\top - \U_j^*\U_j^{*\top}} &\lesssim_m (\sigma^2/\lambda_{\submin}^2 + \fro{\M}^2/\lambda_{\submin}^2) \frac{\big((d^*)^{1/2} + d_j\big)\log^2\dmax}{T_1} \\
		&\quad+ (\sigma/\lambda_{\submin} + \fro{\M}/\lambda_{\submin})\kappa_0\sqrt{\frac{d_j}{T_1}}\log \dmax. 
	\end{align*}

	\hspace{1cm}
	
	\noindent\textit{Estimation of core tensor. }
	Now we consider the accuracy for the core tensor estimation. 
	For notation simplicity, we shall use $\bcalX_t$ instead of $\bcalX_{t-T_1}$ and then $\hat\U_{j}, j\in[m]$ is independent of $\{\bcalX_t\}_{t=1}^{T_2}$. We denote the loss function 
	\begin{align*}
		L(\bcalC) = \frac{1}{2T_2} \sum_{i=1}^{T_2}\big(\inp{\bcalX_t}{\bcalC\times_1\hat\U_1\cdots\times_m\hat\U_m} - Y_t\big)^2. 
	\end{align*}
	And simple computation shows $\nabla L(\bcalC) =\frac{1}{T_2} \sum_{i=1}^{T_2} \big(\inp{\bcalX_t}{\bcalC\times_{j=1}^m\hat\U_j} - Y_t\big)\bcalX_t\times_{j=1}^m\hat\U_j^\top$, and 
	\begin{align*}
		\EE\nabla L(\bcalC) = \bcalC - \bcalC^*\times_{j=1}^m(\hat\U_j^\top\U_j^{*}). 
	\end{align*}
	It would also be helpful to notice $\nabla L(\hat\bcalC) = 0$ since $\hat\bcalC$ is the least square estimator. 
	We now decompose 
	\begin{align*}
		\fro{\hat \bcalC - \bcalC^*}^2 &= \inp{\hat \bcalC - \bcalC^*}{\hat \bcalC - \bcalC^*} = \inp{\hat \bcalC - \bcalC^*}{\EE\nabla L(\hat\bcalC) - \EE\nabla L(\bcalC^*)}\\
		&= \underbrace{\inp{\hat \bcalC - \bcalC^*}{\EE\nabla L(\hat\bcalC) -\nabla L(\hat\bcalC) }}_{=:\beta_1}- \underbrace{\inp{\hat \bcalC - \bcalC^*}{\EE\nabla L(\bcalC^*)}}_{\beta_2}. 
	\end{align*}
	Using the expression for $\EE\nabla L(\hat\C)$ and $\nabla L(\hat\bcalC)$ above, we can further decompose $\beta_1$ as 
	\begin{align*}
		\beta_1 &= - \frac{1}{T_2}\sum_{t=1}^{T_2} \inp{\bcalX_t\times_j\hat\U_j^\top}{\hat\bcalC - \bcalC^*\times_{j=1}^m(\hat\U_j^\top\U_j^{*})}\inp{\bcalX_t\times_j\hat\U_j^\top}{\hat\bcalC - \bcalC^*}\\
		&\quad +\inp{\hat \bcalC - \bcalC^*}{\hat\bcalC - \bcalC^*\times_{j=1}^m(\hat\U_j^\top\U_j^{*})} + \frac{1}{T_2}\sum_{t=1}^{T_2}\epsilon_t\inp{\bcalX_t\times_j\hat\U_j^\top}{\hat\bcalC - \bcalC^*}. 
	\end{align*}
	And we have 
	\begin{align*}
		\beta_{1,1}&:= \bigg|- \frac{1}{T_2}\sum_{t=1}^{T_2} \inp{\bcalX_t\times_j\hat\U_j^\top}{\hat\bcalC - \bcalC^*\times_{j=1}^m(\hat\U_j^\top\U_j^{*})}\inp{\bcalX_t\times_j\hat\U_j^\top}{\hat\bcalC - \bcalC^*} \\
		&\hspace{6cm}+\inp{\hat \bcalC - \bcalC^*}{\hat\bcalC - \bcalC^*\times_{j=1}^m(\hat\U_j^\top\U_j^{*})}\bigg|\\
		&=|\inp{\frac{1}{T_2}\sum_{t=1}^{T_2}\g_i\g_i^\top - \I}{\a\b^\top}|, 
	\end{align*}
	where $\g_i = \vec(\bcalX_t\times_j\hat\U_j^\top)\in\RR^{r^*}$ satisfies $\psitwo{\inp{\g_i}{\m}}\leq C_0\ltwo{\m}$ (notice we here implicitly use the fact $\hat\U_j$ and $\bcalX_i$ are independent), and $$\a = \vec(\hat \bcalC - \bcalC^*), \b = \vec\big(\hat\bcalC - \bcalC^*\times_{j=1}^m(\hat\U_j^\top\U_j^{*})\big). $$ 
	Using standard $\epsilon$-net argument, and we can see with probability exceeding $1-e^{-\delta_1}$ (for some $\delta_1\leq T_2$ to be specified),
	\begin{align*}
		\op{\frac{1}{T_2}\sum_{t=1}^{T_2}\g_i\g_i^\top - \I} \lesssim \sqrt{\frac{r^* + \delta_1}{T_2}}. 
	\end{align*} 
	Therefore, we have 
	\begin{align*}
		|\inp{\frac{1}{T_2}\sum_{t=1}^{T_2}\g_i\g_i^\top - \I}{\a\b^\top}| &\leq \sqrt{\frac{r^* + \delta_1}{T_2}}\nuc{\a\b^\top}\\
		&= \sqrt{\frac{r^* + \delta_1}{T_2}}\fro{\hat \bcalC - \bcalC^*} \fro{\hat\bcalC - \bcalC^*\times_{j=1}^m(\hat\U_j^\top\U_j^{*})}. 
	\end{align*}
	We have 
	\begin{align}\label{decomposition}
		&\quad \hat\bcalC - \bcalC^*\times_{j=1}^m(\hat\U_j^\top\U_j^{*}) = \hat\bcalC - \bcalC^* + \bcalC^* -  \bcalC^*\times_{j=1}^m(\hat\U_j^\top\U_j^{*})\notag\\
		&= \hat\bcalC - \bcalC^* +\sum_{j=1}^m \bcalC^*\times_1\I_{r_1}\cdots \times_j(\I_{r_j} - \hat\U_j^\top\U_j^{*})\times_{j+1}(\hat\U_{j+1}^\top\U_{j+1}^{*}) \cdots \times_m(\hat\U_{m}^\top\U_{m}^{*}).
	\end{align}
	Therefore 
	\begin{align*}
		\fro{\hat\bcalC - \bcalC^*\times_{j=1}^m(\hat\U_j^\top\U_j^{*})} &\leq \fro{\hat\bcalC - \bcalC^*} + \lambda_{\submax}\sum_{j=1}^m\fro{\I_{r_j} - \hat\U_j^\top\U_j^{*}}\\
		&\leq  \fro{\hat\bcalC - \bcalC^*} + \lambda_{\submax}\sum_{j=1}^m\fro{\hat\U_j\hat\U_j^\top - \U_j^*\U_j^{*\top}}.
	\end{align*}
	So we conclude 
	\begin{align*}
		\beta_{1,1} \leq \sqrt{\frac{r^* + \delta_1}{T_2}}\fro{\hat \bcalC - \bcalC^*}\bigg(\fro{\hat \bcalC - \bcalC^*} + \lambda_{\submax}\sum_{j=1}^m\fro{\hat\U_j\hat\U_j^\top - \U_j^*\U_j^{*\top}}\bigg).
	\end{align*}
	On the other hand, using similar $\epsilon$-net argument, we conclude with probability exceeding $1-e^{-\delta_1}$, 
	\begin{align*}
		\beta_{1,2} &:= \bigg|\frac{1}{T_2}\sum_{t=1}^{T_2}\epsilon_t\inp{\bcalX_t\times_j\hat\U_j^\top}{\hat\bcalC - \bcalC^*}\bigg| \leq \sqrt{\frac{r^* + \delta_1}{T_2}}\sigma\fro{\hat\bcalC - \bcalC^*}. 
	\end{align*}
	For $\beta_2$, using once again the decomposition in \eqref{decomposition}, 
	\begin{align*}
		|\beta_2| &= |\inp{\hat \bcalC - \bcalC^*}{\EE\nabla L(\bcalC^*)} = \inp{\hat \bcalC - \bcalC^*}{\bcalC^* - \bcalC^*\times_{j=1}^m(\hat\U_j^\top\U_j^{*})}|\\
		&=\big|\sum_{j=1}^m\inp{\hat \bcalC - \bcalC^*}{ \bcalC^*\times_1\I_{r_1}\cdots \times_j(\I_{r_j} - \hat\U_j^\top\U_j^{*})\times_{j+1}(\hat\U_{j+1}^\top\U_{j+1}^{*}) \cdots \times_m(\hat\U_{m}^\top\U_{m}^{*})}\big|\\
		&\leq \sum_{j=1}^m \fro{\hat\bcalC - \bcalC^*}\cdot \lambda_{\submax}\fro{\hat\U_j\hat\U_j^\top - \U_j^*\U_j^{*\top}}. 
	\end{align*}
	Putting everything together and we have 
	\begin{align*}
		\fro{\hat\bcalC - \bcalC^*}^2 &\leq \sqrt{\frac{r^* + \delta_1}{T_2}}\fro{\hat \bcalC - \bcalC^*}\bigg(\fro{\hat \bcalC - \bcalC^*} + \lambda_{\submax}\sum_{j=1}^m\fro{\hat\U_j\hat\U_j^\top - \U_j^*\U_j^{*\top}} + \sigma\bigg) \\
		&\quad + \sum_{j=1}^m \fro{\hat\bcalC - \bcalC^*}\cdot \lambda_{\submax}\fro{\hat\U_j\hat\U_j^\top - \U_j^*\U_j^{*\top}}.
	\end{align*}
	As a result, as long as $\sqrt{\frac{r^* + \delta_1}{T_2}}\leq \frac{1}{2}$ we have 
	\begin{align*}
		\fro{\hat\bcalC - \bcalC^*}\lesssim \lambda_{\submax}\sum_{j=1}^m\fro{\hat\U_j\hat\U_j^\top - \U_j^*\U_j^{*\top}}  +\sqrt{\frac{r^* + \delta_1}{T_2}}\sigma
	\end{align*}
	Finally, we have 
	\begin{align*}
		\fro{\hat\bcalC\times_{j=1}^m\hat\U_j - \bcalC^*\times_{j=1}^m\U_j^*} \leq \fro{\hat\bcalC - \bcalC^*} + \lambda_{\submax}\sum_{j=1}^m\fro{\hat\U_j\hat\U_j^\top - \U_j^*\U_j^{*\top}} \leq c_m\lambda_{\submin} 
	\end{align*}
	under the given sample size condition and SNR condition. 
\end{proof}

\subsection{Proof of Theorem~\ref{thm:poissonregression}}
Since $h_{\theta}(\theta,y) = -I^{-1}y + e^\theta$ and $g_t(\theta) = I^{-1}e^{\theta}$, 
Assumption \ref{assump:GLM-loss} holds with $\gamma_{\alpha} = e^{-\alpha}$ and $\mu_{\alpha} = e^{\alpha}$. Observe that 
\begin{align*}
	\max_{t=0}^T |h_{\theta}(\inp{\bcalX_t}{\bcalT^*},Y_t)| = I^{-1}\max_{t=0}^T|Y_t - I\exp(\inp{\bcalX_t}{\bcalT^*})|.
\end{align*}
Poisson distribution is sub-exponential (see Lemma \ref{lemma:poisson:psi1}) and $\|Y_t - I\exp(\inp{\bcalX_t}{\bcalT^*})\|_{\psi_1}\leq \log^{-1}(2I^{-1}e^{-\alpha} + 1)$, implying that $\errinf= \bar C_1 I^{-1}\log\dmax$  with probability exceeding $1-2T\dmax^{-10}$, where $\bar C_1 = \log^{-1}(2I^{-1}e^{-\alpha} + 1)$ holds. Similarly, by definition, $	\errtwo^2 = \max_{t=0}^T \max_{|\theta|\leq \alpha}g_t(\theta)= I^{-1}e^{\alpha}.$ 
Consequently, $\errratio = \bar c_1 I^{1/2}\log^{-1}\dmax$ with $\bar c_1 = e^{\alpha/2}\log^{-1}(2I^{-1}e^{-\alpha} + 1)$. The proof is concluded by Theorem \ref{thm:gen}.

\subsection{Proof of Theorem \ref{thm:init:poisson}}
We first state two lemmas that are useful in the proof.

\begin{lemma}\label{lemma:poisson:psi1}
	Let $W\sim\text{Pois}(\nu)$, then $\psione{{W-\nu}}\leq \frac{1}{\log(\frac{2}{\nu} + 1)}$. 
\end{lemma}
\begin{proof}
	This is as a result of $\EE e^{\lambda(W-\nu)} = \exp\big(\nu(e^{\lambda }- \lambda - 1)\big)$. And if we set $\lambda =\log(\frac{2}{\nu} + 1)$, $\EE e^{\lambda(W-\nu)} = 2$. The result is then from the equivalent definition of $\psi_1$ norm. 
\end{proof}

\begin{lemma}\label{lemma:poisson}
	Let $W\sim\text{Pois}(\nu)$ and $W' = W\mathds{1}(\frac{1}{10}\nu\leq W\leq 10\nu)+\nu\mathds{1}(W\notin[\frac{1}{10}\nu,10\nu])$. 
	Then there exists absolute constant $C_0$ such that for all $\nu\geq C_0$, we have 
	\begin{align*}
		\EE\big(\log(W'+\frac{1}{2}) - \log\nu\big)^2 \leq \frac{C}{\nu}
	\end{align*}
	for some absolute constant $C>0$. 
\end{lemma}
\begin{proof}
	We have 
	\begin{align*}
		\EE\big(\log(W'+\frac{1}{2}) - \log\nu\big)^2  = \EE\big(\log(W'+\frac{1}{2}) -\EE\log(W'+\frac{1}{2})\big)^2 + \big(\EE\log(W'+\frac{1}{2}) - \log \nu\big)^2.
	\end{align*}
	Using Lemma B3 in \cite{shi2022high}, we have
	\begin{align*}
		\psitwo{\log(W'+\frac{1}{2}) -\EE\log(W'+\frac{1}{2})}\leq \frac{C}{\sqrt{\nu}}.
	\end{align*}
	And thus 
	\begin{align*}
		\EE\big(\log(W'+\frac{1}{2}) -\EE\log(W'+\frac{1}{2})\big)^2 \leq \frac{C}{\nu}.
	\end{align*}
	On the other hand, from Lemma B1 in \cite{shi2022high}, we have 
	\begin{align*}
		\big(\EE\log(W'+\frac{1}{2}) - \log \nu\big)^2 \leq \frac{C}{\nu}.
	\end{align*}
	These together give the desired result. 
\end{proof}

\hspace{1cm}

\begin{proof}[Proof of Theorem \ref{thm:init:poisson}. ]
We consider the following event
\begin{align*}
	\calE_1 = \bigg\{|\inp{\bcalX_t}{\bcalT^*}|\leq \underbrace{C_0\fro{\bcalT^*}\log^{1/2}\dmax}_{:=\tau_0}, \quad\forall t\in[T]\bigg\},
\end{align*}
which holds with probability exceeding $1-T\dmax^{-100}$. 
Also notice form Assumption \ref{assump:GLM-trueT}, $\tau_0\leq \alpha$.  
We denote $\nu_t = I\exp(\inp{\bcalX_t}{\bcalT^*})$. 
Moreover we set $Y_t' = Y_t\mathds{1}(\frac{1}{10}\nu_t\leq Y_t\leq 10\nu_t)+\nu_t\mathds{1}(Y_t\notin[\frac{1}{10}\nu_t,10\nu_t])$. We also consider the event 
\begin{align*}
	\calE_2 = \bigg\{Y_t = Y_t', \quad \forall t\in[T]\bigg\}. 
\end{align*}
Then we have 
\begin{align*}
	\PP(Y_t\neq Y_t') \leq \PP(Y_t\geq 10\nu_t) + \PP(Y_t\leq \frac{1}{10}\nu_t).  
\end{align*}
And 
\begin{align*}
	\PP(Y_t\geq 10\nu_t)  &\leq   \PP(Y_t\geq 10\nu_t|\calE_1)  + \PP(\calE_1^c)\\
	&= \PP\big(Y_t\geq 10\nu_t\big| \nu_t\in [I\exp(-\tau_0), I\exp(\tau_0)]\big) + \PP(\calE_1^c)\\
	&\leq \exp\big(-I\exp(-\tau_0)\big) + T\dmax^{-100}\\
	&\leq 2T\dmax^{-100}. 
\end{align*}
where second inequality is from the tail probability of Poisson (see e.g. Lemma E.8 in \cite{han2022optimal}), and the last inequality holds as long as $I\geq C_0\exp(\tau_0)\log \dmax$. We can similarly bound the second term using again Lemma E.8 in \cite{han2022optimal}, and we conclude $\PP(\calE_2)\geq 1- 4T\dmax^{-100}$. Now we proceed under the event $\calE_1\cap\calE_2$. Under these events, we have 
\begin{align*}
	\op{\frac{1}{T}\sum_{t=1}^T \log\frac{Y_t' + \frac{1}{2}}{I}\X_t - \T^*} &\leq \op{\frac{1}{T}\sum_{t=1}^T \log\frac{Y_t' + \frac{1}{2}}{I}\X_t - \EE\log\frac{Y_1' + \frac{1}{2}}{I}\X_1} \\
	&\quad+ \op{ \EE\log\frac{Y_1' + \frac{1}{2}}{I}\X_1- \T^*},
\end{align*}
where $\T^* = \calM_{\calI}(\bcalT^*)$. 
Also, we have 
\begin{align*}
	\big|\log\frac{Y_t' + \frac{1}{2}}{I}\big| \leq \big|\log\frac{10Ie^{\tau_0} + \frac{1}{2}}{I}\big|\leq \log 20  +\tau_0,
\end{align*}
where the last inequality is due to $\frac{1}{2}\leq 10Ie^{\tau_0}$.
Therefore
\begin{align*}
	\big\|\op{\log\frac{Y_t' + \frac{1}{2}}{I}\X_t}\big\|_{\psi_2} \leq \sqrt{D_{\submax}}(\log 20  +\tau_0).
\end{align*}
Also, 
\begin{align*}
	\max\big\{\op{\EE \log^2\frac{Y_t' + \frac{1}{2}}{I}\X_t\X_t^\top}, \op{\EE \log^2\frac{Y_t' + \frac{1}{2}}{I}\X_t^\top\X_t}\big\} \leq \dmax(\log 20  +\tau_0)^2. 
\end{align*}
Using the matrix Bernstein inequality (see e.g. Proposition 2 in \cite{koltchinskii2016perturbation}), we have with probability exceeding $1-\dmax^{-100}$, 
\begin{align*}
	\op{\frac{1}{T}\sum_{t=1}^T \log\frac{Y_t' + \frac{1}{2}}{I}\X_t - \EE\log\frac{Y_1' + \frac{1}{2}}{I}\X_1} &\lesssim \bigg(\sqrt{\frac{ \log \dmax}{T}}  + \frac{ \log \dmax}{T}\bigg)\sqrt{\dmax}(\log 20  +\tau_0).
\end{align*}
Next we consider $\op{ \EE\log\frac{Y_1' + \frac{1}{2}}{I}\X_1- \T^*}$:
\begin{align*}
	\op{ \EE\log\frac{Y_1' + \frac{1}{2}}{I}\X_1- \T^*} &= \op{ \EE\big(\log\frac{Y_1' + \frac{1}{2}}{I}- \inp{\bcalX_1}{\bcalT^*}\big)\X_1}\\
	&=\sup_{\ltwo{\x}=\ltwo{\y}=1} \x^\top\EE\big(\log\frac{Y_1' + \frac{1}{2}}{I}- \inp{\bcalX_1}{\bcalT^*}\big)\X_1 \y\\
	&\leq \sup_{\ltwo{\x}=\ltwo{\y}=1}  \big[\EE\big(\log\frac{Y_1' + \frac{1}{2}}{I}- \inp{\bcalX_1}{\bcalT^*}\big)^2\big]^{1/2}\big[\EE(\x^\top\X_1 \y)^2]^{1/2}\\
	&= \big[\EE\big(\log(Y_1' + \frac{1}{2})-\log I- \inp{\bcalX_1}{\bcalT^*}\big)^2\big]^{1/2}\\
	&= \big[\EE \big(\log(Y_1' + \frac{1}{2}) - \log\nu_1\big)^2\big]^{1/2}\\
	&\leq \frac{C\exp(\tau_0/2)}{\sqrt{I}},
\end{align*}
where the last inequality is from Lemma \ref{lemma:poisson} and $\nu_1\geq I\exp(-\tau_0)$. 
In conclusion, we have with probability exceeding $1-5T\dmax^{-100}$, 
\begin{align*}
	\op{\frac{1}{T}\sum_{t=1}^T \log\frac{Y_t' + \frac{1}{2}}{I}\X_t - \T^*}\lesssim \sqrt{\frac{\dmax\log \dmax}{T}}(\log 20  +\tau_0) + \frac{\exp(\tau_0/2)}{\sqrt{I}}. 
\end{align*}
Next use Lemma 18 in \cite{shen2023computationally}, we conclude 
\begin{align*}
	\op{\tilde\T - \T^*} \lesssim \sqrt{\frac{\dmax\log \dmax}{T}}(\log 20  +\tau_0) + \frac{\exp(\tau_0/2)}{\sqrt{I}}
\end{align*}
under the conditions on intensity and sample size. Since $\rank(\hat\T), \rank(\T^*)\leq R$, we have 
\begin{align*}
	\fro{\tilde\T - \T^*} \lesssim \sqrt{\frac{R\dmax\log \dmax}{T}}(\log 20  +\tau_0) + \frac{\sqrt{R}\exp(\tau_0/2)}{\sqrt{I}}.
\end{align*}
Since $\calM_{\calI}$ is linear and preserve the norm, we have 
\begin{align*}
	\fro{\tilde\bcalT - \bcalT^*} \lesssim \sqrt{\frac{R\dmax\log \dmax}{T}}(\log 20  +\tau_0) + \frac{\sqrt{R}\exp(\tau_0/2)}{\sqrt{I}}.
\end{align*}
Finally using Lemma \ref{lemma:perturbation}, we conclude 
\begin{align*}
	\fro{\hat\bcalT - \bcalT^*} \lesssim \sqrt{\frac{R\dmax\log \dmax}{T}}(\log 20  +\tau_0) + \frac{\sqrt{R}\exp(\tau_0/2)}{\sqrt{I}}\leq c_me^{-2\alpha}\lambda_{\submin}^{-1},
\end{align*}
where the last inequality holds under the intensity and sample size conditions. 
\end{proof}

\subsection{Proof of Theorem \ref{thm:localrefinement:completion}}
In this section we present the proof of \cref{thm:localrefinement:completion}. 
We first introduce some notations. For any $\omega\in[d_1]\times\cdots\times[d_m]$, denote $\bcalE_\omega = \e_{\omega_1}\otimes\cdots\otimes\e_{\omega_m}$ be the standard basis for $\RR^{d_1\times\cdots\times d_m}$. Also, recall $\bcalT^*$ admits the decomposition $\bcalT^* = \bcalC^*\cdot(\U_1^*,\ldots,\U_m^*)$, and denote $\bcalT_t = \bcalC_t\cdot(\U_{t,1},\ldots,\U_{t,m})$.
Define the event 
\begin{align*}
	\calE_t &= \bigg\{\forall 0\leq l\leq t, \fro{\bcalT_l-\bcalT^*}^2\leq 2(1-\frac{\eta}{4})^{l}\fro{\bcalT_0-\bcalT^*}^2 + 10\eta\dof\sigma^2,\\
	&\linf{\bcalT_l-\bcalT^*} \lesssim_m (1-\frac{\eta}{4})^{\frac{l}{2}}\kappa_0^2(\frac{\mu_0^mr^*}{d^*})^{1/2}\lambda_{\submax} + \kappa_0^3(\frac{\mu_0^mr^*}{d^*})^{1/2}(\eta\dof)^{1/2}\sigma,\\
	&\hspace{3cm}|\epsilon_l|\lesssim \sigma\sqrt{\log\dmax},\hspace{0.5cm} \forall j\in[m], \ltwoinf{\U_{l,j}}^2 \leq \frac{\mu_0r_j}{d_j}\bigg\},
\end{align*}
where $\mu_0 = 20\kappa_0^2\mu$. 
In the following proof, at time step $t$, we are conditioning on the event $\calE_t$.  For notation simplicity, we shall drop the subscripts $t$ in the expression of $\bcalT_t$, i.e. $\bcalT_t = \bcalC\cdot(\U_1,\cdots,\U_m)$.
We also introduce the spikiness of a tensor, which is closely related to the incoherence. 
\begin{definition}\label{def:spikiness}
	Let $\bcalT\in\RR^{d_1\times \cdots\times d_m}$. The spikiness of $\bcalT$ is defined as 
	$$\spiki(\bcalT) := \frac{\sqrt{d_1\cdots d_m}\linf{\bcalT}}{\fro{\bcalT}}.$$
\end{definition}
Their relation between spikiness and incoherence is summarized in Lemma \ref{lemma:incohspiki}. And therefore, $\spiki(\bcalT_t) \leq \sqrt{r^*/\rmax}\kappa_0\cdot \mu_0^{m/2} =: \nu_0$.

For a tensor $\bcalT$ of Tucker rank $\r$, we define 
$$\sigma_{\submin}(\bcalT) = \min_{i=1}^m\sigma_{r_i}(\bcalT_{(i)}), \quad \sigma_{\submax}(\bcalT) = \max_{i=1}^m\sigma_{1}(\bcalT_{(i)}).$$
And recall we denote $\lambda_{\submin}= \sigma_{\submin}(\bcalT^*)$ and $\lambda_{\submax} = \sigma_{\submax}(\bcalT^*)$.

From $\fro{\bcalT_l-\bcalT^*}^2\leq 2(1-\frac{\eta}{4})^{l}\fro{\bcalT_0-\bcalT^*}^2 + 10\eta\dof\sigma^2$, we see that 
\begin{align}\label{fro:Tl}
	\fro{\bcalT_l} \leq \fro{\bcalT_l-\bcalT^*} + \fro{\bcalT^*} \leq \sqrt{2\rmin}\lambda_{\submax},
\end{align}
where the last inequality holds as long as $\eta\dof(\sigma/\lambda_{\submin})^2\lesssim \rmin\kappa_0^2$. And 
\begin{align}\label{lowerbound:signalstrength:Tl}
	\sigma_{\submin}(\bcalT_l) \geq \sigma_{\submin}(\bcalT^*) - \fro{\bcalT^*-\bcalT_l} \geq \frac{1}{2}\lambda_{\submin},
\end{align}
where the last line holds as long as $\eta\dof\sigma^2\lesssim\lambda_{\submin}^2$.

We now derive some bounds that will be used frequently.
Recall we write $\bcalT_t = \bcalC\cdot(\U_1,\cdots,\U_m)$ where $\U_j$ are $\mu_0$ incoherent under $\calE_t$. For any $\omega\in[d_1]\times\cdots\times[d_m]$,
$$\calP_{\TT_t}\bcalE_\omega = \bcalE_{\omega}\times_{j=1}^m \P_{\U_j} + \sum_{j=1}^m\bcalC\times_{k\neq j}\U_k\times_j\W_j$$
with $\W_j = \P_{\U_j}^\perp(\bcalE_\omega)_{(j)}(\otimes_{k\neq j}\U_k)\bcalC_{(j)}^\dagger$.
Notice these components are mutually orthogonal, using the incoherence of $\U_j$,
\begin{align}\label{uniform:ptE}
	\fro{\calP_{\TT}\bcalE_\omega}^2 &= \fro{\bcalE_{\omega}\times_{j=1}^m \P_{\U_j}}^2 + \sum_{j=1}^m\fro{\bcalC\times_{k\neq j}\U_k\times_j\W_j}^2\leq \frac{\mu_0^mr^*}{d^*}  +\sum_{j=1}^m \frac{\mu_0^{m-1}r_j^-}{d_j^-}\notag\\
	&\leq \frac{\mu_0^{m-1}r^*}{d^*}(\mu_0+\sum_{j=1}^md_j/r_j)\leq (m+1)\frac{\mu_0^{m-1}r^*}{d^*}\frac{\dmax}{\rmin},
\end{align}
where the last inequality holds since from the definition of incoherence, we have $\mu_0\rmin\leq \dmax$.
As a result,
\begin{align}\label{uniform:ptx}
	\fro{\calP_{\TT}\bcalX_t}^2 \leq (m+1)\mu_0^{m-1}r^*\frac{\dmax}{\rmin}.
\end{align}
Meanwhile,
\begin{align}\label{expectation:ptx}
	\EE_t\fro{\bcalP_{\TT_t}\bcalX_t}^2 &= \sum_{\omega} (\fro{\bcalE_{\omega}\times_{j=1}^m \P_{\U_j}}^2 + \sum_{j=1}^m\fro{\bcalC\times_{k\neq j}\U_k\times_j\W_j}^2)\leq  r^* + \sum_{j=1}^m d_jr_j = \dof,
\end{align}
since 
\begin{align*}
	\sum_{\omega}\fro{\bcalC\times_{k\neq j}\U_k\times_j\W_j}^2 = \fro{\P_{\U_j}^\perp}^2\cdot\fro{(\otimes_{k\neq j}\U_k)\bcalC_{(j)}^\dagger\bcalC_{(j)}(\otimes_{k\neq j}\U_k)^\top}^2\leq d_jr_j.
\end{align*}
Next we bound $\linf{\bcalT_t - \bcalT^*}$ that is dependent on $\fro{\bcalT_t-\bcalT^*}$. Notice 
\begin{align*}
	\bcalT_t - \bcalT^* &= (\bcalT_t - \bcalT^*)\times_{j=1}^m(\P_{\U_j^*} + \P_{\U_j^*}^\perp) \\
	&=(\bcalT_t - \bcalT^*)\times_{j=1}^m\P_{\U_j^*}+ \sum_{j=1}^m (\bcalT_t - \bcalT^*)\times_{k=1}^{j-1}\P_{\U_k^*}\times_j \P_{\U_j^*}^\perp\\
	&=(\bcalT_t - \bcalT^*)\times_{j=1}^m\P_{\U_j^*}+ \sum_{j=1}^m \bcalT_t  \times_{k=1}^{j-1}\P_{\U_k^*}\times_j \P_{\U_j^*}^\perp\times_{k=j+1}^m\P_{\U_k},
\end{align*}
where the last inequality holds since $\bcalT^*\times_j \P_{\U_j^*}^\perp = 0$ and $\bcalT_t = \bcalT_t\times_j\P_{\U_j}$. For any $\omega\in[d_1]\times \cdots\times [d_m]$, using the incoherence of $\U_j$ and $\U_j^*$, we obtain
\begin{align*}
	|\inp{(\bcalT_t - \bcalT^*)\times_{j=1}^m\P_{\U_j^*}}{\bcalE_{\omega}}| = |\inp{(\bcalT_t - \bcalT^*)}{\bcalE_{\omega}\times_{j=1}^m\P_{\U_j^*}}| \leq \fro{\bcalT_t - \bcalT^*}\sqrt{\frac{\mu^mr^*}{d^*}},
\end{align*}
and
\begin{align*}
	&\quad |\inp{\bcalT_t  \times_{k=1}^{j-1}\P_{\U_k^*}\times_j \P_{\U_j^*}^\perp\times_{k=j+1}^m\P_{\U_k}}{\bcalE_\omega}|  = |\inp{\bcalT_t  \times_j \P_{\U_j^*}^\perp}{\bcalE_\omega\times_{k=1}^{j-1}\P_{\U_k^*}\times_{k=j+1}^m\P_{\U_k}}|\\
	&= |\inp{(\bcalT_t-\bcalT^*)  \times_j \P_{\U_j^*}^\perp}{\bcalE_\omega\times_{k=1}^{j-1}\P_{\U_k^*}\times_{k=j+1}^m\P_{\U_k}}|\leq \fro{\bcalT_t-\bcalT^*}\sqrt\frac{\mu_0^{m-1}r_j^-}{d_j^-}.
\end{align*}
Therefore
\begin{align}\label{T-T*infty}
	\linf{\bcalT_t - \bcalT^*} \leq \fro{\bcalT_t - \bcalT^*}\big(\sqrt{\frac{\mu^mr^*}{d^*}}+\sum_{j=1}^m\sqrt\frac{\mu_0^{m-1}r_j^-}{d_j^-}\big).
\end{align}
On the other hand, using the spikiness of the components of $\bcalT_t$, we can derive a bound that is independent of $\fro{\bcalT_t - \bcalT^*}$, and from \eqref{fro:Tl},
\begin{align}\label{T-T*infty:idpt}
	\linf{\bcalT_t- \bcalT^*}^2 \leq(\linf{\bcalT_t}+\linf{\bcalT^*})^2\leq 3\frac{\nu_0^2\rmin\lambda_{\submax}^2}{d^*}.
\end{align}
Recall $\bcalG_t = (\inp{\bcalX_t}{\bcalT_t-\bcalT^*} - \epsilon_t)\bcalX_t$ and under the event $\calE_t$, $|\epsilon_t|\lesssim\sigma \sqrt{\log\dmax}$. And $\inp{\bcalX_t}{\bcalT_t - \bcalT^*}^2 \leq d^*\linf{\bcalT_t- \bcalT^*}^2$.
Using \eqref{uniform:ptx} and \eqref{T-T*infty:idpt}, we can derive a uniform bound for $\fro{\calP_{\TT_t}\bcalG_t}^2$:
\begin{align}\label{ptg:uniform}
	\fro{\calP_{\TT_t}\bcalG_t}^2 \leq 2(\inp{\bcalX_t}{\bcalT_t - \bcalT^*}^2 + |\epsilon_t|^2)\fro{\calP_{\TT}\bcalX_t}^2\lesssim_m (\nu_0^2\rmin\lambda_{\submax}^2 + \sigma^2\log \dmax)\mu_0^{m-1}r^*\frac{\dmax}{\rmin}.
\end{align}
We also derive the bound for $\EE_t\fro{\calP_{\TT_t}\bcalG_t}^2$ using \eqref{expectation:ptx}:
\begin{align}\label{expectation:ptg}
	\EE_t \fro{\calP_{\TT_t}\bcalG_t}^2 &= \EE_t \inp{\bcalX_t}{\bcalT_t - \bcalT^*}^2\fro{\calP_{\TT_t}\bcalX_t}^2 +\sigma^2\EE_t\fro{\calP_{\TT_t}\bcalX_t}^2\notag\\
	&\leq \frac{1}{d^*}\sum_{\omega}d^*[\bcalT_t - \bcalT^*]_{\omega}^2 \cdot d^*\fro{\calP_{\TT_t}\bcalE_\omega}^2 + \dof\cdot\sigma^2\notag\\
	&\leq 2m\mu_0^{m-1}r^*\frac{\dmax}{\rmin}\fro{\bcalT_t-\bcalT^*}^2 + \dof\cdot\sigma^2.
\end{align}

\hspace{1cm}

\noindent\textit{Step 1: Bounding $\fro{\bcalT_{t+1}-\bcalT^*}^2.$}
Using triangle inequality and we get 
\begin{align}\label{t+1andt:01}
	\fro{\bcalT_{t+1} - \bcalT^*}^2 &\leq (1+\frac{2}{\eta})\fro{\textsf{HOSVD}_{\r}(\bcalT_t^+ )- \bcalT_t^+ }^2 + (1+\frac{\eta}{2})\fro{\bcalT_t^+  - \bcalT^*}^2.
\end{align}
We can use Lemma \ref{lemma:perturbation} to derive a bound for the first term. 
As shown in \eqref{ptg:uniform}, we have 
$$\eta\fro{\calP_{\TT_t}\bcalG_t}\lesssim_m\eta(\nu_0^2\rmin\lambda_{\submax}^2 + \sigma^2\log \dmax)^{1/2}\mu_0^{(m-1)/2}(r^*)^{1/2}\frac{\dmax^{1/2}}{\rmin^{1/2}}\leq \frac{1}{8}\lambda_{\submin}$$
as long as $\eta\nu_0\mu_0^{(m-1)/2}\kappa_0(\dmax r^*)^{1/2}\lesssim_m 1$ and $\eta\frac{\sigma}{\lambda_{\submin}}\mu_0^{(m-1)/2}(\frac{\dmax}{\rmin}r^*\log\dmax)^{1/2}\lesssim_m 1$. We have verified the condition in Lemma \ref{lemma:perturbation}, and thus
\begin{align}
	\fro{\textsf{HOSVD}_{\r}(\bcalT_t^+ )- \bcalT_t^+ }^2 \leq C_m\frac{\eta^4\fro{\calP_{\TT_t}\bcalG_t}^4}{\lambda_{\submin}^2}
\end{align}
for some $C_m >0$ depending only on $m$. 
From \eqref{t+1andt:01} and \eqref{ptg:uniform}, we get 
\begin{align}\label{t+1andt:02}
	\fro{\bcalT_{t+1} - \bcalT^*}^2 &\leq C_m
	\eta^3(\nu_0^2\kappa_0^2\mu_0^{m-1}\dmax r^* + (\frac{\sigma}{\lambda_{\submin}})^2\mu_0^{m-1}r^*\frac{\dmax}{\rmin}\log\dmax)\fro{\calP_{\TT_t}\bcalG_t}^2 \notag\\
	&\quad + (1+\frac{\eta}{2})\fro{\bcalT_l^+ - \bcalT^*}^2\notag\\
	&\leq 2\eta^2\fro{\calP_{\TT_t}\bcalG_t}^2 + (1+\frac{\eta}{2})\fro{\bcalT_l^+ - \bcalT^*}^2,
\end{align}
where the last inequality holds as long as $\eta\nu_0^2\kappa_0^2\mu_0^{m-1}\dmax r^*\lesssim_m 1$ and $\eta(\frac{\sigma}{\lambda_{\submin}})^2\mu_0^{m-1}r^*\frac{\dmax}{\rmin}\log \dmax\lesssim_m 1$.
Now we consider $\EE_t\fro{\bcalT_t^+ - \bcalT^*}^2$:
\begin{align*}
	\EE_t\fro{\bcalT_t^+ - \bcalT^*}^2 &= \EE_t (\fro{\bcalT_t - \bcalT^*}^2 - 2\eta\inp{\bcalT_t - \bcalT^*}{\calP_{\TT_t}\bcalG_t} + \eta^2\fro{\calP_{\TT_t}\bcalG_t}^2) \\
	&= \fro{\bcalT_t - \bcalT^*}^2 - 2\eta\fro{\calP_{\TT_t}(\bcalT_t -\bcalT^*)}^2 + \eta^2\EE_t\fro{\calP_{\TT_t}\bcalG_t}^2\\
	&\leq (1-\frac{3\eta}{2})\fro{\bcalT_t - \bcalT^*}^2 + \eta^2\EE_t\fro{\calP_{\TT_t}\bcalG_t}^2,
\end{align*}
where in the last inequality we use Lemma \ref{lemma:ptperp}.
From \eqref{expectation:ptg} and \eqref{t+1andt:02}, we get
\begin{align*}
	\fro{\bcalT_{t+1} - \bcalT^*}^2 
	&\leq 2\eta^2\fro{\calP_{\TT_t}\bcalG_t}^2 + (1+\frac{\eta}{2})\fro{\bcalT_l^+ - \bcalT^*}^2 - \EE_t(2\eta^2\fro{\calP_{\TT_t}\bcalG_t}^2 + (1+\frac{\eta}{2})\fro{\bcalT_l^+ - \bcalT^*}^2)\\ 
	&\quad+ \EE_t(2\eta^2\fro{\calP_{\TT_t}\bcalG_t}^2 + (1+\frac{\eta}{2})\fro{\bcalT_t^+ - \bcalT^*}^2)\\
	&\leq 2\eta^2\fro{\calP_{\TT_t}\bcalG_t}^2 + (1+\frac{\eta}{2})\fro{\bcalT_t^+ - \bcalT^*}^2 - \EE_t\big(2\eta^2\fro{\calP_{\TT_t}\bcalG_t}^2 + (1+\frac{\eta}{2})\fro{\bcalT_t^+ - \bcalT^*}^2\big)\\ 
	&\quad + (1-\eta)\fro{\bcalT_t - \bcalT^*}^2  + 4\eta^2\EE_t\fro{\calP_{\TT_t}\bcalG_t}^2\\
	&\leq 2\eta^2\fro{\calP_{\TT_t}\bcalG_t}^2 + (1+\frac{\eta}{2})\fro{\bcalT_t^+ - \bcalT^*}^2 - \EE_t\big(2\eta^2\fro{\calP_{\TT_t}\bcalG_t}^2 + (1+\frac{\eta}{2})\fro{\bcalT_t^+ - \bcalT^*}^2\big)\\ 
	&\quad + (1-\frac{\eta}{2})\fro{\bcalT_t - \bcalT^*}^2  + 4\eta^2\dof\cdot\sigma^2,
\end{align*}
where the last inequality holds as long as $\eta\mu_0^{m-1}r^*\frac{\dmax}{\rmin}\lesssim_m 1$. Telescoping this equality and we get
\begin{align}\label{t+1t}
	&\fro{\bcalT_{t+1} - \bcalT^*}^2  \leq (1-\frac{\eta}{2})^{t+1}\fro{\bcalT_{0} - \bcalT^*}^2+8\eta\dof\cdot\sigma^2\notag\\ &+\sum_{l=0}^{t}\underbrace{(1-\frac{\eta}{2})^{t-l}\big[2\eta^2\fro{\calP_{\TT_l}\bcalG_l}^2 + (1+\frac{\eta}{2})\fro{\bcalT_l^+ - \bcalT^*}^2 - \EE_t\big(2\eta^2\fro{\calP_{\TT_l}\bcalG_l}^2 + (1+\frac{\eta}{2})\fro{\bcalT_l^+ - \bcalT^*}^2\big)\big]}_{=:D_l}.
\end{align}
Now we use martingale concentration inequality to bound $\sum_{l=0}^tD_l$. First we consider the uniform bound for $\fro{\calP_{\TT_l}\bcalG_l}^2$ using \eqref{uniform:ptE} and \eqref{T-T*infty},
\begin{align}\label{uniform:ptg:2}
	\fro{\calP_{\TT_l}\bcalG_l}^2 &\leq 2(\inp{\bcalX_l}{\bcalT_l-\bcalT^*}^2 +\epsilon_l^2) \fro{\calP_{\TT_l}\bcalX_l}^2\notag\\
	&\leq (2d^*\linf{\bcalT_l-\bcalT^*}^2 + 2\sigma^2\log\dmax)\cdot d^*\max_{\omega}\fro{\calP_{\TT_l}\bcalE_{\omega}}^2\notag\\
	&\lesssim_m (\mu_0^{m-1}r^*\frac{\dmax}{\rmin}\fro{\bcalT_l - \bcalT^*}^2 + \sigma^2\log\dmax)\mu_0^{m-1}r^*\frac{\dmax}{\rmin}.
\end{align}
Meanwhile,
\begin{align*}
	\bigg|\fro{\bcalT_l^+ - \bcalT^*}^2 - \EE_l\fro{\bcalT_l^+ - \bcalT^*}^2\bigg| &\leq 2\eta\bigg|\inp{\bcalT_l-\bcalT^*}{\calP_{\TT_l}\G_l} - \EE_l\inp{\bcalT_l-\bcalT^*}{\calP_{\TT_l}\G_l}\bigg|\\
	&\quad + \eta^2\bigg|\fro{\calP_{\TT_l}\G_l}^2 - \EE_l\fro{\calP_{\TT_l}\G_l}^2\bigg|.
\end{align*}
We consider the uniform bound for $|\inp{\bcalT_l-\bcalT^*}{\calP_{\TT_l}\G_l}|$. Using Cauchy-Schwartz inequality, and \eqref{uniform:ptg:2},
\begin{align*}
	|\inp{\bcalT_l-\bcalT^*}{\calP_{\TT_l}\G_l}| &\leq \fro{\bcalT_l-\bcalT^*} \fro{\calP_{\TT_l}\G_l}\\
	&\lesssim_m (\mu_0^{m-1}r^*\frac{\dmax}{\rmin}\fro{\bcalT_l - \bcalT^*}^2 + \sqrt{\mu_0^{m-1}r^*\frac{\dmax}{\rmin}}\sigma\sqrt{\log\dmax}\fro{\bcalT_l-\bcalT^*}).
\end{align*}
And therefore the uniform bound for $D_l$ is as follows
\begin{align}\label{D_l:uniform}
	|D_l|&\lesssim_m (1-\frac{\eta}{2})^{t-l}\eta\bigg(\mu_0^{m-1}r^*\frac{\dmax}{\rmin}\fro{\bcalT_l - \bcalT^*}^2 + \sqrt{\mu_0^{m-1}r^*\frac{\dmax}{\rmin}}\sigma\sqrt{\log d}\fro{\bcalT_l-\bcalT^*}\bigg)\notag\\
	&\quad +  (1-\frac{\eta}{2})^{t-l}\eta^2(\mu_0^{m-1}r^*\frac{\dmax}{\rmin}\fro{\bcalT_l - \bcalT^*}^2 + \sigma^2\log d)\mu_0^{m-1}r^*\frac{\dmax}{\rmin}\notag\\
	&\leq (1-\frac{\eta}{2})^{t-l}\eta\mu_0^{m-1}r^*\frac{\dmax}{\rmin}\fro{\bcalT_l - \bcalT^*}^2 + \eta\sigma^2\log\dmax\notag\\
	&\leq \frac{1}{2}(1-\frac{\eta}{4})^{t+1}\fro{\bcalT_0-\bcalT^*}^2(\log\dmax)^{-1} + \eta\dof\cdot\sigma^2(\log\dmax)^{-1},
\end{align}
where the second inequality uses $2ab\leq a^2 + b^2$ and $\eta\mu_0^{m-1}r^*\frac{\dmax}{\rmin}\log\dmax\lesssim 1$, and the last inequality is from 
$\fro{\bcalT_l-\bcalT^*}^2\leq 2(1-\frac{\eta}{4})^{l}\fro{\bcalT_0-\bcalT^*}^2 + C\eta\dof\sigma^2$ and $\eta\mu_0^{m-1}r^*\frac{\dmax}{\rmin}\log\dmax\lesssim 1$.

We now consider the variance bound for $D_l$. Firstly we consider $\EE_l \fro{\calP_{\TT_l}\bcalG_l}^4$ using \eqref{T-T*infty} and \eqref{uniform:ptx}, 
\begin{align*}
	&\EE_l \fro{\calP_{\TT_l}\bcalG_l}^4 \leq 4\EE_l\inp{\bcalX_l}{\bcalT_l - \bcalT^*}^4 \fro{\calP_{\TT_l}\bcalX_l}^4 + 4\EE_l\epsilon_l^4 \fro{\calP_{\TT_l}\bcalX_l}^4\\
	&\lesssim_m \fro{\bcalT_l-\bcalT^*}^2\frac{\mu_0^{m-1}r^*\dmax}{\rmin}\cdot (\frac{\mu_0^{m-1}r^*\dmax}{\rmin})^2\cdot\fro{\bcalT_l-\bcalT^*}^2 + \dof\cdot\frac{\mu_0^{m-1}r^*\dmax}{\rmin}\sigma^4\\
	&= (\frac{\mu_0^{m-1}r^*\dmax}{\rmin})^3\fro{\bcalT_l-\bcalT^*}^4  + \dof\cdot\frac{\mu_0^{m-1}r^*\dmax}{\rmin}\sigma^4.
\end{align*}
On the other hand, 
\begin{align*}
	\EE_l|\inp{\bcalT_l-\bcalT^*}{\calP_{\TT_l}\bcalG_l}|^2 &= \EE_l|\inp{\bcalT_l-\bcalT^*}{\calP_{\TT_l}\bcalX_l}|^2(\epsilon_l^2 + |\inp{\bcalT_l-\bcalT^*}{\bcalX_l}|^2)\\
	&\leq \sigma^2\EE_l|\inp{\bcalT_l-\bcalT^*}{\calP_{\TT_l}\bcalX_l}|^2 + d^*\linf{\bcalT_l-\bcalT^*}^2\EE_l|\inp{\bcalT_l-\bcalT^*}{\calP_{\TT_l}\bcalX_l}|^2\\
	&\lesssim_m \sigma^2\fro{\bcalT_l-\bcalT^*}^2 + \frac{\mu_0^{m-1}r^*\dmax}{\rmin}\fro{\bcalT_l-\bcalT^*}^4.
\end{align*}
So as long as $\eta\mu_0^{m-1}\frac{r^*\dmax}{\rmin}\lesssim_m 1$,
\begin{align*}
	\text{Var}_l D_l^2&\lesssim_m (1-\frac{\eta}{2})^{2t-2l}\eta^2\frac{\mu_0^{m-1}r^*\dmax}{\rmin}\fro{\bcalT_l - \bcalT^*}^4 +(1-\frac{\eta}{2})^{2t-2l}\eta^4\frac{\mu_0^{m-1}r^*\dmax}{\rmin}\sigma^4 \\
	&\quad + (1-\frac{\eta}{2})^{2t-2l}\eta^2\sigma^2\fro{\bcalT_l - \bcalT^*}^2.
\end{align*}
Together with $\fro{\bcalT_l-\bcalT^*}^2 \leq 2(1-\frac{\eta}{4})^{l}\fro{\bcalT_0-\bcalT^*}^2 + C\eta\dof\sigma^2$, we obtain
\begin{align*}
	\sum_{l=0}^t\text{Var}_l D_l^2 &\lesssim_m (1-\frac{\eta}{4})^{2t+2}\eta\frac{\mu_0^{m-1}r^*\dmax}{\rmin}\fro{\bcalT_0-\bcalT^*}^2 + \eta^3\frac{\mu_0^{m-1}r^*\dmax}{\rmin}\dof^2\sigma^4\\
	&\quad  +\eta^3\frac{\mu_0^{m-1}r^*\dmax}{\rmin}\sigma^4+(1-\frac{\eta}{4})^{t+1}\eta\sigma^2\fro{\bcalT_0-\bcalT^*}^2 + \eta^2\dof\sigma^4\\
	&\leq \frac{1}{4}(1-\frac{\eta}{4})^{2t+2}\fro{\bcalT_0-\bcalT^*}^4(\log\dmax)^{-1} + 
	 \eta^2\dof^2\sigma^4(\log\dmax)^{-1},
\end{align*}
where the last inequality holds as long as $\log\dmax\lesssim \dof$ and $\eta\frac{\mu_0^{m-1}r^*\dmax}{\rmin}\log\dmax\lesssim 1$.
%
Using the variance bound and \eqref{D_l:uniform}, we see that with probability exceeding $1-\dmax^{-10}$,
\begin{align*}
	\sum_{l=0}^{t}D_l &\leq (1-\frac{\eta}{4})^{t+1}\fro{\bcalT_0-\bcalT^*}^2+ 2\eta\dof\cdot\sigma^2.
\end{align*}
Put this back to \eqref{t+1t}, and we see that 
\begin{align}
	\fro{\bcalT_{t+1} - \bcalT^*}^2 \leq 2(1-\frac{\eta}{4})^{t+1}\fro{\bcalT_0-\bcalT^*}^2 + 10 \eta\dof\cdot\sigma^2.
\end{align}

\hspace{1cm}

\noindent\textit{Step 2: Incoherence and entry-wise error bound.} 
Now we bound the incoherence for $\U_{t+1,j}$ for $j\in[m]$ and the entry-wise bound using the representation formula of spectral projectors. Remember we denote $\bcalT_t = \bcalC\cdot(\U_1,\cdots,\U_m)$.

For each $j\in[m], i\in[d_j]$, we are interested in $\e_i^\top\U_{t+1,j}\U_{t+1,j}^\top\e_i$ and $\ltwo{(\U_{t+1,j}\U_{t+1,j}^\top - \U_j^*\U_j^{*\top})\e_i}$.
Since $\U_{t+1,j}$ is the top $r_j$ left singular vectors of $(\bcalT_t - \eta\calP_{\TT_t}\bcalG_t)_{(j)}$.
And $\U_j$ is the left singular vectors of $(\bcalT_t)_{(j)} = \U_j\bcalC_{(j)}(\otimes_{l\neq j}\U_l)^{\top}$.
We can obtain a closed form for $\U_{t+1,j}\U_{t+1,j}^{\top} - \U_j\U_j^\top$ from Section \ref{sec:assymetric} as follows
$$\U_{t+1,j}\U_{t+1,j}^{\top} - \U_j\U_j^\top = \sum_{k \geq 1}\S_{k}.$$
Here $\S_k$ depends on $\U,\bSigma,\V,\Z$ defined as follows (see Section \ref{sec:assymetric} for more details):
suppose $\bcalC_{(j)}$ admits a compact SVD as $\bcalC_{(j)} = \Q_1\bSigma\Q_2^\top$ with $\Q_1\in\RR^{r_j\times r_j}$ and $\Q_2\in\RR^{r_j^-\times r_j}$, and $\U = \U_j\Q_1$, $\V = (\otimes_{l\neq j}\U_l)\Q_2$ and $\Z = -\eta(\calP_{\TT_t}\bcalG_t)_{(j)}$.
Here for each $\S_k$, it can be written as 
$$\S_k = \sum_{\s: s_1+\cdots + s_{k+1} = k}\S_k(\s), $$
where $\S_k(\s)$ takes the following form:
\begin{align}\label{representation}
	\A_1\bSigma^{-s_1}\B_1\bSigma^{-s_2}\B_2\bSigma^{-s_3}\cdots\B_k\bSigma^{-s_{k+1}}\A_2^\top,
\end{align}
where $\A_1,\A_2\in \{\U,\U_{\perp}\}$ and 
$\B_i \in \{\U^\top\Z\V, \U^\top_{\perp}\Z\V,\U^\top\Z\V_{\perp},\U_{\perp}^\top\Z\V_{\perp} \text{ or their transpose}\}.$

As a consequence of \eqref{lowerbound:signalstrength:Tl}, $\op{\bSigma^{-1}} = \sigma_{r_j}^{-1}(\bcalT_{(j)})\leq 2\lambda_{\submin}^{-1}$.
Also,
\begin{align}\label{ptx:junfolding}
	(\calP_{\TT_t}\bcalX_t)_{(j)} = \P_{\U_j}(\bcalX_t)_{(j)}(\otimes_{l\neq j}\P_{\U_l})^\top + \sum_{l\neq j}\U_j\bcalC_{(j)}(\W_l\otimes_{p\neq j,l}\U_p)^\top + \W_j\bcalC_{(j)}(\otimes_{l\neq j}\U_l)^{\top},
\end{align}
where $\W_l = \P_{\U_l}^\perp(\bcalX_t)_{(l)}(\otimes_{q\neq l}\U_q)\bcalC_{(l)}^\dagger,\forall l\in[m]$. 

Now we derive the operator norm bound for $\op{\U^\top\Z\V}, \op{\U^\top_{\perp}\Z\V},\op{\U^\top\Z\V_{\perp}},\op{\U_{\perp}^\top\Z\V_{\perp}}$. 
Since $\bcalG_t = (\inp{\bcalX_t}{\bcalT_t-\bcalT^*}-\epsilon_t)\bcalX_t$, we see $\U_{\perp}^\top\Z\V_{\perp} = 0$.
For $\op{\U^\top\Z\V}$, we have
$$\op{\U^\top\Z\V} = |\eta(\inp{\bcalX_t}{\bcalT_t-\bcalT^*}-\epsilon_t)|\op{\U_j^\top(\bcalX_t)_{(j)}(\otimes_{l\neq j}\U_l)}\leq |\eta (\inp{\bcalX_t}{\bcalT_t-\bcalT^*}-\epsilon_t)|\sqrt{\mu_0^mr^*},$$
where the last inequality holds from the incoherence of $\U_l$ and $(\bcalX_t)_{(j)}$ is of the form $\sqrt{d^*}\e_p\e_q^\top$.

For $\op{\U^\top_{\perp}\Z\V}$, we have $\U^\top_{\perp}(\calP_{\TT_t}\bcalX_t)_{(j)}\V = \U^\top_{\perp}(\bcalX_t)_{(j)}(\otimes_{l\neq j}\U_l)\bcalC_{(j)}^\dagger\bcalC_{(j)}(\otimes_{l\neq j}\U_l)^\top\Q_2$, and thus
$$\op{\U^\top_{\perp}\Z\V}\leq |\eta (\inp{\bcalX_t}{\bcalT_t-\bcalT^*}-\epsilon_t)|\sqrt{d_jr_j^-\mu_0^{m-1}},$$
where the last inequality holds since $(\bcalX_t)_{(j)}$ is of the form $\sqrt{d^*}\e_p\e_q^\top$ and $\bcalC_{(j)}^\dagger\bcalC_{(j)}$ is a projector.

For $\op{\U^\top\Z\V_{\perp}}$, we have 
\begin{align}\label{eq:001}
	\U^\top(\calP_{\TT_t}\bcalX_t)_{(j)}\V_{\perp} = \Q_1^\top\bcalC_{(j)}\sum_{l\neq j}(\W_l\otimes_{p\neq j,l}\U_p)^\top\V_{\perp}.
\end{align}
Notice $\V_{\perp}$ is a block matrix of block size $(2^{m-1}-1)\times 1$, where each block takes the form 
$\otimes_{p\neq j}\A_p,$
where $\A_p$ takes value either $\U_p$ or $\U_{p,\perp}$ and there exists at least one $p\in[m]\backslash j$, $\A_p = \U_{p,\perp}$. For each fixed $l\in[m]\backslash j$, the only term that is non-zero is $\U_{l,\perp}\otimes_{p\neq j,l}\U_p$. 
And each term in \eqref{eq:001} are mutually orthogonal since the non-zero blocks are at different slots in the block matrix.
With this observation, we see that 
\begin{align*}
	\op{\U^\top(\calP_{\TT}\bcalX)_{(j)}\V_{\perp}}^2&\leq \sum_{l\neq j}\fro{\bcalC_{(j)}\bigg(\big((\bcalC_{(l)}^\dagger)^\top(\otimes_{p\neq l}\U_p)^\top(\bcalX_t)_{(l)}^\top\U_{l,\perp}\big)\otimes_{s\neq j,l}\I_{r_s}\bigg)}^2\\
	&= \sum_{l\neq j}\fro{\bcalC\times_l\big(\U_{l,\perp}^{\top}(\bcalX_t)_{(l)}(\otimes_{p\neq l}\U_p)\bcalC_{(l)}^\dagger\big)}^2\\
	&= \sum_{l\neq j}\fro{\U_{l,\perp}^{\top}(\bcalX_t)_{(l)}(\otimes_{p\neq l}\U_p)\bcalC_{(l)}^\dagger\bcalC_{(l)}}^2\\
	&\leq \mu_0^{m-1}\sum_{l\neq j}d_lr_l^-.
\end{align*}
Therefore 
\begin{align}\label{UZV:003}
	\op{\U^\top\Z\V_{\perp}} \leq |\eta (\inp{\bcalX_t}{\bcalT_t-\bcalT^*}-\epsilon_t)|\sqrt{\mu_0^{m-1}\sum_{l\neq j}d_lr_l^-}.
\end{align}
Then we conclude
\begin{align}\label{allbound}
	\max\big\{\op{\U^\top\Z\V}, \op{\U^\top_{\perp}\Z\V},\op{\U^\top\Z\V_{\perp}}\big\}\leq  |\eta (\inp{\bcalX_t}{\bcalT_t-\bcalT^*}-\epsilon_t)|\bigg((m-1)\mu_0^{m-1}\dmax \frac{r^*}{\rmin}\bigg)^{1/2}.
\end{align}
We also bound $|\inp{\bcalX_t}{\bcalT_t-\bcalT^*}-\epsilon_t|$. From \eqref{T-T*infty:idpt} and under $\calE_t$, we see 
\begin{align}\label{inp-eps}
	|\inp{\bcalX_t}{\bcalT_t-\bcalT^*}-\epsilon_t|\lesssim \nu_0\rmin^{1/2}\lambda_{\submax} + \sigma\log^{1/2}\dmax.
\end{align}

For different values of $k$, we control $\e_i^\top\S_k\e_i$ and $\S_k\e_i$. 
Notice that for given $k$ and $\s$ satisfying $s_1+\cdots+s_{k+1} = k$, if there exists $1\leq j\leq k$, $s_{j} = s_{j+1} = 0$, then $\S_k(\s) = 0$. This holds since $\B_j = \U_{\perp}^\top\Z\V_{\perp}$ or $\V_{\perp}^\top\Z^\top\U_{\perp} = 0$.
And thus in the following, we only consider $\s$ such that $\S_k(\s)\neq 0$.

\hspace{1cm}

\noindent{\textit{Case 1: $k = 1$.}} When $k = 1$, the closed form is given in \eqref{S1}. 
Since $\EE_t\Z = -\eta \big(\calP_{\TT_t}(\bcalT_t - \bcalT^*)\big)_{(j)}$, the expectation of $\S_1$ is given by 
\begin{align}\label{expectation:S1}
	\EE_t\S_1 &= \EE_t\P_{\U}^\perp\Z\V\bSigma^{-1}\U^\top + \U\bSigma^{-1}\V^\top\Z^\top\P_{\U}^\perp
	= \eta\P_{\U_j}^\perp\bcalT_{(j)}^*(\bcalT_t)_{(j)}^\dagger + \eta(\bcalT_t)_{(j)}^{\dagger\top}(\bcalT_{(j)}^*)^\top\P_{\U_j}^\perp.
\end{align}
We compute the conditional expectation
$
\EE_t\e_i^\top\S_1\e_i. 
$
Due to symmetry, we have
\begin{align*}
	\EE_t\e_i^\top\S_1\e_i &= 2\eta\e_i^\top\P_{\U_j}^\perp\bcalT_{(j)}^*(\bcalT_t)_{(j)}^\dagger\e_i = 2\eta\e_i^\top(\I-\P_{\U_j})\bcalT_{(j)}^*(\bcalT_t)_{(j)}^\dagger\e_i \\
	&=  2\eta\e_i^\top\bcalT^*_{(j)}(\bcalT_t)_{(j)}^\dagger\e_i
	-2\eta\e_i^\top\P_{\U_j}(\bcalT^*_{(j)}-(\bcalT_t)_{(j)})(\bcalT_t)_{(j)}^\dagger\e_i
	-2\eta\e_i^\top\P_{\U_j}(\bcalT_t)_{(j)}(\bcalT_t)_{(j)}^\dagger\e_i.
\end{align*}
For the first term on the RHS, using the incoherence of $\U_j^*$ and $\U_j$,
\begin{align*}
	\e_i^\top\bcalT^*_{(j)}(\bcalT_t)_{(j)}^\dagger\e_i \leq 2\sqrt{\frac{\mu r_j}{d_j}}\sqrt{\frac{\mu_0r_j}{d_j}}\kappa_0.
\end{align*}
For the second term on the RHS, using $\fro{\bcalT_t-\bcalT^*}\leq \frac{1}{40}\lambda_{\submin} + \sqrt{\eta\dof}\sigma\leq\frac{1}{20}\lambda_{\submin}$ and the incoherence of $\U_j$,
\begin{align*}
	\e_i^\top\P_{\U_j}(\bcalT^*_{(j)}-(\bcalT_t)_{(j)})(\bcalT_t)_{(j)}^\dagger\e_i\leq \frac{1}{10}\frac{\mu_0r_j}{d_j}
\end{align*}
And $\e_i^\top\P_{\U_j}(\bcalT_t)_{(j)}(\bcalT_t)_{(j)}^\dagger\e_i = \e_i^\top\U_j\U_j^\top\e_i$. So we conclude that (given $\mu_0\gtrsim \mu\kappa_0^2$)
\begin{align}\label{k=1}
	\EE_t\e_i^\top\S_1\e_i \leq \frac{2\eta}{5}\frac{\mu_0r_j}{d_j}  - 2\eta\e_i^\top\U_j\U_j^\top\e_i.
\end{align}

\noindent\textit{Computing $\S_1\e_i$.} To give an entry-wise bound, we are interested in $\EE_t\S_1\e_i$. 
Notice
\begin{align*}
\eta\P_{\U_j}^\perp\bcalT_{(j)}^*(\bcalT_t)_{(j)}^\dagger 
	= \eta[\bcalT_{(j)}^*-(\bcalT_t)_{(j)}](\bcalT_t)_{(j)}^\dagger +  \eta\P_{\U_j}
	- \eta\P_{\U_j}[\bcalT_{(j)}^* - (\bcalT_t)_{(j)}](\bcalT_t)_{(j)}^\dagger - \eta\P_{\U_j},
\end{align*}
and 
\begin{align*}
	\eta(\bcalT_t)_{(j)}^{\dagger\top}(\bcalT_{(j)}^*)^\top\P_{\U_j}^\perp
	= \eta[(\bcalT_t)_{(j)}^{\dagger} - (\bcalT^*_{(j)})^{\dagger}]^\top(\bcalT_{(j)}^*)^\top + \eta\P_{\U_j^*}
	-\eta(\bcalT_t)_{(j)}^{\dagger\top}(\bcalT_{(j)}^*-(\bcalT_t)_{(j)})^\top\P_{\U_j} -\eta\P_{\U_j}.
\end{align*}
Notice the bound of $\op{(\bcalT_t)_{(j)}^{\dagger} - (\bcalT^*_{(j)})^{\dagger}}$ can be obtained using Lemma \ref{pseudoinversedistance},
\begin{align}\label{pseudoT}
	\op{(\bcalT_t)_{(j)}^{\dagger} - (\bcalT^*_{(j)})^{\dagger}}\leq 30\lambda_{\submin}^{-2}\kappa_0\fro{\bcalT_t-\bcalT^*}.
\end{align}
Plug in these representations, we have 
\begin{align*}
	\EE\S_1\e_i &= \big(\eta\P_{\U_j^*}-\eta\P_{\U_j}\big)\e_i + \eta[\bcalT_{(j)}^*-(\bcalT_t)_{(j)}](\bcalT_t)_{(j)}^\dagger \e_i
	- \eta\P_{\U_j}[\bcalT_{(j)}^* - (\bcalT_t)_{(j)}](\bcalT_t)_{(j)}^\dagger \e_i\\ &\quad+\eta[(\bcalT_t)_{(j)}^{\dagger} - (\bcalT^*_{(j)})^{\dagger}]^\top(\bcalT_{(j)}^*)^\top \e_i
	-\eta(\bcalT_t)_{(j)}^{\dagger\top}(\bcalT_{(j)}^*-(\bcalT_t)_{(j)})^\top\P_{\U_j} \e_i\\
	&=: \big(\eta\P_{\U_j^*}-\eta\P_{\U_j}\big)\e_i + \v.
\end{align*}
Using the incoherence of $\U_j,\U_j^*$ and \eqref{pseudoT}, we can bound $\ltwo{\v}$ as follows
\begin{align*}
	\ltwo{\v} \leq C\eta\lambda_{\submin}^{-1}\kappa_0^2\fro{\bcalT_t -\bcalT^*}\sqrt{\frac{\mu_0r_j}{d_j}}.
\end{align*}
And as a result, 
\begin{align}\label{k=1entrywise}
	\S_{1}\e_i = \S_{1}\e_i - \EE_t\S_{1}\e_i+ \big(\eta\P_{\U_j^*}-\eta\P_{\U_j}\big)\e_i+\v,
\end{align}
where $\v$ satisfying $\ltwo{\v} \leq C\eta\lambda_{\submin}^{-1}\kappa_0^2\fro{\bcalT_t -\bcalT^*}\sqrt{\frac{\mu_0r_j}{d_j}}$.

\hspace{1cm}

\noindent{\textit{Case 2: $k = 2$.}} We discuss according to different $\s$. If $\s = (0,0,2)\text{~or~} (2,0,0)$, then $\S_2(\s) = 0$. If $\s = (0,1,1)\text{~or~}(1,1,0)$, 
then $\S_2(\s) = \U_{\perp}\U_{\perp}^\top\Z\V^\top\bSigma^{-1}\U^\top\Z\V\bSigma^{-1}\U^\top$ or its transpose. Together with \eqref{inp-eps}, we see 
\begin{align*}
	|\e_i^\top\S_2(\s)\e_i| &\leq |\eta (\inp{\bcalX_t}{\bcalT_t-\bcalT^*}-\epsilon_t)|^2\sqrt{\mu_0r_j/d_j}
	\sqrt{\mu_0^mr^*}\sqrt{\mu_0^{m-1}d_jr_j^-}\sigma_{\min}^{-2}(\bcalT_t)\\
	&\lesssim_m \eta^2(\nu_0^2\rmin\lambda_{\submax}^2 + \sigma^2\log\dmax)\mu_0^mr^*\lambda_{\submin}^{-2}.
\end{align*}
More specifically, when $\s = (1,1,0)$, 
\begin{align*}
	\ltwo{\S_2(\s)\e_i} &= \ltwo{\U\bSigma^{-1}\V^\top\Z^\top\U\bSigma^{-1}\V^\top\Z^\top\U_{\perp}\U_{\perp}^\top\e_i}\\
	&\leq (2\lambda_{\submin}^{-1})^2\op{\V^\top\Z^\top\U}\op{\V^\top\Z^\top\U_{\perp}}\\
	&\leq |\eta (\inp{\bcalX_t}{\bcalT_t-\bcalT^*}-\epsilon_t)|^2(\mu_0^mr^*)^{1/2}(\mu_0^{m-1}d_jr_j^-)^{1/2}(2\lambda_{\submin}^{-1})^2\\
	&\lesssim \eta^2(d^*\linf{\bcalT_t-\bcalT^*}^2 + \sigma^2\log\dmax)\mu_0^{m-1/2}(d_jr_j^-r^*)^{1/2}\lambda_{\submin}^{-2}.
\end{align*}
When $\s = (0,1,1)$,
\begin{align*}
	\ltwo{\S_2(\s)\e_i} &= \ltwo{\U_{\perp}\U_{\perp}^\top\Z\V^\top\bSigma^{-1}\U^\top\Z\V\bSigma^{-1}\U^\top\e_i}\\
	&\leq (2\lambda_{\submin}^{-1})^2 \op{\V^\top\Z^\top\U}\op{\V^\top\Z^\top\U_{\perp}}\ltwoinf{\U}\\
	&\lesssim \eta^2(d^*\linf{\bcalT_t-\bcalT^*}^2 + \sigma^2\log\dmax)\mu_0^{m-1/2}(d_jr_j^-r^*)^{1/2}\lambda_{\submin}^{-2}\sqrt{\frac{\mu_0r_j}{d_j}}.
\end{align*}
If $\s = (1,0,1)$, then $\S_2(\s) = \U\bSigma^{-1}\V^\top\Z^\top\U_{\perp}\U_{\perp}^\top\Z\V\bSigma^{-1}\U^\top$, and 
\begin{align*}
	|\e_i^\top\S_2(\s)\e_i| &\leq |\eta (\inp{\bcalX_t}{\bcalT_t-\bcalT^*}-\epsilon_t)|^2\frac{\mu_0r_j}{d_j}
	\mu_0^{m-1}d_jr_j^-\sigma_{\min}^{-2}(\bcalT_t)\\
	&\lesssim_m \eta^2(\nu_0^2\rmin\lambda_{\submax}^2 + \sigma^2\log\dmax)\mu_0^mr^*\lambda_{\submin}^{-2},
\end{align*}
and
\begin{align*}
	\ltwo{\S_2(\s)\e_i} &= \ltwo{\U\bSigma^{-1}\V^\top\Z^\top\U_{\perp}\U_{\perp}^\top\Z\V\bSigma^{-1}\U^\top\e_i}\\
	&\leq (2\lambda_{\submin}^{-1})^2\op{\V^\top\Z^\top\U_{\perp}}^2\ltwoinf{\U}\\
	&\lesssim \eta^2(d^*\linf{\bcalT_t-\bcalT^*}^2 + \sigma^2\log\dmax)\mu_0^{m-1}d_jr_j^-\lambda_{\submin}^{-2}\sqrt{\frac{\mu_0r_j}{d_j}}.
\end{align*}
For the case $\s = (0,2,0)$, from the closed form given in \eqref{S21},
$\S_{2,1} := \S_2(\s) = \P_{\U}^\perp\Z\V\bSigma^{-2}\V^\top\Z^\top \P_{\U}^\perp$.From \eqref{ptx:junfolding} and recall $\V = (\otimes_{l\neq j}\U_l)\Q_2$, we see
\begin{align*}
	\EE_t \S_{2,1} &= \EE_t|\eta (\inp{\bcalX_t}{\bcalT_t-\bcalT^*}-\epsilon_t)|^2 \P_{\U}^{\perp}(\bcalX_t)_{(j)}(\bcalT_t)_{(j)}^\dagger(\bcalT_t)_{(j)}^{\dagger\top}(\bcalX_t)_{(j)}^\top\P_{\U}^{\perp}\\
	&\lesssim \eta^2(d^*\linf{\bcalT_t-\bcalT^*}^2 + \sigma^2) \sum_{p,q}\P_{\U}^{\perp}\e_p\e_q^\top(\bcalT_t)_{(j)}^\dagger(\bcalT_t)_{(j)}^{\dagger\top}\e_q\e_p^\top\P_{\U}^{\perp}\\
	&= \eta^2(d^*\linf{\bcalT_t-\bcalT^*}^2 + \sigma^2) \trace\bigg((\bcalT_t)_{(j)}^\dagger(\bcalT_t)_{(j)}^{\dagger\top}\bigg)\P_{\U}^{\perp}\\
	&\leq \eta^2(d^*\linf{\bcalT_t-\bcalT^*}^2 + \sigma^2) r_j\lambda_{\submin}^{-2}\P_{\U}^{\perp}
\end{align*}
where $\A\leq \B$ means $\B-\A$ is SPSD and $\A\lesssim \B$ means there exists absolute constant $C>0$, $C\B-\A$ is SPSD. 
And the last inequality holds since $\trace\bigg((\bcalT_t)_{(j)}^\dagger(\bcalT_t)_{(j)}^{\dagger\top}\bigg) \leq \fro{(\bcalT_t)_{(j)}^\dagger}^2 \lesssim r_j\lambda_{\submin}^{-2}$.
As a result, from the above computation and \eqref{T-T*infty:idpt}, we have 
\begin{align}\label{k=2}
	\EE_t \e_i^\top\S_{2,1}\e_i \lesssim \eta^2(\nu_0^2\rmin\lambda_{\submax}^2+ \sigma^2)r_j\lambda_{\submin}^{-2},
\end{align}
and
\begin{align*}
	\ltwo{\EE_t \S_{2,1}\e_i} \lesssim \eta^2(d^*\linf{\bcalT_t-\bcalT^*}^2 + \sigma^2) r_j\lambda_{\submin}^{-2}.
\end{align*}
Therefore we conclude
\begin{align}\label{k=2entrywise}
	\S_{2}\e_i = \S_{2,1}\e_i - \EE_t\S_{2,1}\e_i+ \u,
\end{align}
fro some $\u$ satisfying $\ltwo{\u}\lesssim \eta^2(d^*\linf{\bcalT_t - \bcalT^*}^2 + \sigma^2\log\dmax)\mu_0^{m-1/2}(d_jr_j^-r^*)^{1/2}\lambda_{\submin}^{-2}.$


\hspace{1cm}

\noindent{\textit{Case 3: $k \geq 3$ and $k$ is odd.}} For each $\s$ such that $\S_k(\s)\neq 0$, if at least one of $\A_1,\A_2$ takes $\U$, then using \eqref{lowerbound:signalstrength:Tl} and \eqref{allbound}, we see
$$|\e_i^\top\S_k(\s)\e_i| \leq |\eta (\inp{\bcalX_t}{\bcalT_t-\bcalT^*}-\epsilon_t)|^k(\mu_0r_j/d_j)^{1/2}\bigg((m-1)\mu_0^{m-1}\dmax r^*/\rmin\bigg)^{k/2}(2\lambda_{\submin}^{-1})^{k}.$$
If both $\A_1,\A_2$ take $\U_{\perp}$, there exists at least one $l\in[k]$, $\B_l = \U^\top\Z\V \text{~or~} \V^\top\Z^\top\U$, for otherwise there exists $l'\in[k]$, $s_l = s_{l+1} = 0$ and this term will vanish. Then 
$$|\e_i^\top\S_k(\s)\e_i| \leq |\eta (\inp{\bcalX_t}{\bcalT_t-\bcalT^*}-\epsilon_t)|^{k}(\mu_0^mr^*)^{1/2}\bigg((m-1)\mu_0^{m-1}\dmax r^*/\rmin\bigg)^{(k-1)/2}(2\lambda_{\submin}^{-1})^{k}.$$
Since there are at most $\binom{2k}{k}\leq 4^k$ possible $\s$, 
$$|\e_i^\top\S_k\e_i| \leq 4^k|\eta (\inp{\bcalX_t}{\bcalT_t-\bcalT^*}-\epsilon_t)|^k(\mu_0r_j/d_j)^{1/2}\bigg((m-1)\mu_0^{m-1}\dmax r^*/\rmin\bigg)^{k/2}(2\lambda_{\submin}^{-1})^{k}.$$
Notice that from \eqref{inp-eps}
\begin{align}\label{ratio}
	&\quad4^2|\eta (\inp{\bcalX_t}{\bcalT_t-\bcalT^*}-\epsilon_t)|^2(m-1)\mu_0^{m-1}\dmax r^*/\rmin(2\lambda_{\submin}^{-1})^2\notag\\
	&\lesssim_m  \eta^2 \kappa_0^2\nu_0^2\mu_0^{m-1}\dmax r^* + \eta^2\mu_0^{m-1} (\frac{\sigma}{\lambda_{\submin}})^2\dmax \frac{r^*}{\rmin}\log\dmax \leq 1/2,
\end{align}
where the last inequality holds given $\eta^2\nu_0^2\kappa_0^2\mu_0^{m-1}\dmax r^*\lesssim_m 1$ and $(\frac{\lambda_{\submin}^2}{\sigma^2})\gtrsim_m\eta^2 \mu_0^{m-1}\dmax \frac{r^*}{\rmin}\log\dmax$.
From \eqref{ratio}, the contribution for such $k$ is bounded by
\begin{align}\label{kgeq3odd}
	\sum_{k\geq 3, k \text{ is odd}}\e_i^\top\S_k\e_i \lesssim_m (\mu_0r_j/d_j)^{1/2}\eta^3\mu_0^{3(m-1)/2}\dmax^{3/2} (r^*)^{3/2}\bigg(\kappa_0^3\nu_0^3 + \sigma^3\lambda_{\submin}^{-3}\log^{3/2}\dmax/\rmin^{3/2}\bigg).
\end{align}

\hspace{1cm}

\noindent{\textit{Case 4: $k \geq 4$ and $k$ is even.}} We shall apply \eqref{allbound} to bound $\e_i^\top\S_k\e_i$:
\begin{align*}
	|\e_i^\top\S_k(\s)\e_i| \leq |\eta (\inp{\bcalX_t}{\bcalT_t-\bcalT^*}-\epsilon_t)|^k\bigg((m-1)\mu_0^{m-1}\dmax r^*/\rmin\bigg)^{k/2}(2\lambda_{\submin}^{-1})^k.
\end{align*}
Since there are $\binom{2k}{k}\leq 4^k$ legal $\s$, 
$$|\e_i^\top\S_k\e_i| \leq 4^k |\eta (\inp{\bcalX_t}{\bcalT_t-\bcalT^*}-\epsilon_t)|^k\bigg((m-1)\mu_0^{m-1}\dmax r^*/\rmin\bigg)^{k/2}(2\lambda_{\submin}^{-1})^k.$$
From \eqref{ratio}, the contribution for such $k$ is bounded by
\begin{align}\label{kgeq4even}
	\sum_{k\geq 4, k \text{ is even}}\e_i^\top\S_k\e_i \lesssim_m \eta^4 \kappa_0^4\nu_0^4\mu_0^{2m-2}\dmax^2 (r^*)^2 + \eta^4 \sigma^4\lambda_{\submin}^{-4}\mu_0^{2m-2}\dmax^2 (r^*)^2\log^2\dmax/\rmin^2.
\end{align}

\noindent\textit{Computing $\sum_{k\geq 3}\S_k\e_i$.} For the entry-wise bound, we can consider simultaneously the cases when $k\geq 3$. In fact, using \eqref{allbound},
\begin{align*}
	\ltwo{\S_k\e_i} \lesssim 8^k\eta^k\bigg(\sqrt{d^*}\linf{\bcalT_t - \bcalT^*} + \sigma\sqrt{\log\dmax}\bigg)^k\lambda_{\submin}^{-k} (\mu_0^{m-1}\dmax\frac{r^*}{\rmin})^{k/2},
\end{align*}
and thus from \eqref{ratio},
\begin{align}\label{kgeq3entrywise}
	\sum_{k\geq 3}\ltwo{\S_k\e_i} \lesssim \eta^3\bigg(\sqrt{d^*}\linf{\bcalT_t - \bcalT^*} + \sigma\sqrt{\log\dmax}\bigg)^3\lambda_{\submin}^{-3} (\mu_0^{m-1}\dmax\frac{r^*}{\rmin})^{3/2}.
\end{align}

\hspace{1cm}

\noindent \textit{Step 2.1: Bounding the Frobenius norm error $\fro{\bcalT_{t+1} - \bcalT^*}$.}
Now from \eqref{k=1}, \eqref{k=2}, \eqref{kgeq3odd} and \eqref{kgeq4even}, we see that 
\begin{align*}
	& \e_i^\top\U_{t+1,j}\U_{t+1,j}^\top\e_i \leq (1-2\eta)\e_i^\top\U_{t,j}\U_{t,j}^\top\e_i + (\e_i^\top\S_1\e_i - \EE_t\e_i^\top\S_1\e_i) + \frac{2\eta}{5}\frac{\mu_0r_j}{d_j}\\
	&\hspace{2cm}+ \underbrace{(\e_i^\top\S_{2,1}\e_i - \EE_t \e_i^\top\S_{2,1}\e_i) + C_m\eta^2(\nu_0^2\rmin\lambda_{\submax}^2 + \sigma^2\log\dmax)\mu_0^mr^*\lambda_{\submin}^{-2}}_{k = 2}\\
	& \hspace{2cm}+ \underbrace{C_m(\mu_0r_j/d_j)^{1/2}\eta^3\mu_0^{3(m-1)/2}\dmax^{3/2} (r^*)^{3/2}\bigg(\kappa_0^3\nu_0^3 + \sigma^3\lambda_{\submin}^{-3}\log^{3/2}\dmax/\rmin^{3/2}\bigg)}_{k\geq3,~k\text{ is odd}}\\
	&\hspace{2cm}+ \underbrace{C_m\eta^4 \kappa_0^4\nu_0^4\mu_0^{2m-2}\dmax^2 (r^*)^2 + \eta^4 \sigma^4\lambda_{\submin}^{-4}\mu_0^{2m-2}\dmax^2 (r^*)^2\log^2\dmax/\rmin^2}_{k\geq4,~k\text{ is even}}\\
	&\leq (1-2\eta)\e_i^\top\U_{t,j}\U_{t,j}^\top\e_i + (\e_i^\top\S_1\e_i - \EE_t\e_i^\top\S_1\e_i) + (\e_i^\top\S_{2,1}\e_i - \EE_t \e_i^\top\S_{2,1}\e_i) + \frac{\eta}{2}\frac{\mu_0r_j}{d_j},
\end{align*}
where the last inequality holds given $\eta\dmax r^*\mu_0^{m-1}\nu_0^2\kappa_0^2\lesssim_m 1$ and 
$$
\eta \mu_0^{m-1} (\sigma/\lambda_{\submin})^2\frac{\dmax r^*}{\rmin}\lesssim_m 1, \eta \mu_0^{\frac{3}{4}m-1} (\sigma/\lambda_{\submin})^{\frac{3}{2}}\frac{\dmax (r^*)^\frac{3}{4}}{\rmin}\lesssim_m 1, \eta \mu_0^{\frac{2}{3}m-1} (\sigma/\lambda_{\submin})^\frac{4}{3}\frac{\dmax (r^*)^{\frac{2}{3}}}{\rmin}\lesssim_m 1,
$$
Notice this boils down to $\eta \mu_0^{m-1} (\sigma/\lambda_{\submin})^2\frac{\dmax r^*}{\rmin}\lesssim_m 1$ if $\eta\mu_0^{-1}\frac{\dmax}{\rmin}\leq 1$.

 To emphasize the dependence on $t$, we add subscripts for $\S_1$ and $\S_{2,1}$, namely, $\S_{1,t}$ and $\S_{2,1,t}$. Now telescoping this inequality, and we get 
\begin{align}\label{incoh:main}
	&\e_i^\top\U_{t+1,j}\U_{t+1,j}^\top\e_i \leq (1-2\eta)^{t+1}\e_i^\top\U_{0,j}\U_{0,j}^\top\e_i +\frac{1}{4}\frac{\mu_0r_j}{d_j},\notag\\
	&\hspace{3cm}+ \sum_{l=0}^t \underbrace{(1-2\eta)^{t-l}\big[(\e_i^\top\S_{1,l}\e_i - \EE_t\e_i^\top\S_{1,l}\e_i) + (\e_i^\top\S_{2,1,l}\e_i - \EE_t \e_i^\top\S_{2,1,l}\e_i)\big]}_{=:D_l}
\end{align}
Now we use martingale concentration inequality to bound $\sum_{l=0}^t D_l$. We first consider the uniform bound. Notice 
\begin{align*}
	|\e_i^\top\S_{1,l}\e_i| &\leq 2\eta|\inp{\bcalX_l}{\bcalT_l-\bcalT^*} - \epsilon_l|\cdot|\e_i^\top\P_{\U}^\perp(\bcalP_{\TT_l}\bcalX_l)_{(j)}\V\bSigma^{-1}\U^\top\e_i|\\
	&\leq2\eta(2\nu_0\sqrt{\rmin}\lambda_{\submax} + C\sigma\log^{1/2} \dmax)\sqrt{d^*}\sqrt{\frac{\mu_0^mr^*}{d^*}}2\lambda_{\submin}^{-1}\\
	&\leq 8\eta\nu_0\kappa_0\mu_0^{m/2}(\rmin r^*)^{1/2} + C\eta(\mu_0^m r^*)^{1/2}(\sigma/\lambda_{\submin})\log^{1/2}\dmax.
\end{align*}
Using $|X-\EE X|\leq 2B$ when $|X|\leq B$, we get $$|\e_i^\top\S_{1,l}\e_i - \EE_l\e_i^\top\S_{1,l}\e_i|\leq 16\eta\nu_0\kappa_0\mu_0^{m/2}(\rmin r^*)^{1/2} + C\eta(\mu_0^m r^*)^{1/2}(\sigma/\lambda_{\submin})\log^{1/2}\dmax.$$
On the other hand, again from \eqref{inp-eps}, we have
\begin{align*}
	|\e_i^\top\S_{2,1,l}\e_i| &= |\e_i^\top \P_{\U}^\perp\Z\V\bSigma^{-2}\V^\top\Z^\top \P_{\U}^\perp\e_i|\\
	&\leq 2\eta^2(3\nu_0^2\rmin\lambda_{\submax}^2 + C\sigma^2\log \dmax)4\lambda_{\submin}^{-2}\mu_0^{m-1}r_j^-d_j\\
	&\leq 24\eta^2\nu_0^2\kappa_0^2\mu_0^{m-1}\rmin r_j^-d_j + C\eta^2 (\sigma/\lambda_{\submin})^2\mu_0^{m-1}r_j^-d_j\log\dmax.
\end{align*}
Therefore $$|\e_i^\top\S_{2,1,l}\e_i- \EE_l \e_i^\top\S_{2,1,l}\e_i| \leq 48\eta^2\nu_0^2\kappa_0^2\mu_0^{m-1}\rmin r_j^-d_j + C\eta^2 (\sigma/\lambda_{\submin})^2\mu_0^{m-1}r_j^-d_j\log\dmax.$$
As long as $\eta\kappa_0\nu_0\mu_0^{m/2-1}\dmax(r^*/\rmin)^{1/2}\lesssim 1$ and $\eta(\sigma/\lambda_{\submin})\mu_0^{m/2-1}\dmax\frac{(r^*)^{1/2}}{\rmin}
	\log^{1/2}\dmax \lesssim 1$, we have 
\begin{align}\label{azuma-bernstein:uniform}
	|D_l|\lesssim (1-2\eta)^{t-l}\bigg(\eta\nu_0\kappa_0\mu_0^{m/2}(\rmin r^*)^{1/2} + \eta(\mu_0^m r^*)^{1/2}(\sigma/\lambda_{\submin})\log^{1/2}\dmax\bigg).
\end{align}
We also need to bound the variance. In fact, 
\begin{align*}
	\EE_l D_l^2 \leq 2(1-2\eta)^{2t-2l}[\EE_l|\e_i^\top\S_{1,l}\e_i|^2+\EE_l|\e_i^\top\S_{2,1,l}\e_i|^2].
\end{align*}
And 
\begin{align*}
	&\EE_l|\e_i^\top\S_{1,l}\e_i|^2 = 4\eta^2\EE_l(\inp{\bcalX_l}{\bcalT_l-\bcalT^*}^2 + \epsilon_l^2)(\e_i^\top\P_{\U}^\perp(\bcalP_{\TT_l}\bcalX_l)_{(j)}\V\bSigma^{-1}\U^\top\e_i)^2\\
	&\leq 4\eta^2(3\nu_0^2\rmin\lambda_{\submax}^2 + \sigma^2)
	\sum_{p,q}(\e_i^\top\P_{\U}^\perp\e_p)^2(\e_q^\top(\otimes_{k\neq j}\U_k)\bcalC_{(j)}^\dagger\bcalC_{(j)}(\otimes_{k\neq j}\U_k)^\top\V\bSigma^{-1}\U^\top\e_i)^2\\
	&\lesssim \eta^2(\nu_0^2\rmin\lambda_{\submax}^2 + \sigma^2)\lambda_{\submin}^{-2}\frac{\mu_0r_j}{d_j},
\end{align*}
where in the last inequality we use $\sum_{p\in[d_j]}(\e_i^\top\P_{\U}^\perp\e_p)^2\leq 2$ and
\begin{align*}
	&\quad\sum_{q\in[d_j^-]}(\e_q^\top(\otimes_{k\neq j}\U_k)\bcalC_{(j)}^\dagger\bcalC_{(j)}(\otimes_{k\neq j}\U_k)^\top\V\bSigma^{-1}\U^\top\e_i)^2\\
	&=\ltwo{(\otimes_{k\neq j}\U_k)\bcalC_{(j)}^\dagger\bcalC_{(j)}(\otimes_{k\neq j}\U_k)^\top\V\bSigma^{-1}\U^\top\e_i}^2\\
	&\leq \ltwo{\bSigma^{-1}\U^\top\e_i}^2 \leq (2\lambda_{\submin}^{-1})^2\frac{\mu_0r_j}{d_j}.
\end{align*}
On the other hand, 
\begin{align*}
	&\EE_l|\e_i^\top\S_{2,1,l}\e_i|^2 = \EE_l |\e_i^\top\P_{\U}^\perp\Z\V\bSigma^{-2}\V^\top\Z^\top \P_{\U}^\perp\e_i|^2\\
	&= \EE_l\eta^4(\inp{\bcalX_l}{\bcalT_l-\bcalT^*} + \epsilon_l)^4(|\e_i^\top\P_{\U}^\perp(\bcalP_{\TT_l}\bcalX_l)_{(j)}\V\bSigma^{-2}\V^\top(\bcalP_{\TT_l}\bcalX_l)_{(j)}^\top \P_{\U}^\perp\e_i|^2)\\
	&\lesssim \eta^4(\nu_0^4\rmin^2\lambda_{\submax}^4 + \sigma^4)d^*\sum_p(\e_i^\top\P_{\U}^\perp\e_p)^4\fro{(\otimes_{k\neq j}\U_k)\bcalC_{(j)}^\dagger\bcalC_{(j)}(\otimes_{k\neq j}\U_k)^\top\V\bSigma^{-1}}^2\cdot\frac{\mu_0^{m-1}r_j^-}{d_j^-}\lambda_{\submin}^{-2}\\
	&\lesssim \eta^4(\nu_0^4\rmin^2\lambda_{\submax}^4 + \sigma^4)d^* r_j\lambda_{\submin}^{-4}\frac{\mu_0^{m-1}r_j^-}{d_j^-}.
\end{align*}
Therefore as long as $\eta\nu_0\kappa_0\mu_0^{\frac{m}{2}-1}\dmax (r^*)^{1/2}\lesssim 1$ and $\eta(\sigma/\lambda_{\submin})\mu_0^{\frac{m}{2}-1}\dmax(\frac{r^*}{\rmin})^{1/2}\lesssim 1$,
\begin{align*}
	\EE_l D_l^2 \lesssim (1-2\eta)^{2t-2l} \eta^2(\nu_0^2\rmin\lambda_{\submax}^2 + \sigma^2)\lambda_{\submin}^{-2}\frac{\mu_0r_j}{d_j}.
\end{align*}
And the summation has the following upper bound, 
\begin{align}\label{azuma-bernstein:variance}
	\sum_{l=0}^t \EE_l D_l^2 \lesssim \eta(\nu_0^2\rmin\lambda_{\submax}^2 + \sigma^2)\lambda_{\submin}^{-2}\frac{\mu_0r_j}{d_j}.
\end{align}
From \eqref{azuma-bernstein:uniform} and \eqref{azuma-bernstein:variance}, and as a result of Azuma-Bernstein inequality, we see that with probability exceeding $\dmax^{-12}$, as long as $\eta\log^2\dmax\lesssim \frac{1}{d_j}$, 
\begin{align}
	\sum_{l=0}^tD_l &\lesssim \sqrt{\eta(\nu_0^2\rmin\lambda_{\submax}^2 + \sigma^2)\lambda_{\submin}^{-2}\frac{\mu_0r_j}{d_j}\log\dmax}+ \eta\nu_0\kappa_0\mu_0^{m/2}(\rmin r^*)^{1/2}\log\dmax  \notag\\
	&\hspace{6cm}+ \eta(\mu_0^m r^*)^{1/2}(\sigma/\lambda_{\submin})\log^{3/2}\dmax\notag\\
	&\leq \frac{1}{4}\frac{\mu_0r_j}{d_j},
\end{align}
where the last inequality holds as long as $\eta\nu_0^2\kappa_0^2\mu_0^{-1}\dmax\log\dmax\lesssim 1$ and $\eta(\sigma^2/\lambda_{\submin}^2)\dmax\rmin^{-1}\mu_0^{-1}\log\dmax\lesssim 1$
and $\eta\mu_0^{m-1}\dmax\frac{r^*}{\rmin}\log^2\dmax\lesssim 1$.
Now we go back to \eqref{incoh:main}, and we see as long as $\e_i^\top\U_{0,j}\U_{0,j}^\top\e_i\leq \frac{1}{2} \frac{\mu_0r_j}{d_j}$, we have 
\begin{align*}
	\e_i^\top\U_{t+1,j}\U_{t+1,j}^\top\e_i \leq  \frac{\mu_0r_j}{d_j}.
\end{align*}
Taking a union bound over all $i,j$ and we see that with probability exceeding $1-m\dmax^{-11}$, 
$$\ltwoinf{\U_{t+1,j}}^2\leq \frac{\mu_0r_j}{d_j}, \forall j\in[m].$$
\hspace{1cm}

\noindent\textit{Step 2.2: Bounding the entry-wise error $\linf{\bcalT_{t+1} - \bcalT^*}$. } 
To remind the readers that $\S_1,\S_{2,1},\w$ depend on $t$, we add subscripts. 
From \eqref{k=1entrywise}, \eqref{k=2entrywise} and \eqref{kgeq3entrywise}, we see that 
\begin{align}\label{twoinf:recursive}
	\P_{\U_{t+1,j}}\e_i - \P_{\U^*_j}\e_i = (1-\eta)\big(\P_{\U_{t,j}}\e_i - \P_{\U^*_j}\e_i\big) + \S_{1,t}\e_i - \EE_t\S_{1,t}\e_i + \S_{2,1,t}\e_i - \EE_t\S_{2,1,t}\e_i + \w_t,
\end{align}
where $\w_t$ satisfies
\begin{align*}
	\ltwo{\w_t} &\lesssim \eta\lambda_{\submin}^{-1}\kappa_0^2\fro{\bcalT_t -\bcalT^*}\sqrt{\frac{\mu_0r_j}{d_j}} + \eta^2(d^*\linf{\bcalT_t - \bcalT^*}^2 + \sigma^2\log\dmax)\mu_0^{m-1/2}(d_jr_j^-r^*)^{1/2}\lambda_{\submin}^{-2}\\
	&\quad + \eta^3\bigg(\sqrt{d^*}\linf{\bcalT_t - \bcalT^*} + \sigma\sqrt{\log\dmax}\bigg)^3\lambda_{\submin}^{-3} (\mu_0^{m-1}\dmax\frac{r^*}{\rmin})^{3/2}.
\end{align*}
Now we plug in the bound for $\fro{\bcalT_t - \bcalT^*}$ and $\linf{\bcalT_t - \bcalT^*}$, and as long as $\eta\kappa_0^4\mu_0^{2m-1}\dmax\frac{(r^*)^2}{\rmin}\lesssim_m 1$, $\lambda_{\submin}/\sigma \gtrsim_m \eta^{1/2}\mu_0^{m-1}\kappa_0^{-2}\frac{\dmax r^*}{(\dof)^{1/2}\rmin}\log\dmax$ and $\fro{\bcalT_0 - \bcalT^*}\leq \lambda_{\submin}$, we have 
\begin{align}\label{bound:wl}
	\ltwo{\w_t} \lesssim \eta(1-\frac{\eta}{4})^{t/2}\kappa_0^2(\frac{\mu_0r_j}{d_j})^{1/2} + C\eta\lambda_{\submin}^{-1}\kappa_0^2(\frac{\mu_0r_j}{d_j})^{1/2} (\eta\dof)^{1/2}\sigma.
\end{align}
Now telescoping \eqref{twoinf:recursive}, we see that 
\begin{align*}
	\P_{\U_{t+1,j}}\e_i - \P_{\U^*_j}\e_i &= (1-\eta)^{t+1}\big(\P_{\U_{0,j}}\e_i - \P_{\U^*_j}\e_i\big) +\sum_{l=0}^t(1-\eta)^{t-l}\w_l\\
	&\quad +  \sum_{l=0}^t\underbrace{(1-\eta)^{t-l}\bigg(\S_{1,l}\e_i - \EE_t\S_{1,l}\e_i + \S_{2,1,l}\e_i - \EE_t\S_{2,1,l}\e_i\bigg)}_{\f_l}.
\end{align*}
We shall use martingale concentration inequality to bound $\ltwo{\sum_{l=0}^t \f_l}$. Recall $\S_{1,l}\e_i = \P_{\U}^\perp \Z\V\bSigma^{-1}\U^\top\e_i + \U\bSigma^{-1}\V^\top\Z^\top\P_{\U}^\perp\e_i$. And therefore
\begin{align*}
	\ltwo{\S_{1,l}\e_i}\lesssim \eta(d^*)^{1/2}(\frac{\mu_0^{m-1}r_j^-}{d_j^-})^{1/2}\bigg((d^*)^{1/2}\linf{\bcalT_l-\bcalT^*} + \sigma\log^{1/2}\dmax\bigg)\lambda_{\submin}^{-1}.
\end{align*}
Using the bound for $\linf{\bcalT_t-\bcalT^*}$ and we see (given $\eta\dof\kappa_0^6\mu_0^mr^*\lesssim \log\dmax$)
\begin{align*}
	\ltwo{\S_{1,l}\e_i}\lesssim \eta(\mu_0^{m-1}r_j^-d_j)^{1/2}\bigg((1-\frac{\eta}{4})^{l/2}\kappa_0^2(\mu_0^mr^*)^{1/2}\lambda_{\submax} + \sigma\log^{1/2}\dmax\bigg)\lambda_{\submin}^{-1}.
\end{align*}
 On the other hand, recall $\S_{2,1,l}\e_i = \P_{\U}^\perp\Z\V\bSigma^{-2}\V^\top\Z^\top\P_{\U}^\perp \e_i$, and similarly
 \begin{align*}
 	\ltwo{\S_{2,1,l}\e_i} \leq \eta^2\mu_0^{m-1}r_j^-d_j\bigg((1-\frac{\eta}{4})^{l}\kappa_0^4(\mu_0^mr^*)\lambda_{\submax}^2 + \sigma^2\log\dmax\bigg)\lambda_{\submin}^{-2}.
 \end{align*}
So as long as $\eta^2\mu_0^{m-1}r_j^-d_j\bigg((1-\frac{\eta}{4})^{l}\kappa_0^4(\mu_0^mr^*)\lambda_{\submax}^2 + \sigma^2\log\dmax\bigg)\lambda_{\submin}^{-2}\lesssim 1$, we see 
\begin{align}\label{uniform:1}
	\ltwo{\f_l} \lesssim \eta(\mu_0^{m-1}r_j^-d_j)^{1/2}\bigg((1-\frac{\eta}{4})^{t/2}\kappa_0^2(\mu_0^mr^*)^{1/2}\lambda_{\submax} + \sigma\log^{1/2}\dmax\bigg)\lambda_{\submin}^{-1}
\end{align}
for all $0\leq l\leq t$. Now we consider the variance bound,
\begin{align*}
	\EE\e_i^\top\S_{1,l}^\top\S_{1,l}\e_i &\lesssim \eta^2(d^*\linf{\bcalT_l-\bcalT^*}^2 + \sigma^2\log\dmax)\sum_{p,q} \e_i^\top\P_{\U}^\perp\e_p\e_q^\top \V\bSigma^{-2}\V^\top\e_q\e_p^\top\P_{\U}^\perp\e_i\\
	&\quad + \eta^2(d^*\linf{\bcalT_l-\bcalT^*}^2 + \sigma^2\log\dmax)\sum_{p,q} \e_i^\top\U\bSigma^{-1}\V^\top\e_q\e_p^\top\P_{\U}^\perp\e_p\e_q^\top\V\bSigma^{-1}\U^\top\e_i\\
	&\lesssim \eta^2(d^*\linf{\bcalT_l-\bcalT^*}^2 + \sigma^2\log\dmax)\mu_0r_j\lambda_{\submin}^{-2}\\
	&\lesssim \eta^2\bigg((1-\frac{\eta}{4})^t\kappa_0^4\mu_0^mr^*\lambda_{\submax}^2 + \sigma^2\log\dmax\bigg)\mu_0r_j\lambda_{\submin}^{-2}.
\end{align*}
On the other hand, 
\begin{align*}
	\S_{1,l}\e_i\e_i^\top\S_{1,l}^\top = (\P_{\U}^\perp\Z\V\bSigma^{-1}\U^\top + \U\bSigma^{-1}\V^\top\Z^\top\P_{\U}^\perp)\e_i\e_i^\top(\U\bSigma^{-1}\V^\top\Z^\top\P_{\U}^\perp + \P_{\U}^\perp \Z\V\bSigma^{-1}\U^\top).
\end{align*}
And similar computation as above gives
\begin{align*}
	\op{\EE \S_{1,l}\e_i\e_i^\top\S_{1,l}^\top} \lesssim \eta^2\bigg((1-\frac{\eta}{4})^t\kappa_0^4\mu_0^mr^*\lambda_{\submax}^2 + \sigma^2\log\dmax\bigg)\lambda_{\submin}^{-2}.
\end{align*}
Meanwhile, $\EE_l\e_i^\top\S_{2,1,l}^\top\S_{2,1,l}\e_i$ and $\op{\EE_l\S_{2,1,l}\e_i\e_i^\top\S_{2,1,l}^\top}$ can be bounded as follows
\begin{align*}
	&\quad\max\{\EE_l\e_i^\top\S_{2,1,l}^\top\S_{2,1,l}\e_i, \op{\EE_l\S_{2,1,l}\e_i\e_i^\top\S_{2,1,l}^\top}\}\\
	&\leq \eta^4\bigg((1-\frac{\eta}{4})^{2t}\kappa_0^8\mu_0^{2m}(r^*)^2\lambda_{\submax}^4 + \sigma^4\log^2\dmax\bigg)\mu_0^{m-1}r^*d_j\lambda_{\submin}^{-4}.
\end{align*}
Therefore we conclude as long as $\eta^2(\kappa_0^4\mu_0^mr^*\lambda_{\submax}^2 + \sigma^2\log\dmax)\mu_0^{m-2}d_jr_j^-\lambda_{\submin}^{-2}\lesssim 1$, 
\begin{align*}
	\max\{\op{\EE_l\f_l\f_l^\top}, \EE_l\f_l^\top\f_l\}\lesssim (1-\eta)^{2t-2l}\eta^2\bigg((1-\frac{\eta}{4})^t\kappa_0^4\mu_0^mr^*\lambda_{\submax}^2 + \sigma^2\log\dmax\bigg)\lambda_{\submin}^{-2}.
\end{align*}
And 
\begin{align}\label{variance:1}
	\max\{\op{\sum_{l=0}^t\EE_l\f_l\f_l^\top}, \sum_{l=0}^t\EE_l\f_l^\top\f_l\}\lesssim \eta\bigg((1-\frac{\eta}{4})^{t+1}\kappa_0^4\mu_0^mr^*\lambda_{\submax}^2 + \sigma^2\log\dmax\bigg)\lambda_{\submin}^{-2}.
\end{align}
Then from \eqref{uniform:1} and \eqref{variance:1}, we see with probability exceeding $1-\dmax^{-11}$, 
\begin{align*}
	\ltwo{\sum_{l=0}^t\f_l} &\lesssim \eta(\mu_0^{m-1}r_j^-d_j)^{1/2}\bigg((1-\frac{\eta}{4})^{t/2}\kappa_0^2(\mu_0^mr^*)^{1/2}\lambda_{\submax} + \sigma\log^{1/2}\dmax\bigg)\lambda_{\submin}^{-1}\log\dmax\\
	&\quad+ \bigg[\eta\bigg((1-\frac{\eta}{4})^{t+1}\kappa_0^4\mu_0^mr^*\lambda_{\submax}^2 + \sigma^2\log\dmax\bigg)\lambda_{\submin}^{-2}\log\dmax\bigg]^{1/2}\\
	&\lesssim \bigg[\eta\bigg((1-\frac{\eta}{4})^{t+1}\kappa_0^4\mu_0^mr^*\lambda_{\submax}^2 + \sigma^2\log\dmax\bigg)\lambda_{\submin}^{-2}\log\dmax\bigg]^{1/2},
\end{align*}
where the last inequality holds as long as $\eta\dmax\frac{r^*}{\rmin}\mu_0^{m-1}\log\dmax\lesssim 1$.
From this and \eqref{bound:wl}, we see 
\begin{align*}
	\ltwo{\P_{\U_{t+1,j}}\e_i - \P_{\U^*_j}\e_i} &\leq (1-\eta)^{t+1}\ltwo{\P_{\U_{0,j}}\e_i - \P_{\U^*_j}\e_i} + C(1-\frac{\eta}{4})^{(t+1)/2}\kappa_0^2(\frac{\mu_0r_j}{d_j})^{1/2} \\
	&\quad+ C\lambda_{\submin}^{-1}\kappa_0^2(\frac{\mu_0r_j}{d_j})^{1/2} (\eta\dof)^{1/2}\sigma\\
	&\quad + C\bigg[\eta\bigg((1-\frac{\eta}{4})^{t+1}\kappa_0^4\mu_0^mr^*\lambda_{\submax}^2 + \sigma^2\log\dmax\bigg)\lambda_{\submin}^{-2}\log\dmax\bigg]^{1/2}\\
	&\leq (1-\eta)^{t+1}\ltwo{\P_{\U_{0,j}}\e_i - \P_{\U^*_j}\e_i} + C(1-\frac{\eta}{4})^{(t+1)/2}\kappa_0^2(\frac{\mu_0r_j}{d_j})^{1/2} \\
	&\quad+ C\lambda_{\submin}^{-1}\kappa_0^2(\frac{\mu_0r_j}{d_j})^{1/2} (\eta\dof)^{1/2}\sigma,
\end{align*}
where the last inequality holds as long as $\eta\mu_0^{m-1}\kappa_0^2\dmax\frac{r^*}{\rmin}\log\dmax\lesssim 1$.
Now we see as long as $\U_{0,j}$ is $\mu_0$-incoherent, we have 
\begin{align}\label{twoinfUj}
	\ltwo{\P_{\U_{t+1,j}}\e_i - \P_{\U^*_j}\e_i} 
	&\leq  C(1-\frac{\eta}{4})^{(t+1)/2}\kappa_0^2 (\frac{\mu_0 r_j }{d_j})^{1/2} + 
	C\sigma\lambda_{\submin}^{-1}\kappa_0^2 (\frac{\mu_0 r_j }{d_j})^{1/2}(\eta\dof)^{1/2}.
\end{align}
Taking a union bound and with probability exceeding $1-m\dmax^{-11}$, \eqref{twoinfUj} holds for all $j\in[m]$. Now we are ready to bound $\linf{\bcalT_{t+1} - \bcalT^*}$. Since
\begin{align*}
	\bcalT_{t+1} - \bcalT^* &= \bcalT_{t+1}\times_{j=1}^m\P_{\U_{t+1,j}} - \bcalT^*\times_{j=1}^m\P_{\U_{j}^*} \\
	&=(\bcalT_{t+1} - \bcalT^*)\times_j\P_{\U_{t+1},j}+ \sum_{j=1}^m \bcalT^*\times_{l<j}\P_{\U_l^*}\times_j(\P_{U_{t+1,j}} - \P_{\U_j^*})\times_{l>j}\P_{\U_{t+1,j}},
\end{align*}
Using triangle inequality and \eqref{twoinfUj}, we have 
\begin{align*}
	\linf{\bcalT_{t+1} - \bcalT^*} &\leq \fro{\bcalT_{t+1}-\bcalT^*}(\frac{\mu_0^mr^*}{d^*})^{1/2} + \sum_{j=1}^m\ltwoinf{\P_{U_{t+1,j}} - \P_{\U_j^*}}(\frac{\mu_0^{m-1}r_j^-}{d_j^-})^{1/2}\lambda_{\submax}\\
	&\leq \bigg(\sqrt{2}(1-\frac{\eta}{4})^{(t+1)/2}\fro{\bcalT_0 -\bcalT^*} + C (\eta\dof)^{1/2}\sigma\bigg)(\frac{\mu_0^mr^*}{d^*})^{1/2}\\
	&\quad + \sum_{j=1}^m[C(1-\frac{\eta}{4})^{(t+1)/2}\kappa_0^2 (\frac{\mu_0 r_j }{d_j})^{1/2} + 
	C\sigma\lambda_{\submin}^{-1}\kappa_0^2 (\frac{\mu_0 r_j }{d_j})^{1/2}(\eta\dof)^{1/2}](\frac{\mu_0^{m-1}r_j^-}{d_j^-})^{1/2}\lambda_{\submax}\\
	&\leq Cm(1-\frac{\eta}{4})^{\frac{t+1}{2}}\kappa_0^2(\frac{\mu_0^mr^*}{d^*})^{1/2}\lambda_{\submax} + Cm\kappa_0^3(\frac{\mu_0^mr^*}{d^*})^{1/2}(\eta\dof)^{1/2}\sigma.
\end{align*}

\noindent\textit{Step 3: Controlling the probability.} From Step 1 and Step 2, we have proved under the event $\calE_t$, $\calE_{t+1}$ holds with probability exceeding $1-(1+2m\dmax^{-1})\dmax^{-10} \geq 1-3\dmax^{-10}$. Then a similar argument as in Step 3 in the proof of Theorem \ref{thm:gen} shows  
$$\PP(\calE_T^c) = \sum_{t= 1}^T\PP(\calE_{t-1}\cap\calE_t^c)\leq 3T\dmax^{-10}.$$

\subsection{Proof of Theorem \ref{thm:init:completion}}
Notice for each $j\in[m]$, $\tilde\U_j$ is actually the top $r_j$ left singular vectors of the following matrix:
\begin{align*}
	\tilde\N_j = \frac{1}{T_1(T_1-1)}\sum_{1\leq i<i'\leq T_1} Y_iY_{i'}(\calM_j(\bcalX_{i})\calM_j(\bcalX_{i'})^\top+\calM_j(\bcalX_{i'})\calM_j(\bcalX_{i})^\top).
\end{align*}
We denote $\N_j = \calM_j(\bcalT^*)\calM_j(\bcalT^*)^\top$. 
Using Wedin's sin$\Theta$ theorem, we have
\begin{align*}
	\op{\hat\U_j\hat\U_j^\top- \U_j^*\U_j^{*\top}} \leq \frac{\sqrt{2}\op{\hat\N_j - \N_j}}{\lambda_{\submin}^2}. 
\end{align*}
Now $\hat\N_j  -\N_j$ is a U-statistics of order 2, using standard decoupling techniques for U-statistics (see e.g. Theorem 3.4.1 in \cite{de2012decoupling}), we have 
\begin{align*}
	\PP(\op{\tilde\N_j - \N_j}\geq t )\leq 15\PP(\op{\bar\N_j - \N_j}\geq t), 
\end{align*}
where 
\begin{align*}
	\bar\N_j =  \frac{1}{T_1(T_1-1)}\sum_{1\leq i<i'\leq T_1} Y_i\tilde Y_{i'}(\calM_j(\bcalX_{i})\calM_j(\tilde\bcalX_{i'})^\top+\calM_j(\tilde\bcalX_{i'})\calM_j(\bcalX_{i})^\top),
\end{align*}
where $\tilde\bcalX_i = \sqrt{d^*}\bcalE_{\tilde\omega_{i}}$, and $\tilde\omega_{i}, \tilde Y_i$ i.i.d. copy of $\omega_{i},Y_i$ such that 
\begin{align*}
	\tilde Y_{i} = \inp{\tilde\bcalX_{i}}{\bcalT^*} +\tilde\epsilon_i. 
\end{align*}
For notation simplicity, we drop the subscript $j$, and we denote $m_1 = d_j$, $m_2=  d_j^-$, $\M = \calM_j(\bcalT^*)\in\RR^{m_1\times m_2}$, and $\X_i = \calM_j(\bcalX_{i})$, $\tilde\X_i = \calM_j(\tilde\bcalX_{i})$. 
We define 
\begin{align*}
	\S_1 = \bDel_1 + \Z_1, \quad \S_2 = \bDel_2 + \Z_2,
\end{align*}
where 
\begin{align*}
	&\bDel_1 = \left(\frac{1}{T_1}\sum_{i=1}^{T_1}\inp{\X_i}{\M}\X_i - \M\right),  &\bDel_2 = \left(\frac{1}{T_1}\sum_{i=1}^{T_1}\inp{\tilde\X_i}{\M}\tilde\X_i - \M\right),\\
	&\Z_1 = \frac{1}{T_1}\sum_{i=1}^{T_1}\epsilon_i\X_i, &\Z_2 = \frac{1}{T_1}\sum_{i=1}^{T_1}\tilde\epsilon_i\tilde\X_i.
\end{align*}
Recall we write $\N_j = \M\M^\top$ and thus 
\begin{align*}
	\tilde\N_j - \N_j &= \frac{T_1}{2(T_1 - 1)}(\S_1\S_2^\top +\S_2\S_1^\top)+ \frac{T_1}{2(T_1-1)}(\S_1+\S_2)\M^\top + \frac{T_1}{2(T_1-1)}\M(\S_1+\S_2)^\top \\
	&\quad+ \frac{1}{T_1-1}\left(\frac{1}{2T_1}\sum_{i=1}^{T_1}Y_i\tilde Y_i(\X_i\tilde\X_i^\top + \tilde\X_i\X_i^\top) - \M\M^\top\right).
\end{align*}
We denote $M = \max\{m_1,m_2\}$. 
Using matrix Bernstein inequality (see e.g. \cite{koltchinskii2011nuclear}) and Lemma 2.1 in \cite{koltchinskii2016perturbation}, we have the following event 
\begin{align*}
	\calE_1 = \bigg\{\max\{\op{\Z_1},\op{\Z_2}\} \leq C\sigma\frac{\sqrt{M\log M}}{\sqrt{T_1}}\bigg\}
\end{align*}
holds with probability exceeding $1-M^{-10}$ as long as $T_1\gtrsim \min\{m_1,m_2\}\log^2 M$. And from matrix Bernstein inequality, 
\begin{align*}
	\calE_2 = \bigg\{\max\{\op{\bDel_1}, \op{\bDel_2}\}\leq \linf{\M}\frac{\sqrt{Md^*\log M}}{\sqrt{T_1}}\bigg\}
\end{align*}
holds with probability exceeding $1-M^{-10}$ as long as $T_1\gtrsim \min\{m_1,m_2\}\log M$. 
Moreover, we consider the event 
\begin{align*}
	\calE_3 = \bigg\{\oneinf{\frac{1}{T_1}\sum_{t=1}Y_t\X_t^\top},& \oneinf{\frac{1}{T_1}\sum_{t=1}\tilde Y_t\tilde\X_t^\top}\\
	&\lesssim \sqrt{d^*}\big(\sqrt{d^*}\linf{\M} +\sigma\log^{1/2}M\big)(\frac{1}{m_2}+\frac{\log M}{T_1})\bigg\},
\end{align*}
where $\oneinf{\S^\top} = \max_{i\in[m_2]}\ltwo{\S\e_i}$ for $\S\in\RR^{m_1\times m_2}$. From Chernoff bound (see e.g. proof of Theorem 2 in \cite{yuan2017incoherent}), we see $\PP(\calE_3)\geq 1 - T_1M^{-10}$. 
We now proceed our proof conditioning on $\calE_1\cap \calE_2 \cap \calE_3$. 

\hspace{1cm}

\noindent\textit{Upper bound for $\op{\S_1\S_2^\top}, \op{\S_2\S_1^\top}$. }
We only consider the upper bound for $\op{\S_2\S_1^\top}$. Notice $\S_2$ is independent of $Y_i\X_i$. We shall proceed conditioning on $\S_2$. In fact, 
\begin{align*}
	\S_2\S_1^\top = \frac{1}{T_1}\sum_{i=1}^{T_1}(Y_i\S_2\X_i^\top - \S_2\M^\top), 
\end{align*}
and that 
\begin{align*}
	\EE \S_2(Y_i\X_i^\top -\M^\top)(Y_i\X_i -\M)\S_2^\top \preccurlyeq (d^*\linf{\M}^2 + \sigma^2)m_1\S_2\S_2^\top,
\end{align*}
and 
\begin{align*}
	\EE (Y_i\X_i -\M)\S_2^\top\S_2(Y_i\X_i^\top -\M^\top) \preccurlyeq (d^*\linf{\M}^2 + \sigma^2)\tr(\S_2^\top\S_2)\I_{m_1}. 
\end{align*}
Moreover, we have 
\begin{align*}
	&\quad \bigg\|\op{\S_2(Y_i\X_i^\top - \M^\top)}\bigg\|_{\psi_2} \lesssim  \sqrt{m_1}\oneinf{\S_2^\top}(\sqrt{d^*}\linf{\M} + \sigma) + \op{\M}\op{\S_2}. 
\end{align*}
Now using the matrix Bernstein inequality (see e.g. Proposition 2 in \cite{koltchinskii2016perturbation}), we have with probability exceeding $1-M^{-10}$, the following event holds
\begin{align*}
	\calE_4 = \bigg\{\op{\S_2\S_1^\top}&\lesssim \frac{\sqrt{\log^3 M}}{T_1}\bigg(\sqrt{m_1}\oneinf{\S_2^\top}(\sqrt{d^*}\linf{\M} + \sigma) + \op{\M}\op{\S_2}\bigg)\\
	&\quad+(\sqrt{d^*}\linf{\M} + \sigma)\op{\S_2}\sqrt{\frac{m_1\log M}{T_1}}\bigg\}.
\end{align*}
And thus under $\calE_1\cap \calE_2\cap\calE_3\cap \calE_4$, as long as $T_1\gtrsim \sqrt{\min\{m_1,m_2\}}\log^2M$,
\begin{align*}
	\op{\S_2\S_1^\top}\lesssim (\sqrt{d^*}\linf{\M} + \sigma)^2\frac{\sqrt{m_1M}\log M}{T_1}.
\end{align*}

\hspace{1cm}

\noindent\textit{Upper bound for $\op{\S_1\M^\top}, \op{\S_2\M^\top}$. }
Notice 
\begin{align*}
	\S_1\M^\top  = \frac{1}{T_1}\sum_{i=1}^{T_1}(Y_i\X_i - \M)\M^\top.
\end{align*}
And it is easy to verify that 
\begin{align*}
	&\quad\max\bigg\{\big\|\EE(Y_i\X_i - \M)\M^\top\M(Y_i\X_i - \M)\big\|, \big\|\EE\M(Y_i\X_i - \M)(Y_i\X_i - \M)\M^\top\big\|\bigg\}\\
	&\lesssim m_1(\sigma^2 + d^*\linf{\M}^2)\op{\M}^2, 
\end{align*}
and 
\begin{align*}
	&\quad\bigg\|\op{(Y_i\X_i - \M)\M^\top}\bigg\|_{\psi_2} \leq \sqrt{m_1d^*}(\sqrt{d^*}\linf{\M} + \sigma)\linf{\M}.
\end{align*}
Using matrix Bernstein inequality again, and we see with probability exceeding $1-M^{-10}$, the following event holds,
\begin{align*}
	\calE_5 = \bigg\{\op{\S_1\M^\top},\op{\S_2\M^\top} \lesssim \frac{\sqrt{m_1d^*}(\sqrt{d^*}\linf{\M} + \sigma)\linf{\M}\sqrt{\log M}}{\sqrt{T_1}}\bigg\}. 
\end{align*}

\hspace{1cm}

\noindent\textit{Upper bound for $\frac{1}{T_1-1}\left(\frac{1}{2T_1}\sum_{i=1}^{T_1}Y_i\tilde Y_i(\X_i\tilde\X_i^\top + \tilde\X_i\X_i^\top) - \M\M^\top\right)$. }
In fact, we have
\begin{align*}
	\frac{1}{T_1}\sum_{i=1}^{T_1}Y_i\tilde Y_i\X_i\tilde\X_i^\top- \M\M^\top &= \frac{1}{T_1}\sum_{i=1}^{T_1}Y_i\X_i(\tilde Y_i \tilde \X_i^\top - \M^\top) + \frac{1}{T_1}\sum_{i=1}^{T_1}(Y_i\X_i - \M)\M^\top.
\end{align*}
Here the second term is $\S_1\M^\top$, which is just bounded above. Now we conditioned on $Y_i\X_i, i = 1,\cdots, T_1$. One can similarly show 
\begin{align*}
	\bigg\|\op{Y_i\X_i(\tilde Y_i \tilde \X_i^\top - \M^\top)}\bigg\|_{\psi_2} \lesssim \sqrt{d^*}(\sigma + \sqrt{d^*}\linf{\M})\op{Y_i\X_i}. 
\end{align*}
And 
\begin{align*}
	&\quad\max\bigg\{\big\|\EE \sum_{i=1}^{T_1}Y_i^2\X_i(\tilde Y_i\tilde\X_i - \M)^\top(\tilde Y_i\tilde\X_i - \M)\X_i^\top\big\|, \big\|\EE\sum_{i=1}^{T_1}Y_i^2(\tilde Y_i\tilde\X_i - \M)\X_i^\top\X_i(\tilde Y_i\tilde\X_i - \M)^\top\big\|\bigg\}\\
	&\lesssim m_1(\sigma^2 + d^*\linf{\M}^2)\max\bigg\{\left\|\sum_{i=1}^{T_1}Y_i^2\X_i\X_i^\top\right\|, \left\|\sum_{i=1}^{T_1}Y_i^2\X_i^\top\X_i\right\|\bigg\}.
\end{align*}
And we have from matrix Bernstein inequality again, and we see with probability exceeding $1-M^{-10}$, the following event holds,
\begin{align*}
	\calE_6 = \bigg\{&\op{\frac{1}{T_1}\sum_{i=1}^{T_1}Y_i\X_i(\tilde Y_i \tilde \X_i^\top - \M^\top)} \lesssim \sqrt{d^*}(\sigma + \sqrt{d^*}\linf{\M})\frac{\log M}{T_1}\max_{i}\op{Y_i\X_i}\\
	&\quad + \sqrt{m_1}(\sigma + \sqrt{d^*}\linf{\M})\frac{\sqrt{\log M}}{T_1}\sqrt{\max\bigg\{\left\|\sum_{i=1}^{T_1}Y_i^2\X_i\X_i^\top\right\|, \left\|\sum_{i=1}^{T_1}Y_i^2\X_i^\top\X_i\right\|\bigg\}}
	\bigg\}.
\end{align*}
Now we consider the following event 
\begin{align*}
	\calE_7 = \bigg\{\max_i \op{Y_i\X_i}\lesssim (\sigma\sqrt{\log M} + \sqrt{d^*}\linf{\M})\bigg\},
\end{align*}
which holds with probability exceeding $1-T_1\dmax^{-10}$. 
So we conclude on $\calE_6\cap\calE_7$, we have 
\begin{align*}
	\op{\frac{1}{T_1}\sum_{i=1}^{T_1}Y_i\X_i(\tilde Y_i \tilde \X_i^\top - \M^\top)} \lesssim(\sigma + \sqrt{d^*}\linf{\M})^2 \bigg(\frac{\sqrt{d^*\log^3 M}}{T_1} + \frac{\sqrt{m_1}\log M}{\sqrt{T_1}}\bigg).
\end{align*}

\hspace{1cm}

\noindent\textit{Finalize the proof for subspace. }
Under $\bigcap_{k=1}^7\calE_k$, we have 
\begin{align*}
	\op{\bar\N_j - \N_j} &\lesssim (\sqrt{d^*}\linf{\M} + \sigma)^2\frac{\sqrt{m_1M}\log M}{T_1} + \frac{\sqrt{m_1d^*}(\sqrt{d^*}\linf{\M} + \sigma)\linf{\M}\sqrt{\log M}}{\sqrt{T_1}}\\
	&\quad + \frac{1}{T_1}(\sigma + \sqrt{d^*}\linf{\M})^2 \bigg(\frac{\sqrt{d^*\log^3 M}}{T_1} + \frac{\sqrt{m_1}\log M}{\sqrt{T_1}}\bigg)\\
	&\lesssim (\sqrt{d^*}\linf{\M} + \sigma)^2\frac{\sqrt{m_1M}\log M}{T_1} + \frac{\sqrt{m_1d^*}(\sqrt{d^*}\linf{\M} + \sigma)\linf{\M}\sqrt{\log M}}{\sqrt{T_1}}.
\end{align*}
We now plug in the $m_1, m_2, M$ and we obtain 
\begin{align*}
	\op{\bar\N_j - \N_j} &\lesssim_m(\sqrt{d^*}\linf{\M} + \sigma)^2\frac{\sqrt{d_j^2 + d^*}\log\dmax}{T_1} + \frac{\sqrt{d_jd^*}(\sqrt{d^*}\linf{\M} + \sigma)\linf{\M}\sqrt{\log\dmax}}{\sqrt{T_1}},
\end{align*}
where $\lesssim_m$ hides constant depending only on the dimension $m$ of the tensor. Finally, notice $\sqrt{d^*}\linf{\M}\leq \mu^{m/2}\sqrt{r^*}\lambda_{\submax}$, we conclude with probability exceeding $1-4T_1\dmax^{-10}$, 
\begin{align*}
	\op{\tilde\U_j\tilde\U_j^\top - \U_j^*\U_j^{*\top}} &\lesssim_m (\mu^mr^*\kappa_0^2+ \sigma^2/\lambda_{\submin}^2)\frac{(d_j + \sqrt{d^*})\log\dmax}{T_1} \\
	&\quad + \frac{\sqrt{d_j}(\mu^{m/2}\sqrt{r^*}\lambda_{\submax}+ \sigma)\mu^{m/2}\sqrt{r^*}\lambda_{\submax}\sqrt{\log\dmax}}{\lambda_{\submin}^2\sqrt{T_1}}.
\end{align*}
Finally from Remark 6.2 in \cite{keshavan2010matrix}, we conclude 
\begin{align*}
	\op{\hat\U_j\hat\U_j^\top - \U_j^*\U_j^{*\top}} &\lesssim_m (\mu^mr^*\kappa_0^2+ \sigma^2/\lambda_{\submin}^2)\frac{(d_j + \sqrt{d^*})\log\dmax}{T_1} \\
	&\quad + \frac{\sqrt{d_j}(\mu^{m/2}\sqrt{r^*}\lambda_{\submax}+ \sigma)\mu^{m/2}\sqrt{r^*}\lambda_{\submax}\sqrt{\log\dmax}}{\lambda_{\submin}^2\sqrt{T_1}}.
\end{align*}
and $\incoh(\hat\U_j)\leq 3\mu$.

\hspace{1cm}

\noindent\textit{Core tensor estimation. }
Now we consider the accuracy for the core tensor estimation. 
For notation simplicity, we shall use $\bcalX_t$ instead of $\bcalX_{t-T_1}$ and then $\hat\U_{j}, j\in[m]$ is independent of $\{\bcalX_t\}_{t=1}^{T_2}$. We denote the loss function 
\begin{align*}
	L(\bcalC) = \frac{1}{2T_2} \sum_{i=1}^{T_2}\big(\inp{\bcalX_t}{\bcalC\times_1\hat\U_1\cdots\times_m\hat\U_m} - Y_t\big)^2. 
\end{align*}
And simple computation shows $\nabla L(\bcalC) =\frac{1}{T_2} \sum_{i=1}^{T_2} \big(\inp{\bcalX_t}{\bcalC\times_{j=1}^m\hat\U_j} - Y_t\big)\bcalX_t\times_{j=1}^m\hat\U_j^\top$, and 
\begin{align*}
	\EE\nabla L(\bcalC) = \bcalC - \bcalC^*\times_{j=1}^m(\hat\U_j^\top\U_j^{*}). 
\end{align*}
It would also be helpful to notice $\nabla L(\hat\bcalC) = 0$ since $\hat\bcalC$ is the least square estimator. 
We now decompose 
\begin{align*}
	\fro{\hat \bcalC - \bcalC^*}^2 &= \inp{\hat \bcalC - \bcalC^*}{\hat \bcalC - \bcalC^*} = \inp{\hat \bcalC - \bcalC^*}{\EE\nabla L(\hat\bcalC) - \EE\nabla L(\bcalC^*)}\\
	&= \underbrace{\inp{\hat \bcalC - \bcalC^*}{\EE\nabla L(\hat\bcalC) -\nabla L(\hat\bcalC) }}_{=:\beta_1}- \underbrace{\inp{\hat \bcalC - \bcalC^*}{\EE\nabla L(\bcalC^*)}}_{\beta_2}. 
\end{align*}
Using the expression for $\EE\nabla L(\hat\C)$ and $\nabla L(\hat\bcalC)$ above, we can further decompose $\beta_1$ as 
\begin{align*}
	\beta_1 &= - \frac{1}{T_2}\sum_{t=1}^{T_2} \inp{\bcalX_t\times_j\hat\U_j^\top}{\hat\bcalC - \bcalC^*\times_{j=1}^m(\hat\U_j^\top\U_j^{*})}\inp{\bcalX_t\times_j\hat\U_j^\top}{\hat\bcalC - \bcalC^*}\\
	&\quad +\inp{\hat \bcalC - \bcalC^*}{\hat\bcalC - \bcalC^*\times_{j=1}^m(\hat\U_j^\top\U_j^{*})} + \frac{1}{T_2}\sum_{t=1}^{T_2}\epsilon_t\inp{\bcalX_t\times_j\hat\U_j^\top}{\hat\bcalC - \bcalC^*}. 
\end{align*}
And we have 
\begin{align*}
	\beta_{1,1}&:= \bigg|- \frac{1}{T_2}\sum_{t=1}^{T_2} \inp{\bcalX_t\times_j\hat\U_j^\top}{\hat\bcalC - \bcalC^*\times_{j=1}^m(\hat\U_j^\top\U_j^{*})}\inp{\bcalX_t\times_j\hat\U_j^\top}{\hat\bcalC - \bcalC^*} \\
	&\hspace{6cm}+\inp{\hat \bcalC - \bcalC^*}{\hat\bcalC - \bcalC^*\times_{j=1}^m(\hat\U_j^\top\U_j^{*})}\bigg|\\
	&=|\inp{\frac{1}{T_2}\sum_{t=1}^{T_2}\g_i\g_i^\top - \I}{\a\b^\top}|, 
\end{align*}
where $\g_i = \vec(\bcalX_t\times_j\hat\U_j^\top)\in\RR^{r^*}$ satisfies $\psitwo{\inp{\g_i}{\m}}\leq C_0\ltwo{\m}$ (notice we here implicitly use the fact $\hat\U_j$ and $\bcalX_i$ are independent), and $$\a = \vec(\hat \bcalC - \bcalC^*), \b = \vec\big(\hat\bcalC - \bcalC^*\times_{j=1}^m(\hat\U_j^\top\U_j^{*})\big). $$ 
Using matrix Bernstein inequality, and we see with probability exceeding $1-e^{-\delta_1}$ (for some $\delta_1\leq T_2$ to be specified),
\begin{align*}
	\op{\frac{1}{T_2}\sum_{t=1}^{T_2}\g_i\g_i^\top - \I} \lesssim \sqrt{\frac{\mu^mr^*(\log r^* + \delta_1)}{T_2}}. 
\end{align*} 
Therefore, we have 
\begin{align*}
	|\inp{\frac{1}{T_2}\sum_{t=1}^{T_2}\g_i\g_i^\top - \I}{\a\b^\top}| &\leq \sqrt{\frac{\mu^mr^*(\log r^* + \delta_1)}{T_2}}\nuc{\a\b^\top}\\
	&=\sqrt{\frac{\mu^mr^*(\log r^* + \delta_1)}{T_2}}\fro{\hat \bcalC - \bcalC^*} \fro{\hat\bcalC - \bcalC^*\times_{j=1}^m(\hat\U_j^\top\U_j^{*})}. 
\end{align*}
We have 
\begin{align}\label{decomposition:tc}
	&\quad \hat\bcalC - \bcalC^*\times_{j=1}^m(\hat\U_j^\top\U_j^{*}) = \hat\bcalC - \bcalC^* + \bcalC^* -  \bcalC^*\times_{j=1}^m(\hat\U_j^\top\U_j^{*})\notag\\
	&= \hat\bcalC - \bcalC^* +\sum_{j=1}^m \bcalC^*\times_1\I_{r_1}\cdots \times_j(\I_{r_j} - \hat\U_j^\top\U_j^{*})\times_{j+1}(\hat\U_{j+1}^\top\U_{j+1}^{*}) \cdots \times_m(\hat\U_{m}^\top\U_{m}^{*}).
\end{align}
Therefore 
\begin{align*}
	\fro{\hat\bcalC - \bcalC^*\times_{j=1}^m(\hat\U_j^\top\U_j^{*})} &\leq \fro{\hat\bcalC - \bcalC^*} + \lambda_{\submax}\sum_{j=1}^m\fro{\I_{r_j} - \hat\U_j^\top\U_j^{*}}\\
	&\leq  \fro{\hat\bcalC - \bcalC^*} + \lambda_{\submax}\sum_{j=1}^m\fro{\hat\U_j\hat\U_j^\top - \U_j^*\U_j^{*\top}}.
\end{align*}
So we conclude 
\begin{align*}
	\beta_{1,1} \leq \sqrt{\frac{\mu^mr^*(\log r^* + \delta_1)}{T_2}}\fro{\hat \bcalC - \bcalC^*}\bigg(\fro{\hat \bcalC - \bcalC^*} + \lambda_{\submax}\sum_{j=1}^m\fro{\hat\U_j\hat\U_j^\top - \U_j^*\U_j^{*\top}}\bigg).
\end{align*}
On the other hand, using matrix Bernstein inequality, we conclude with probability exceeding $1-e^{-\delta_1}$, 
\begin{align*}
	\beta_{1,2} &:= \bigg|\frac{1}{T_2}\sum_{t=1}^{T_2}\epsilon_t\inp{\bcalX_t\times_j\hat\U_j^\top}{\hat\bcalC - \bcalC^*}\bigg| \leq \sqrt{\frac{\mu^mr^*(\log r^* + \delta_1)}{T_2}}\sigma\fro{\hat\bcalC - \bcalC^*}. 
\end{align*}
For $\beta_2$, using once again the decomposition in \eqref{decomposition:tc}, 
\begin{align*}
	|\beta_2| &= |\inp{\hat \bcalC - \bcalC^*}{\EE\nabla L(\bcalC^*)} = \inp{\hat \bcalC - \bcalC^*}{\bcalC^* - \bcalC^*\times_{j=1}^m(\hat\U_j^\top\U_j^{*})}|\\
	&=\big|\sum_{j=1}^m\inp{\hat \bcalC - \bcalC^*}{ \bcalC^*\times_1\I_{r_1}\cdots \times_j(\I_{r_j} - \hat\U_j^\top\U_j^{*})\times_{j+1}(\hat\U_{j+1}^\top\U_{j+1}^{*}) \cdots \times_m(\hat\U_{m}^\top\U_{m}^{*})}\big|\\
	&\leq \sum_{j=1}^m \fro{\hat\bcalC - \bcalC^*}\cdot \lambda_{\submax}\fro{\hat\U_j\hat\U_j^\top - \U_j^*\U_j^{*\top}}. 
\end{align*}
Putting everything together and we have 
\begin{align*}
	\fro{\hat\bcalC - \bcalC^*}^2 &\leq \sqrt{\frac{\mu^mr^*(\log r^* + \delta_1)}{T_2}}\fro{\hat \bcalC - \bcalC^*}\bigg(\fro{\hat \bcalC - \bcalC^*} + \lambda_{\submax}\sum_{j=1}^m\fro{\hat\U_j\hat\U_j^\top - \U_j^*\U_j^{*\top}} + \sigma\bigg) \\
	&\quad + \sum_{j=1}^m \fro{\hat\bcalC - \bcalC^*}\cdot \lambda_{\submax}\fro{\hat\U_j\hat\U_j^\top - \U_j^*\U_j^{*\top}}.
\end{align*}
As a result, as long as $\sqrt{\frac{\mu^mr^*(\log r^* + \delta_1)}{T_2}}\leq \frac{1}{2}$ we have 
\begin{align*}
	\fro{\hat\bcalC - \bcalC^*}\lesssim \lambda_{\submax}\sum_{j=1}^m\fro{\hat\U_j\hat\U_j^\top - \U_j^*\U_j^{*\top}}  +\sqrt{\frac{\mu^mr^*(\log r^* + \delta_1)}{T_2}}\sigma
\end{align*}
Finally, we set $\delta_1 = \log\dmax$, and we have
\begin{align*}
	\fro{\hat\bcalC\times_{j=1}^m\hat\U_j - \bcalC^*\times_{j=1}^m\U_j^*} \leq \fro{\hat\bcalC - \bcalC^*} + \lambda_{\submax}\sum_{j=1}^m\fro{\hat\U_j\hat\U_j^\top - \U_j^*\U_j^{*\top}} \leq c_m\lambda_{\submin} 
\end{align*}
under the given sample size condition and SNR condition.

	\subsection{Proof of Theorem \ref{thm:binary}}
The proof of this theorem is similar with the proof of Theorem \ref{thm:localrefinement:completion}. We shall frequently use the results therein.
Notice from Assumption \ref{assump:binary}, we see that $\sqrt{d^*}\linf{\bcalT^*}\leq \alpha/2$. 
At time $t$, we shall condition on the following event,
\begin{align*}
	\calE_t = \bigg\{\forall 0\leq l\leq t, \fro{\bcalT_l - \bcalT^*}^2 \leq 2(1-\frac{1}{4}\eta\gamma_{\alpha})^{l}\fro{\bcalT_{0} - \bcalT^*}^2 + C\frac{\eta L_{\alpha}^2\dof}{\gamma_{\alpha}},\\
	\forall j\in[m], \ltwoinf{\U_{l,j}}^2 \leq \frac{\mu_0r_j}{d_j}\bigg\}.
\end{align*}
for some absolute constant $C>0$ and $\mu_0 = 20\kappa_0^2\mu$. From Lemma \ref{lemma:incohspiki}, we see that $\spiki(\bcalT_t)\leq (\frac{r^*}{\rmax})^{1/2}\mu_0^{m/2}\kappa_0=:\nu_0$.
Then from Assumption \ref{assump:binary} and under the event $\calE_t$, we see that 
$$\sqrt{d^*}\linf{\bcalT_t} \leq \sqrt{d^*}\frac{\nu_0\fro{\bcalT_t}}{\sqrt{d^*}}\leq \nu_0(\fro{\bcalT_t-\bcalT^*} + \fro{\bcalT^*}) \leq \nu_0\bigg(2\fro{\bcalT_0 - \bcalT^*} + C\sqrt{\frac{\eta L_{\alpha}^2\dof}{\gamma_{\alpha}}}\bigg) + \frac{\alpha}{2}\leq \alpha,$$
as long as $\fro{\bcalT_0 - \bcalT^*}\lesssim\frac{\alpha}{\nu_0}$ and $\sqrt{\frac{\eta L_{\alpha}^2\dof}{\gamma_{\alpha}}} \lesssim \alpha$.

Recall the loss function is $\ell(\bcalT_t, \mathfrak{D}_t)=h(\langle \bcalT, \bcalX_t\rangle, Y_t)$ and thus the gradient $\bcalG_t$ is given by $\bcalG_t = h_{\theta}(\inp{\bcalX_t}{\bcalT_t},Y_t)\bcalX_t$. 

\hspace{1cm} 

\noindent\textit{Step 1: Relation between $\fro{\bcalT_{t+1}-\bcalT^*}^2$ and $\fro{\bcalT_{t}-\bcalT^*}^2$.} 
We first bound $\fro{\calP_{\TT_t}\bcalG_t}$. Recall $\bcalG_t = h_{\theta}(\inp{\bcalX_t}{\bcalT_t},Y_t)\bcalX_t$ and $|h_{\theta}(\inp{\bcalX_t}{\bcalT_t},Y_t)|\leq L_{\alpha}$ since $|\inp{\bcalX_t}{\bcalT_t}|\leq \alpha$ and for $y\in\{0,1\}$, $\sup_{|\theta|\leq \alpha}|h_{\theta}(\theta,y)|\leq L_{\alpha}$ from the definition of $L_{\alpha}$.
Moreover, $\bcalX_t$ has the same form as in the tensor completion case, so from \eqref{uniform:ptx}, 
\begin{align*}
	\fro{\calP_{\TT_t}\bcalG_t} \leq L_{\alpha}(m+1)^{1/2}(\mu_0^{m-1}r^*\frac{\dmax}{\rmin})^{1/2}.
\end{align*}
So as long as $\eta L_{\alpha}(\mu_0^{m-1}r^*\frac{\dmax}{\rmin})^{1/2}\lesssim_m \lambda_{\submin}$, $\eta\fro{\calP_{\TT_t}\bcalG_t} \leq \frac{\lambda_{\submin}}{8}$. Since we have verified the condition of Lemma \ref{lemma:perturbation}, we see
\begin{align}\label{binary:t+1andt}
	\fro{\bcalT_{t+1} - \bcalT^*}^2 &\leq (1+\frac{2}{\eta\gamma_{\alpha}})\fro{\calR(\bcalT_t^+ )- \bcalT_t^+ }^2 + (1+\frac{\eta\gamma_{\alpha}}{2})\fro{\bcalT_t^+  - \bcalT^*}^2\notag\\
	&\leq C_m\frac{\eta^3\fro{\calP_{\TT_t}\bcalG_t}^4}{\gamma_{\alpha}\lambda_{\submin}^2}+ (1+\frac{\eta\gamma_{\alpha}}{2})\fro{\bcalT_t^+  - \bcalT^*}^2\notag\\
	&\leq \eta^2L_{\alpha}^2\dof + (1+\frac{\eta\gamma_{\alpha}}{2})\fro{\bcalT_t^+  - \bcalT^*}^2,
\end{align}
as long as $\eta L_{\alpha}^2(\mu_0^{m-1}r^*\frac{\dmax}{\rmin})^2\frac{1}{\gamma_{\alpha}\lambda_{\submax}^2\dof}\lesssim_m 1$. 

Denote $\bcalG_t^* = \nabla_{\bcalT} \ell(\bcalT^*,\mathfrak{D}_t) = h_{\theta}(\inp{\bcalX_t}{\bcalT^*},Y_t)\bcalX_t$. Then $\EE_t\bcalG_t^* = \EE_{\bcalX_t}\EE_{Y_t}\bcalG_t^*\cdot\mathds{1}(\calE_t) = 0$, and thus
\begin{align*}
	\EE_t\inp{\bcalT_t - \bcalT^*}{\calP_{\TT_t}\bcalG_t} &= \EE_t\inp{\bcalT_t - \bcalT^*}{\calP_{\TT_t}(\bcalG_t-\bcalG_t^*)}  \\
	&= \EE_t\inp{\bcalT_t - \bcalT^*}{\bcalG_t-\bcalG_t^*} - 	\EE_t\inp{\bcalT_t - \bcalT^*}{\calP_{\TT_t}^{\perp}(\bcalG_t-\bcalG_t^*)}  
\end{align*}
Recall from the definition of $\gamma_{\alpha},\mu_{\alpha}$, $h(\theta,Y_t)$ is $\gamma_{\alpha}$-strongly convex and $h_{\theta}(\theta,y)$ is $\mu_{\alpha}$-Lipschitz w.r.t. $|\theta|\leq\alpha$. Therefore,
\begin{align*}
	\EE_t\inp{\bcalT_t - \bcalT^*}{\bcalG_t-\bcalG_t^*} 
	&= \EE_t\inp{\bcalT_t-\bcalT^*}{\bcalX_t} \cdot \left(h_{\theta}(\inp{\bcalX_t}{\bcalT_t},Y_t) - h_{\theta}(\inp{\bcalX_t}{\bcalT^*}),Y_t\right)\\
	&\geq \gamma_{\alpha}\EE_t\inp{\bcalT_t-\bcalT^*}{\bcalX_t}^2= \gamma_{\alpha}\fro{\bcalT_t-\bcalT^*}^2,
\end{align*}
and
\begin{align*}
	\EE_t\inp{\bcalT_t - \bcalT^*}{\calP_{\TT_t}^{\perp}(\bcalG_t-\bcalG_t^*)} 
	&\leq \EE_t|\inp{\bcalT_t - \bcalT^*}{\calP_{\TT_t}^{\perp}\bcalX_t}|\cdot|h_{\theta}(\inp{\bcalX_t}{\bcalT_t},Y_t) - h_{\theta}(\inp{\bcalX_t}{\bcalT^*},Y_t)|\\
	&\leq \mu_{\alpha}\EE_t|\inp{\calP_{\TT_t}^{\perp}(\bcalT_t - \bcalT^*)}{\bcalX_t}|\cdot|\inp{\bcalT_t-\bcalT^*}{\bcalX_t}|\\
	&\leq C\mu_{\alpha}\fro{\calP_{\TT_t}^{\perp}\bcalT^*}\fro{\bcalT_t-\bcalT^*}
	\leq C_m \mu_{\alpha}\frac{\fro{\bcalT_t-\bcalT^*}^3}{\lambda_{\submin}},
\end{align*}
where the last inequality is from Lemma \ref{lemma:ptperp}. So as long as $\fro{\bcalT_0 - \bcalT^*}\lesssim_m \frac{\gamma_{\alpha}}{\mu_{\alpha}}\lambda_{\submin}$ and $\sqrt{\eta\frac{L_{\alpha}^2}{\gamma_{\alpha}}\dof}\lesssim_m \frac{\gamma_{\alpha}}{\mu_{\alpha}}\lambda_{\submin}$, we see 
$$\EE_t\inp{\bcalT_t - \bcalT^*}{\calP_{\TT_t}\bcalG_t}\geq 0.5\gamma_{\alpha}\fro{\bcalT_t-\bcalT^*}^2.$$
Also from \eqref{expectation:ptx}, we see 
\begin{align*}
	\EE_t\fro{\calP_{\TT_t}\bcalG_t}^2 = \EE_t|h_{\theta}(\inp{\bcalT_t}{\bcalX_t},Y_t)|^2\fro{\calP_{\TT_t}\bcalX_t}^2\leq L_{\alpha}^2\EE_t\fro{\calP_{\TT_t}\bcalX_t}^2 \leq L_{\alpha}^2\dof.
\end{align*}
So we have 
\begin{align*}
	\fro{\bcalT_t^+  - \bcalT^*}^2 &= \fro{\bcalT_t - \bcalT^*}^2 - 2\eta\EE_t\inp{\bcalT_t - \bcalT^*}{\calP_{\TT_t}\bcalG_t}+ \eta^2\EE_t\fro{\calP_{\TT_t}\bcalG_t}^2\\ 
	&\quad+\bigg(2\eta\EE_t\inp{\bcalT_t - \bcalT^*}{\calP_{\TT_t}\bcalG_t} - 2\eta\inp{\bcalT_t - \bcalT^*}{\calP_{\TT_t}\bcalG_t}\bigg) \\
	&\quad + \eta^2(\fro{\calP_{\TT_t}\bcalG_t}^2 - \EE_t \fro{\calP_{\TT_t}\bcalG_t}^2)\\
	&\leq (1-\eta\gamma_{\alpha})\fro{\bcalT_t - \bcalT^*}^2 + \eta^2L_{\alpha}^2\dof\\&\quad+\bigg(2\eta\EE_t\inp{\bcalT_t - \bcalT^*}{\calP_{\TT_t}\bcalG_t} - 2\eta\inp{\bcalT_t - \bcalT^*}{\calP_{\TT_t}\bcalG_t}\bigg) \\
	&\quad + \eta^2(\fro{\calP_{\TT_t}\bcalG_t}^2 - \EE_t \fro{\calP_{\TT_t}\bcalG_t}^2).
\end{align*}
Together with \eqref{binary:t+1andt}, we obtain 
\begin{align*}
	\fro{\bcalT_{t+1} - \bcalT^*}^2 &\leq 
	 (1-\frac{\eta\gamma_{\alpha}}{2})\fro{\bcalT_{t} - \bcalT^*}^2 + 2\eta^2L_{\alpha}^2\dof\\
	&\quad +2\eta(1+\frac{\eta\gamma_{\alpha}}{2})\bigg(\EE_t\inp{\bcalT_t - \bcalT^*}{\calP_{\TT_t}\bcalG_t} - \inp{\bcalT_t - \bcalT^*}{\calP_{\TT_t}\bcalG_t}\bigg)\\
	&\quad + \eta^2(1+\frac{\eta\gamma_{\alpha}}{2})(\fro{\calP_{\TT_t}\bcalG_t}^2 - \EE_t \fro{\calP_{\TT_t}\bcalG_t}^2).
\end{align*}
Telescoping this inequality and we get
\begin{align*}
	\fro{\bcalT_{t+1} - \bcalT^*}^2 &\leq (1-\frac{1}{2}\eta\gamma_{\alpha})^{t+1}\fro{\bcalT_{0} - \bcalT^*}^2 + C\frac{\eta L_{\alpha}^2\dof}{\gamma_{\alpha}}\\
	&\quad + \sum_{l=0}^t \underbrace{(1-\frac{1}{2}\eta\gamma_{\alpha})^{t-l}2\eta(1+\frac{1}{2}\eta\gamma_{\alpha})\bigg(\EE_l\inp{\bcalT_l - \bcalT^*}{\calP_{\TT_l}\bcalG_l} - \inp{\bcalT_l - \bcalT^*}{\calP_{\TT_l}\bcalG_l}\bigg)}_{=:D_l}\\
	&\quad + \sum_{l=0}^t\underbrace{(1-\frac{1}{2}\eta\gamma_{\alpha})^{t-l}\eta^2(1+\frac{1}{2}\eta\gamma_{\alpha})(\fro{\calP_{\TT_l}\bcalG_l}^2 - \EE_t \fro{\calP_{\TT_l}\bcalG_l}^2)}_{=:F_l}.
\end{align*}
Now we first use martingale concentration inequality to bound $\sum_{l=0}^tD_l$. In fact, from \eqref{uniform:ptx}
\begin{align*}
	|\inp{\bcalT_l - \bcalT^*}{\calP_{\TT_l}\bcalG_l}| &\leq \fro{\bcalT_l - \bcalT^*}\fro{\calP_{\TT_l}\bcalG_l} \leq C_m\fro{\bcalT_l - \bcalT^*}L_{\alpha}\sqrt{\mu_0^{m-1}r^*\frac{\dmax}{\rmin}}\\
	&\leq C_m\bigg[(1-\frac{1}{4}\eta\gamma_{\alpha})^{\frac{l}{2}}\fro{\bcalT_0-\bcalT^*} + (\eta\frac{L_{\alpha}^2}{\gamma_{\alpha}}\dof)^{1/2}\bigg]L_{\alpha}\sqrt{\mu_0^{m-1}r^*\frac{\dmax}{\rmin}}.
\end{align*}
And thus 
\begin{align*}
	|D_l| &\leq C_m\eta\bigg[(1-\frac{1}{4}\eta\gamma_{\alpha})^{\frac{t}{2}}\fro{\bcalT_0-\bcalT^*} + (\eta\frac{L_{\alpha}^2}{\gamma_{\alpha}}\dof)^{1/2}\bigg]L_{\alpha}\sqrt{\mu_0^{m-1}r^*\frac{\dmax}{\rmin}}\\
	&\leq \frac{1}{2}(1-\frac{1}{4}\eta\gamma_{\alpha})^{t+1}\fro{\bcalT_0-\bcalT^*}^2(\log\dmax)^{-1} + C_m\eta^2L_{\alpha}^2\mu_0^{m-1}r^*\frac{\dmax}{\rmin}\log\dmax\\
	&\hspace{4cm}+ C_m\eta^{3/2}(\frac{L_{\alpha}^2}{\gamma_{\alpha}}\dof)^{1/2} L_{\alpha}\sqrt{\mu_0^{m-1}r^*\frac{\dmax}{\rmin}}\\
	&\leq \frac{1}{2}(1-\frac{1}{4}\eta\gamma_{\alpha})^{t+1}\fro{\bcalT_0-\bcalT^*}^2(\log\dmax)^{-1} + C_m\eta^{3/2}(\frac{L_{\alpha}^2}{\gamma_{\alpha}}\dof)^{1/2} L_{\alpha}\sqrt{\mu_0^{m-1}r^*\frac{\dmax}{\rmin}},
\end{align*}
as long as $\eta\gamma_{\alpha}\mu_0^{m-1}r^*\frac{\dmax}{\rmin \dof}\log^2\dmax\lesssim_m 1$.
On the other hand, 
\begin{align*}
	\EE_l\inp{\bcalT_l - \bcalT^*}{\calP_{\TT_l}\bcalG_l}^2 \leq L_{\alpha}^2 \EE_l\inp{\bcalT_l - \bcalT^*}{\calP_{\TT_l}\bcalX_l}^2\leq L_{\alpha}^2\fro{\bcalT_l - \bcalT^*}^2.
\end{align*}
Therefore 
\begin{align*}
	\sum_{l=0}^t\text{Var}_l D_l &\leq C\sum_{l = 0}^t(1-\frac{1}{2}\eta\gamma_{\alpha})^{2t-2l}\eta^2L_{\alpha}^2\fro{\bcalT_l - \bcalT^*}^2\\
	&\leq C\sum_{l = 0}^t(1-\frac{1}{2}\eta\gamma_{\alpha})^{2t-2l}\eta^2L_{\alpha}^2\bigg(2(1-\frac{1}{4}\eta\gamma_{\alpha})^l\fro{\bcalT_0 - \bcalT^*}^2 + \eta\frac{L_{\alpha}^2}{\gamma_{\alpha}}\dof\bigg)\\
	&\leq C\eta\frac{L_{\alpha}^2}{\gamma_{\alpha}}(1-\frac{1}{4}\eta\gamma_{\alpha})^{t+1}\fro{\bcalT_0 - \bcalT^*}^2 + C\eta^2\frac{L_{\alpha}^4}{\gamma_{\alpha}^2}\dof\\
	&\leq \frac{1}{4}(1-\frac{1}{4}\eta\gamma_{\alpha})^{2t+2}\fro{\bcalT_0 - \bcalT^*}^4(\log\dmax)^{-1} +  C\eta^2\frac{L_{\alpha}^4}{\gamma_{\alpha}^2}\dof.
\end{align*}
As a result, with probability exceeding $1-2\dmax^{-10}$, as long as $\eta\gamma_{\alpha}\mu_0^{m-1}r^*\frac{\dmax}{\rmin}\log\dmax\lesssim 1$, 
\begin{align*}
	|\sum_{l=0}^t D_l| &\leq (1-\frac{1}{4}\eta\gamma_{\alpha})^{t+1}\fro{\bcalT_0 - \bcalT^*}^2 + C\eta\frac{L_{\alpha}^2}{\gamma_{\alpha}}(\dof\log\dmax)^{1/2}.
\end{align*}
Now we consider $\sum_{l=0}^t F_l$. From \eqref{uniform:ptx},
$
	\fro{\calP_{\TT_t}\bcalG_t}^2 \leq (m+1)L_{\alpha}^2\mu_0^{m-1}r^*\frac{\dmax}{\rmin}.
$
And thus 
\begin{align*}
	|F_l|\lesssim_m \eta^2L_{\alpha}^2\mu_0^{m-1}r^*\frac{\dmax}{\rmin}.
\end{align*}
On the other hand, 
$
	\EE_l\fro{\calP_{\TT_l}\bcalG_l}^4 \lesssim_m L_{\alpha}^4\mu_0^{m-1}r^*\frac{\dmax}{\rmin}\dof.
$
And 
\begin{align*}
	\sum_{l=0}^t \text{Var}_l F_l^2 &\lesssim_m \sum_{l=0}^t(1-\frac{1}{2}\eta\gamma_{\alpha})^{2t-2l}\eta^4L_{\alpha}^4\dof\mu_0^{m-1}r^*\frac{\dmax}{\rmin}\\
	&\lesssim_m \eta^3L_{\alpha}^4(\gamma_{\alpha})^{-1}\dof\mu_0^{m-1}r^*\frac{\dmax}{\rmin}.
\end{align*}
So as long as $\eta\gamma_{\alpha}\mu_0^{m-1}r^*\frac{\dmax}{\rmin \dof}\log\dmax \lesssim 1$, with probability exceeding $1-2\dmax^{-10}$, 
\begin{align*}
	|\sum_{l=0}^t F_l| \lesssim_m \eta^{3/2}L_{\alpha}^2\gamma_{\alpha}^{-1/2}(\dof\mu_0^{m-1}r^*\frac{\dmax}{\rmin})^{1/2}\log^{1/2}\dmax.
\end{align*}
These conclude under event $\calE_t$, with probability exceeding $1-4\dmax^{-10}$,
\begin{align*}
	\fro{\bcalT_{t+1} - \bcalT^*}^2 &\leq 2(1-\frac{1}{4}\eta\gamma_{\alpha})^{t+1}\fro{\bcalT_{0} - \bcalT^*}^2+ C\eta\frac{L_{\alpha}^2}{\gamma_{\alpha}}(\dof\log\dmax)^{1/2}\\
	&\quad+ C_m \eta^{3/2}L_{\alpha}^2\gamma_{\alpha}^{-1/2}(\dof\mu_0^{m-1}r^*\frac{\dmax}{\rmin})^{1/2}\log^{1/2}\dmax + C\eta\gamma_{\alpha}^{-1}L_{\alpha}^2\dof\\
	&\leq 2(1-\frac{1}{4}\eta\gamma_{\alpha})^{t+1}\fro{\bcalT_{0} - \bcalT^*}^2+ C\eta\gamma_{\alpha}^{-1}L_{\alpha}^2\dof,
\end{align*}
as long as $\eta\gamma_{\alpha}\mu_0^{m-1}r^*\frac{\dmax}{\rmin \dof}\log\dmax \lesssim 1$.
                                                                                                
\hspace{1cm}                                                                                    

\noindent\textit{Step 2: Incoherence of $\bcalT_{t+1}$.} The proof of this part is mostly similar with the proof in Step 2 of Theorem \ref{thm:localrefinement:completion}. We will use the same notation and point out the differences. Recall we write $\bcalT_t  = \bcalC\cdot(\U_1,\cdots,\U_m)$.

Now for each $j\in[m], i\in[d_j]$, we consider $\e_i^\top\U_{t+1,j}\U_{t+1,j}^{\top}\e_i$. Since $\U_{t+1,j}$ is the top $r_j$ left singular vectors of $(\bcalT_t - \eta\calP_{\TT_t}\bcalG_t)_{(j)}$.
And $\U_j$ is the left singular vectors of $(\bcalT_t)_{(j)} = \U_j\bcalC_{(j)}(\otimes_{l\neq j}\U_l)^{\top}$.
We can obtain a closed form for $\U_{t+1,j}\U_{t+1,j}^{\top} - \U_j\U_j^\top$ from Section \ref{sec:assymetric} as follows
$$\U_{t+1,j}\U_{t+1,j}^{\top} - \U_j\U_j^\top = \sum_{k \geq 1}\S_{k}.$$
Here $\S_k$ depends on $\U,\bSigma,\V,\Z$ defined as follows (see Section \ref{sec:assymetric} for more details):
suppose $\bcalC_{(j)}$ admits a compact SVD as $\bcalC_{(j)} = \Q_1\bSigma\Q_2^\top$ with $\Q_1\in\RR^{r_j\times r_j}$ and $\Q_2\in\RR^{r_j^-\times r_j}$, and $\U = \U_j\Q_1$, $\V = (\otimes_{l\neq j}\U_l)\Q_2$ and $\Z = -\eta(\calP_{\TT_t}\bcalG_t)_{(j)}$.
Here for each $\S_k$, it can be written as 
$$\S_k = \sum_{\s: s_1+\cdots + s_{k+1} = k}\S_k(\s), $$
where $\S_k(\s)$ takes the following form:
\begin{align}\label{binary:representation}
	\A_1\bSigma^{-s_1}\B_1\bSigma^{-s_2}\B_2\bSigma^{-s_3}\cdots\B_k\bSigma^{-s_{k+1}}\A_2^\top,
\end{align}
where $\A_1,\A_2\in \{\U,\U_{\perp}\}$ and 
$\B_i \in \{\U^\top\Z\V, \U^\top_{\perp}\Z\V,\U^\top\Z\V_{\perp},\U_{\perp}^\top\Z\V_{\perp} \text{ or their transpose}\}.$
Under the event $\calE_t$, as long as $\eta\frac{L_{\alpha}^2}{\gamma_{\alpha}}\dof\lesssim \lambda_{\submin}^2$, we have $\op{\bSigma^{-1}}\geq 2\lambda_{\submin}^{-1}$. 
Since now $\bcalG_t = h_{\theta}(\inp{\bcalX_t}{\bcalT_t},Y_t)\bcalX_t$, and $|h_{\theta}(\inp{\bcalX_t}{\bcalT_t},Y_t)|\leq L_{\alpha}$, we have 
\begin{align*}
	\op{\U^\top\Z\V} = \eta| h_{\theta}(\inp{\bcalX_t}{\bcalT_t},Y_t)|\cdot\op{\U_j^\top\bcalX_{(j)}(\otimes_{l\neq j}\U_l)}\leq \eta L_{\alpha} \sqrt{\mu_0^mr^*},\\
	\op{\U^\top_{\perp}\Z\V} \leq \eta L_{\alpha}\sqrt{\mu_0^{m-1}d_jr_j^-},\quad\op{\U^\top\Z\V_{\perp}} \leq \eta L_{\alpha}\sqrt{\mu_0^{m-1}\sum_{l\neq j}d_lr_l^-}.
\end{align*}
Now we bound $\e_i^\top\S_k\e_i$.

\hspace{1cm}

\noindent{\textit{Case 1: $k = 1$.}} Recall $\Z = -\eta (\calP_{\TT_t}\bcalG_t)_{(j)} = -\eta h_{\theta}(\inp{\bcalX_t}{\bcalT_t},Y_t)(\calP_{\TT_t}\bcalX_t)_{(j)}$. Unfortunately, we are unable to write the expectation $\EE_t\Z$ as before. As a remedy, we notice for the logistic link \blue{$f(\theta) = (1+e^{-\theta})^{-1}$}, 
\begin{align*}
	h_{\theta}(\theta,Y_t) = -Y_t\frac{f'(\theta)}{f(\theta)} + (1-Y_t)\frac{f'(\theta)}{1- f(\theta)} = (f(\theta)-Y_t)\frac{f'(\theta)}{f(\theta)(1-f(\theta))} = f(\theta)-Y_t.
\end{align*}
Also notice that $\EE_t\bcalG_t^* = \EE_{\bcalX_t}\EE_{Y_t}\bcalG_t^* = \EE_{\bcalX_t}\EE_{Y_t}(f(\inp{\bcalX_t}{\bcalT^*}) - Y_t)\bcalX_t = 0$ since $Y_t\sim\text{Ber}(f(\inp{\bcalX_t}{\bcalT^*}))$. Therefore,
\begin{align*}
	\EE_t\Z &= -\eta\EE_t\calP_{\TT_t}(\bcalG_t-\bcalG_t^*)_{(j)}  = -\eta\EE_t\bigg(f(\inp{\bcalX_t}{\bcalT_t}) - f(\inp{\bcalX_t}{\bcalT^*})\bigg)(\calP_{\TT_t}\bcalX_t)_{(j)} \\
	&= -\eta(\frac{1}{4}\inp{\bcalX_t}{\bcalT_t} - \frac{1}{4}\inp{\bcalX_t}{\bcalT^*} + \zeta)(\calP_{\TT_t}\bcalX_t)_{(j)} ,
\end{align*}
where $\zeta$ is some number satisfying $|\zeta|\leq \frac{1}{12\sqrt{3}}d^*(\linf{\bcalT_t}^2 + \linf{\bcalT^*}^2) \leq \frac{1}{2\sqrt{3}}\nu_0^2\rmin\lambda_{\submax}^2$ as long as $\eta\frac{L_{\alpha}^2}{\gamma_{\alpha}}\dof\leq \fro{\bcalT^*}^2$.
Now we are able to consider the following expectation, 
\begin{align*}
	\EE_t\e_i^\top\P_{\U}^\perp\Z\V\bSigma^{-1}\U^\top\e_i = -\eta\e_i^\top\P_{\U}^\perp\bigg(\calP_{\TT_t}(\bcalT_t-\bcalT^*)_{(j)}  + \EE_t\zeta(\calP_{\TT_t}\bcalX_t)_{(j)}\bigg)\V\bSigma^{-1}\U^\top\e_i.
\end{align*}
We consider $-\eta\e_i^\top\P_{\U}^\perp \EE_t\zeta(\calP_{\TT_t}\bcalX_t)_{(j)}\V\bSigma^{-1}\U^\top\e_i$, using \eqref{ptx:junfolding},
\begin{align*}
	|\eta\e_i^\top\P_{\U}^\perp \EE_t\zeta(\calP_{\TT_t}\bcalX_t)_{(j)}\V\bSigma^{-1}\U^\top\e_i| &\leq \frac{1}{2\sqrt{3}}\eta\nu_0^2\rmin\lambda_{\submax}^2\sqrt{d^*}\sqrt{\frac{\mu_0^{m-1}r_j^-}{d_j^-}}2\lambda_{\submin}^{-1}\sqrt{\frac{\mu_0r_j}{d_j}}\\
	&= \frac{1}{\sqrt{3}}\eta\nu_0^2\kappa_0\mu_0^{m/2} \rmin \sqrt{r^*}\lambda_{\submax}.
\end{align*}
Now as long as $\nu_0^2\kappa_0\mu_0^{m/2} \rmin \sqrt{r^*}\lambda_{\submax} \leq \frac{\mu_0r_j}{d_j}$, we have similarly with \eqref{k=1},
\begin{align}\label{binary:k=1}
		\EE_t\e_i^\top\S_1\e_i \leq \frac{\eta}{5}\frac{\mu_0r_j}{d_j}  - \frac{\eta}{2}\e_i^\top\U_j\U_j^\top\e_i.
\end{align}

\hspace{1cm}

\noindent{\textit{Case 2: $k = 2$.}} We discuss according to different $\s$. If $\s = (0,0,2)\text{~or~} (2,0,0)$, then $\S_2(\s) = 0$. If $\s = (1,1,0)\text{~or~}(0,1,1)$, then $\B_1$ or $\B_2 = \U^\top\Z\V$ or its transpose, as a result
\begin{align*}
	|\e_i^\top\S_2(\s)\e_i| &\leq \sqrt{\mu_0r_j/d_j}\eta^2L_{\alpha}^2\sqrt{\mu_0^mr^*}\sqrt{\mu_0^{m-1}\sum_{l\neq j}d_lr_l^-}\sigma_{\submin}^{-2}(\bcalT_t)\\
	&\lesssim_m \eta^2L_{\alpha}^2\mu_0^mr^*\sqrt{\frac{r_j\dmax}{\rmin d_j}}\lambda_{\submin}^{-2}.
\end{align*}
If $\s = (1,0,1)$, then 
\begin{align*}
	|\e_i^\top\S_2(\s)\e_i| &\leq \frac{\mu_0r_j}{d_j}\eta^2L_{\alpha}^2\mu_0^{m-1}\sum_{l\neq j}d_lr_l^-\sigma_{\submin}^{-2}(\bcalT_t)
	\lesssim_m \eta^2L_{\alpha}^2\mu_0^m r^*\frac{r_j\dmax}{\rmin d_j}\lambda_{\submin}^{-2}.
\end{align*}
For the case $\s = (0,2,0)$, the closed form is given in \eqref{S21},
we keep it and compute its conditional expectation 
\begin{align}\label{binary:k=2}
	&\quad\EE_t \e_i^\top\S_{2,1}\e_i = \EE_t\e_i^\top\P_{\U}^\perp\Z\V\bSigma^{-2}\V^\top\Z^\top \P_{\U}^\perp\e_i\notag\\
	&\leq\eta^2L_{\alpha}^2 \EE_t\e_i^\top \P_{\U}^{\perp}\bcalX_{(j)}(\otimes_{l\neq j}\U_l)\bcalC_{(j)}^\dagger\bcalC_{(j)}(\otimes_{l\neq j}\U_l)^\top\V\bSigma^{-2}\V^\top\cdot\notag\\
	&\hspace{5cm}(\otimes_{l\neq j}\U_l)\bcalC_{(j)}^\top(\bcalC_{(j)}^\dagger)^\top(\otimes_{l\neq j}\U_l)^\top\bcalX_{(j)}^\top\P_{\U}^{\perp}\e_i\notag\\
	&\leq \eta^2L_{\alpha}^2\cdot\sum_{p\in[d_j],q\in[d_j^-]}|\e_i^\top\P_{\U}^{\perp}\e_p|^2 \ltwo{\e_q^\top(\otimes_{l\neq j}\U_l)\bcalC_{(j)}^\dagger\bcalC_{(j)}(\otimes_{l\neq j}\U_l)^\top}^2\sigma_{\submin}^{-2}(\bcalT_t)\notag\\
	&\leq 4\eta^2L_{\alpha}^2r_j\lambda_{\submin}^{-2}.
\end{align}

\hspace{1cm}

\noindent{\textit{Case 3: $k \geq 3$ and $k$ is odd.}} For each legal $\s$, if at least one of $\A_1,\A_2$ takes $\U$, then 
$$|\e_i^\top\S_k(\s)\e_i| \leq \sqrt{\mu_0r_j/d_j}(\eta L_{\alpha})^k(\mu_0^{m-1}\sum_{l\neq j}d_lr_l^-)^{k/2}\sigma_{\submin}^{-k}(\bcalT_t).$$
If both $\A_1,\A_2$ take $\U_{\perp}$, there exists at least one $l\in[k]$, $\B_l = \U^\top\Z\V \text{~or~} \V^\top\Z^\top\U$. Then 
$$|\e_i^\top\S_k(\s)\e_i| \leq (\eta L_{\alpha})^k(\mu_0^mr^*)^{1/2}(\mu_0^{m-1}\sum_{l\neq j}d_lr_l^-)^{(k-1)/2}\sigma_{\submin}^{-k}(\bcalT_t).$$
Since there are at most $\binom{2k}{k}\leq 4^k$ legal $\s$, 
$$|\e_i^\top\S_k\e_i| \leq 4^k\sqrt{\mu_0r_j/d_j}(\eta L_{\alpha})^k(\mu_0^{m-1}\sum_{l\neq j}d_lr_l^-)^{k/2}\sigma_{\submin}^{-k}(\bcalT_t).$$
Now as long as $\eta^2L_{\alpha}^2\mu_0^{m-1}\dmax r^*\rmin^{-1}\lambda_{\submin}^{-2}\lesssim_m 1$, we have 
$4^2(\eta L_{\alpha})^2(\mu_0^{m-1}\sum_{l\neq j}d_lr_l^-)\sigma_{\submin}^{-2}(\bcalT_t)\leq 1/2$.
Therefore the contribution for such $k$ is bounded by
\begin{align}\label{binary:kgeq3odd}
	\sum_{k\geq 3, k \text{ is odd}}\e_i^\top\S_k\e_i &\lesssim_m (\eta L_{\alpha}\lambda_{\submin}^{-1})^3(\mu_0^m r^*)^{1/2}\mu_0^{m-1}\dmax r^*\rmin^{-1}.
\end{align}

\hspace{1cm}

\noindent{\textit{Case 4: $k \geq 4$ and $k$ is even.}} We shall apply trivial bound to $\e_i^\top\S_k\e_i$.
\begin{align*}
	|\e_i^\top\S_k(\s)\e_i| \leq (\eta L_{\alpha})^k(\mu_0^{m-1}\sum_{l\neq j}d_lr_l^-)^{k/2}\sigma_{\submin}^{-k}(\bcalT_t).
\end{align*}
Since there are $\binom{2k}{k}\leq 4^k$ legal $\s$, 
$$|\e_i^\top\S_k\e_i| \leq 4^k(\eta L_{\alpha})^k(\mu_0^{m-1}\sum_{l\neq j}d_lr_l^-)^{k/2}\sigma_{\submin}^{-k}(\bcalT_t).$$
And the contribution for such $k$ is bounded by
\begin{align}\label{binary:kgeq4even}
	\sum_{k\geq 4, k \text{ is even}}\e_i^\top\S_k\e_i \lesssim_m  (\eta L_{\alpha}\lambda_{\submin}^{-1})^4 (\mu_0^{m-1}\dmax r^*\rmin^{-1})^2.
\end{align}
Now from \eqref{binary:k=1} - \eqref{binary:kgeq4even}, we see if $\dmax/\dmin\leq C,\rmax/\rmin\leq C$, 
\begin{align*}
	&\e_i^\top\U_{t+1,j}\U_{t+1,j}^\top\e_i \leq (1-\frac{1}{2}\eta)\e_i^\top\U_{t,j}\U_{t,j}^\top\e_i + (\e_i^\top\S_1\e_i - \EE_t\e_i^\top\S_1\e_i) + \frac{\eta}{5}\frac{\mu_0r_j}{d_j}\\
	&+ \underbrace{C_m\eta^2L_{\alpha}^2\mu_0^m\sqrt{\frac{r_j}{d_j}r^*\frac{\dmax r^*}{\rmin}}\lambda_{\submin}^{-2} +C_m\eta^2L_{\alpha}^2\mu_0^m\frac{r_j}{d_j}\frac{\dmax r^*}{\rmin}\lambda_{\submin}^{-2} +(\e_i^\top\S_{2,1}\e_i - \EE_t\e_i^\top\S_{2,1}\e_i)}_{k = 2}\\
	&+\underbrace{C_m(\eta L_{\alpha}\lambda_{\submin}^{-1})^3(\mu_0^m r^*)^{1/2}\mu_0^{m-1}\dmax r^*\rmin^{-1}}_{k\geq 3 \text{ and } k \text{ is odd}} + \underbrace{C_m(\eta L_{\alpha}\lambda_{\submin}^{-1})^4 (\mu_0^{m-1}\dmax r^*\rmin^{-1})^2}_{k\geq 4 \text{ and } k \text{ is even}}\\
	&\leq (1-\frac{1}{2}\eta)\e_i^\top\U_{t,j}\U_{t,j}^\top\e_i + \frac{2\eta}{5}\frac{\mu_0r_j}{d_j} + (\e_i^\top\S_1\e_i - \EE_t\e_i^\top\S_1\e_i) +(\e_i^\top\S_{2,1}\e_i - \EE_t\e_i^\top\S_{2,1}\e_i),
\end{align*}
as long as $\eta\dmax r_j^- (\frac{L_{\alpha}}{\lambda_{\submin}})^2\mu_0^{m-1}\lesssim_m 1$, $\eta^2\dmax^2 (r^*)^{3/2}\rmin^{-2} (\frac{L_{\alpha}}{\lambda_{\submin}})^3\mu_0^{3m/2-2}\lesssim_m 1$, $\eta^3\dmax^3 (r^*)^{2}\rmin^{-1} (\frac{L_{\alpha}}{\lambda_{\submin}})^4\mu_0^{2m-3}\lesssim_m 1$.
 
Now telescoping this inequality and we get 
\begin{align}\label{binary:incoherence}
	\e_i^\top\U_{t+1,j}\U_{t+1,j}^\top\e_i &\leq (1-\frac{1}{2}\eta)^{t+1}\e_i^\top\U_{0,j}\U_{0,j}^\top\e_i + \frac{4}{5}\frac{\mu_0r_j}{d_j}
	\notag\\
	&\quad + \sum_{l=0}^t\underbrace{(1-\frac{1}{2}\eta)^{t-l}\big[(\e_i^\top\S_{1,l}\e_i - \EE_t\e_i^\top\S_{1,l}\e_i) +(\e_i^\top\S_{2,1,l}\e_i - \EE_t\e_i^\top\S_{2,1,l}\e_i)\big]}_{=:D_l}
\end{align}
Now we use martingale concentration inequality to bound $|\sum_{l=0}^t D_l|$. We first consider the uniform bound. Notice 
\begin{align*}
	|\e_i^\top\S_{1,l}\e_i| &\leq 2\eta L_{\alpha}\cdot|\e_i^\top\P_{\U}^\perp(\bcalP_{\TT_l}\bcalX_l)_{(j)}\V\bSigma^{-1}\U^\top\e_i|\\
	&\leq 2\eta L_{\alpha}\sqrt{d^*}\sqrt{\frac{\mu_0^mr^*}{d^*}}\sigma_{\submin}^{-1}(\bcalT_l)
	\leq 4\eta L_{\alpha}\sqrt{\mu_0^mr^*}\lambda_{\submin}^{-1}.
\end{align*}
Using $|X-\EE X|\leq 2B$ when $|X|\leq B$, we get $$|\e_i^\top\S_{1,l}\e_i - \EE_l\e_i^\top\S_{1,l}\e_i|\leq 8\eta L_{\alpha}\sqrt{\mu_0^mr^*}\lambda_{\submin}^{-1}.$$
On the other hand, 
\begin{align*}
	|\e_i^\top\S_{2,1,l}\e_i| &= |\e_i^\top \P_{\U}^\perp\Z\V\bSigma^{-2}\V^\top\Z^\top \P_{\U}^\perp\e_i|\leq \eta^2L_{\alpha}^2 \sigma_{\submin}^{-2}(\bcalT_l)\mu_0^{m-1}r_j^-d_j\\
	&\leq 4\eta^2L_{\alpha}^2 \lambda_{\submin}^{-2}\mu_0^{m-1}r_j^-d_j.
\end{align*}
Therefore $|\e_i^\top\S_{2,1,l}\e_i- \EE_l \e_i^\top\S_{2,1,l}\e_i| \leq 8\eta^2L_{\alpha}^2 \lambda_{\submin}^{-2}\mu_0^{m-1}r_j^-d_j.$
So we have the uniform bound for $D_l$ as follows,
\begin{align}\label{binary:uniform:D}
|D_l|\leq 8\eta L_{\alpha}\sqrt{\mu_0^mr^*}\lambda_{\submin}^{-1} + 8\eta^2L_{\alpha}^2 \lambda_{\submin}^{-2}\mu_0^{m-1}r_j^-d_j.
\end{align}
We also need to bound the variance. In fact, 
\begin{align*}
	\EE_l D_l^2 \leq 2(1-\frac{1}{2}\eta)^{2t-2l}[\EE_l|\e_i^\top\S_{1,l}\e_i|^2+\EE_l|\e_i^\top\S_{2,1,l}\e_i|^2].
\end{align*}
And 
\begin{align*}
	\EE_l|\e_i^\top\S_{1,l}\e_i|^2 &\leq 4\eta^2L_{\alpha}^2\EE_l(\e_i^\top\P_{\U}^\perp(\bcalP_{\TT_l}\bcalX_l)_{(j)}\V\bSigma^{-1}\U^\top\e_i)^2\\
	&= 4\eta^2L_{\alpha}^2
	\sum_{p,q}(\e_i^\top\P_{\U}^\perp\e_p)^2(\e_q^\top(\otimes_{k\neq j}\U_k)\bcalC_{(j)}^\dagger\bcalC_{(j)}(\otimes_{k\neq j}\U_k)^\top\V\bSigma^{-1}\U^\top\e_i)^2\\
	&\lesssim \eta^2L_{\alpha}^2\lambda_{\submin}^{-2}\frac{\mu_0^mr^*}{d_j}.
\end{align*}
Also,
\begin{align*}
	\EE_l|\e_i^\top\S_{2,1,l}\e_i|^2 &= \EE_l |\e_i^\top\P_{\U}^\perp\Z\V\bSigma^{-2}\V^\top\Z^\top \P_{\U}^\perp\e_i|^2\\
	&\leq \EE_l\eta^4L_{\alpha}^4(|\e_i^\top\P_{\U}^\perp(\bcalP_{\TT_l}\bcalX_l)_{(j)}\V\bSigma^{-2}\V^\top(\bcalP_{\TT_l}\bcalX_l)_{(j)}^\top \P_{\U}^\perp\e_i|^2)\\
	&\lesssim \eta^4L_{\alpha}^4d^*\sum_p(\e_i^\top\P_{\U}^\perp\e_p)^4\fro{(\otimes_{k\neq j}\U_k)\bcalC_{(j)}^\dagger\bcalC_{(j)}(\otimes_{k\neq j}\U_k)^\top\V\bSigma^{-1}}^2\cdot\frac{\mu_0^{m-1}r_j^-}{d_j^-}\lambda_{\submin}^{-2}\\
	&\lesssim \eta^4L_{\alpha}^4d^* r_j\lambda_{\submin}^{-4}\frac{\mu_0^{m-1}r_j^-}{d_j^-}.
\end{align*}
Therefore as long as $\eta^2\dmax^2 L_{\alpha}^2\lambda_{\submin}^{-2}\mu_0^{-1}\lesssim 1$,
\begin{align*}
	\EE_l D_l^2 \lesssim (1-\frac{1}{2}\eta)^{2t-2l} \eta^2L_{\alpha}^2\lambda_{\submin}^{-2}\frac{\mu_0^mr^*}{d_j}.
\end{align*}
And 
\begin{align}\label{binary:variance:D}
	\sum_{l=0}^t \EE_l D_l^2 \lesssim \eta L_{\alpha}^2\lambda_{\submin}^{-2}\frac{\mu_0^mr^*}{d_j}.
\end{align}
From \eqref{binary:uniform:D} and \eqref{binary:variance:D}, and as a result of Azuma-Bernstein inequality, we see that with probability exceeding $\dmax^{-12}$, as long as $\eta d_j\log\dmax\lesssim 1$, $\eta^2(\frac{L_{\alpha}}{\lambda_{\submin}})^2\mu_0^{m-2}d_jr_j^-\log\dmax\lesssim \frac{\mu_0r_j}{d_j}$,
\begin{align}
	|\sum_{l=0}^tD_l| \lesssim \sqrt{\eta L_{\alpha}^2\lambda_{\submin}^{-2}\frac{\mu_0^mr^*}{d_j}\log d}\lesssim \frac{\mu_0 r_j}{d_j},
\end{align}
where the last inequality holds as long as $\eta\dmax r^*\rmin^{-2} L_{\alpha}^2\lambda_{\submin}^{-2}\mu_0^{m-2}\log\dmax \lesssim 1$. Now we go back to \eqref{binary:incoherence}, and we see as long as $\e_i^\top\U_{0,j}\U_{0,j}^\top\e_i\leq \frac{1}{10} \frac{\mu_0r_j}{d_j}$, we have 
\begin{align*}
	\e_i^\top\U_{t+1,j}\U_{t+1,j}^\top\e_i \leq  \frac{\mu_0r_j}{d_j}.
\end{align*}
Taking a union bound over all $i,j$ and we see that with probability exceeding $1-m\dmax^{-11}$, 
$$\ltwoinf{\U_{t+1,j}}^2\leq \frac{\mu_0r_j}{d_j}, \forall j\in[m].$$

\hspace{1cm}

\noindent\textit{Step 3: Controlling the probability.} From Step 1 and Step 2, we have proved under the event $\calE_t$, $\calE_{t+1}$ holds with probability exceeding $1-(4+m\dmax^{-1})\dmax^{-10} \geq 1-5\dmax^{-10}$. Then a similar argument as in Step 3 in the proof of Theorem \ref{thm:gen} shows  
$$\PP(\calE_T^c) = \sum_{t= 1}^T\PP(\calE_{t-1}\cap\calE_t^c)\leq 5Td^{-10}.$$

\subsection{Proof of Theorem~\ref{thm:init-binary}}
	Using Theorem 1 in \cite{davenport20141}, we have
	\begin{align*}
		\frac{1}{d^*}\fro{\hat\M - \M^*}^2 \leq C_{1,\alpha}\frac{1}{d^*}\sqrt{\frac{R(D_1+D_2)}{T}}\sqrt{1+ \frac{(D_1+D_2)\log(D_1D_2)}{T}},
	\end{align*}
	where $C_{1,\alpha} = C\alpha L_{\alpha}\gamma_{\alpha}^{-1}$. And from Lemma \ref{lemma:perturbation}, we have
	\begin{align*}
		\fro{\hat\bcalT - \bcalT^*} \leq C_{1,\alpha}^{1/2}\left(\frac{R(D_1+D_2)}{T}\right)^{1/4}\left(1+ \frac{(D_1+D_2)\log(D_1D_2)}{T}\right)^{1/4}\leq c_m\gamma_{\alpha}\lambda_{\submin}
	\end{align*}
	under the given condition.

\subsection{Proof of Theorem \ref{thm:regret}}
\begin{proof}
	The proof is mostly the same as the ones in Theorem \ref{thm:gen} and Theorem \ref{thm:localrefinement:completion}. 
	We shall consider online tensor linear regression. The proof for online tensor completion is similar and is thus omitted.
	
	\noindent\textit{Regret bound for online tensor linear regression. } 
	At time $t$, we denote the event 
	\begin{align*}
		&\tilde\calE_t = \bigg\{\forall 0\leq l\leq t, \fro{\bcalT_l - \bcalT^*}^2\leq 2\fro{\bcalT_0 - \bcalT^*}^2 + C\eta\dof\sigma^2,\\
		&\hspace{1.5cm}\fro{\calP_{\TT_l}(\bcalX_l)}\leq C(\dof\cdot \log T)^{1/2},
		|\inp{\bcalX_l}{\calP_{\TT_l}(\bcalT_l-\bcalT^*)}|\leq C\fro{\calP_{\TT_l}(\bcalT_l-\bcalT^*)}\log^{1/2}T,\\ 
		&\hspace{1.5cm}|\inp{\bcalT^*}{\bcalX_l}| \leq C\fro{\bcalT^*}\log^{1/2}T,
		\quad|\inp{\bcalT_l}{\bcalX_l}| \leq C\fro{\bcalT_l}\log^{1/2}T,\\
		&\hspace{8cm}|\inp{\bcalX_l}{\bcalT_l-\bcalT^*}|\leq C\fro{\bcalT_l-\bcalT^*}\log^{1/2} T\bigg\},
	\end{align*}
	and
	\begin{align*}
		\tilde\calY_t &=\bigg\{\fro{\bcalT_t - \bcalT^*}^2\leq 2\fro{\bcalT_0 - \bcalT^*}^2 + C\eta\dof\sigma^2,\notag\\
		&\hspace{1.5cm}\fro{\calP_{\TT_t}(\bcalX_t)}\leq C(\dof\cdot \log T)^{1/2},
		|\inp{\bcalX_t}{\calP_{\TT_t}(\bcalT_t-\bcalT^*)}|\leq C\fro{\calP_{\TT_t}(\bcalT_t-\bcalT^*)}\log^{1/2}T,\notag\\ 
		&\hspace{1.5cm}|\inp{\bcalT^*}{\bcalX_t}| \leq C\fro{\bcalT^*}\log^{1/2}T,
		\quad|\inp{\bcalT_t}{\bcalX_t}| \leq C\fro{\bcalT_t}\log^{1/2}T,\notag\\
		&\hspace{8cm}|\inp{\bcalX_t}{\bcalT_t-\bcalT^*}|\leq C\fro{\bcalT_t-\bcalT^*}\log^{1/2} T\bigg\}.
	\end{align*}
	Notice these events are slightly different with the one defined in the proof of Theorem \ref{thm:gen} in that we replace $\log \dmax$ with $\log T$. Notice under the assumption $T\leq \dmax^{C_0m}$ for some large but absolute constant $C_0>0$, the proof of Theorem \ref{thm:gen} is still valid. 
	Therefore, $\PP(\tilde\calE_t)\geq 1- 14tT^{-10}\geq 1-14T^{-9}$ from the previous proof. 
	Notice $\tilde\calE_t\subset \tilde\calY_t$.
	Moreover, using law of total probability, 
	\begin{align*}
		\PP(\tilde\calY_t^c) &= \PP(\tilde\calY_t^c | \tilde\calE_{t-1}) \PP(\tilde\calE_{t-1}) +  \PP(\tilde\calY_t^c | \tilde\calE_{t-1}^c) \PP(\tilde\calE_{t-1}^c)\\
		&\leq \PP(\tilde\calY_t^c \cap \tilde\calE_{t-1})+   \PP(\tilde\calE_{t-1}^c)\\
		&\leq \PP(\tilde\calE_t^c \cap \tilde\calE_{t-1})+   \PP(\tilde\calE_{t-1}^c)\\
		&\leq  14T^{-10}+ 14T^{-9} \leq 15T^{-9}. 
	\end{align*}
	where the last inequality holds from Step 3 in the proof of Theorem \ref{thm:gen}.

	We split $\EE_{\leq t} \fro{\bcalT_{t+1} - \bcalT^*}^2$ into two terms,
	\begin{align}\label{regret:split}
		\EE_{\leq t} \fro{\bcalT_{t+1} - \bcalT^*}^2 = \underbrace{\EE_{\leq t}\fro{\bcalT_{t+1} - \bcalT^*}^2\cdot\mathds{1}(\tilde\calE_t)}_{\textsf{A}_t} + \underbrace{\EE_{\leq t} \fro{\bcalT_{t+1} - \bcalT^*}^2\cdot \mathds{1}(\tilde\calE_t^c)}_{\textsf{B}_t}.
	\end{align}
	
	\noindent\textit{Controlling $\sfA_t$. } 
	Notice $\EE_{\leq t}\fro{\bcalT_{t+1} - \bcalT^*}^2\cdot\mathds{1}(\tilde\calE_t) = \EE_{\leq t-1}\mathds{1}(\tilde\calE_{t-1})\EE_t\fro{\bcalT_{t+1} - \bcalT^*}^2\cdot\mathds{1}(\tilde\calY_t)$. 
	We first control $\EE_t \fro{\bcalT_{t+1} - \bcalT^*}^2\cdot\mathds{1}(\tilde\calY_{t})$. Using Cauchy-Schwarz inequality, we obtain,
	\begin{align*}
		\fro{\bcalT_{t+1} - \bcalT^*}^2 &\leq (1+\eta^{-1})\fro{\hosvd(\bcalT_t^+) - \bcalT_t^+}^2 + (1+\eta) \fro{\bcalT_t^+ - \bcalT^*}^2.
	\end{align*}
	And 
	\begin{align}\label{eq:eptg}
		\EE_t \fro{\calP_{\TT_{t}}\bcalG_t}^4 &= \EE_t(\inp{\bcalT_t - \bcalT^*}{\bcalX_t} - \epsilon_t)^4 \fro{\calP_{\TT_{t}}\bcalX_t}^4\notag\\
		&\leq \big(\EE_t(\inp{\bcalT_t - \bcalT^*}{\bcalX_t} - \epsilon_t)^8\EE_t \fro{\calP_{\TT_{t}}\bcalX_t}^8\big)^{1/2}\notag\\
		&\leq C(\fro{\bcalT_t -\bcalT^*}^4 +\sigma^4)\dof^2.
	\end{align}
	Under $\tilde\calY_t$, we can apply Lemma \ref{lemma:perturbation}, and we obtain 
	\begin{align*}
		\EE_t\fro{\hosvd(\bcalT_t^+) - \bcalT_t^+}^2\cdot\mathds{1}(\tilde\calY_{t}) &\leq C_m\EE_t\frac{\eta^4\fro{\calP_{\TT_{t}}\bcalG_t}^4\cdot\mathds{1}(\tilde\calY_{t})}{\lambda_{\submin}^2}\\
		&\leq C_m\eta^4(\fro{\bcalT_t -\bcalT^*}^4 +\sigma^4)\dof^2\lambda_{\submin}^{-2}\cdot\mathds{1}(\tilde\calY_{t})\\
		&\leq C_m\eta^4\fro{\bcalT_t -\bcalT^*}^2\dof^2 + C_m\eta^4\sigma^4\dof^2\lambda_{\submin}^{-2}.
	\end{align*}
	From \eqref{lowerbound:inp}, we have 
	\begin{align*}
		\EE_t\inp{\calP_{\TT_t}\bcalG_t}{\bcalT_t - \bcalT^*}\cdot\mathds{1}(\tilde\calY_{t}) \geq \frac{3}{4}\fro{\bcalT_t - \bcalT^*}^2. 
	\end{align*}
	And from Lemma \ref{lemma:psioneofptx}, we obtain
	\begin{align*}
		\EE_t\fro{\calP_{\TT_t}\bcalG_t}^2\cdot\mathds{1}(\tilde\calY_{t}) \leq (3\fro{\bcalT_t - \bcalT^*}^2  + \sigma^2)\dof.
	\end{align*}
	These together give, 
	\begin{align*}
		\EE_t  \fro{\bcalT_t^+ - \bcalT^*}^2\cdot\mathds{1}(\tilde\calY_{t})&= \fro{\bcalT_t - \bcalT^*}^2 -2\eta\EE_t\inp{\calP_{\TT_t}\bcalG_t}{\bcalT_t - \bcalT^*}\cdot\mathds{1}(\tilde\calY_{t}) + \eta^2\EE_t\fro{\calP_{\TT_t}\bcalG_t}^2\cdot\mathds{1}(\tilde\calY_{t})\\
		&\leq (1-\frac{3}{2}\eta+3\eta^2)\fro{\bcalT_t - \bcalT^*}^2+ \eta^2\sigma^2 \dof. 
	\end{align*}
	Therefore for $\eta$ that satisfies the conditions in Theorem \ref{thm:regression},
	\begin{align}\label{exp:bound:T-T}
		\EE_t\fro{\bcalT_{t+1} - \bcalT^*}^2 \cdot\mathds{1}(\tilde\calY_{t}) &\leq C_m\eta^3\fro{\bcalT_t -\bcalT^*}^2\dof^2 + C_m\eta^3\sigma^4\dof^2\lambda_{\submin}^{-2}\dof^2\lambda_{\submin}^{-2} \notag\\
		&\quad + (1-\frac{1}{2}\eta)\fro{\bcalT_t - \bcalT^*}^2+ 2\eta^2\sigma^2 \dof\notag\\
		&\leq (1-\frac{\eta}{4})\fro{\bcalT_t - \bcalT^*}^2 + 3\eta^2\sigma^2\dof. 
	\end{align}
	Notice this is valid for $t\geq 0$. Therefore, 
	\begin{align}
		\sfA_t &\leq (1-\frac{\eta}{4})\sfA_{t-1} + 3\eta^2\sigma^2\dof \leq (1-\frac{\eta}{4})^{t}\fro{\bcalT^* - \bcalT_0}^2 + 12\eta\dof\sigma^2.\label{bound:sfA1}
	\end{align}
	Therefore 
	\begin{align}\label{regret:A}
		\sum_{t=0}^{T-1}\sfA_t \leq C(\eta^{-1}\fro{\bcalT_0 -\bcalT^*}^2 + \eta T\dof \sigma^2). 
	\end{align}
	
	\noindent\textit{Controlling $\sfB_t$. } 
	Since $\tilde\calE_{t} = \tilde\calE_{t-1}\cap \tilde\calY_{t}$, we have $\tilde\calE_{t}^c = \tilde\calE_{t-1}^c\cup \tilde\calY_{t}^c$. And therefore, 
	\begin{align*}
		\mathds{1}(\tilde\calE_{t}^c) = \mathds{1}(\tilde\calE_{t-1}^c\cap \tilde\calY_{t}) + \mathds{1}(\tilde\calE_{t-1}\cap \tilde\calY_{t}^c) + \mathds{1}(\tilde\calE_{t-1}^c\cap \tilde\calY_{t}^c).
	\end{align*}
	So we can further split $\textsf{B}_t$ into
	\begin{align*}
		\textsf{B}_t= \underbrace{\EE_{\leq t} \fro{\bcalT_{t+1} - \bcalT^*}^2\cdot\mathds{1}(\tilde\calE_{t-1}^c\cap \tilde\calY_{t})}_{\textsf{B}_{t,1}}+ \underbrace{\EE_{\leq t} \fro{\bcalT_{t+1} - \bcalT^*}^2\cdot\mathds{1}(\tilde\calE_{t-1}\cap \tilde\calY_{t}^c)}_{\textsf{B}_{t,2}} + 
		\underbrace{\EE_{\leq t} \fro{\bcalT_{t+1} - \bcalT^*}^2\cdot\mathds{1}(\tilde\calE_{t-1}^c\cap \tilde\calY_{t}^c)}_{\textsf{B}_{t,3}}.
	\end{align*}
	From the tower property and $\mathds{1}(\calF_1\cap\calF_2) = \mathds{1}(\calF_1)\mathds{1}(\calF_2)$ for any events $\calF_1,\calF_2$, we have 
	\begin{align*}
		\sfB_{t,1} = \EE_{\leq t-1}\mathds{1}(\tilde\calE_{t-1}^c) \EE_t\fro{\bcalT_{t+1} - \bcalT^*}^2\cdot\mathds{1}(\tilde\calY_{t}).
	\end{align*}
	From \eqref{exp:bound:T-T}, we have 
	\begin{align}\label{bound:sfB1}
		\sfB_{t,1} \leq (1-\frac{\eta}{4})\sfB_{t-1} + 3\eta^2\sigma^2\dof.
	\end{align}
	For $\sfB_{t,2}, \sfB_{t,3}$, we have 
	\begin{align*}
		\sfB_{t,2} &\leq \EE_{\leq t-1}(\EE_t \fro{\bcalT_{t+1} - \bcalT^*}^4)^{1/2}\cdot\PP^{1/2}(\tilde\calY_{t}^c) \cdot \mathds{1}(\tilde\calE_{t-1}),\\
		\sfB_{t,3} &\leq \EE_{\leq t-1}(\EE_t \fro{\bcalT_{t+1} - \bcalT^*}^4)^{1/2}\cdot\PP^{1/2}(\tilde\calY_{t}^c) \cdot \mathds{1}(\tilde\calE_{t-1}^c).
	\end{align*}
	This motivates us to consider the bound for $\EE_t \fro{\bcalT_{t+1} - \bcalT^*}^4$. 
	We shall use the bound for $\hosvd$ in \cite{de2000multilinear}, which states $\fro{\bcalT - \hosvd(\bcalT)}\leq \sqrt{m} \fro{\bcalT - \calP_{\MM_{\r}}(\bcalT)}$, where $\calP_{\MM_{\r}}(\bcalT)$ is the best rank $\r$ approximation (under Frobenius norm) of $\bcalT$. 
	Using this inequality, we get, 
	\begin{align*}
		\fro{\bcalT_{t+1} - \bcalT^*} &= \fro{\hosvd(\bcalT_t^+) - \bcalT^*} \leq \fro{\hosvd(\bcalT_t^+) - \bcalT_t^+} +\fro{\bcalT_t^+ - \bcalT^*}\\
		&\leq \sqrt{m}\fro{\bcalT - \calP_{\MM_{\r}}(\bcalT)} +\fro{\bcalT_t^+ - \bcalT^*}\\
		&\leq (\sqrt{m} +1)\fro{\bcalT_t^+ - \bcalT^*}\\
		&\leq C_m\big(\fro{\bcalT_t - \bcalT^*}+\eta\fro{\calP_{\TT_{t}}\bcalG_t}\big).
	\end{align*}
	Using $(a+b)^4\leq 8(a^4 + b^4)$, we obtain from Lemma \ref{lemma:psioneofptx},
	\begin{align*}
		\EE_t\fro{\bcalT_{t+1} - \bcalT^*}^4 
		&\leq C_m(\fro{\bcalT_{t}-\bcalT^*}^4 + \eta^4\EE_t\fro{\calP_{\TT_{t}}\bcalG_t}^4). 
	\end{align*}
	Using $C_1\eta\dof\leq 1$ and \eqref{eq:eptg}, we have
	\begin{align}\label{eq:4th}
		(\EE_t\fro{\bcalT_{t+1} - \bcalT^*}^4)^{1/2} \leq C_m(\fro{\bcalT_t -\bcalT^*}^2 + \eta^2\dof\sigma^2).
	\end{align}
	So we conclude 
	\begin{align}
		\sfB_{t,2} &\leq \frac{C_m}{T^4}(\sfA_{t-1} + \eta^2\dof\sigma^2)\label{bound:sfB2}\\
		\sfB_{t,3} &\leq \frac{C_m}{T^4}(\sfB_{t-1} + \eta^2\dof\sigma^2)\label{bound:sfB3}.
	\end{align}
	From \eqref{bound:sfB1}, \eqref{bound:sfB2}, and \eqref{bound:sfB3}, as long as $T^4\geq C_m$,  we obtain
	\begin{align*}
		\sfB_t \leq (1-\frac{\eta}{4}+\frac{C_m}{T^4})\sfB_{t-1} + \frac{C_m}{T^4}\sfA_{t-1}  + C\eta^2\dof\sigma^2.
	\end{align*}
	Notice $\sfB_0 = \EE_{0} \fro{\bcalT_{1} - \bcalT^*}^2\cdot \mathds{1}(\tilde\calE_0^c)$, and we have from \eqref{eq:4th},
	\begin{align*}
		\sfB_0 &\leq \frac{1}{T^4}(\EE_0\fro{\bcalT_{1} - \bcalT^*}^4)^{1/2}\leq \frac{C_m}{T^4}(\fro{\bcalT_0 -\bcalT^*}^2 + \eta^2\dof\sigma^2).
	\end{align*}
	Telescoping the inequality about $\sfB_t,\sfB_{t-1}$ and from \eqref{bound:sfA1}, we get,
	\begin{align*}
		\sfB_t&\leq (1-\frac{\eta}{4}+\frac{C_m}{T^4})^t\sfB_0 +\frac{C_m}{T^4}\sum_{\tau = 0}^{t-1}(1-\frac{\eta}{4}+\frac{C_m}{T^4})^{t-1-\tau}\sfA_{\tau}+C\sum_{\tau = 0}^{t-1}(1-\frac{\eta}{4}+\frac{C_m}{T^4})^{t-1-\tau}\eta^2\dof\sigma^2.
	\end{align*}
	On the one hand, if $\frac{\eta}{8}\geq \frac{C_m}{T^4}$, 
	\begin{align*}
		\sfB_t&\leq (1-\frac{\eta}{8})^t\sfB_0 +\frac{C_m}{T^4}\sum_{\tau = 0}^{t-1}(1-\frac{\eta}{8})^{t-1-\tau}\sfA_{\tau}+C\sum_{\tau = 0}^{t-1}(1-\frac{\eta}{8})^{t-1-\tau}\eta^2\dof\sigma^2\\
		&\leq (1-\frac{\eta}{8})^t\frac{C_m}{T^4}(\fro{\bcalT_0 -\bcalT^*}^2 + \eta^2\dof\sigma^2) + C\eta\dof\sigma^2\\
		&\quad + \frac{C_m}{T^4}\sum_{\tau = 0}^{t-1}(1-\frac{\eta}{8})^{t-1-\tau}\big((1-\frac{\eta}{4})^{\tau+1}\fro{\bcalT_0 - \bcalT^*}^2 + \eta\dof\sigma^2\big)\\
		&\leq 2(1-\frac{\eta}{8})^{t+1}\fro{\bcalT_0 -\bcalT^*}^2 + C\eta\dof\sigma^2,
	\end{align*}
	where the last inequality uses $\frac{\eta}{8}\geq \frac{C_m}{T^4}$. 
	And therefore 
	\begin{align*}
		\sum_{t=0}^{T-1} \sfB_t \leq C\eta^{-1}\fro{\bcalT_0 - \bcalT^*}^2 + C\eta T\dof\sigma^2. 
	\end{align*} 
	On the other hand, if $\frac{\eta}{8}\leq \frac{C_m}{T^4}$, 
	\begin{align*}
		\sfB_t&\leq (1+\frac{C_m}{T^4})^t\sfB_0 +\frac{C_m}{T^4}\sum_{\tau = 0}^{t-1}(1+\frac{C_m}{T^4})^{t-1-\tau}\sfA_{\tau}+C\sum_{\tau = 0}^{t-1}(1+\frac{C_m}{T^4})^{t-1-\tau}\eta^2\dof\sigma^2\\
		&\leq (1+\frac{C_m}{T^4})^t\frac{C_m}{T^4}(\fro{\bcalT_0 -\bcalT^*}^2 + \eta^2\dof\sigma^2) + C\frac{T^4}{C_m}\eta^2\dof\sigma^2\\
		&\quad + \frac{C_m}{T^4}\sum_{\tau = 0}^{t-1}(1+\frac{C_m}{T^4})^{t-1-\tau}\big((1-\frac{\eta}{4})^{\tau+1}\fro{\bcalT^* - \bcalT_0}^2 + 12\eta\dof\sigma^2\big)\\
		&\leq (1+\frac{C_m}{T^4})^t\frac{C_m}{T^4}(\fro{\bcalT_0 -\bcalT^*}^2 + \eta^2\dof\sigma^2) + C\eta\dof\sigma^2+ \frac{C_m}{T^4}(1+\frac{C_m}{T^4})^t\eta^{-1}\fro{\bcalT_0 -\bcalT^*}^2,
	\end{align*}
	where in the last line we use $\frac{\eta}{8}\leq \frac{C_m}{T^4}$.
	And thus 
	\begin{align}
		\sum_{t=0}^{T-1} \sfB_t &\leq \exp(\frac{C_m}{T^3})((1+\eta^{-1})\fro{\bcalT_0 -\bcalT^*}^2 + \eta^2\dof\sigma^2)+ C\eta T\dof\sigma^2\notag\\
		&\leq C(\eta^{-1}\fro{\bcalT_0 -\bcalT^*}^2 + \eta T\dof\sigma^2),
	\end{align}
	where in the last line we use $C_m\leq T^3$. Therefore for either choice of $\eta$, we have 
	\begin{align} \label{regret:B}
		\sum_{t=0}^{T-1} \sfB_t  \leq C(\eta^{-1}\fro{\bcalT_0 -\bcalT^*}^2 + \eta T\dof\sigma^2). 
	\end{align}
	From \eqref{regret:A} and \eqref{regret:B}, we conclude 
	\begin{align*}
		R_T \leq C(\eta^{-1}\fro{\bcalT_0 -\bcalT^*}^2 + \eta T\dof\sigma^2). 
	\end{align*}

\end{proof}

\subsection{Proof of Theorem \ref{thm:adaptive}}
We shall consider online tensor linear regression. The proof for online tensor completion is similar and is thus omitted.
For simplicity, suppose the true but unknown horizon is $T = (2^K - 1)t_0$ for some positive integers $t_0$. 
Denote square of the error after phase $k$ by $E_k$, namely $E_k:= \fro{\bcalT_{(2^{k}-1)t_0}-\bcalT^*}^2$. Then from Theorem \ref{thm:regression}, we have 
\begin{align}
	E_{k+1} &\leq 2(1-\frac{\eta_{k+1}}{8})^{2^kt_0}E_k + C_0\eta_{k+1}\dof\sigma^2\notag\\
	&\leq 2\exp(-\frac{\eta_{k+1}}{8}2^kt_0)E_k + C_0\eta_{k+1}\dof\sigma^2\label{Ekrecursion}\\
	&= 2\exp(-2^{-4}t_0\eta_0) E_k+ C_02^{-k-1}\eta_0\dof\sigma^2. \notag
\end{align}
Suppose $2\exp(-2^{-4}t_0\eta_0)\leq \frac{1}{2}$.
Telescoping this inequality and we get 
\begin{align*}
	E_K &\leq 2^{K}\exp(-2^{-4}\eta_0t_0 K)\fro{\bcalT_0-\bcalT^*}^2 + C_1\eta_0\dof\sigma^2.
\end{align*}

Now we consider the regret under the adaptive setting. We denote the regret occurred at phase $k$ by $R_{(k)}$, and the start time of each phase by $s_k$, namely $s_k = (2^{k-1}-1)t_0$. 
Using the notation as in the proof of Theorem \ref{thm:regret} (see \eqref{regret:split}), we have,
$$R_{(k)} = \sum_{t = s_k}^{s_{k+1}-1}\EE_{\leq t}\fro{\bcalT_t - \bcalT^*}^2 = \sum_{t = s_k}^{s_{k+1}-1} (\textsf{A}_t  + \textsf{B}_t). $$
To this end, we define $\textsf{E}_{k} := \EE_{\leq s_{k} - 1}\fro{\bcalT_{s_{k}}-\bcalT^*}^2$, which is the expected error at the beginning of phase $k$. 
In fact, from the proof of Theorem \ref{thm:regret}, we have 
\begin{align*}
	\textsf{E}_k &\leq 2\exp(-\frac{\eta_{k}}{8}2^{k-1}t_0)\textsf{E}_{k-1}+ C\eta_{k}\dof\sigma^2.
\end{align*}
We set $\eta_k = O(2^{-k}(t_0\dof)^{-1/2}\frac{\lambda_{\submin}}{\sigma})$. This gives us 
\begin{align*}
	\textsf{E}_k &\leq 4\exp(-c\frac{\lambda_{\submin}\sqrt{t_0}}{\sqrt{\dof}\sigma})\textsf{E}_{k-1} + C2^{-k}(t_0)^{-1/2}\dof^{1/2}\lambda_{\submin}\sigma.
\end{align*}
Since $\frac{\lambda_{\submin}\sqrt{t_0}}{\sqrt{\dof}\sigma}\geq c_1$, we obtain 
\begin{align*}
	\textsf{E}_k \leq 2\cdot 4^{-k}\lambda_{\submin}^{2}.
\end{align*}
%
Within each phase, similar as the proof in Theorem \ref{thm:regret}, we can show
\begin{align*}
	\EE_{\leq t}\fro{\bcalT_t - \bcalT^*}^2 \leq (1-\frac{\eta_k}{8})^{t-s_k}\textsf{E}_{k-1} + C\eta_k\dof\sigma^2. 
\end{align*}
And thus 
\begin{align*}
	R_{(k)} \leq C\frac{1}{\eta_k}\textsf{E}_{k-1} + 2^{k}t_0\cdot C\eta_k\dof\sigma^2\leq C\frac{1}{\eta_k}4^{-k}\lambda_{\submin}^2 + C\eta_k2^{k}t_0\dof\sigma^2
\end{align*}
With the previously chosen $\eta_k$, we obtain
\begin{align}\label{bound:Ratk}
	R_{(k)} \leq C\sqrt{t_0\dof}\sigma\lambda_{\submin}. 
\end{align}
%
%
Together with \eqref{bound:Ratk}, we obtain
\begin{align*}
	R_T = \sum_{k=1}^K R_{(k)} \leq \sum_{k=1}^K O(\sqrt{t_0\dof}\sigma\lambda_{\submin})\leq O\big(\sqrt{t_0\dof}\sigma\lambda_{\submin}\log(\frac{T}{t_0})\big). 
\end{align*}

\end{document}